\documentclass[twoside]{article}

\usepackage[accepted]{aistats2026}
\usepackage[round]{natbib}

\usepackage[utf8]{inputenc} 
\usepackage[T1]{fontenc}    
\usepackage{hyperref}       
\usepackage{url}            
\usepackage{booktabs}       
\usepackage{amsfonts}       
\usepackage{nicefrac}       
\usepackage{microtype}      
\usepackage{xcolor}         
\usepackage{bm}

\usepackage{graphicx}
\usepackage{dsfont}
\usepackage{subcaption}
\usepackage{wrapfig}
\usepackage{multirow}
\usepackage{adjustbox}

\newcommand{\mP}{\mathbb{P}}
\newcommand{\mE}{\mathbb{E}}

\newcommand{\mV}{\mathrm{Var}}

\newcommand{\fhat}{\widehat{f}}
\newcommand{\ghat}{\widehat{g}}
\newcommand{\gtilde}{\widetilde{g}}

\newcommand{\lb}{\left[}
\newcommand{\rb}{\right]}
\newcommand{\lp}{\left(}
\newcommand{\rp}{\right)}
\newcommand{\tr}{\mathrm{tr}}

\newcommand{\given}{\,|\,}

\newcommand{\convP}{\overset{P}{\longrightarrow}}
\newcommand{\convD}{\overset{D}{\longrightarrow}}

\usepackage{amsmath}
\usepackage{amssymb}
\usepackage{mathtools}
\usepackage{amsthm}

\newcommand{\xssMMD}{\mathrm{xss}\widehat{\mathrm{MMD}}^2}
\newcommand{\xMMD}{\mathrm{x}\widehat{\mathrm{MMD}}^2}

\usepackage[capitalize,noabbrev]{cleveref}
\theoremstyle{plain}
\newtheorem{theorem}{Theorem}[section]

\newtheorem{lemma}[theorem]{Lemma}
\newtheorem{corollary}[theorem]{Corollary}
\theoremstyle{definition}

\newtheorem{assumption}[theorem]{Assumption}
\theoremstyle{remark}
\newtheorem{remark}[theorem]{Remark}

\begin{document}

%

%

\twocolumn[

\aistatstitle{A Semi-Supervised Kernel Two-Sample Test}

\aistatsauthor{ Gyumin Lee \And Shubhanshu Shekhar \And  Ilmun Kim }

\aistatsaddress{ Dept. of Statistics \\ Pennsylvania State University \And  Dept. of EECS \\ University of Michigan \And  Dept. of Mathematical Sciences \\ KAIST } ]

\begin{abstract}
  We consider the problem of two-sample testing in a semi-supervised setting with abundant unlabeled covariate data. Standard two-sample tests neglect covariate information, which has the potential to significantly boost performance. However, incorporating covariates potentially breaks the exchangeability assumption under the null, which further complicates a calibration procedure. To address these issues, we propose a semi-supervised method that produces a test statistic with asymptotic normality, while effectively integrating additional information from covariates. Our test is straightforward to calibrate due to the asymptotic normality under the null and achieves asymptotic power that is often much higher than existing kernel tests without covariates. Furthermore, we formally show that the proposed method is consistent in power against fixed and local alternatives. Simulations confirm the practical and theoretical strengths of our approach.
\end{abstract}

\section{\MakeUppercase{Introduction}}\label{section 1}

In recent years, the realm of statistics and machine learning has seen notable progress in the development of semi-supervised methodologies that exploit both labeled and unlabeled data. These techniques present promising options for addressing numerous issues where labeled data are scarce or costly to gather, while large quantities of unlabeled data are typically available. The integration of both data types in semi-supervised learning has drawn considerable interest for its efficacy in enhancing predictive modeling for various tasks, including classification and regression \citep[e.g.,][]{chapelle2006semi,zhu2008semi,van2020survey}. These methods are widely adopted across various domains. In healthcare, for example, obtaining sufficient labeled medical data is challenging due to privacy concerns and the rarity of certain diseases \citep[e.g.,][]{han2024deep, Jiao2024learning, chebli2018semi}. Medical image annotation also demands substantial time and effort. Similar challenges occur in cyber-security \citep[e.g.,][]{mvula2024survey, watkins2017using}, drug discovery \citep[e.g.,][]{zhang2018seq3seq}, and part-of-speech tagging \citep[e.g.,][]{wang2007semi}. In such fields, semi-supervised learning has allowed practitioners to leverage unlabeled data for more accurate predictions.

Traditionally, semi-supervised learning research has focused on improving classification performance. However, more recent work has expanded its scope to address a wider range of problems, including tasks in statistical estimation and inference. A central goal of this body of work is to use unlabeled data to improve statistical methods when labeled data are scarce but unlabeled data are abundant. In line with this direction, our objective in this paper is to adapt and extend traditional two-sample testing to a semi-supervised setting to effectively incorporate unlabeled data. 

The goal of two-sample testing is to determine whether two samples originate from the same underlying distribution. Numerous methods have been proposed to address this problem \citep[see][for a recent review]{stolte2024methods}. Among them, a popular technique leverages the kernel-based framework introduced by \citet{gretton2012kernel}. This method employs an estimator of the kernel maximum mean discrepancy (MMD) as the test statistic where it measures the maximum difference in expectations over functions within the unit ball of a reproducing kernel Hilbert space (RKHS). The MMD has gained widespread adoption due to its nonparametric nature and the strong theoretical guarantees provided by RKHS theory. Despite these advantages, its application is often constrained by the intractability of its null distribution, which complicates the implementation of direct inference.

To circumvent this issue, permutation-based methods are commonly used to determine the threshold, setting $\tau:=\tau(\alpha)$ as the $(1-\alpha)$-quantile of the statistic over $B$ permuted datasets. While these methods ensure finite-sample validity, they are computationally intensive, requiring $B+1$ evaluations of the test statistic, with $B$ typically exceeding $100$. Moreover, incorporating additional covariates (say $V$ and $W$), which are possibly correlated with the primary samples ($X$ and $Y$), introduces further complications: under the null hypothesis, the distributions of \(V\) and \(W\) do not need to match, which violates the \emph{exchangeability} assumption. This lack of exchangeability undermines the validity of permutation-based methods, which poses a critical challenge in the semi-supervised setting.

Various permutation-free methods for determining the threshold $\tau$ have been proposed specifically to address the computational challenges associated with the permutation test. However, these methods entail notable limitations that warrant further attention. Some are overly conservative, resulting in a type-I error rate much smaller than the target size $\alpha$ \citep[e.g.,][]{gretton2006kernel, kim2021comparing}. Others rigorously control the size only under restrictive conditions, such as when the kernel remains fixed as the sample size increases \citep[e.g.,][]{gretton2009fast, chwialkowski2015fast, jitkrittum2016interpretable}. Additionally, some methods lack theoretical guarantees of size control, as they are heuristic in nature \citep[e.g.,][]{gretton2006kernel, gretton2009fast}.
 In contrast to these approaches, \citet{shekhar2022permutation} proposed a permutation-free kernel two-sample test that leverages the dimension-agnostic framework introduced by \citet{kim2024dimension}. It offers rigorous theoretical guarantees, achieving consistency and minimax-rate optimality against local alternatives, but it only utilizes labeled data. To our knowledge, there exists no method for two-sample testing in a semi-supervised setting, and in this paper we develop a general framework that addresses this gap in the literature.

\subsection{Contributions}
With the preceding background in place, the main contributions of this work are summarized below. 

\textbf{General framework.} We propose a general framework for semi-supervised two-sample testing that effectively leverages unlabeled data. This framework does not require permutation-based inference and is provably more powerful than the corresponding supervised methods in various scenarios. Additionally, we introduce a cross-fitting procedure to broaden the applicability of our method. \\[.5em]
\textbf{Semi-supervised kernel two-sample test.} As a specific implementation of our general framework, we propose a semi-supervised kernel two-sample test. This test can be seen as a natural extension of the method introduced by \citet{shekhar2022permutation}, adapted to effectively leverage the additional information available from unlabeled data in a semi-supervised setting.  \\[.5em]
\textbf{Power analysis.} We provide both theoretical and empirical evidence showing that our method retains the desirable properties of existing approaches while achieving higher power across diverse scenarios. A key element of our analysis is the asymptotic normality of a studentized test statistic. Unlike prior work, we establish the asymptotic normality of the test statistic under both the null and alternative hypotheses, which is critical to our power analysis. Furthermore, we demonstrate that our test statistic maintains consistency in power against both fixed and local alternatives.

\subsection{Related Work} \label{section: related work}

\noindent \textbf{Semi-Supervised Inference.}
Semi-supervised inference has emerged as an important area in statistics, with numerous studies exploring how unlabeled data can enhance estimation and testing. For instance, \citet{zhang2019semi} introduced methods for semi-supervised mean estimation, demonstrating the potential of unlabeled data to improve inference accuracy. \citet{chakrabortty2019high} examined the use of semi-supervised techniques in high-dimensional settings, while \citet{tony2020semisupervised} focused on variance estimation. In the context of linear regression, \citet{chakrabortty2018efficient,azriel2022semi} improved standard estimators by incorporating unlabeled data, and \citet{chakrabortty2022semi} extended this idea to quantile estimation. Further advancements were made by \citet{angelopoulos2023prediction,zrnic2023cross}, who introduced the concept of \textit{prediction-powered inference}, providing a unified framework for constructing predictive models that leverage both labeled and unlabeled data. More recently, \citet{kim2025semi} analyzed semi-supervised U-statistics, which offers a comprehensive framework for integrating unlabeled data into nonparametric inference. 

\noindent \textbf{Kernel Two-Sample Tests.}
Kernel-based two-sample testing has gained widespread attention for handling complex and high-dimensional data. Since its introduction by \citet{gretton2012kernel}, numerous advancements have followed: developing optimal kernels to enhance test power \citep{gretton2012optimal,sutherland2017generative,liu2020learning}, extending applicability to manifold data \citep{cheng2024kernel}, and devising alternative methods to reduce computational overhead \citep{zaremba2013b,song2021fast,schrab2022efficient,choi2024computational}. A more recent line of work has focused on boosting test power by aggregating MMD estimates over multiple kernels~\citep{schrab2023mmd,biggs2024mmd,hagrass2024spectral,chatterjee2025boosting}. Taking a different approach, \citet{tian2024unified} proposed a unified representation learning framework that utilizes the entire dataset to learn discriminative features. While their approach focuses on embedding learning via self-supervised learning, our method explicitly leverages the functional relationship between the target variable and abundant unlabeled covariates to reduce the variance of the test statistic.

\noindent \textbf{Permutation-Free Approaches.} 
Permutation-free methods have been actively studied to address the computational challenge of traditional permutation-based methods for large-scale analyses. Among these, \citet{kim2024dimension} proposed a dimension-agnostic framework that uses sample-splitting to construct a studentized test statistic, asymptotically Gaussian under the null. This approach proves especially valuable when the null distribution is intractable or computationally expensive to estimate. Building on this, \citet{shekhar2022permutation} extended these ideas to the kernel-MMD setting, introducing the cross-MMD statistic. This innovative approach overcomes the degeneracy issues of the classical kernel-MMD statistic under the null.
A similar framework has been applied in several studies to develop kernel-based independence testing \citep{shekhar2023permutation}, conditional independence testing \citep{lundborg2024projected}, and kernel-based treatment effect testing \citep{martinez2023efficient}. 

The above works collectively form the foundation of our approach, integrating semi-supervised inference, kernel-based testing, and advanced estimation techniques to address challenges in two-sample testing with additional covariates.

\section{\MakeUppercase{General Semi-Supervised Two-Sample Test}}\label{Section: General Semi-Supervised Two-Sample Test}
In this section, we start by describing the problem setup and presenting the key idea underlying our approach. We then formally introduce a general semi-supervised two-sample test, which serves as the cornerstone for the semi-supervised kernel two-sample test detailed in \Cref{Section: Semi-Supervised Kernel Two-Sample Test}. 

We first clarify the terminology of labeled and unlabeled data as used in this paper. While the term \textit{label} often refers to class variables in supervised learning tasks such as classification, we adopt a broader usage that is standard in recent literature on semi-supervised inference \citep[e.g.,][]{zhang2019semi, angelopoulos2023prediction, kim2025semi}. In our setting, labeled data refers to observations for which the primary response variable is available, whereas unlabeled data consists of covariates without associated responses. Although these covariates are not directly analyzed, they are typically easier to obtain and exhibit meaningful associations with the primary variables of interest. 

In particular, we refer to the paired samples $(X, V)$ and $(Y, W)$ as labeled data, where $X$ and $Y$ are the primary variables of interest, and $V$ and $W$ denote covariates associated with the responses $X$ and $Y$, respectively. While $V$ and $W$ come from the same feature space, we distinguish them to reflect the two different groups. In this context, supervised approaches rely solely on the labeled pairs, while semi-supervised methods additionally exploit the unlabeled covariates to improve statistical power. 

\noindent \textbf{Problem Setting.} Let us formalize the setting of semi-supervised two-sample testing where we observe mutually independent labeled and unlabeled datasets as follows:

\noindent $\bullet$ \emph{Labeled data}: $\mathcal{L}_{XV} \coloneqq \{(X_i,V_i)\}_{i = 1}^{n_1} \overset{\mathrm{i.i.d.}}{\sim} P_{XV}$ and $\mathcal{L}_{YW}\coloneqq \{(Y_i,W_i)\}_{i=1}^{n_2} \overset{\mathrm{i.i.d.}}{\sim} P_{YW}$

\noindent $\bullet$ \emph{Unlabeled data}: $\mathcal{U}_{V} \coloneqq \{V_i\}_{i=n_1+1}^{n_1+m_1} \overset{\mathrm{i.i.d.}}{\sim} P_{V}$ and $\mathcal{U}_{W} \coloneqq \{W_i\}_{i=n_2+1}^{n_2+m_2} \overset{\mathrm{i.i.d.}}{\sim} P_{W}$ 


Using these observations, we would like to test 
the null hypothesis that the marginal distributions of $X$ and $Y$ are equal, that is, $H_0: P_X = P_Y$ against the alternative $H_1: P_X \neq P_Y$. 
Unlike the classical two-sample testing, covariates $V$ and $W$ are available, and our goal is to create a testing procedure that boosts statistical power by incorporating these covariates, while ensuring robustness when they are independent of $X$ and $Y$.

\subsection{Oracle Test}\label{Section: Oracle Test}
Before presenting a practical version, we first build intuition by considering an oracle test, assuming that we know the true conditional expectation.

\noindent \textbf{Key idea.} To clarify the key idea behind our approach, let us revisit semi-supervised mean estimation \citep[e.g.,][]{zhang2019semi,zhang2022high}. Specifically, consider the problem of estimating the population mean of some real-valued function $f(X)$. A natural idea is to use the sample mean, $n_1^{-1}\sum_{i=1}^{n_1} f(X_i)$, which has the minimum variance among all possible unbiased estimators. However, the situation changes when additional unlabeled datasets become available. For simplicity of our discussion, we assume the conditional expectation $\mE[f(X_i) \given V_i]$ is known, and address the unknown case in \Cref{Section: Procedure with CF}. Under this setup, one can construct another estimator 
\begin{align*}
	\widehat{\mu}_{X,f}  \coloneqq  \, & \frac{1}{n_1} \sum_{i=1}^{n_1} \{f(X_i) - \mE[f(X_i) \given V_i]\} \\
	& + \frac{1}{n_1 + m_1}\sum_{i=1}^{n_1+m_1} \mE[f(X_i) \given V_i],
\end{align*}
which is also an unbiased estimator. Importantly, the variance of $\widehat{\mu}_{X,f}$ is never greater than that of the ordinary sample mean. This property arises from the observation that the two summations in $\widehat{\mu}_{X,f}$ are uncorrelated. As a result, the variance of $\widehat{\mu}_{X,f}$ can be expressed as
$\sigma_{X,f}^2 \coloneqq n_1^{-1} \sigma_{1,X,f}^2 + (n_1+m_1)^{-1}\sigma_{2,X,f}^2$ where $\sigma_{1,X,f}^2 \coloneqq \mE[\mV\{f(X) \given V\}] $ and $\sigma_{2,X,f}^2 \coloneqq \mV[\mE\{ f(X) \given V \}]$. Moreover, the ordinary sample mean is equivalent to $\widehat{\mu}_{X,f}$ with $m_1=0$, whose variance equals $n_1^{-1} \mV[f(X)] = n_1^{-1}\sigma_{1,X,f}^2 +  n_1^{-1}\sigma_{2,X,f}^2 $ by the law of total variance. This directly confirms the variance reduction achieved by incorporating the additional unlabeled data. Furthermore, $\widehat{\mu}_{X,f}$ is a linear statistic that is expected to converge to a normal distribution under regularity conditions. As a result, statistical inference based on $\widehat{\mu}_{X,f}$ would be more efficient than that based on the ordinary sample mean, which lies at the heart of recent advancements in semi-supervised inference.  

\noindent \textbf{Oracle Test Construction.} Building on the idea that incorporating unlabeled data can lead to variance reduction, we now introduce a general semi-supervised two-sample test. To delineate the procedure, define $\widehat{\mu}_{Y,f}$ analogously to $\widehat{\mu}_{X,f}$ using $\mathcal{L}_{YW}$ and $\mathcal{U}_W$. Here, $f$ is treated as a certain feature map, mapping inputs to $\mathbb{R}$, chosen to effectively distinguish $P_X$ and $P_Y$ under the alternative. For instance, $f$ can be a certain basis function~\citep{zhou2017}, an estimated witness function of an integral probability metric~\citep{kim2024dimension} or a deep kernel feature map~\citep{liu2020learning}. Importantly, we assume that $f$ is independent of $\mathcal{L}_{XV}$, $\mathcal{L}_{YW}$, $\mathcal{U}_V$, and $\mathcal{U}_W$. 
\medskip 
\begin{remark} \label{Remark: conditions for f}
	When $f$ is a random function, all expectations and variances below are \emph{implicitly conditional on $f$} unless stated otherwise. For example, $\mE[f(X)]$ then denotes the conditional expectation of $f(X)$ given the $\sigma$-algebra generated by the randomness of $f$.
\end{remark}

Our general procedure compares the studentized difference between $\widehat{\mu}_{X,f}$ and $\widehat{\mu}_{Y,f}$. To this end, we estimate the variance of $\widehat{\mu}_{X,f}$ as $\widehat{\sigma}^2_{X,f} = n_1^{-1}\widehat{\sigma}_{1,X,f}^2 + (n_1+m_1)^{-1} \widehat{\sigma}_{2,X,f}^2$ by combining two components defined in \eqref{eq: semi-supervised variance estimate} of \Cref{section: proof of main theorem - general power expression}. We similarly define $\widehat{\sigma}^2_{Y,f} = n_2^{-1} \widehat{\sigma}_{1,Y,f}^2 + (n_2+m_2)^{-1} \widehat{\sigma}_{2,Y,f}^2$ as an estimator of $\sigma_{Y,f}^2 \coloneqq n_2^{-1} \sigma_{1,Y,f}^2 + (n_2+m_2)^{-1} \sigma_{2,Y,f}^2$, which is the variance of $\widehat{\mu}_{Y,f}$. Using these estimates, we define an oracle test statistic as
\begin{align}
	T_{\mathrm{oracle}} = \frac{\widehat{\mu}_{X,f} - \widehat{\mu}_{Y,f}}{\sqrt{\widehat{\sigma}^2_{X,f} + \widehat{\sigma}^2_{Y,f}}}.
    \label{eq:oracle-test}
\end{align} 
Given $\alpha \in (0,1)$, the resulting oracle test rejects the null when $T_{\mathrm{oracle}} > z_{1-\alpha}$ (or $|T_{\mathrm{oracle}}| > z_{1-\alpha/2}$ for a two-sided test) without requiring permutations. Here $z_{1-\alpha}$ is the $1-\alpha$ quantile of $N(0,1)$. To analyze the oracle test, we make the following moment assumption, using \( n \coloneqq n_1 \wedge n_2 \) to denote the minimum sample size throughout. Additional notational conventions are provided in \Cref{Section: notation}.

\begin{assumption} \label{Assumption: moment condition}
	Suppose there exists $\delta >0$ such that 
	\begin{align*}
		& \frac{\mE[|f(X) - \mE[f(X)]|^{2+\delta}]}{\sigma_{1,X,f}^{2+\delta} \wedge \sigma_{2,X,f}^{2+\delta}} = o_P\bigl(n_1^{\delta/2}\bigr) \quad \text{and} \\
		&  \frac{\mE[|f(Y) - \mE[f(Y)]|^{2+\delta}]}{\sigma_{1,Y,f}^{2+\delta} \wedge \sigma_{2,Y,f}^{2+\delta}} = o_P\bigl(n_2^{\delta/2}\bigr) \quad \text{as $n \rightarrow \infty$. }
	\end{align*}
\end{assumption}
The $o_P$ notation above is used to accommodate the random case of $f$ discussed in \Cref{Remark: conditions for f}. The next theorem derives an asymptotic power expression of the oracle test and highlights the power gain obtained through the unlabeled dataset. 

\begin{theorem} \label{Theorem: General Power Expression}
	Under \Cref{Assumption: moment condition}, the power function of the oracle test defined in~\eqref{eq:oracle-test} approximates, unconditionally on $f$, that
	\begin{align*}
		\Phi \biggl( z_{\alpha} + \frac{\mE[f(X)] - \mE[f(Y)]}{\sqrt{\smash[b]{\sigma_{X,f}^2 + \sigma_{Y,f}^2}}} \biggr) \quad \text{as $n \rightarrow \infty$. }
	\end{align*}
\end{theorem}
Under the null, \Cref{Theorem: General Power Expression} indicates that the oracle test asymptotically maintains the correct size $\alpha$. Regarding power, the oracle test achieves an asymptotic power \emph{never less} than that of the standard two-sample $t$-test. This is evident from the fact that the standard $t$-test is a special case of the oracle test when $m_1 = m_2 = 0$, and that $\sigma_{X,f}^2 + \sigma_{Y,f}^2$ is non-decreasing in $m_1$ and $m_2$. In other words, incorporating unlabeled data reduces variance while keeping the mean unchanged, ultimately resulting in increased power. \Cref{Theorem: General Power Expression} allows $f$ to change with the sample sizes. This flexibility requires conditions stronger than the finite second moment of $f(X)$ and $f(Y)$ as specified in \Cref{Assumption: moment condition}. It is also worth highlighting that \Cref{Theorem: General Power Expression} puts no restrictions on $m_1$ and $m_2$, which can grow much faster than $n = n_1 \wedge n_2$. 

\subsection{Procedure with Cross-Fitting} \label{Section: Procedure with CF}
In the previous subsection, we constructed the oracle test under the assumption that both the conditional expectations $\mE[f(X) \given V]$ and $\mE[f(Y) \given W]$ are known. We now eliminate this assumption and propose a practical procedure using the estimated conditional expectations $\widehat{\mE}[f(X) \given V]$ and $\widehat{\mE}[f(Y) \given W]$. For this purpose, we employ cross-fitting, a commonly used technique in semi-parametric statistics. Cross-fitting is a practical, efficient method of data splitting, typically applied to correct for bias arising from nuisance estimation, ease stringent conditions on the parameter space, and regain full efficiency \citep[e.g.,][]{zheng2010asymptotic, chernozhukov2018double, newey2018cross, wasserman2020universal, kennedy2020towards, kim2025semi}. This method involves partitioning the dataset into two segments: one is used to estimate nuisance parameters, while the other is employed to form an initial estimator. The roles of these partitions are then alternated, and the procedure is repeated. Finally, the two resulting statistics are aggregated to yield the final estimator. For simplicity, we assume that $n_1, n_2, m_1, m_2$ are even numbers, which allows us to avoid asymmetry in the cross-fitting procedure and thus simplifies the analysis.

\noindent \textbf{Cross-Fit Test Construction.} To describe the idea, we split the dataset $\mathcal{L}_{XV}$ into two parts: $\mathcal{L}_{XV,a} \coloneqq \{(X_i,V_i): i \in [n_1], \, \text{$i$ is odd} \}$ and $\mathcal{L}_{XV,b} \coloneqq  \{(X_i,V_i): i \in [n_1], \, \text{$i$ is even} \}$. Write
$\widehat{\mE}[f(X_i) \given V_i]$
as an estimator of $\mE[f(X_i) \given V_i]$ trained on $\mathcal{L}_{XV,a}$ if the index $i$ is even and on $\mathcal{L}_{XV,b}$ if $i$ is odd. This estimator can be obtained using methods such as neural nets or random forests by regressing $f(X)$ on $V$. We similarly construct $\widehat{\mE}[f(Y_i) \given W_i]$ as an estimator of $\mE[f(Y_i) \given W_i]$. The test statistic ${T}_{\mathrm{cross}}$ is then computed in the same way as $T_{\mathrm{oracle}}$, replacing $\mE[f(X_i) \given V_i]$ and $\mE[f(Y_i) \given W_i]$ with their estimators. We finally reject the null if $T_{\mathrm{cross}}$ exceeds $z_{1-\alpha}$. This cross-fit test retains the same asymptotic properties as the oracle test, provided that the estimated conditional expectations satisfy the required convergence conditions.
\begin{corollary} \label{Corollary: power expression for cross-fit test}
	Suppose \Cref{Assumption: moment condition} holds and, additionally, the following condition is satisfied:
	\begin{equation}
	\begin{aligned} \label{Eq: cross-fit condition}
		\frac{ \mE\bigl[ \{\widehat{\mE}[f(X) \given V] - \mE[f(X) \given V]\}^2\bigr]}{\sigma_{1,X,f}^{2} \wedge \sigma_{2,X,f}^{2}} = o_P(1),
	\end{aligned}
	\end{equation}
	as $n \to \infty$, and the analogous condition holds for $(Y, W)$. Then the power function of the cross-fit test approximates that of the oracle test as in \Cref{Theorem: General Power Expression}.
\end{corollary}
Similarly to \Cref{Assumption: moment condition}, the $o_P$ notation accounts for the randomness of $f$. The validity of \Cref{Corollary: power expression for cross-fit test} primarily depends on accurately estimating the conditional expectation associated with $f$, a problem well-studied in the statistical literature \citep[e.g.,][]{gyorfi2006distribution, wainwright2019high}.

Up to this point, we have developed a general semi-supervised two-sample test with a generic function $f$ and demonstrated its power gain through the incorporation of unlabeled data. We next focus on a specific instantiation of $f$ constructed as the difference between empirical kernel mean embeddings.

\section{\MakeUppercase{Semi-Supervised Kernel Test}} \label{Section: Semi-Supervised Kernel Two-Sample Test}
In this section, we introduce a semi-supervised kernel two-sample test, regarded as a semi-supervised extension of the xMMD test \citep{shekhar2022permutation}. In \Cref{Section: xMMD Test}, we first provide a brief overview of the xMMD test and then describe our proposed semi-supervised extension. \Cref{Section: Theoretical Analysis} presents the theoretical analysis of the proposed test.

\subsection{Testing Procedure} \label{Section: xMMD Test}
As mentioned earlier, one notable method for addressing the two-sample testing problem involves using an empirical version of the kernel-MMD \citep{gretton2012kernel}. For a positive definite kernel $k$ and its associated RKHS $\mathcal{H}_k$, the kernel-MMD quantifies the distance between distributions $P$ and $Q$ by computing the supremum of the difference in expectations $\mathbb{E}_{X \sim P}[f(X)] - \mathbb{E}_{Y \sim Q}[f(Y)]$ over all functions $f$ in the unit ball of $\mathcal{H}_k$.
The empirical MMD statistic, based on U- or V-statistics, has an intractable limiting distribution under the null, which is often addressed using the permutation method. However, this resampling method is computationally expensive due to repeated evaluation of a test statistic. Beyond the computational issue, the permutation method may not be valid in the semi-supervised setting where $V$ and $W$ do not necessarily share the same distribution under the null. This violates the exchangeability assumption, which is crucial for the validity of the permutation test.

To address the computational issue of permutation-based MMD tests, \citet{shekhar2022permutation} introduced the xMMD test, which is essentially the two-sample $t$-test applied to data projected onto the optimal witness function. To provide a brief overview, let $\{\widetilde{X}_i\}_{i=1}^{n_1}$ and $\{\widetilde{Y}_i\}_{i=1}^{n_2}$ be i.i.d.~copies of $\{X_i\}_{i=1}^{n_1}$ and $\{Y_i\}_{i=1}^{n_2}$, respectively, which can be obtained through sample splitting. The xMMD test is then implemented through the following two steps:

\textbf{1.~Optimal Witness Function Estimation.} Estimate the optimal witness function that achieves the supremum in the definition of MMD based on $\{\widetilde{X}_i\}_{i=1}^{n_1}$ and $\{\widetilde{Y}_i\}_{i=1}^{n_2}$:
\begin{align*}
	\fhat(\cdot) := \frac{1}{n_1}\sum_{i=1}^{n_1} k(\widetilde{X}_i,\cdot) - \frac{1}{n_2} \sum_{i=1}^{n_2} k(\widetilde{Y}_i,\cdot).
\end{align*}
 \textbf{2.~Projection and $t$-Test.} Project $\{X_i\}_{i=1}^{n_1}$ and $\{Y_i\}_{i=1}^{n_2}$ onto the direction $\fhat$, which results in (conditionally independent) univariate two samples:
\begin{align} \label{Eq: projected samples}
	\fhat(X_1),\ldots,\fhat(X_{n_1}) \quad \text{and} \quad \fhat(Y_1),\ldots,\fhat(Y_{n_2}).
\end{align}
The xMMD test rejects the null when the corresponding two-sample $t$-statistic exceeds $z_{1-\alpha}$.

\citet{shekhar2022permutation} showed that the xMMD test is asymptotically level $\alpha$ under a certain moment condition, consistent in power and minimax-rate optimal against local $L_2$ alternatives. Moreover, the xMMD test offers a notable computational advantage over the permutation-based MMD test as it avoids the need for repeated resampling to determine a critical value. 


\noindent \textbf{xssMMD Test.} Building on the work of \citet{shekhar2022permutation}, we propose a new method called the xssMMD test, which extends the xMMD test to a semi-supervised setting. The main idea is to apply the general semi-supervised two-sample test introduced in \Cref{Section: General Semi-Supervised Two-Sample Test} to the two projected samples based on the estimated witness function $\widehat{f}$. Specifically, we define the cross-fit statistic $T_{\mathrm{cross}}$ based on the projected samples in \eqref{Eq: projected samples} as 
\begin{align} \label{Eq: xssMMD}
	\xssMMD =  \frac{\widehat{\mu}_{X,\fhat}^{\dagger} - \widehat{\mu}_{Y,\fhat}^\dagger}{\sqrt{\widehat{\sigma}^{\dagger 2}_{X,\fhat} + \widehat{\sigma}^{\dagger 2}_{Y,\fhat}}},
\end{align}
where $\dagger$ denotes the use of cross-fitting. The exact mathematical formulas for the cross-fitted components are detailed in Appendix \ref{appendix: proof of the main theorem}. The xssMMD test then rejects the null when $\xssMMD > z_{1-\alpha}$. A schematic illustration is provided in \Cref{fig:summary_of_testing} of \ref{appendix: overview and roadmap}.

\subsection{Theoretical Analysis} \label{Section: Theoretical Analysis}
We now shift our focus to the theoretical analysis of the xssMMD test. It is already clear from the results of \Cref{Section: General Semi-Supervised Two-Sample Test} that the xssMMD test is asymptotically level $\alpha$ and provably more powerful than the xMMD test under certain conditions. Our goal is to present more concrete conditions for these properties tailored to the kernel-MMD setting.

\noindent \textbf{Comparison with xMMD.} To establish the asymptotic properties of the tests, we require certain regularity and moment conditions. First, we assume that the witness function $\fhat$ is measurable and that the Bochner integral $\int_{\mathcal{X}} \|k(x,\cdot)\|_{\mathcal{H}} dP(x)$ is finite to ensure the existence of the mean embedding. Note that for bounded kernels, such as the Gaussian kernel, this condition is automatically satisfied since $\|k(x,\cdot)\|_{\mathcal{H}} = \sqrt{k(x,x)}$ is bounded.

Next, we state the assumptions required for our theoretical guarantees. These involve key quantities defined through the centered kernel $\overline{k}_X$ and its expected product $\overline{g}_X.$ Formal definitions are deferred to \Cref{appendix: Assumptions for Theorem xssMMD}.

\begin{assumption}[Null Condition] \label{Assumption: xssMMD under H0}
	Assume $X_1,X_2,X_{3}\overset{\mathrm{i.i.d.}}{\sim} P_{X,n}$ where $P_{X} \coloneqq P_{X,n}$ and the kernel $k \coloneqq k_n$ potentially changing with $n$ satisfy 
	\begin{align*} 
		& \frac{\mE[\overline{k}(X_1,X_2)^4] + n_1\mE[\overline{k}(X_1,X_2)^2 \overline{k}(X_1,X_3)^2]}{n_1^2\{\mE[\overline{g}_X(X,X)]\}^2} = o(1).
	\end{align*}
\end{assumption}

\begin{assumption}[Consistency of Conditional Expectation] \label{Assumption: consistency of conditional expectation}
	Assume that the estimated conditional expectation of $\fhat(X)$ given $V,\fhat$ and that of $\fhat(Y)$ given $W,\fhat$ satisfy 
\begin{equation}\label{Eq: cross-fitting condition}
\begin{aligned}
	\frac{ \mE\bigl[ \{\mE[\fhat(X) \given V, \fhat] - \widehat{\mE}[\fhat(X) \given V, \fhat]\}^2 \given \fhat \bigr]}{\mV\{\fhat(X) \given \fhat\}} &= o_P(1),\;\text{and} \\
	\frac{ \mE\bigl[ \{\mE[\fhat(Y) \given W, \fhat] - \widehat{\mE}[\fhat(Y) \given W, \fhat]\}^2 \given \fhat \bigr]}{\mV\{\fhat(Y) \given \fhat\}} &= o_P(1).
\end{aligned}
\end{equation} 
\end{assumption}

\begin{assumption}[Alternative Condition] \label{Assumption: xssMMD under H1}
	Assume that $P_{X} \coloneqq P_{X,n}$ and $P_{Y} \coloneqq P_{Y,n}$ have Lebesgue density functions $p_X$ and $p_Y$, respectively, satisfying $\|p_X/p_Y\|_{L_{\infty}} \vee \|p_Y/p_X\|_{L_{\infty}} < C$ for some constant $C>0$, with $\|f\|_{L_\infty}$ denoting $\inf \{M\geq 0: \text{Leb.}(\{x: |f(x)| > M\})=0\}$. Furthermore, assume that
	\begin{align} \label{Eq: xssMMD under H1}
		\frac{\mathrm{MMD}(P_X,P_Y)^4\times \mE[\overline{k}_X(X,X)^2] }{\{n_1\mE[\overline{g}_X(X,X)] + n_1^2\mE[\overline{g}_X(Y_1,Y_2)]\}^2} = o(1),
	\end{align}
	where $X \sim P_X$ and $Y_1,Y_2 \overset{\mathrm{i.i.d.}}{\sim} P_Y$.   
\end{assumption}

\Cref{Assumption: xssMMD under H0} is essentially a Lyapunov-type condition required for the Central Limit Theorem. At a high level, it ensures that the ``tails'' of the test statistic's distribution are not too heavy relative to its variance, guaranteeing that no single data point dominates the statistic, thus allowing it to converge to a normal distribution. \Cref{Assumption: consistency of conditional expectation} guarantees that the estimation error from the cross-fitting procedure decays sufficiently fast relative to the variance, preventing it from dominating the asymptotic behavior of the test statistic. \Cref{Assumption: xssMMD under H1} provides sufficient conditions to establish asymptotic normality under the alternative. The bounded density ratio simplifies the mathematical derivations, while the moment condition \eqref{Eq: xssMMD under H1} ensures the Lyapunov central limit theorem holds for a broad range of alternatives. We discuss further implications of these assumptions in \Cref{appendix: Assumptions for Theorem xssMMD}.

The next theorem compares the xssMMD test and the xMMD test based on their asymptotic properties under specific conditions. For brevity, these conditions are presented and discussed in \Cref{appendix: Assumptions for Theorem xssMMD}. In the following, we let $\Psi_{\mathrm{x}} \coloneqq \mathds{1}(\xMMD > z_{1-\alpha})$ and $\Psi_{\mathrm{xss}} \coloneqq \mathds{1}(\xssMMD > z_{1-\alpha})$ denote the xMMD and xssMMD tests, respectively.

\begin{theorem} \label{Theorem: xssMMD}
	 The tests $\Psi_{\mathrm{x}}$ and $\Psi_{\mathrm{xss}}$ satisfy the following asymptotic guarantees: \\[.5em]
    \textbf{Level.} Suppose \Cref{Assumption: xssMMD under H0} and \Cref{Assumption: consistency of conditional expectation} hold with $n_1 \asymp m_1$ and $n_2 \asymp m_2$. Then both tests control the size $\alpha$ under $H_0$ such that
	$\lim_{n \rightarrow \infty} \mE_{H_0}[\Psi_{\mathrm{x}}] = \lim_{n\rightarrow \infty} \mE_{H_0}[\Psi_{\mathrm{xss}}] = \alpha$. \\[.5em]	 
    \textbf{Power.} Suppose \Cref{Assumption: xssMMD under H1} also holds under $H_1$. Then the asymptotic power of $\Psi_{\mathrm{xss}}$ is at least as that of $\Psi_{\mathrm{x}}$, satisfying $\lim_{n \rightarrow \infty} \{\mE_{H_1}[\Psi_{\mathrm{xss}}] - \mE_{H_1}[\Psi_{\mathrm{x}}] \} \geq 0$.
\end{theorem}

The above theorem confirms that $\Psi_{\mathrm{xss}}$ is asymptotically level $\alpha$ under the null and achieves at least the same power as $\Psi_{\mathrm{x}}$ under the alternative. As for the general test, the key insight behind the power gain of $\Psi_{\mathrm{xss}}$ lies in the effective use of unlabeled data, which reduces the variance of the test statistic while preserving the same mean. This insight, together with the asymptotic normality, enables a direct comparison of the power of the xssMMD and xMMD tests.

The novelty of \Cref{Theorem: xssMMD} is in extending the conditions for asymptotic normality to the alternative, whereas prior work has primarily focused on the null. This extension is crucial for power comparisons and requires substantial effort to establish. Unlike \Cref{Theorem: General Power Expression}, \Cref{Theorem: xssMMD} additionally assumes $n_1 \asymp m_1$ and $n_2 \asymp m_2$. These conditions are imposed to facilitate a comparison of $\Psi_{\mathrm{x}}$ and $\Psi_{\mathrm{xss}}$ under common and concrete moment assumptions, which could be relaxed under more abstract conditions. Alternatively, when $m_1 \geq n_1$ and $m_2 \geq n_2$, one could discard a portion of the unlabeled samples to ensure the asymptotic balance condition.

\noindent \textbf{Consistency in Power.} The power property of the xssMMD test, as stated in \Cref{Theorem: xssMMD}, is established under the assumptions that the centered test statistic converges to a normal distribution under the alternative. Here, we present independent conditions under which the xssMMD test remains consistent in power (i.e., the power approaches one), without relying on the asymptotic normality. Below, a subscript $n$ is added to indicate that the corresponding sequence may vary with $n = n_1 \wedge n_2$.

\begin{lemma} \label{Lemma: consistency}
	Let $\{\delta_{n}: n \geq 2\}$ be any positive sequence such that $\delta_n \rightarrow 0$, and $\gamma_{n} \coloneqq \mathrm{MMD}(P_{X,n},P_{Y,n})$. If
	\begin{align*}
		\sup_{(P_{X,n},P_{Y,n}) \in \mathcal{P}_n} \Biggl\{&  \frac{\mE_{P_{X,n},P_{Y,n}}\bigl[\widehat{\sigma}^{\dagger 2}_{X,\fhat} + \widehat{\sigma}^{\dagger 2}_{Y,\fhat}\bigr]}{\delta_n\gamma_n^4}  \\
		& + \frac{\mathrm{Var}_{P_{X,n},P_{Y,n}}\bigl[\widehat{\mu}_{X,\fhat}^{\dagger} - \widehat{\mu}_{Y,\fhat}^\dagger\bigr]}{\gamma_n^4} \Biggr\} = o(1),
	\end{align*}
	then $\Psi_{\mathrm{xss}}$ is consistent in power uniformly over $\mathcal{P}_n$ as $\inf_{(P_{X,n},P_{Y,n}) \in \mathcal{P}_n} \mE_{P_{X,n},P_{Y,n}}[\Psi_{\mathrm{xss}}] = 1$.
\end{lemma}
The lemma above corresponds to \citet[][Theorem 8]{shekhar2022permutation}, which forms the primary foundation for their other results, including minimax-rate optimality. In \Cref{Section: Consistency in Power}, we show that the condition in \citet[][Theorem 8]{shekhar2022permutation} is stronger than that in \Cref{Lemma: consistency}, provided that $\widehat{\mE}[\fhat(X) \given V, \fhat]$ and $\widehat{\mE}[\fhat(Y) \given W, \fhat]$ exhibit ``well-behaved'' properties. This implies that $\Psi_{\mathrm{xss}}$ is consistent in power whenever $\Psi_{\mathrm{x}}$ is. Importantly, this result does not rely on the consistency of $\widehat{\mE}[\fhat(X) \given V, \fhat]$ and $\widehat{\mE}[\fhat(Y) \given W, \fhat]$ with the true conditional expectations. Instead, it requires that the residuals $\fhat(X) - \widehat{\mE}[\fhat(X) \given V, \fhat]$ and $\fhat(Y) - \widehat{\mE}[\fhat(Y) \given W, \fhat]$ have second moments comparable to the variances of $\fhat(X)$ and $\fhat(Y)$, respectively---a much weaker condition than the full consistency of the conditional expectations. We discuss further implications of \Cref{Lemma: consistency} in \Cref{Section: Consistency in Power}.

\section{\MakeUppercase{Experiments}} \label{Section: Experiments}

We now experimentally validate the theoretical results stated in the previous sections. In particular, our experiments show that (i) the limiting null distribution of the proposed test statistic in \eqref{Eq: xssMMD} follows a $N(0,1)$ distribution across a wide range of dimensions $d$, sample sizes $n_1, n_2, m_1, m_2$, and kernel $k$, and (ii) the power of the xssMMD test is comparable to and often much higher than that of the xMMD test and the kernel-MMD permutation (MMD-perm) test. Moreover, we examine its performance on several real-world datasets. Additional experimental findings can be found in \Cref{appendix: Additional Experiments}.

\noindent \textbf{Limiting Null Distribution.} We demonstrated in \Cref{Theorem: xssMMD} that the xssMMD test is asymptotically level $\alpha$ under the null, given some assumptions. We empirically validate this result by considering the case where $P_{XV}=P_{YW}=N\left(\mathbf{0}_{2d}, I_{2d}\right)$. We study the effects of dimensionality, sample skewness, labeled-unlabeled sample size ratio, methods for estimating conditional expectation, and choice of kernel on the null distribution of the test statistic. Specifically, we consider two scenarios: 




\noindent $\bullet$ \textbf{Scenario 1 (Null).}  $d=10$, $n_1/n_2=1$, $n_1/m_1=n_2/m_2=1$, Gaussian kernel with the median heuristic.

\noindent $\bullet$ \textbf{Scenario 2 (Null).}  $d=100$, $n_1/n_2=0.1$, $n_1/m_1=n_2/m_2=0.5$, bilinear kernel.



Note that we applied a bandwidth determined by the median heuristic when using a Gaussian kernel. Each scenario considers different methods for estimating the conditional expectation, including $k$-nearest neighbors (knn), kernel regression (kernel), and random forest (rf). As shown in \Cref{fig:null_dist}, the null distribution of $\xssMMD$ is robust to all these factors and closely approximates $N(0,1)$. This confirms that our test, calibrated using the normal quantile, successfully controls the level $\alpha$ in the scenarios considered. Additional results and implementation details are provided in \Cref{appendix: Additional Experiments}.

\begin{figure}[t!]	\centering
	\includegraphics[width=\columnwidth]{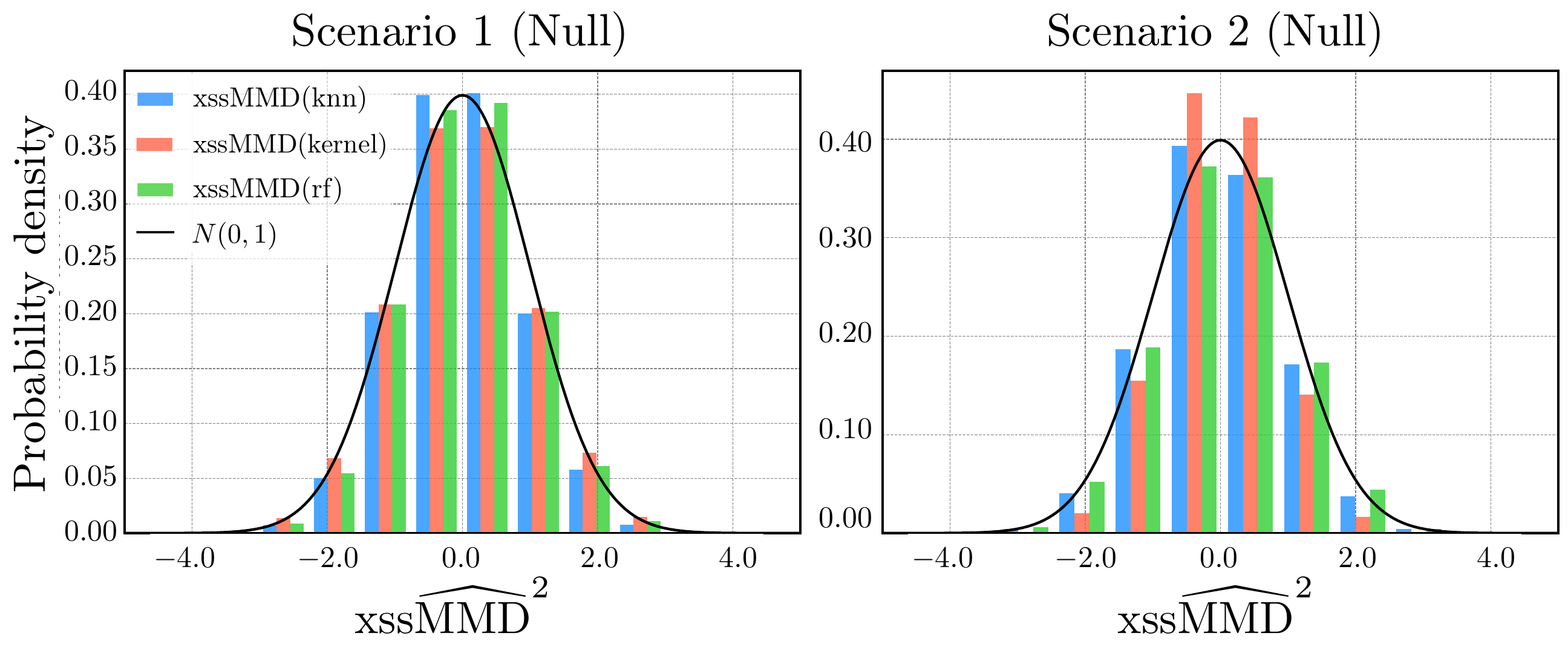}
	\caption{Experimental results for the distribution of $\xssMMD$ under the null hypothesis. The plots demonstrate that the test statistic closely follows a $N(0,1)$ distribution across various parameter settings, confirming the validity of the xssMMD test.} 
    \label{fig:null_dist}
    \vspace{-0.5cm}
\end{figure}

\begin{figure*}[t!]
    \centering
    \begin{subfigure}[t]{0.22\textwidth} 
        \centering
        \includegraphics[width=\linewidth]{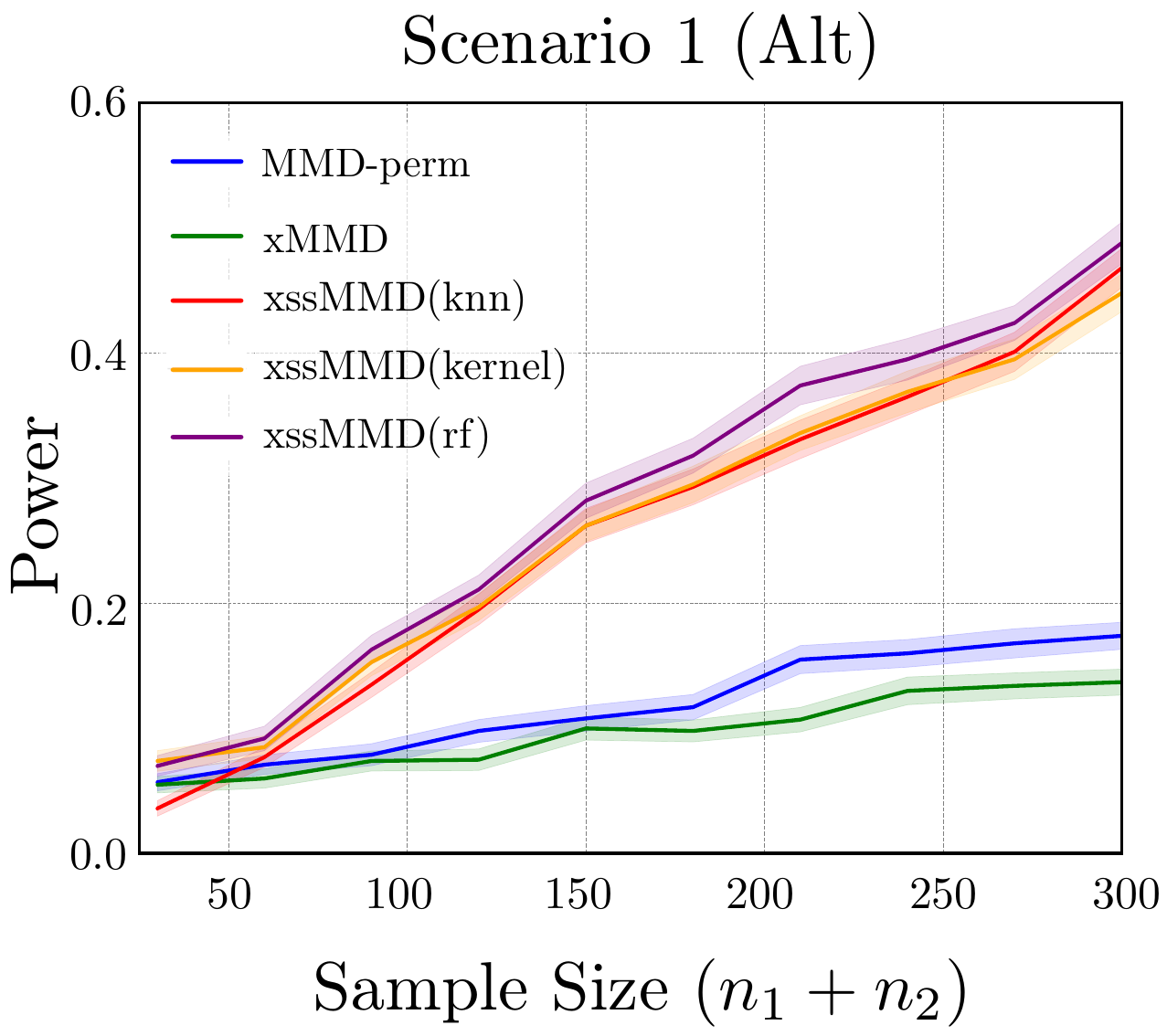}
    \end{subfigure}%
    \hfill%
    \begin{subfigure}[t]{0.21\textwidth}
        \centering
        \includegraphics[width=\linewidth]{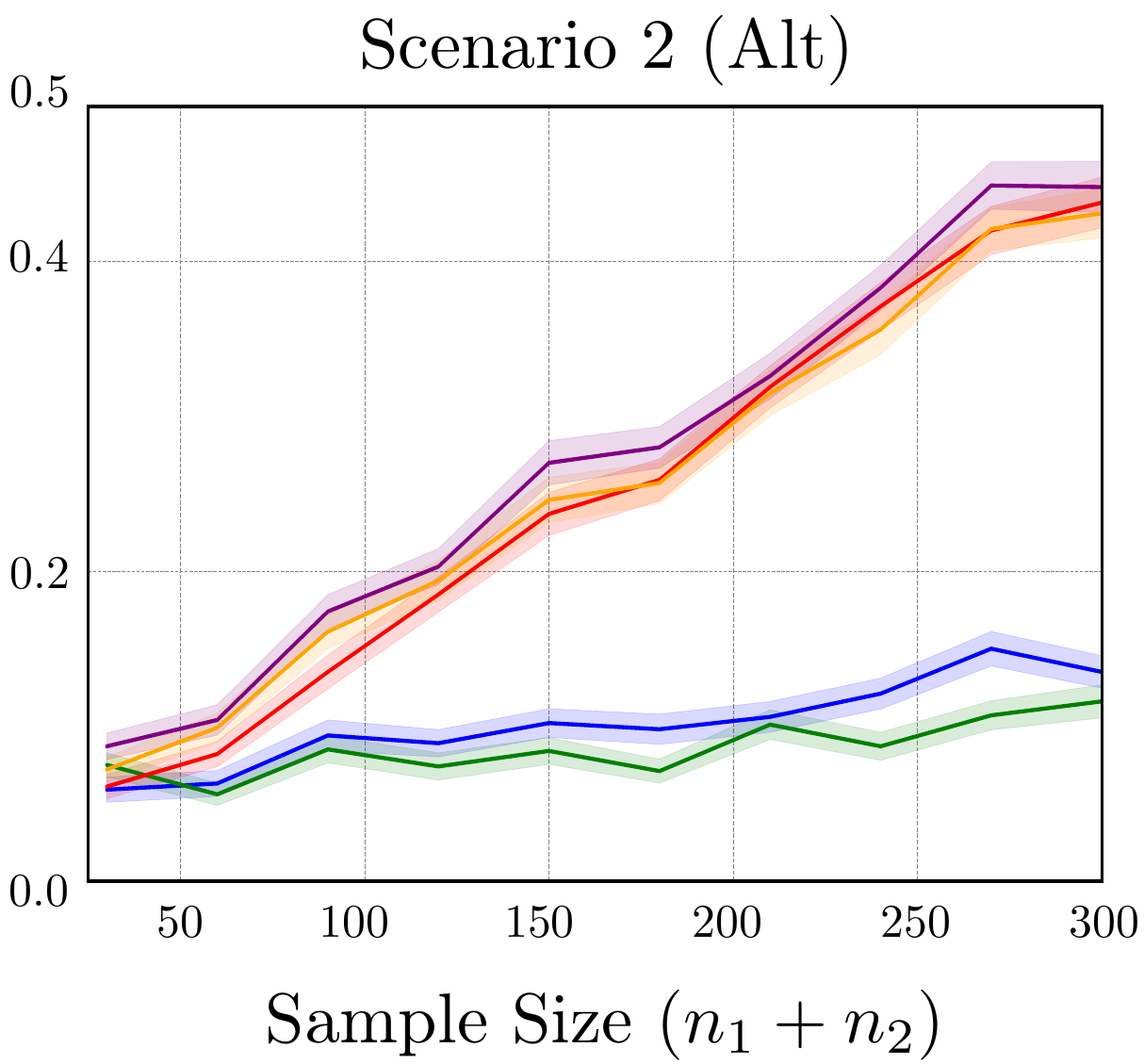}
    \end{subfigure}%
    \hfill%
    \begin{subfigure}[t]{0.21\textwidth} 
        \centering
        \includegraphics[width=\linewidth]{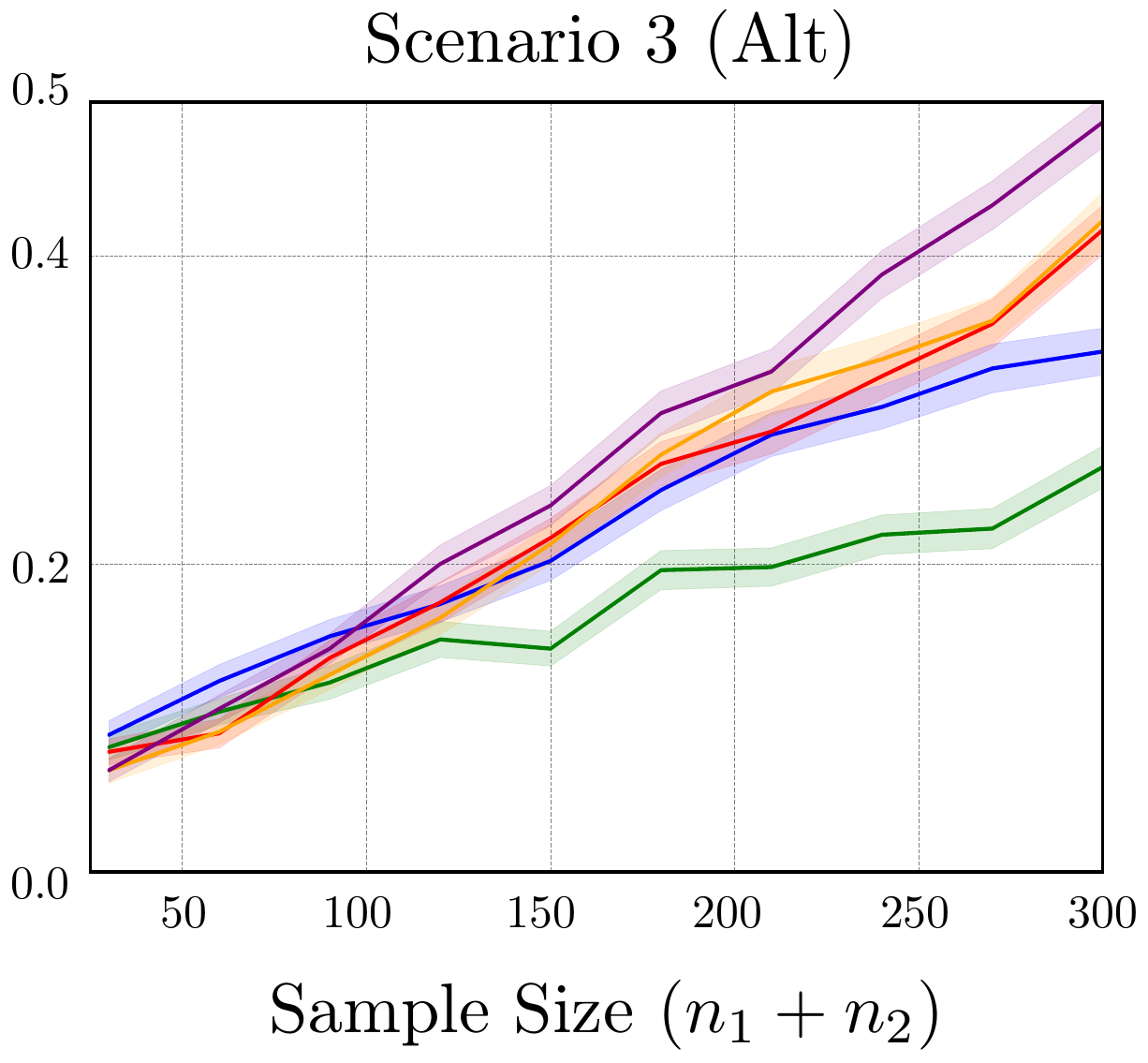}
    \end{subfigure}%
    \hfill%
    \begin{subfigure}[t]{0.21\textwidth} 
        \centering
        \includegraphics[width=\linewidth]{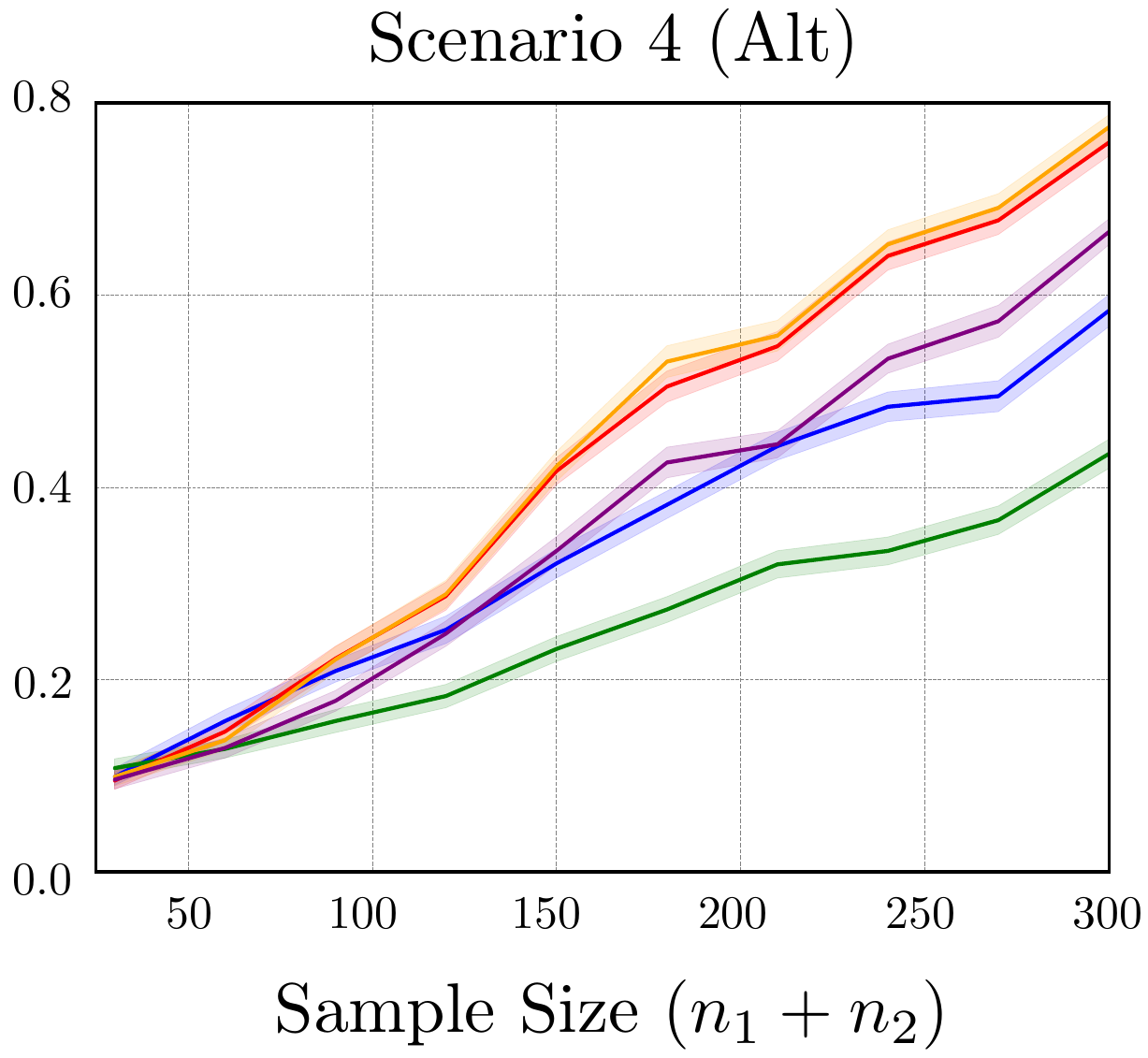}
    \end{subfigure}%
    \caption{Power comparisons across different dependence scenarios. The xssMMD tests, employing various regression methods, outperform existing approaches in the considered scenarios, particularly when $X$ and $Y$ exhibit strong dependence on $V$ and $W$.}
    \label{fig:power curves}
\end{figure*}

\noindent \textbf{Power Analysis.} \Cref{Theorem: xssMMD} shows that the xssMMD test achieves power at least as high as the xMMD test under certain conditions. We empirically validate this by comparing our test ($\Psi_{\mathrm{xss}}$) with the xMMD test ($\Psi_{\mathrm{x}}$) using a Gaussian kernel and the MMD-perm test ($B=200$ permutations). We perform simulations at $\alpha=0.05$ and present results from 1,000 trials in \Cref{fig:power curves}.

We consider the case of  $P_{V}=N(\mathbf{0}_{d}, \Sigma_V)$ and $P_{W}=N(\mathbf{a}_{\epsilon, j}, \Sigma_W)$, where $\mathbf{a}_{\epsilon, j} \in \mathbb{R}^d$ has its first $j$ entries equal to $\epsilon$ and the rest zero. We let $\Sigma_V=\Sigma_W=\rho\mathds{1}_d\mathds{1}^\top_d+(1-\rho)I_d$ and obtain $\{V_i\}_{i=1}^{n_1+m_1}$ by sampling $n_1+m_1$ independent samples from $P_V$. We then construct $\mathbb{V} = (V^\top_1, \ldots, V^\top_{n_1})^\top \in \mathbb{R}^{n_1 \times d}$ and obtain a set of $n_1$ labeled samples, $\mathbb{X} = \mathbb{V} \cdot \mathbf{b}$, where $\mathbf{b} = (b_i)_{i=1}^{d} \in \mathbb{R}^d$ with $b_i=1$ if $i$ belongs to an index set $\mathcal{I}$ and $b_i=0$ otherwise. A similar construction is applied to $\mathbb{Y}$ and $\mathbb{W}.$ Each scenario differs based on how we construct $X$ and $V$, which determines the dependence between the labeled and unlabeled data. In particular, we consider four scenarios: 

\noindent $\bullet$ \textbf{Scenario 1 (Alt).} $\rho=0.95$, $\mathcal{I}=\{1, d-1, d\}$

\noindent $\bullet$ \textbf{Scenario 2 (Alt).} $\rho=0.95$, $\mathcal{I}=[d]$

\noindent $\bullet$ \textbf{Scenario 3 (Alt).} $\rho=0.1$, $\mathcal{I}=\{1, d-1, d\}$

\noindent $\bullet$ \textbf{Scenario 4 (Alt).} $\rho=0.1$, $\mathcal{I}=[d]$

In all scenarios in \Cref{fig:power curves}, we fix parameters at $\epsilon=0.3$, $j=3$, $d=10$, $n_1/n_2=1$, and $n_1/m_1=n_2/m_2=0.1$, using a Gaussian kernel with the median heuristic. The main factors controlling the dependence between $X$ and $V$ are $\rho$ and $\mathbf{b}$. For example, in Scenarios 1 and 3 (Alt), $X$ and $Y$ are sums of the first and last two entries of $V$ and $W$, so the covariance vector has $1+2\rho$ for the first and last two entries and $3\rho$ otherwise. In contrast, in Scenarios 2 and 4 (Alt), $X$ and $Y$ are sums of all entries, yielding uniform covariance of $1+(d-1)\rho$. This leads to stronger dependence when $d\geq 3$, with larger $\rho$ further enhancing it and improving the performance of the xssMMD test.


As shown in \Cref{fig:power curves}, the xssMMD test significantly outperforms other methods when additional covariates strongly correlate with the labeled data (Scenarios 1 and 2). Even when the correlation is weaker (Scenarios 3 and 4), the xssMMD still demonstrates consistently better or comparable performance. These results highlight the advantage of leveraging auxiliary covariates, particularly when dependencies are strong. Further implementation details are provided in \Cref{appendix: Additional Experiments}.

\textbf{Experiment on HTRU2 dataset.} We next evaluate the performance of xssMMD using the HTRU2 dataset \citep{htru2_372}, which involves the classification of pulsars versus non-pulsars based on radio signal features of the integrated pulse profile (IP) and DM-SNR curve (DM). We examine several scenarios of labeled data with various levels of Gaussian noise added. A detailed description of the experimental setup and results are provided in \Cref{appendix: HTRU2 dataset}. As shown in \Cref{table: pulsar experiment}, xssMMD consistently outperforms baseline methods across most of the settings and noise levels. Even when Gaussian noise degrades the labeled data, xssMMD maintains a high power, leveraging auxiliary covariates effectively. These results highlight the strength of the method in extracting signals from complementary, unlabeled information under semi-supervised conditions.

\textbf{Experiment on Caltech-UCSD Bird dataset.} We also examine the performance of our proposed methods on the Caltech-UCSD Bird dataset~\citep{wah2011caltech}, which contains 11,788 images of 200 bird species. Each image has 10 detailed single-sentence descriptions collected through the Amazon Mechanical Turk (AMT) platform \citep{reed2016learning}. In this experiment, we conduct two-sample tests to detect differences among groups of birds categorized by their diet and habitat. A detailed description of the experimental setup is provided in \Cref{appendix: CUB dataset}. The MMD-perm and xMMD tests rely solely on textual descriptions, while the xssMMD test additionally incorporates image data as auxiliary covariates. As shown in \Cref{table: cub dataset}, the xssMMD test consistently outperforms the other tests in all cases. This confirms that the use of additional covariates improves the power of the test.  
This superior performance is likely due to strong dependency between the textual descriptions (labeled data) and images (unlabeled data), allowing xssMMD to extract informative representations from the additional covariates. 

\textbf{Experiment on MNIST dataset.} We further evaluate the performance of our proposed methods on the MNIST dataset \citep{lecun2010mnist}. We construct a testing problem to detect distributional differences between two distinct groups of handwritten digits. We use clean images as labeled data and images with Gaussian noise as unlabeled data. A detailed description of the experimental setup and results are provided in \Cref{appendix: MNIST dataset}. As shown in \Cref{table: mnist experiment}, the xssMMD test outperforms the MMD-perm and xMMD tests across most of the tested conditions, especially when the noise level is low. Despite the increase of the noise level, the power of the xssMMD test is still higher than that of the xMMD test. This result demonstrates that xssMMD effectively utilizes information from the additional covariates, even when those covariates are corrupted by Gaussian noise. The successful integration of noisy auxiliary data underscores the strength of the method, boosting the power in semi-supervised settings.

\begin{table}[t]
    \footnotesize
    \setlength{\tabcolsep}{3pt}
	\caption{Estimated test power for detecting the difference between two bird groups with test level $\alpha= 0.05$.}
	\label{table: cub dataset}
	\begin{center}
	\begin{tabular}{lllr}
	\toprule
	Group 1         & Group 2     & Test        & Power \\
	\midrule
	                &             & MMD-perm    & 0.957 \\
	Insect          & Forest      & xMMD        & 0.837 \\
                    &             & xssMMD      & \textbf{0.989} \\
	\hline
	                &               & MMD-perm    & 0.626 \\
	Fish			& Wetland       & xMMD        & 0.471 \\
					&             & xssMMD      & \textbf{0.808} \\
	\hline
    	           &             & MMD-perm    & 0.992 \\
	Seed		    & Scrub      & xMMD        & 0.920 \\
					&             & xssMMD      & \textbf{0.998} \\
	\bottomrule
	\end{tabular}
	\end{center}
\end{table}

\section{\MakeUppercase{Discussion}} \label{Section: Discussion}
In this paper, we present a semi-supervised framework for two-sample testing that incorporates both labeled and unlabeled covariate data to improve power while maintaining asymptotic level control. Leveraging sample-splitting and cross-fitting, the proposed method integrates covariate information and achieves asymptotic properties such as power consistency. Our analysis highlights the benefits of utilizing unlabeled data and provides conditions ensuring the validity of our tests. Along with numerical experiments, these results emphasize the potential of the framework as a theoretically sound tool for semi-supervised inference.

Several promising directions remain for exploration. First, extending the framework to broader contexts, such as $k$-sample testing and independence testing, would expand its applicability to complex semi-supervised problems. Exploring witness functions beyond MMD offers another avenue for future research. Moreover, studying methods for estimating conditional mean embeddings and exploring alternative variance reduction techniques, such as control covariates, may further refine the proposed framework.

\subsubsection*{Acknowledgements}
Ilmun Kim gratefully acknowledges support from the Korean government (RS-2023-00211073) and KAIST startup funding (KAIST-G04250059).

\bibliography{reference}

\section*{Checklist}

\begin{enumerate}

  \item For all models and algorithms presented, check if you include:
  \begin{enumerate}
    \item A clear description of the mathematical setting, assumptions, algorithm, and/or model. [Yes] The paper clearly states all assumptions for each theorem. Our general framework is outlined in \Cref{Section: General Semi-Supervised Two-Sample Test}. We subsequently detail our method in \Cref{Section: Semi-Supervised Kernel Two-Sample Test} and discuss assumptions in \Cref{appendix: Assumptions for Theorem xssMMD}.
    \item An analysis of the properties and complexity (time, space, sample size) of any algorithm. [Yes] In \Cref{Section: Theoretical Analysis}, we prove several asymptotic properties of our model. Further theoretical results and all proofs are provided in \Cref{appendix: Additional Theoretical Results} and \Cref{appendix: proof of main results}.
    \item (Optional) Anonymized source code, with specification of all dependencies, including external libraries. [Yes] The supplementary materials include all code and instructions, with dependencies clearly specified.
  \end{enumerate}

  \item For any theoretical claim, check if you include:
  \begin{enumerate}
    \item Statements of the full set of assumptions of all theoretical results. [Yes] We provide \Cref{Assumption: moment condition} for \Cref{Theorem: General Power Expression} and state additional assumptions for theoretical analysis in \Cref{Section: Theoretical Analysis}.
    \item Complete proofs of all theoretical results. [Yes] All proofs for our theoretical results are provided in \Cref{appendix: proof of main results}.
    \item Clear explanations of any assumptions. [Yes] All assumptions are provided with intuitive sketches and explanations.
  \end{enumerate}

  \item For all figures and tables that present empirical results, check if you include:
  \begin{enumerate}
    \item The code, data, and instructions needed to reproduce the main experimental results (either in the supplemental material or as a URL). [Yes] In the supplementary material, we provide code and detailed instructions to reproduce the experiments.
    \item All the training details (e.g., data splits, hyperparameters, how they were chosen). [Yes] \Cref{Section: Experiments} details our experiments, with further settings and additional results provided in \Cref{appendix: Additional Experiments}.
    \item A clear definition of the specific measure or statistics and error bars (e.g., with respect to the random seed after running experiments multiple times). [Yes] In \Cref{Section: Experiments} and \Cref{appendix: Additional Experiments}, we specify the experimental settings and describe the evaluation measures used in our experiments.
    \item A description of the computing infrastructure used. (e.g., type of GPUs, internal cluster, or cloud provider). [Yes] In \Cref{appendix: Additional Experiments}, we describe the computing environments used in the experiments.
  \end{enumerate}

  \item If you are using existing assets (e.g., code, data, models) or curating/releasing new assets, check if you include:
  \begin{enumerate}
    \item Citations of the creator If your work uses existing assets. [Yes] In \Cref{Section: Experiments}, several datasets used in our experiments are publicly available and cited properly.    
    \item The license information of the assets, if applicable. [Not Applicable]
    \item New assets either in the supplemental material or as a URL, if applicable. [Yes] In \Cref{Section: Experiments}, we demonstrate our method using several datasets. The supplementary materials provide the necessary code.
    \item Information about consent from data providers/curators. [Not Applicable]
    \item Discussion of sensible content if applicable, e.g., personally identifiable information or offensive content. [Not Applicable]
  \end{enumerate}

  \item If you used crowdsourcing or conducted research with human subjects, check if you include:
  \begin{enumerate}
    \item The full text of instructions given to participants and screenshots. [Not Applicable]
    \item Descriptions of potential participant risks, with links to Institutional Review Board (IRB) approvals if applicable. [Not Applicable]
    \item The estimated hourly wage paid to participants and the total amount spent on participant compensation. [Not Applicable]
  \end{enumerate}

\end{enumerate}

\clearpage
\appendix
\thispagestyle{empty}

\onecolumn
\aistatstitle{A Semi-Supervised Kernel Two-Sample Test: \\
Supplementary Materials}




\section{Overview of the xssMMD Framework and Theoretical Contributions} \label{appendix: overview and roadmap}

\subsection{Visual Overview of the xssMMD Framework}
To complement the main text, we include here an illustration of the xssMMD construction.
\begin{figure}[h!]	\centering
	\includegraphics[width=0.5\columnwidth]{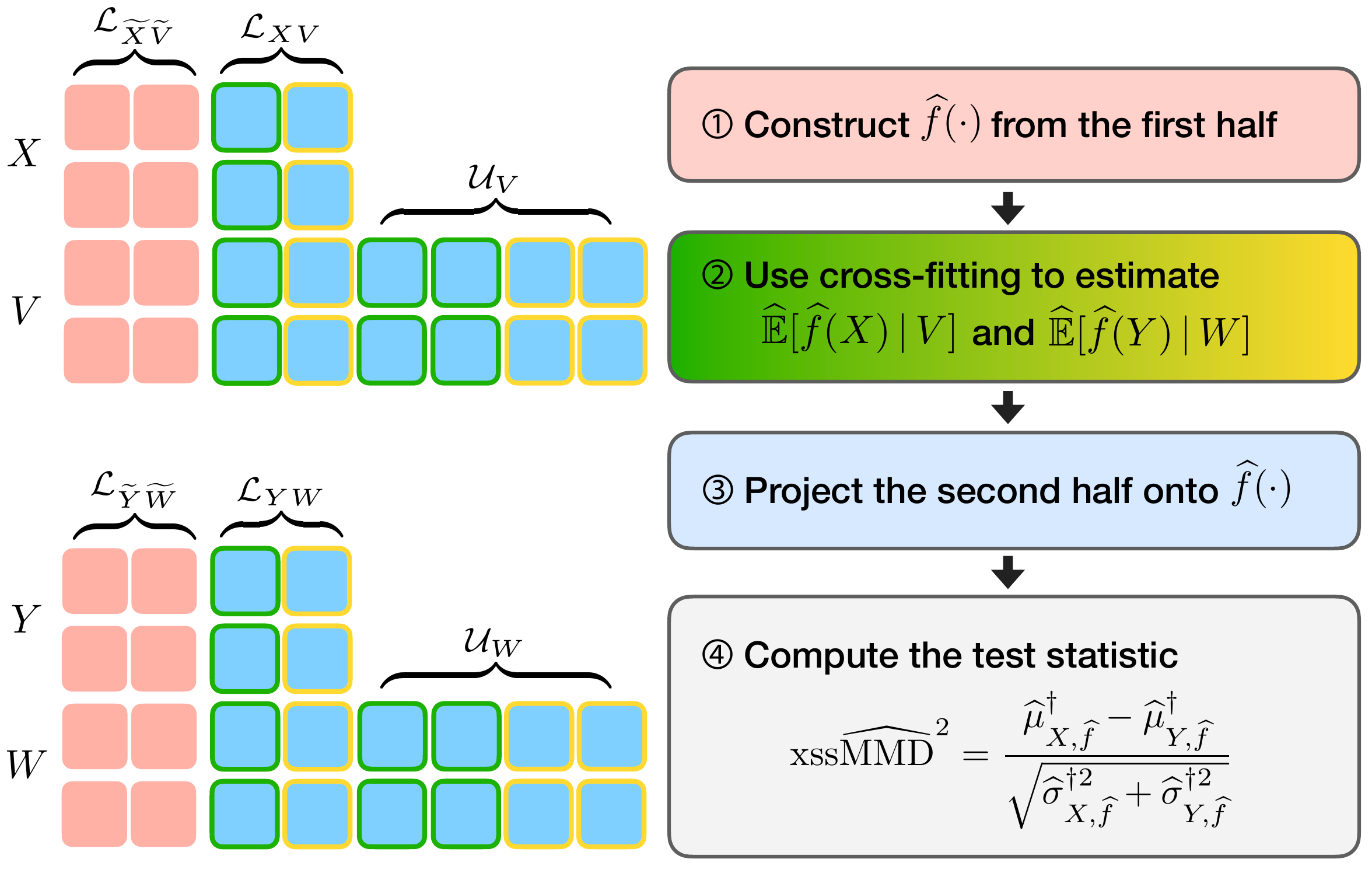}
	\caption{An illustration of the construction of the xssMMD statistic based on the same principles as the general framework and using an empirical estimate of the MMD witness function.}
    \label{fig:summary_of_testing}
\end{figure}

This figure provides a visual breakdown of how the test statistic $\xssMMD$ is derived in practice. Specifically, it illustrates the key steps involved in the cross-fitting procedure: splitting the data, estimating the witness function $f$ from the first half, and computing the statistic by projecting the second half onto the learned function $\fhat$. By leveraging auxiliary covariates (such as $V$ and $W$), the method estimates conditional expectations $\mE[\fhat(X)|V]$ and $\mE[\fhat(Y)|W]$, thereby integrating semi-supervised information into the testing framework. This enables the test to maintain nonparametric flexibility while enhancing power.

\subsection{Roadmap of Theoretical Results}
This section provides a structured overview of the theoretical guarantees established in this work. To facilitate navigation through the various lemmas, propositions, and theorems, we summarize our key theoretical contributions and their exact locations in Table \ref{tab:theoretical_summary_app}.

\begin{table}[!bhtp]
\centering
\footnotesize
\caption{Summary of Theoretical Results and Guarantees.}
\label{tab:theoretical_summary_app}
\begin{adjustbox}{max width=\linewidth}
\begin{tabular}{llc}
\toprule
Result & Brief Description & Location \\
\midrule
\Cref{Theorem: General Power Expression} & Asymptotic power and size control of the oracle semi-supervised test & \Cref{Section: Oracle Test} \\
\Cref{Corollary: power expression for cross-fit test} & Asymptotic power approximation of the cross-fit test to the oracle test & \Cref{Section: Procedure with CF} \\
\Cref{Theorem: xssMMD} & Asymptotic level and power guarantees of the xssMMD test compared to xMMD & \Cref{Section: Theoretical Analysis} \\
\Cref{Lemma: consistency} & Power consistency of the xssMMD test against fixed and local alternatives & \Cref{Section: Theoretical Analysis} \\
\midrule
\Cref{Theorem: consistency condition} & Power consistency of xssMMD under weaker conditions on estimated conditional expectations & \Cref{Section: Consistency in Power} \\
\Cref{corollary:fixed-alternative} & Power consistency of the xssMMD test against fixed alternatives & \Cref{Section: Consistency in Power} \\
\Cref{corollary: smooth local alternatives} & Power consistency of the xssMMD test against smooth local alternatives & \Cref{Section: Consistency in Power} \\
\Cref{Corollary: cross-fitting using linear operator} & Sufficient conditions for Assumption C.2 using linear smoothers & \Cref{Section: Linear Smoother} \\
\Cref{thm: asymptotic power expression} & Explicit asymptotic power expression of xssMMD under Gaussianity and a linear kernel & \Cref{Section: Asymptotic Power Expression using a Linear Kernel} \\
\bottomrule
\end{tabular}
\end{adjustbox}
\end{table}

\section{Notation} \label{Section: notation}

For a sequence of random variables $(X_n)_{n\geq 1}$ and another random variable $X$, we write $X_n \convD X$ when $X_n$ converges in distribution to $X$. Likewise, we write $X_n \convP X$ when $X_n$ converges in probability to $X$. For a sequence of positive numbers $(a_n)_{n\geq 1}$, we denote $a_n \lesssim b_n$ if there exists some constant $C>0$, which may depend on some fixed parameters, such that $a_n \leq C b_n$ for all $n \geq 1$. Also, we write $a_n \asymp b_n$ if there exist some positive constants $C_1, C_2$ such that $C_1 \leq |a_n/b_n| \leq C_2$ for all $n \geq 1$. We say $X_n=o_P(a_n)$ when $X_n/a_n \convP 0$, and $a_n=o(1)$ when $a_n \rightarrow 0$ as $n \rightarrow \infty$. Also, we write $a_n=O(1)$ when $|a_n| \leq C$ for some constant $C>0$ for all large $n$. The symbol $\Phi$ represents the cumulative distribution function of the standard normal random variable $N(0,1)$ and the $\alpha$ quantile of $N(0,1)$ is denoted as $z_{\alpha} = \Phi^{-1}(\alpha)$. For two real numbers $a$ and $b$, we use $a \wedge b$ and $a \vee b$ to denote $\min(a,b)$ and $\max(a,b)$, respectively. In numerical studies, we denote $\mathbf{0}_{d}$ as the all-zeros vector in $\mathbb{R}^{d}$, $\mathds{1}_{d}$ as the all-ones vector in $\mathbb{R}^{d}$, and $I_{d}$ as the $d \times d$ identity matrix.

\section{Detailed Discussion on Assumptions for \Cref{Theorem: xssMMD}} \label{appendix: Assumptions for Theorem xssMMD}

In this section, we formally define the centered kernel quantities and provide a more in-depth discussion of the assumptions introduced in \Cref{Section: Theoretical Analysis}. These assumptions are crucial for establishing the asymptotic properties of the xssMMD test and provide a framework for understanding the conditions under which the xssMMD test achieves improved performance. The assumptions involve key quantities defined through the centered kernel $\overline{k}_X$, which captures pairwise relationships while removing marginal effects. Specifically, $\overline{k} \coloneqq \overline{k}_X$ with respect to $X_1,X_2 \overset{\mathrm{i.i.d.}}{\sim} P_X$ is defined as:
\begin{align} \label{Eq: centered kernel} 
	\overline{k}(x_1,x_2) \coloneqq k(x_1,x_2) - \mE[k(X_1,X_2) \given X_1=x_1] - \mE[k(X_1,X_2) \given X_2=x_2] + \mE[k(X_1,X_2)]. 
\end{align}
Based on this, we further define $\overline{g}_X(x_1,x_2) \coloneqq \mE[\overline{k}(x_1,X_1)\overline{k}(x_2,X_1)]$, which quantifies the dependence structure of $X$ and encodes covariance-like properties. 

\textbf{Discussion on \Cref{Assumption: xssMMD under H0}:}
We note that similar conditions have been considered in the literature. For example, \citet[][page 23, condition (29b)]{kim2024dimension} introduced a related assumption in the context of high-dimensional testing with U-statistics. Likewise, \citet[][Theorem 5]{shekhar2022permutation} proposed a similar but stronger condition to establish the asymptotic normality of the test statistic in the xMMD test. This moment condition is also used in the proof of \citet[][Theorem 1]{li2024optimality}, where it is shown to hold for the Gaussian kernel with bandwidths that grow at specific rates relative to the sample size.

As highlighted in the main text, \Cref{Assumption: xssMMD under H0} is a Lyapunov-type condition. A set of bounded kernels (such as the Gaussian kernel used in our experiments) serves as a primary example of a class satisfying this condition, provided that the variance is lower-bounded by a constant. Because bounded kernels imply that the centered kernel functions are uniformly bounded, all higher-order moments in the numerator of \Cref{Assumption: xssMMD under H0} are bounded by finite constants. Meanwhile, the denominator, with the variance term not vanishing too rapidly, grows with the sample size $n$. Consequently, the ratio vanishes asymptotically ($o(1)$), implying that the assumption is satisfied for this practical class of kernels.

\textbf{Discussion on \Cref{Assumption: consistency of conditional expectation}:}
\Cref{Assumption: consistency of conditional expectation} ensures that the conditional expectations of $\fhat(X)$ and $\fhat(Y)$ are estimated with sufficient accuracy relative to their variances. Under these conditions, we rigorously show that the test statistic, after applying cross-fitting, is asymptotically equivalent to the oracle test statistic using the true conditional expectations.

Importantly, the boundedness of the kernel ensures that the witness function is uniformly bounded. In non-parametric regression theory \citep{gyorfi2006distribution}, the boundedness of the target function is a sufficient condition to establish the consistency of standard estimators like $k$-NN and kernel regression. Therefore, the use of a bounded kernel ensures that the numerator in \Cref{Assumption: consistency of conditional expectation} converges to zero, satisfying the condition for any consistent regression method. In \Cref{Section: Linear Smoother}, we provide a more concrete condition for the consistency of the conditional expectations in the context of linear smoothers.

\textbf{Discussion on \Cref{Assumption: xssMMD under H1}:}
The condition for bounded ratios between density functions is not necessary to obtain the asymptotic normality of the test statistic, but it greatly simplifies our conditions. In particular, this allows the expectations associated with $X$ and $Y$ to be comparable up to a constant factor, which results in a more concise expression of the condition. For instance, under the condition, we can write $\mE[\overline{g}_X(Y,Y)] \asymp \mE[\overline{g}_X(X,X)]$ which helps simplify the necessary conditions. The additional condition in \Cref{Assumption: xssMMD under H1} expected to be satisfied under a broad range of alternatives with small $\mathrm{MMD}(P_X,P_Y)$. At a high-level, this condition is derived while verifying the Lyapunov central limit theorem for the test statistic. For further details, refer to the proof of \Cref{Theorem: xssMMD}.small $\mathrm{MMD}(P_X,P_Y)$.

It is important to clarify the implications if \Cref{Assumption: xssMMD under H1} fails. Technically, this assumption is a sufficient condition for establishing asymptotic normality under the alternative, rather than a strict condition for variance reduction. Consequently, if it fails, we cannot use the analytic formula derived in \Cref{Theorem: xssMMD} to directly compare the power. However, this does not automatically imply that the asymptotic power of $\Psi_{xss}$ is lower than that of $\Psi_{x}$. 

That being said, there are specific scenarios where $\Psi_{xss}$ could underperform relative to $\Psi_{x}$. This typically occurs in finite-sample regimes when the conditional expectation is estimated poorly, leading to increased variance. For instance, if one uses a complex regression model (e.g., a deep neural network) on a small sample size where covariates $V$ are completely independent of $X$, the model may overfit to the noise in $V$. In this case, the estimated residuals will have higher variance than the original data, which can lead to a power loss compared to the supervised baseline $\Psi_{x}$. This highlights the important distinction between the condition required for asymptotic normality and the finite-sample estimation risks.

\section{Additional Theoretical Results}\label{appendix: Additional Theoretical Results}
In this section, we extend our theoretical findings in several directions. First, we compare the conditions under which the xssMMD test achieves power consistency with those of the xMMD test. We demonstrate that the xssMMD test attains consistency under weaker assumptions than the xMMD test, thereby highlighting its broader applicability. Moreover, we consider a linear smoothing approach for estimating conditional expectations, which enables us to derive a simplified version of \Cref{Theorem: xssMMD}.
 
\subsection{Consistency in Power} \label{Section: Consistency in Power}
In this subsection, we further examine the consistency in power of the xssMMD test and its relationship with the xMMD test. Notably, the xMMD statistic can be regarded as a special case of the xssMMD statistic with $m_1 = m_2 = 0$. To explicitly define the xMMD statistic, we set $\widetilde{\mu}_{X,\fhat} = n_1^{-1} \sum_{i=1}^{n_1} \widehat{f}(X_i)$, $\widetilde{\mu}_{Y,\fhat} = n_2^{-1} \sum_{i=1}^{n_2} \widehat{f}(Y_i)$, $\widetilde{\sigma}^2_{X,\fhat} = n_1^{-2} \sum_{i=1}^{n_1} \{\widehat{f}(X_i) - \widetilde{\mu}_{X,\fhat}\}^2$, and $\widetilde{\sigma}^2_{Y,\fhat} = n_2^{-2} \sum_{i=1}^{n_2} \{\widehat{f}(Y_i) - \widetilde{\mu}_{Y,\fhat}\}^2$. The xMMD test statistic is then defined as
\begin{align*}
	\xMMD = \frac{\widetilde{\mu}_{X,\fhat} - \widetilde{\mu}_{Y,\fhat}}{\sqrt{\widetilde{\sigma}^2_{X,\fhat} + \widetilde{\sigma}^2_{Y,\fhat}}}.
\end{align*}
Similarly to \Cref{Lemma: consistency}, \citet[][Theorem 8]{shekhar2022permutation} shows that the xMMD test is consistent in power under the following condition:
\begin{align} \label{Eq: xMMD consistency}
	\sup_{(P_{X,n},P_{Y,n}) \in \mathcal{P}_n} \Biggl\{ \frac{\mE_{P_{X,n},P_{Y,n}}\bigl[\widetilde{\sigma}^2_{X,\fhat} + \widetilde{\sigma}^2_{Y,\fhat}\bigr]}{\delta_n\gamma_n^4}  + \frac{\mathrm{Var}_{P_{X,n},P_{Y,n}}\bigl[\widetilde{\mu}_{X,\fhat} - \widetilde{\mu}_{Y,\fhat}\bigr]}{\gamma_n^4} \Biggr\} = o(1),
\end{align}
where $\delta_n$ is any positive sequence converging to zero and $\gamma_n = \mathrm{MMD}(P_{X,n},P_{Y,n})$. We now show that the condition in \eqref{Eq: xMMD consistency} is stronger than that in \Cref{Lemma: consistency} whenever
the second moments of the residuals $\widehat{f}(X) - \widehat{\mE}[\widehat{f}(X) \given V, \widehat{f}]$ and $\widehat{f}(Y) - \widehat{\mE}[\widehat{f}(Y) \given W, \widehat{f}]$ are comparable to the variances of $\widehat{f}(X)$ and $\widehat{f}(Y)$, respectively. 
\begin{theorem} \label{Theorem: consistency condition}
	Suppose the consistency condition for the xMMD test in \eqref{Eq: xMMD consistency} holds. Moreover, suppose that $\widehat{\mE}[\fhat(X) \given V, \fhat] \coloneqq \widehat{u}_X(V)$ and $\widehat{\mE}[\fhat(Y) \given W, \fhat] \coloneqq \widehat{u}_Y(W)$ satisfy the following conditions:
	\begin{align} \label{Eq: consistency condition}
		& \frac{\mE\bigl[\bigl\{\fhat(X) - \widehat{u}_X(V) \bigr\}^2 \bigr]}{\mE\bigl[\mV\bigl\{\fhat(X) \given \fhat \bigr\}\bigr]} \lesssim 1 \quad \text{and} \quad  \frac{\mE\bigl[\bigl\{\fhat(Y) - \widehat{u}_Y(W) \bigr\}^2 \bigr]}{\mE\bigl[\mV\bigl\{\fhat(Y) \given \fhat \bigr\}\bigr]} \lesssim 1.
	\end{align}
	Then \Cref{Lemma: consistency} remains valid.
	\begin{proof}
		The proof can be found in Appendix \ref{Section: proof of Theorem consistency condition}.
	\end{proof}
\end{theorem}
As emphasized in the main text, the conditions in \eqref{Eq: consistency condition} are much weaker than the full consistency of the conditional expectations $\widehat{\mE}[\fhat(X) \given V, \fhat]$ and $\widehat{\mE}[\fhat(Y) \given W, \fhat]$. In particular, when $\widehat{\mE}[\fhat(X) \given V, \fhat]$ and $\widehat{\mE}[\fhat(Y) \given W, \fhat]$ are the true conditional expectations, the conditions in \eqref{Eq: consistency condition} are automatically satisfied by the law of total variance. This finding demonstrates that the xssMMD test is consistent in power whenever the xMMD test is consistent in power under weak conditions on the residuals of the estimated conditional expectations. Furthermore, this result can be applied to establish other consistency results from \citet{shekhar2022permutation} such as consistency against fixed alternatives and against $L_2$ local alternatives.

\textbf{Fixed alternatives.} We begin by applying \Cref{Theorem: consistency condition} to the setting where $P_X$ and $P_Y$ are fixed distributions, and show that the xssMMD test equipped with a characteristic kernel achieves asymptotic power of one in distinguishing $P_X$ and $P_Y$.
\begin{corollary}
    \label{corollary:fixed-alternative}
	Suppose that distributions $P_X, P_Y$ and a kernel $k$ do not vary with $n$. If $k$ is a characteristic kernel with $\mE_{P_X}[{k}(X_1,X_1)] < \infty$ and $\mE_{P_Y}[{k}(Y_1,Y_1)] < \infty$, and condition \eqref{Eq: consistency condition} holds, then the {\normalfont xssMMD} test is consistent against the fixed alternative $H_1: P_X \neq P_Y$.
\end{corollary}
The proof of this statement is given in~\Cref{proof:fixed-alternative}. 

\textbf{Smooth local alternatives.} We next demonstrate \Cref{Theorem: consistency condition} to the setting where the distributions $P_{X,n}$ and $P_{Y,n}$ admit Lebesgue densities $p_{X,n}$ and $p_{Y,n}$ which belong to a Sobolev ball of order $\beta$ for some $\beta>0$. Specifically, we consider the following class of smooth densities:
\begin{align*}
	\mathcal{W}^{\beta, 2}(M) \coloneqq \left\{ f: \mathcal{X} \rightarrow \mathbb{R} \mid f \text{ is almost surely continuous and} \ \int \left(1+\omega^2\right)^{\beta/2} \|\mathcal{F}(f)(\omega)\|^2 d\omega < M < \infty \right\},
\end{align*}
where $\mathcal{F}(f)$ is the Fourier transform of $f$ and $\|\cdot\|$ denotes the Euclidean norm. We then define a class of alternative distributions that is $\Delta_n$-close to the null hypothesis in the $L_2$-norm:
\begin{equation*}
\begin{aligned}
\mathcal{P}^{(1)}_n=\left\{(P_X, P_Y) \text { with densities } p_X, p_Y \in \mathcal{W}^{\beta, 2}(M):\|p_X-p_Y\|_{L_2} \geq \Delta_n\right\}
\end{aligned}
\end{equation*}
for some sequence $\Delta_n$ decaying to zero. The next theorem, which is the corresponding result to \citet[][Theorem 9]{shekhar2022permutation}, establishes the consistency in power of the xssMMD test against $\mathcal{P}^{(1)}_n$.

\begin{corollary}\label{corollary: smooth local alternatives}
 Consider the {\normalfont xssMMD} test $\Psi_{\mathrm{xss}}$ with the Gaussian kernel $k_{s_n}(x, y)=\exp(-s_n\|x-y\|^2)$ with the scale parameter $s_n \asymp n^{4 /(d+4 \beta)}$. If condition \eqref{Eq: consistency condition} holds and $\lim _{n \rightarrow \infty} \Delta_n n^{2 \beta /(d+4 \beta)} = \infty$ with $n_1=n_2=n$, then $\Psi_{\mathrm{xss}}$  is consistent against $\mathcal{P}^{(1)}_n$ as
$$
\lim _{n \rightarrow \infty} \inf _{\left(P_{X,n}, P_{Y,n} \right) \in \mathcal{P}_n^{(1)}} \mathbb{E}_{P_{X,n}, P_{Y,n}}[\Psi_{\mathrm{xss}}]=1.
$$
\begin{proof}
    We omit the proof of \Cref{corollary: smooth local alternatives} since it is a direct consequence of \Cref{Theorem: consistency condition} and the proof of \citet[][Theorem 9]{shekhar2022permutation}. 
\end{proof}
\end{corollary} 
The corollary shows that the xssMMD test is consistent against smooth local alternatives under the same conditions as the xMMD test. This result highlights that the xssMMD test achieves the same separation rate as the xMMD test, provided the estimated conditional expectations satisfy the requirements in \eqref{Eq: consistency condition}.

\subsection{Linear Smoother for Estimating Conditional Expectations} \label{Section: Linear Smoother}
In this subsection, we consider a linear smoother (e.g., $k$-nearest neighbors and kernel regression) for estimating conditional expectations, which provides a more interpretable condition for the consistency of the conditional expectations in \Cref{Assumption: consistency of conditional expectation}. The simplification of condition (\ref{Eq: cross-fitting condition}) is derived using the spectral decomposition of the centered kernel $\overline{k}$ in \eqref{Eq: centered kernel}:
\begin{align*}
\overline{k}(x_1, x_2) = \sum_{i=1}^\infty \lambda_i \phi_i(x_1) \phi_i(x_2),
\end{align*}
where $\{\lambda_i\}_{i=1}^\infty$ are the eigenvalues, and $\{\phi_i\}_{i=1}^{\infty}$ are the orthonormal eigenfunctions under the marginal distribution of \( X \). Given this decomposition, we express the conditional expectation and variance in terms of eigenfunctions and derive a clearer condition as follows:

\begin{corollary}\label{Corollary: cross-fitting using linear operator}
    Suppose that the estimators for the conditional expectations satisfy $\widehat{\mE}[a \fhat(X) +b \given V, \fhat] = a\widehat{\mE}[\fhat(X) \given V, \fhat] + b$ and $\widehat{\mE}[a \fhat(Y) +b \given W, \fhat] = a\widehat{\mE}[\fhat(Y) \given W, \fhat] + b$ for all $a,b \in \mathbb{R}$. Suppose further that
    \begin{align}
	   \sup_{i \geq 1} \mE[\Delta_{X,i}^2] = o(1),\quad\text{and}\quad \sup_{i \geq 1} \mE[\Delta_{Y,i}^2] = o(1) \label{Eq: linear operator}
    \end{align} 
	where $\Delta_{X,i} = \mE[\phi_i(X) \given V]  - \widehat{\mE}[\phi_i(X) \given V]$ and $\Delta_{Y,i} = \mE[\phi_i(Y) \given W]  - \widehat{\mE}[\phi_i(Y) \given W]$, and \Cref{Assumption: xssMMD under H0} holds with $n_1 \asymp m_1$ and $n_2 \asymp m_2$. Then \Cref{Assumption: consistency of conditional expectation} is satisfied.
    \begin{proof} 
        The proof can be found in \Cref{Section: proof of Corollary linear operator}. 
    \end{proof}
\end{corollary}
The condition in \eqref{Eq: linear operator} essentially states that if the regression estimator is linear and consistent for estimating the conditional expectations of the eigenfunctions $\mE[\phi_i(X) \given V]$ and $\mE[\phi_i(Y) \given W]$, then \Cref{Assumption: consistency of conditional expectation} holds. This condition is particularly notable as it translates the relatively abstract stochastic requirement in \eqref{Eq: cross-fitting condition} into a deterministic one that does not depend on $\fhat$. Moreover, we note that the linearity of the regression estimator can be relaxed to asymptotic linearity with more technical efforts.

\subsection{Asymptotic Power Expression using a Linear Kernel}\label{Section: Asymptotic Power Expression using a Linear Kernel}
In the main text, we showed that the xssMMD test achieves asymptotic power at least as high as that of the xMMD test while maintaining controlled type-I error. To further support this finding, we formalize the asymptotic expression of the power of our statistic using the linear kernel $k(x,y) =  \langle x, y \rangle$, and compare it to the heuristic results of the kernel-MMD test and the xMMD test. Before proceeding with further discussion, let us assume the following conditions to ease our analysis.

\begin{assumption}\label{Assumption: assumptions for asymptotic power expression}
Suppose that the following assumptions are satisfied.

(a) Gaussianity: We observe $d$-dimensional i.i.d.~copies of random vectors $(X,V)^\top$ and $(V,W)^\top$ from a Gaussian distributions \begin{align*}
    P_{XV}=\begin{pmatrix}P_X\\
P_V
\end{pmatrix} &\sim  N
\begin{pmatrix}
\begin{pmatrix}
\mu_X\\
\mu_V
\end{pmatrix}\!\!,&
\begin{pmatrix}
\Sigma_{11} & \Sigma_{12}\\
\Sigma_{21} & \Sigma_{22}
\end{pmatrix}
\end{pmatrix} \quad \text{and} \\
P_{YW}=\begin{pmatrix}P_Y\\
P_W
\end{pmatrix} &\sim  N
\begin{pmatrix}
\begin{pmatrix}
\mu_Y\\
\mu_W
\end{pmatrix}\!\!,&
\begin{pmatrix}
\Sigma_{11} & \Sigma_{12}\\
\Sigma_{21} & \Sigma_{22}
\end{pmatrix}
\end{pmatrix}.
\end{align*}

(b) Bounded eigenvalues: For $i=1,2,$ there exist constants $c$ and $C>0$ such that $c \leq \lambda_1(\Sigma_{ii})\leq \cdots \leq \lambda_d(\Sigma_{ii}) \leq C.$

(c) Local alternative: $\mu_X^\top \mu_X=O(\sqrt{d}/n_1),$ $\mu_Y^\top \mu_Y=O(\sqrt{d}/n_2).$ 

(d) Dimension-to-sample size ratio: $d/n_1 \rightarrow \tau_1 \in (0,\infty),$ $d/n_2 \rightarrow \tau_2 \in (0,\infty).$

(e) Labeled-unlabeled sample size ratio: $m_1/(n_1+m_1) \rightarrow r_1 \in (0,1),$ $m_2/(n_2+m_2) \rightarrow r_2 \in (0,1)$.
\end{assumption}

We note that these conditions are only necessary for deriving the concrete, asymptotic power expression of the proposed test. These conditions are analogous to those given by \citet[][Assumption 2.5 ]{kim2024dimension} and \Cref{Assumption: assumptions for asymptotic power expression} can be seen as its two-sample testing version extension in the semi-supervised setting. 

For the xssMMD test defined as $\Psi_{\mathrm{xss}} \coloneqq \mathds{1}(\xssMMD > z_{1-\alpha})$, we analyze its power assuming the previous \Cref{Assumption: assumptions for asymptotic power expression} holds. 

\begin{theorem}\label{thm: asymptotic power expression}
    Suppose that \Cref{Assumption: consistency of conditional expectation} and \Cref{Assumption: assumptions for asymptotic power expression} are fulfilled under the alternative. Assume that $(X,V)^\top$ and $(V,W)^\top$ have equal sample sizes and equal covariance matrices, i.e., $n_1=n_2=m_1=m_2=n$ and $\Sigma_{ij}=\tilde{\Sigma}_{ij}$ for $i,j\in \{1,2\}$. Then, it holds that 
    \begin{align*}
        \mE[\Psi_{\mathrm{xss}}]
    =\Phi\Biggl(z_\alpha + \frac{n(\mu_X-\mu_Y)^\top(\mu_X-\mu_Y)}{\sqrt{4\tr(\Sigma_{11}^2)-2\tr(\Sigma_{12}\Sigma^{-1}_{22}\Sigma_{21}\Sigma_{11})}}\Biggr)+o_P(1).
    \end{align*}
\begin{proof}
    The proof can be found in \Cref{Section: proof of power using linear kernel}
\end{proof}
\end{theorem}

We note that the constant value of $2$ comes from $4r$ where $r$ denotes the labeled-unlabeled sample size ratio $r \coloneqq r_1 = r_2$ and we assumed $r=1/2$ in the above theorem. From this, we can also obtain the result when there is no unlabeled data, meaning $r=0,$ 
which is identical to the heuristic result of the permutation-free kernel two-sample test from \citet{shekhar2022permutation}. This suggests that our finding can be viewed as an extension of the earlier results from the xMMD test, incorporating additional covariates, which may result in a reduction of variance calculated in the denominator and an increase in power.

Conversely, the asymptotic expressions for the power of the kernel two-sample test using permutation on MMD, as suggested by \citet{gretton2012kernel}, denoted as $\Psi_{\mathrm{perm}}$, and the permutation-free kernel two-sample test using studentized MMD by \citet{shekhar2022permutation}, represented as $\Psi_{\mathrm{x}}$, can be expressed as follows:
\begin{align*}
    & \mE[\Psi_{\mathrm{perm}}]
     \approx \Phi\Biggl(z_\alpha + \frac{n(\mu_X-\mu_Y)^\top(\mu_X-\mu_Y)}{\sqrt{2\tr(\Sigma_{11}^2)}}\Biggr),\quad\text{and} \\
    & \mE[\Psi_{\mathrm{x}}]
    \approx \Phi\Biggl(z_\alpha + \frac{n(\mu_X-\mu_Y)^\top(\mu_X-\mu_Y)}{\sqrt{4\tr(\Sigma_{11}^2)}}\Biggr).
\end{align*} 
We note that the result for $\Psi_{\mathrm{perm}}$ is estimated in heuristic manner, while that of $\Psi_{\mathrm{x}}$ is derived from its asymptotic normality under the alternative, shown in \Cref{Theorem: xssMMD}. Observe that the power of $\Psi_{\mathrm{x}}$ is $\sqrt{2}$ times lower than that of $\Psi_{\mathrm{perm}}$, a result stemming from sample-splitting. In this context, the value of our test statistic is straightforward, as we initially assumed there were $2n_1$ and $2n_2$ labeled samples at first, then utilized one half to construct the witness function and the other to compute the studentized statistic. These distinctions in power highlight the inherent trade-off introduced by sample-splitting, where the power of the test is reduced in exchange for computational efficiency. However, incorporating additional covariates into our xssMMD framework mitigates this drawback by utilizing the unlabeled data to reduce variance, thereby narrowing the gap in power performance while maintaining robustness. This is evident when comparing the powers of $\Psi_{\mathrm{x}}$ and $\Psi_{\mathrm{xss}},$ where $\tr(\Sigma_{12}\Sigma^{-1}_{22}\Sigma_{21}\Sigma_{11}) > 0 $ indicates that $\Psi_{\mathrm{xss}}$ has greater power than $\Psi_{\mathrm{x}}.$ Furthermore, under \Cref{Assumption: assumptions for asymptotic power expression}, we identified explicit conditions under which the power of $\Psi_{\mathrm{xss}}$ exceeds that of $\Psi_{\mathrm{perm}}.$ Specifically, when $2\tr(\Sigma_{11}^2) < \tr(\Sigma_{12}\Sigma^{-1}_{22}\Sigma_{21}\Sigma_{11}),$ $\Psi_{\mathrm{xss}}$ demonstrates superior power compared to $\Psi_{\mathrm{perm}},$ which implies that leveraging highly informative additional covariates can enhance the power. This underscores the significant practical advantage of our proposed test, especially in scenarios where unlabeled data is abundant and can be effectively utilized.

\section{Additional Experiments}\label{appendix: Additional Experiments}
In this section, we present additional numerical results and provide detailed information about our experiments. In implementing our proposal, we incorporate various methods for conditional expectation estimation, including $k$-nearest neighbors (knn), kernel regression, and random forest (rf). In all tables throughout this section, boldface indicates the best performance: the highest test power under the alternative, and the lowest Type-I error rate under the null.

We have limited our scope to a standard kernel-based baseline (e.g., MMD with a Gaussian kernel using median heuristic) to clearly isolate and evaluate the contribution of unlabeled data to statistical power. Incorporating recent advanced methods \citep[e.g.,][]{biggs2024mmd,schrab2023mmd,kubler2022witness} could potentially yield stronger empirical performance. However, doing so at this stage may conflate gains attributable to semi-supervised information with those arising from more refined kernel choices. We believe that establishing the value of unlabeled data in a controlled setting is a necessary first step.

Experiments under the null and alternative are lightweight and can be conducted on a local machine without GPU acceleration, taking approximately an hour each. The HTRU2 experiments run efficiently on CPU and complete within minutes. The MNIST experiments also run on CPU but complete within an hour. In contrast, the CUB experiments require computing image embeddings using a pretrained model, for which GPU acceleration is beneficial. Once embeddings are obtained, the remaining computations are lightweight. We used an NVIDIA RTX A6000 GPU and it is done within minutes for each embedding setting. Reproducible code is available at \href{https://github.com/gyumin-lee68/ssk2st}{https://github.com/gyumin-lee68/ssk2st} under the MIT License.

\subsection{Limiting Null Distribution with Different Settings}\label{appendix: experiments under the null}

Along with \Cref{Section: Experiments}, this subsection examines the asymptotic normality of the xssMMD test statistic under the null across various scenarios and parameter settings. 

\begin{figure*}[t!]
    \centering
    \begin{subfigure}[t]{0.225\textwidth} 
        \centering
        \includegraphics[width=\linewidth]{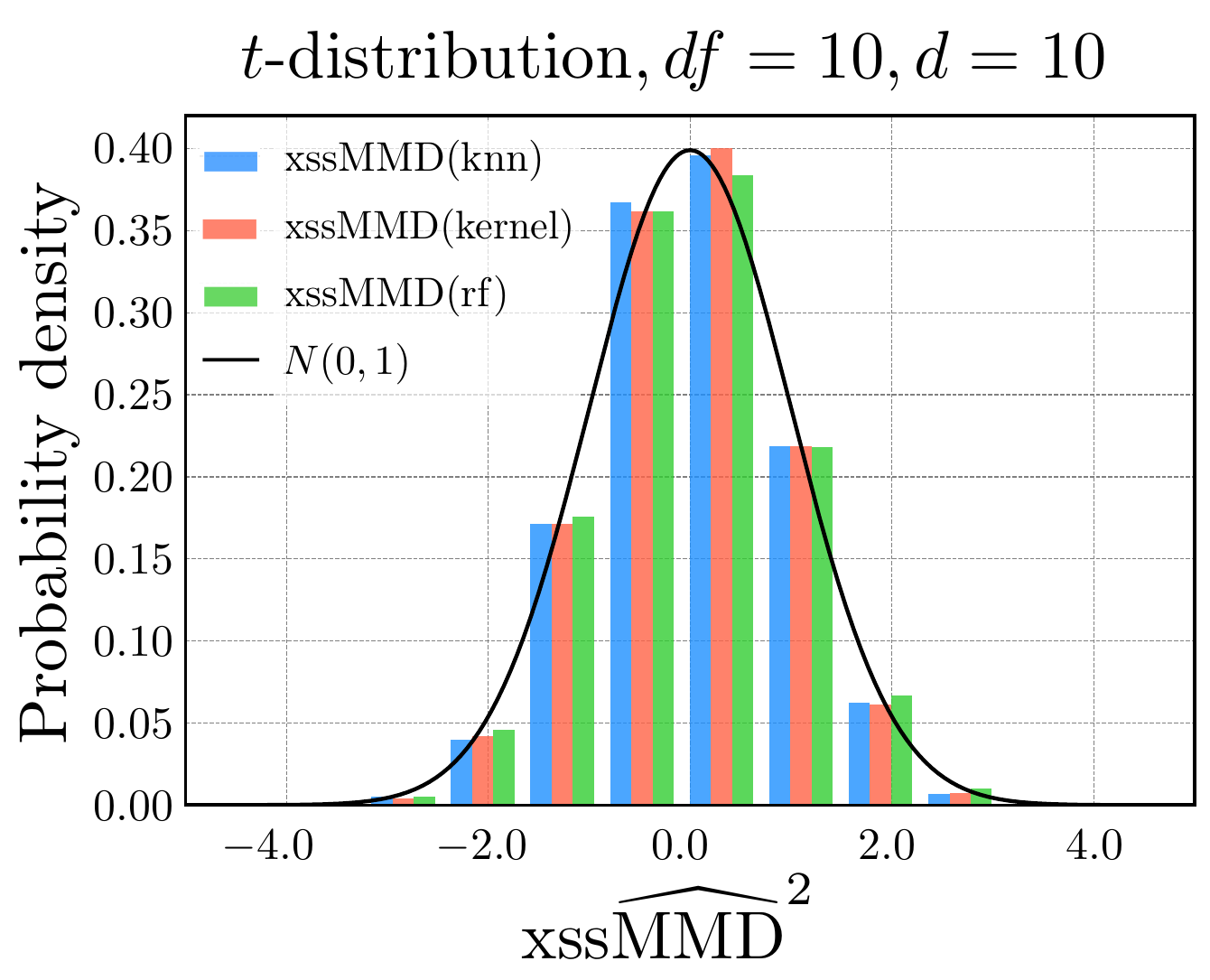}
    \end{subfigure}
    \hfill
    \begin{subfigure}[t]{0.21\textwidth}
        \centering
        \includegraphics[width=\linewidth]{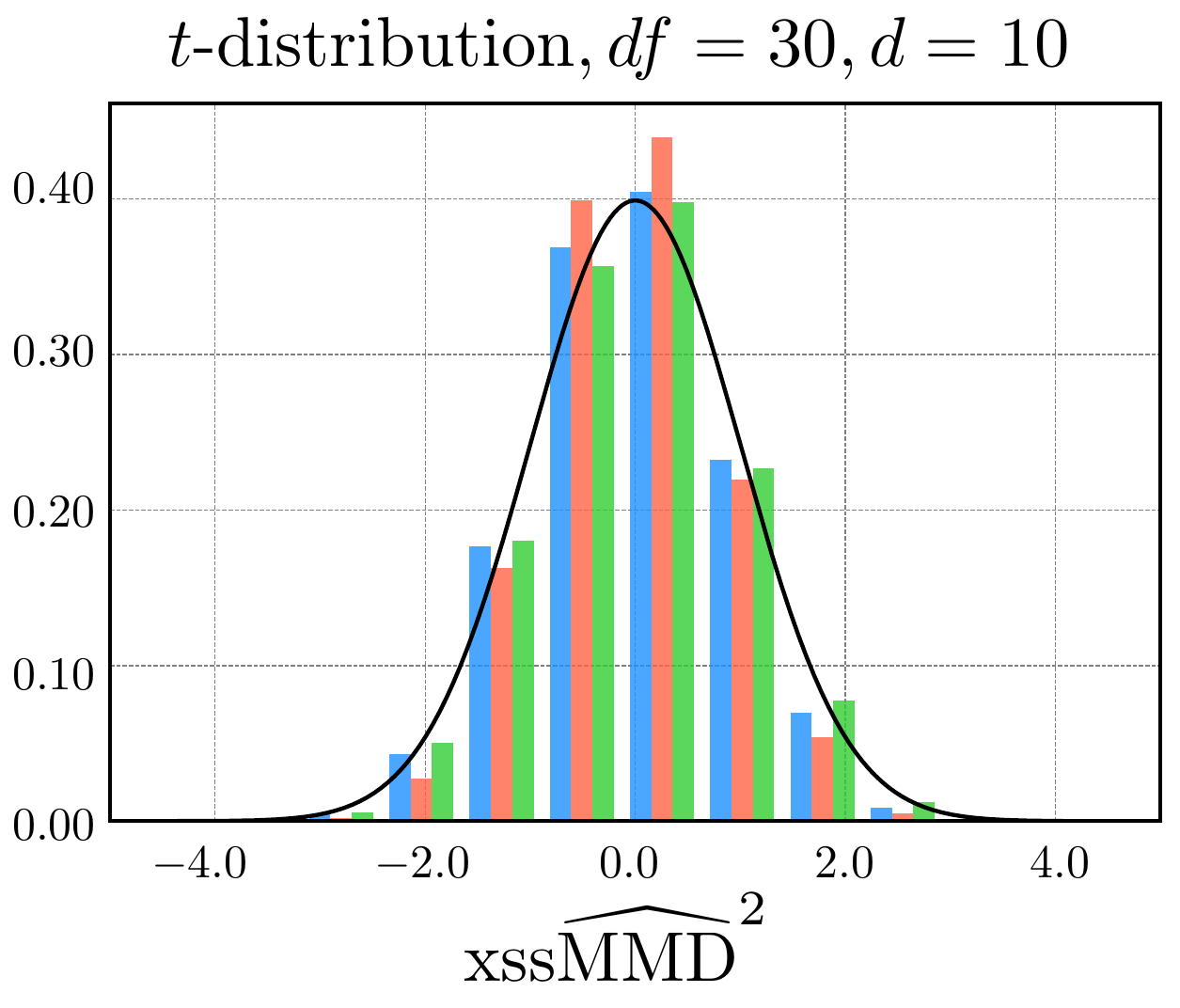}
    \end{subfigure}
    \hfill
    \begin{subfigure}[t]{0.21\textwidth} 
        \centering
        \includegraphics[width=\linewidth]{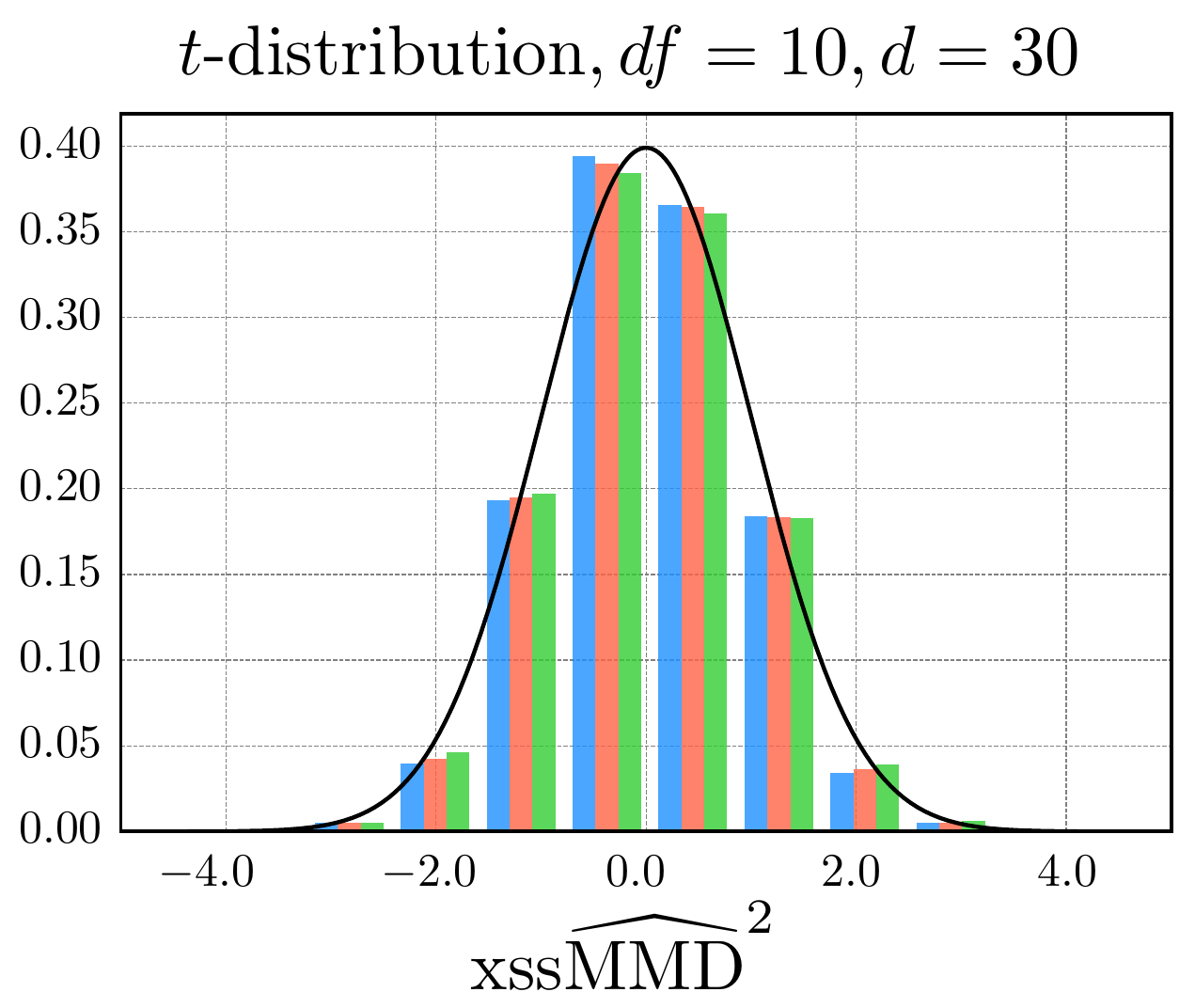}
    \end{subfigure}
    \hfill
    \begin{subfigure}[t]{0.21\textwidth} 
        \centering
        \includegraphics[width=\linewidth]{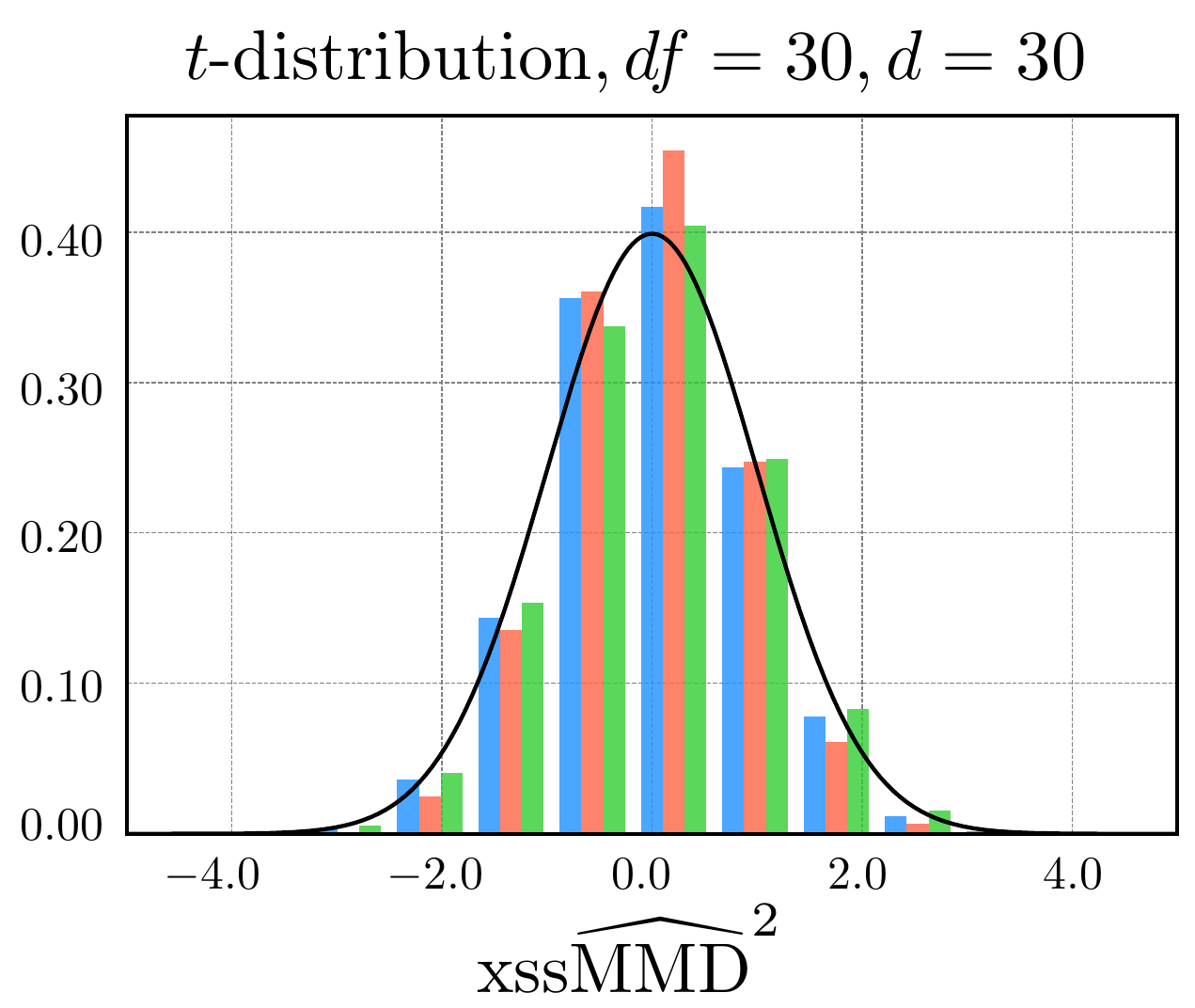}
    \end{subfigure}
    \caption{Experimental results for the distribution of $\xssMMD$ under the null, using $t$-distributed data with varying dimension and degrees of freedom. The plots demonstrate that the test statistic $\xssMMD$ asymptotically follows a $N(0,1)$ distribution under the null, even when the data deviates from Gaussianity.}
    \label{fig:additional limiting null distributions 1}
\end{figure*}
To validate the asymptotic normality of the xssMMD test under the null, we conduct experiments across various settings, including different data distributions, dimensions, and dependency structures. First, we examine the behavior of the test statistic when the data follows a $t$-distribution. The results are presented in \Cref{fig:additional limiting null distributions 1}, with varying degrees of freedom $df$ and dimension $d$: $df=10$ \& $d=10,$ $df=30$ \& $d=10,$ $df=10$ \& $d=30,$ and $df=30$ \& $d=30$ from left to right. Note that we fixed the other parameters as $n_1=n_2=100$ and $m_1=m_2=200$ using a Gaussian kernel with the median heuristic. The results confirm that $\xssMMD$ follows $N(0,1)$ under the null, even when the data deviate from Gaussianity. This demonstrates that the asymptotic normality of the proposed statistic remains robust across different distributional settings.

Next, we show that the xssMMD test consistently achieves asymptotic normality under the null, regardless of the specific factors outlined in \Cref{Section: Experiments}. In detail, we demonstrate that its asymptotic normality is consistently achieved despite the effects of dimensionality, sample skewness, labeled-unlabeled sample size ratio, methods for estimating conditional expectation, and choice of kernel on the null distribution of the test statistic. The experimental results are summarized in \Cref{fig:additional limiting null distributions 2} whose each column represents the different cases of sample skewness $r_{sample}\coloneqq n_1/n_2$ and labeled-unlabeled sample size ratio $r_{label}\coloneqq n_1/m_1 = n_2/m_2$: $r_{sample}=1$ \& $r_{label}=1,$ $r_{sample}=0.1$ \& $r_{label}=0.5,$ $r_{sample}=1$ \& $r_{label}=1,$ $r_{sample}=0.1$ \& $r_{label}=0.5$ from left to right with fixed $n_1=100.$ Each row corresponds to the different cases based on dimension $d$ and kernel choice: $d=10$ \& bilinear kernel, $d=100$ \& bilinear kernel, $d=10$ \& Gaussian kernel, $d=100$ \& Gaussian kernel from top to bottom. Note that we used the median heuristic as the bandwidth for the Gaussian kernel. These results confirm that $\xssMMD$ follows $N(0,1)$ under the null, regardless of variations in dimensionality, sample skewness, the labeled-unlabeled sample size ratio, estimation methods for conditional expectations, and kernel choice. This robustness further supports the validity of our theoretical findings, demonstrating that the asymptotic normality of the xssMMD test consistently holds across diverse settings. 

\begin{figure*}[b!]
    \centering
    \begin{subfigure}[t]{0.22\textwidth} 
        \centering
        \includegraphics[width=\linewidth]{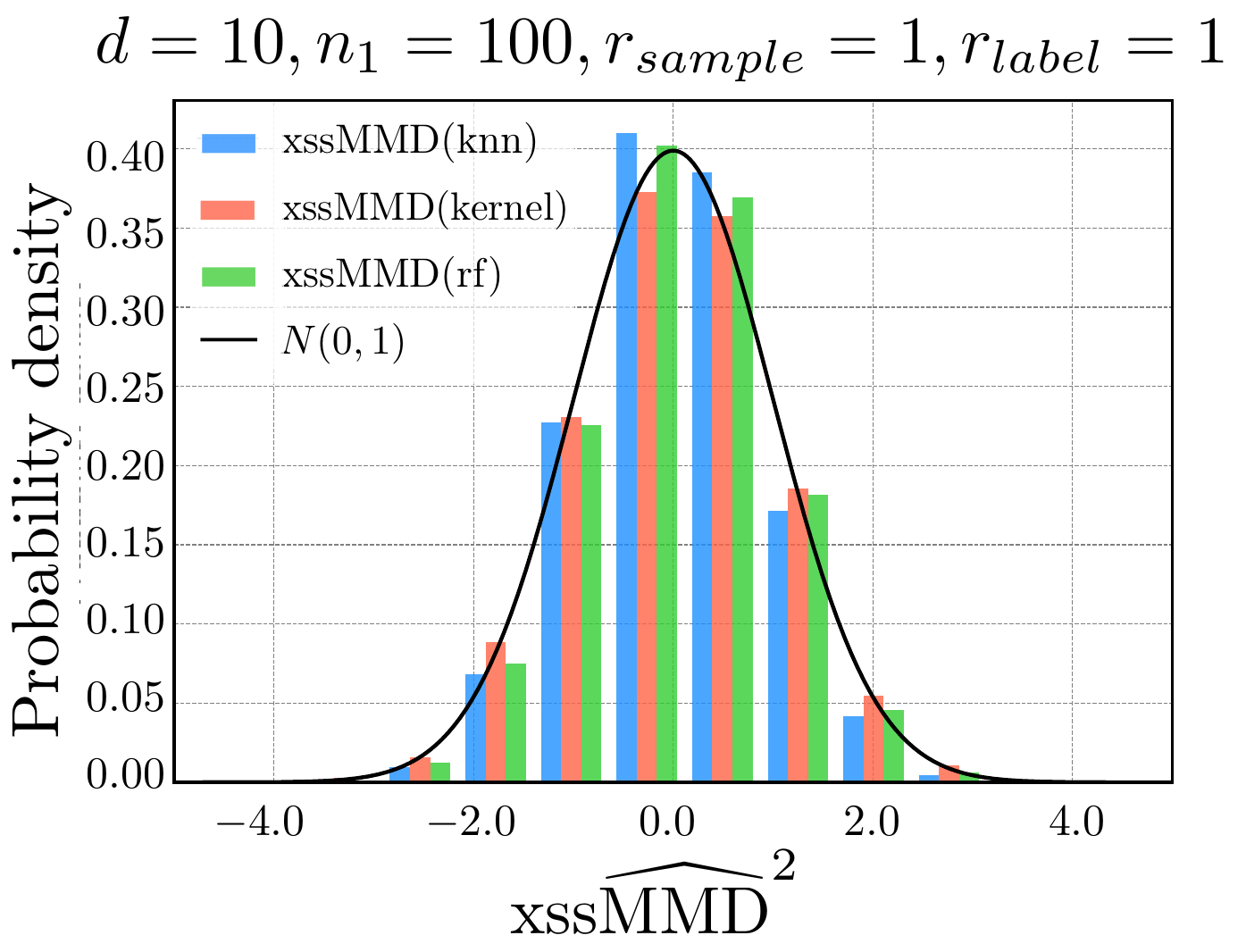}
    \end{subfigure}
    \hspace{0.01\textwidth}
    \begin{subfigure}[t]{0.21\textwidth}
        \centering
        \includegraphics[width=\linewidth]{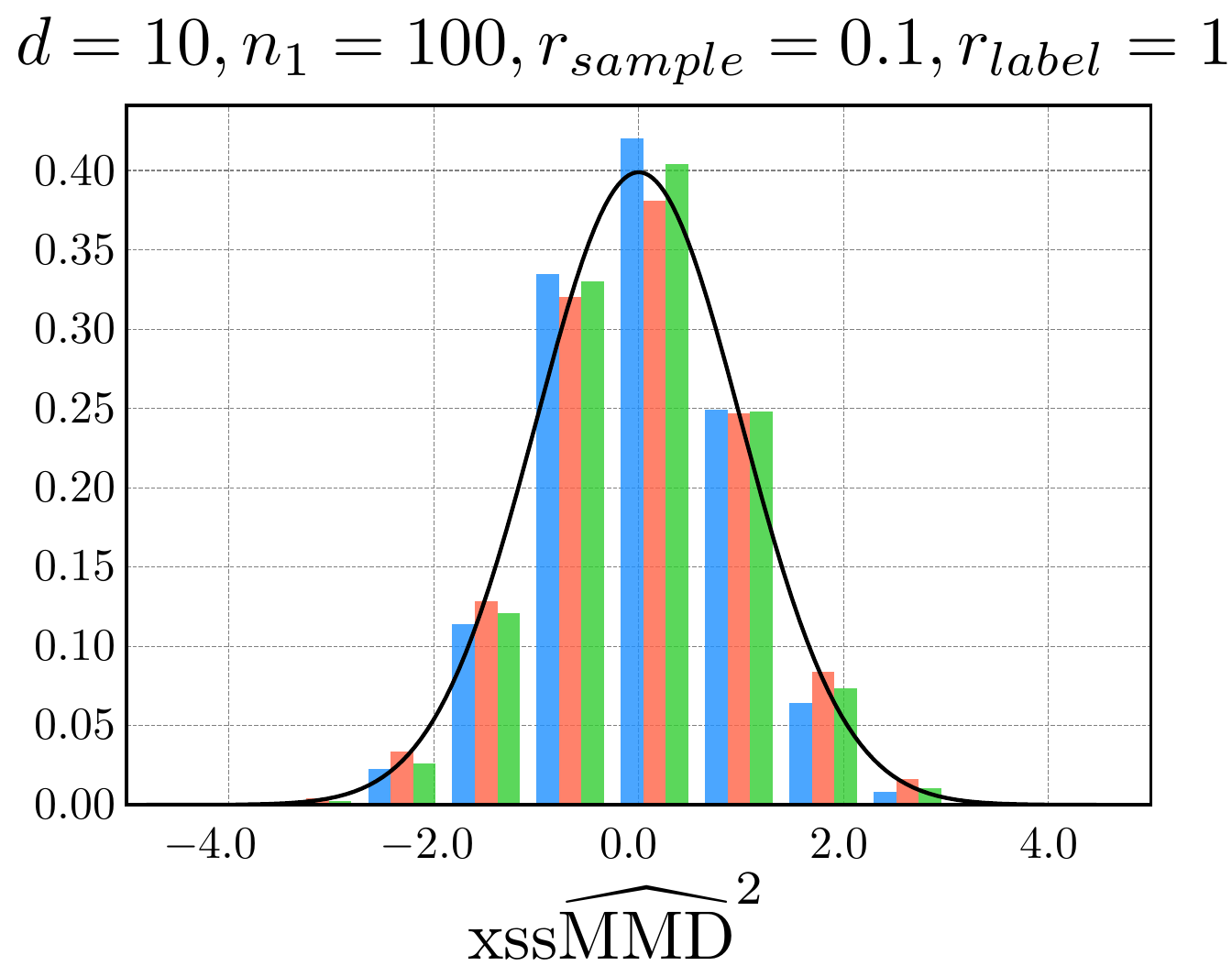}
    \end{subfigure}
    \hspace{0.01\textwidth}
    \begin{subfigure}[t]{0.21\textwidth} 
        \centering
        \includegraphics[width=\linewidth]{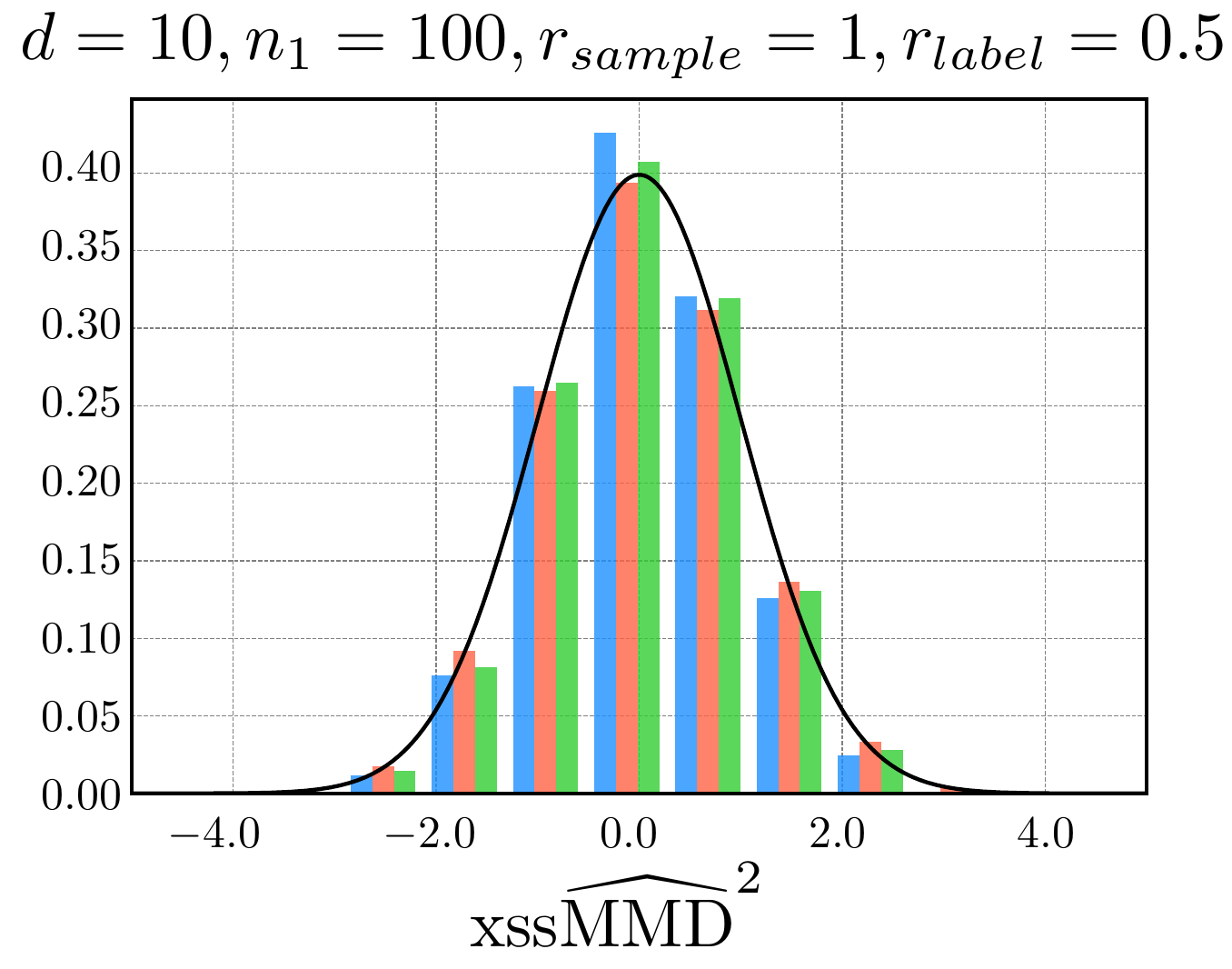}
    \end{subfigure}
    \hspace{0.01\textwidth}
    \begin{subfigure}[t]{0.21\textwidth} 
        \centering
        \includegraphics[width=\linewidth]{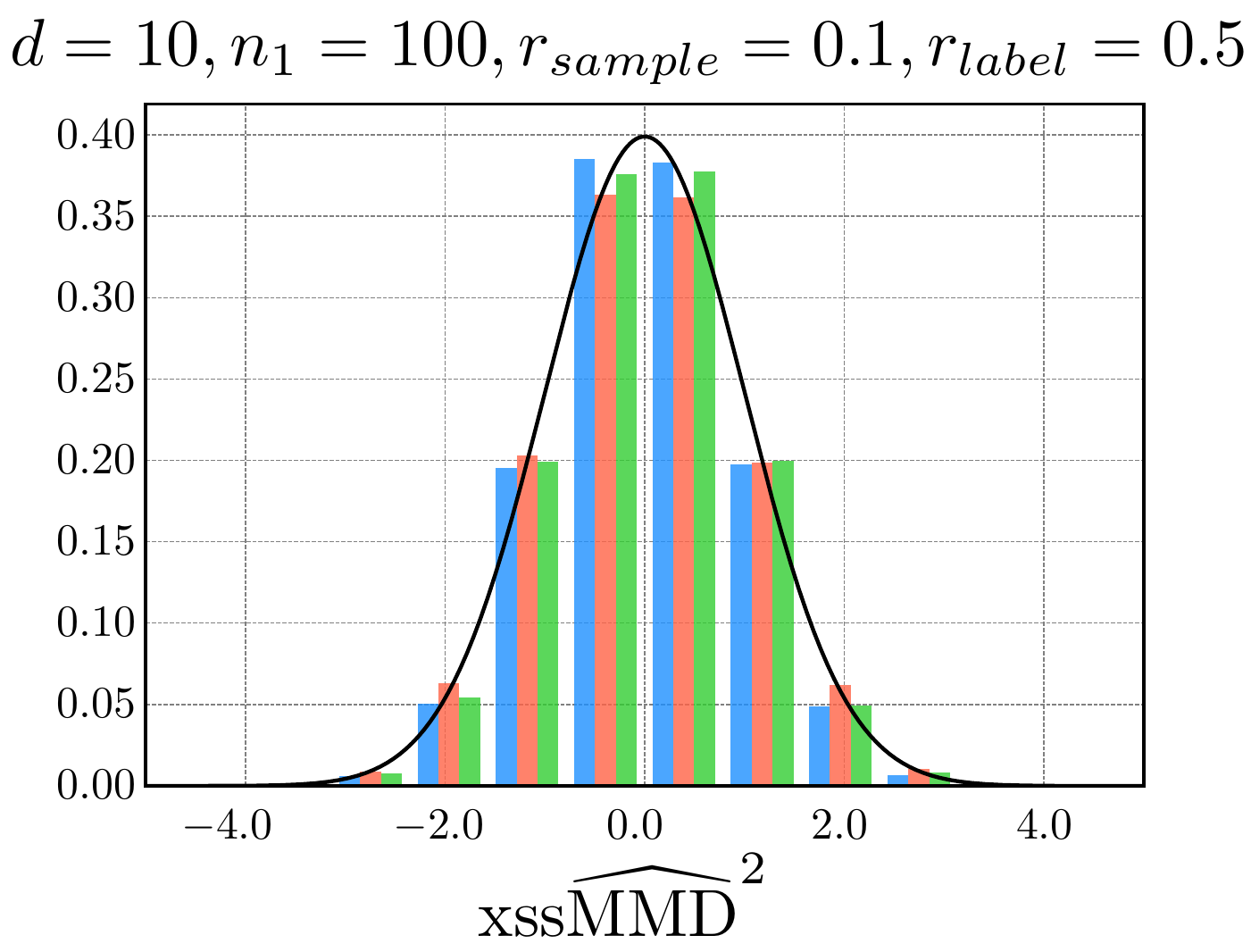}
    \end{subfigure}
    \begin{subfigure}[t]{0.21\textwidth} 
        \centering
        \includegraphics[width=\linewidth]{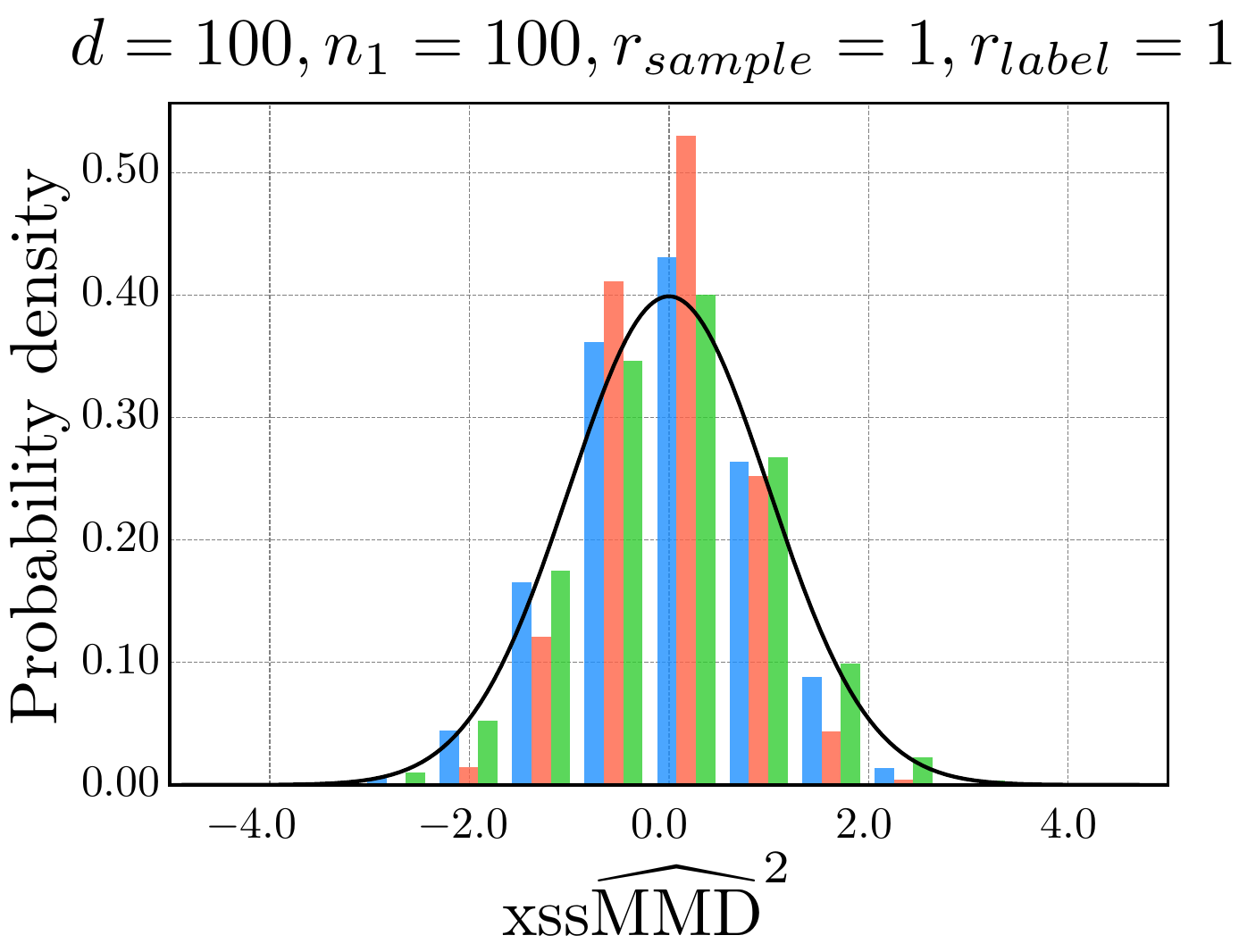}
    \end{subfigure}
    \hspace{0.01\textwidth}
    \begin{subfigure}[t]{0.21\textwidth}
        \centering
        \includegraphics[width=\linewidth]{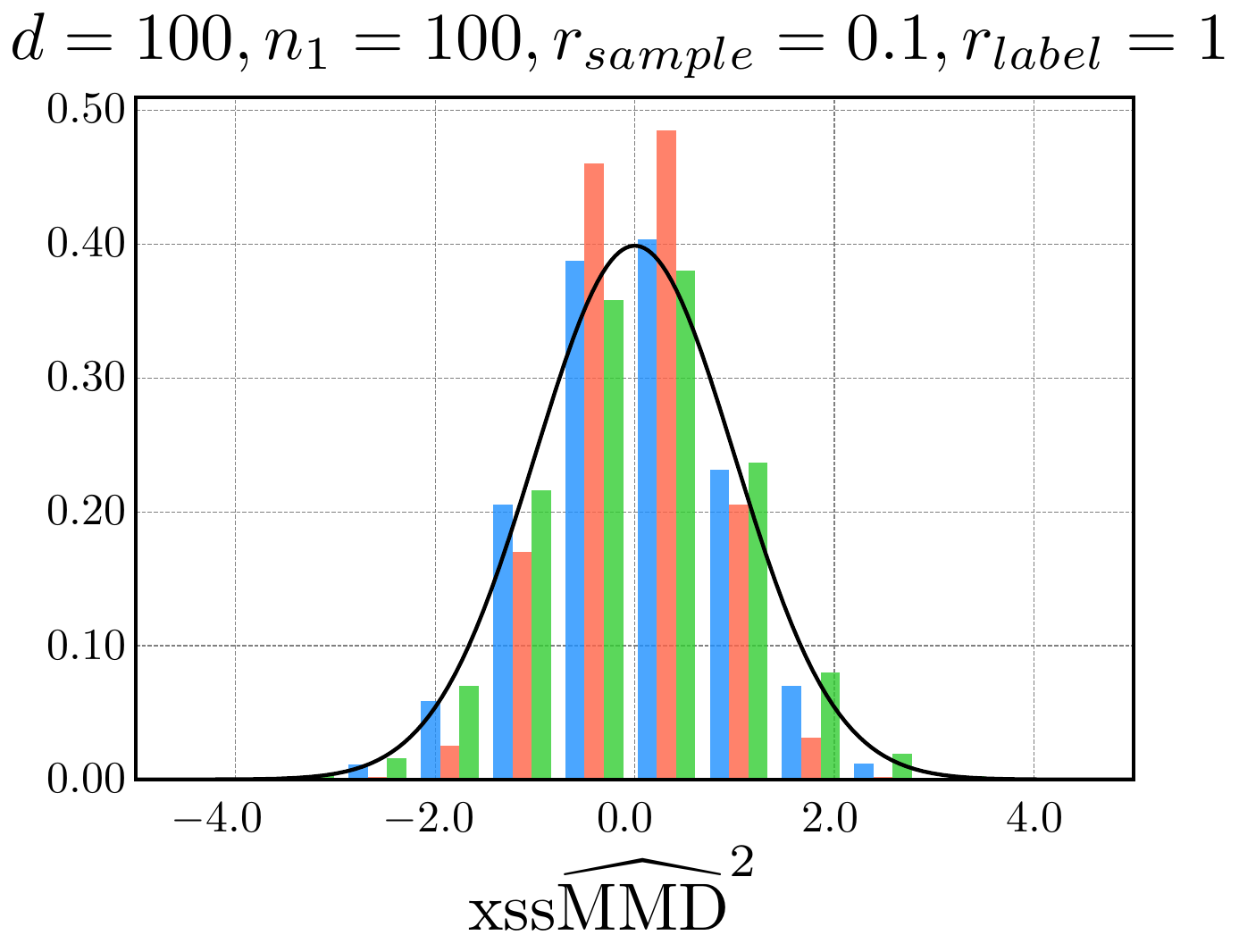}
    \end{subfigure}
    \hspace{0.01\textwidth}
    \begin{subfigure}[t]{0.21\textwidth} 
        \centering
        \includegraphics[width=\linewidth]{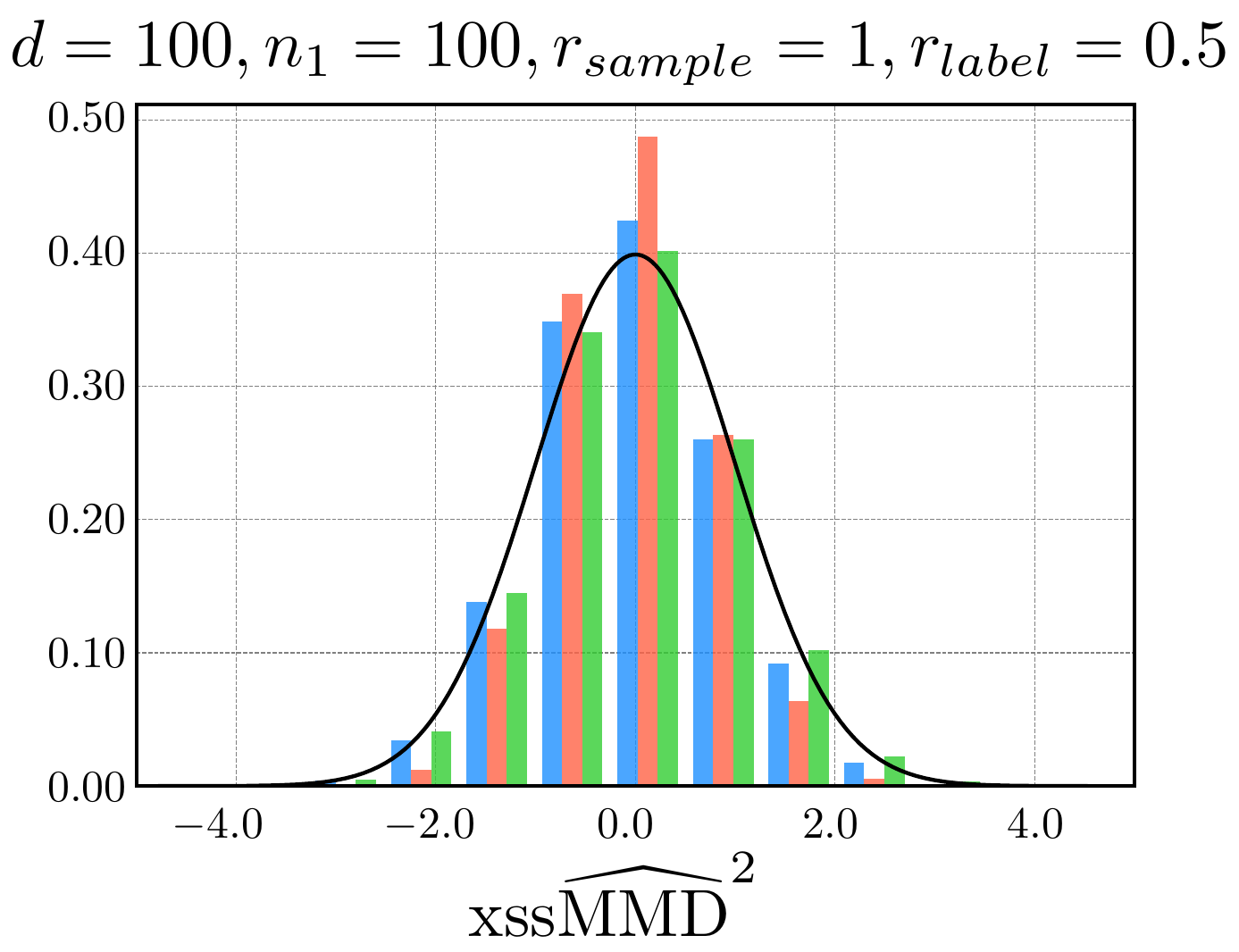}
    \end{subfigure}
    \hspace{0.01\textwidth}
    \begin{subfigure}[t]{0.21\textwidth} 
        \centering
        \includegraphics[width=\linewidth]{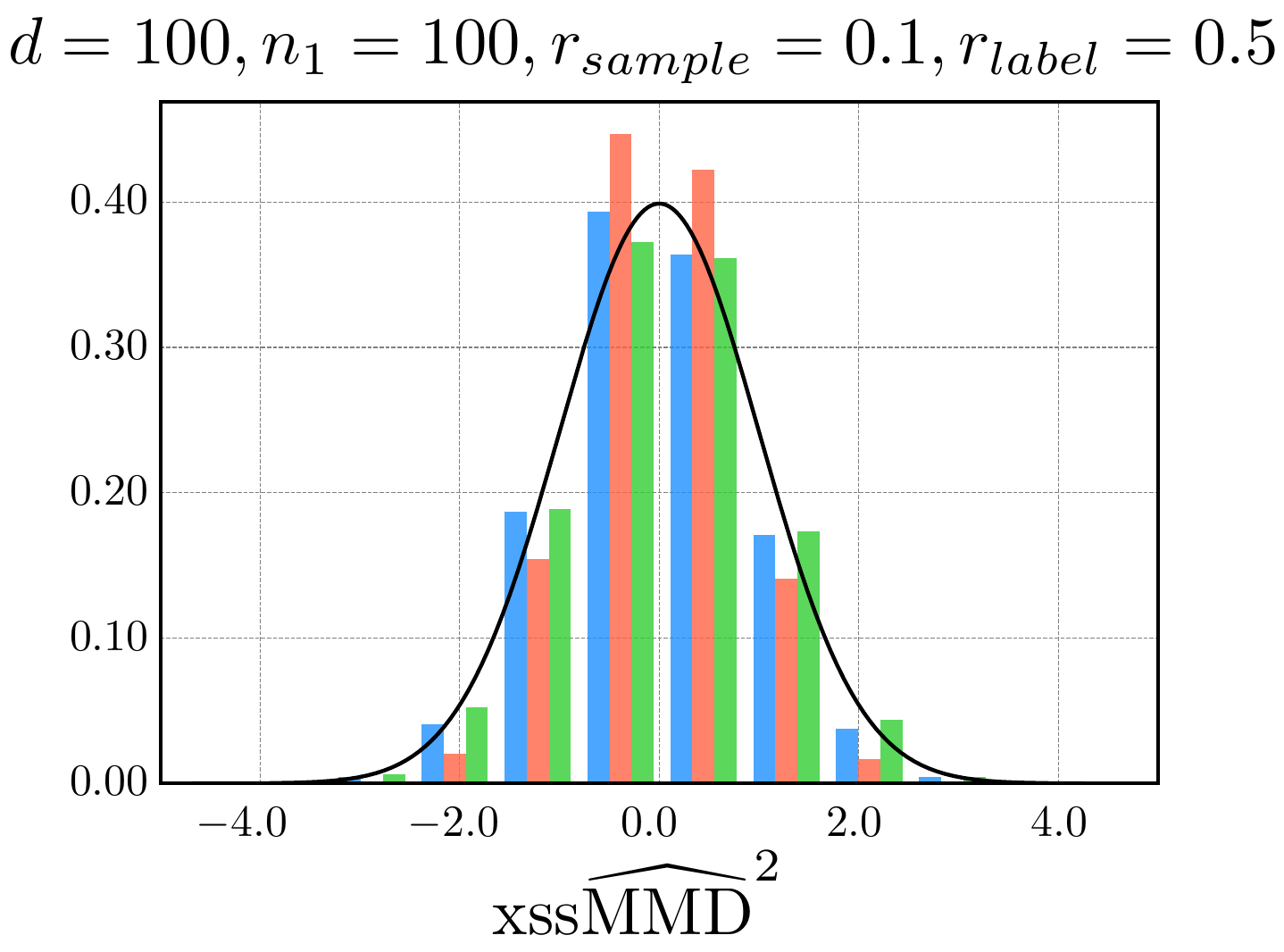}
    \end{subfigure}
    \begin{subfigure}[t]{0.22\textwidth} 
        \centering
        \includegraphics[width=\linewidth]{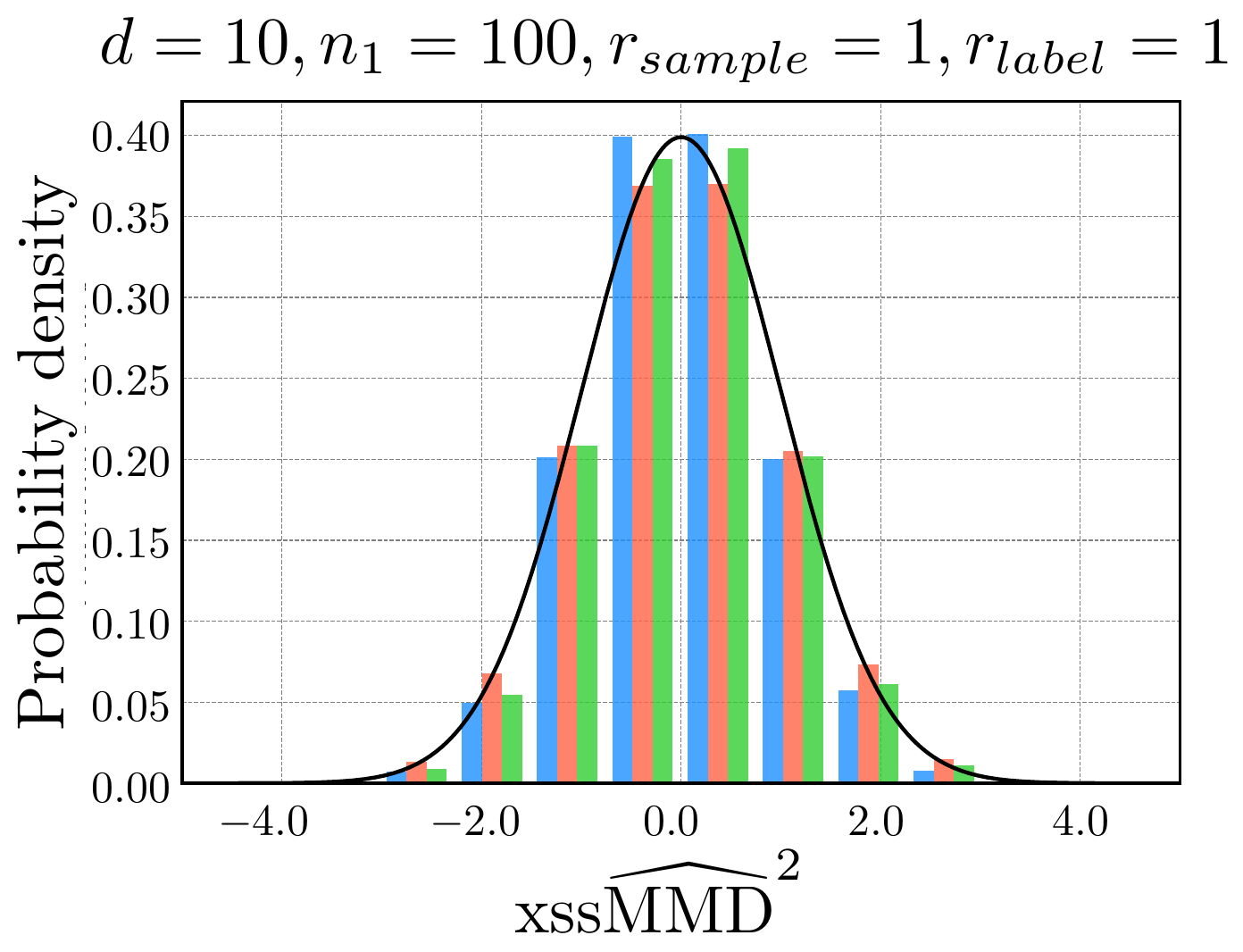}
    \end{subfigure}
    \hspace{0.01\textwidth}
    \begin{subfigure}[t]{0.21\textwidth}
        \centering
        \includegraphics[width=\linewidth]{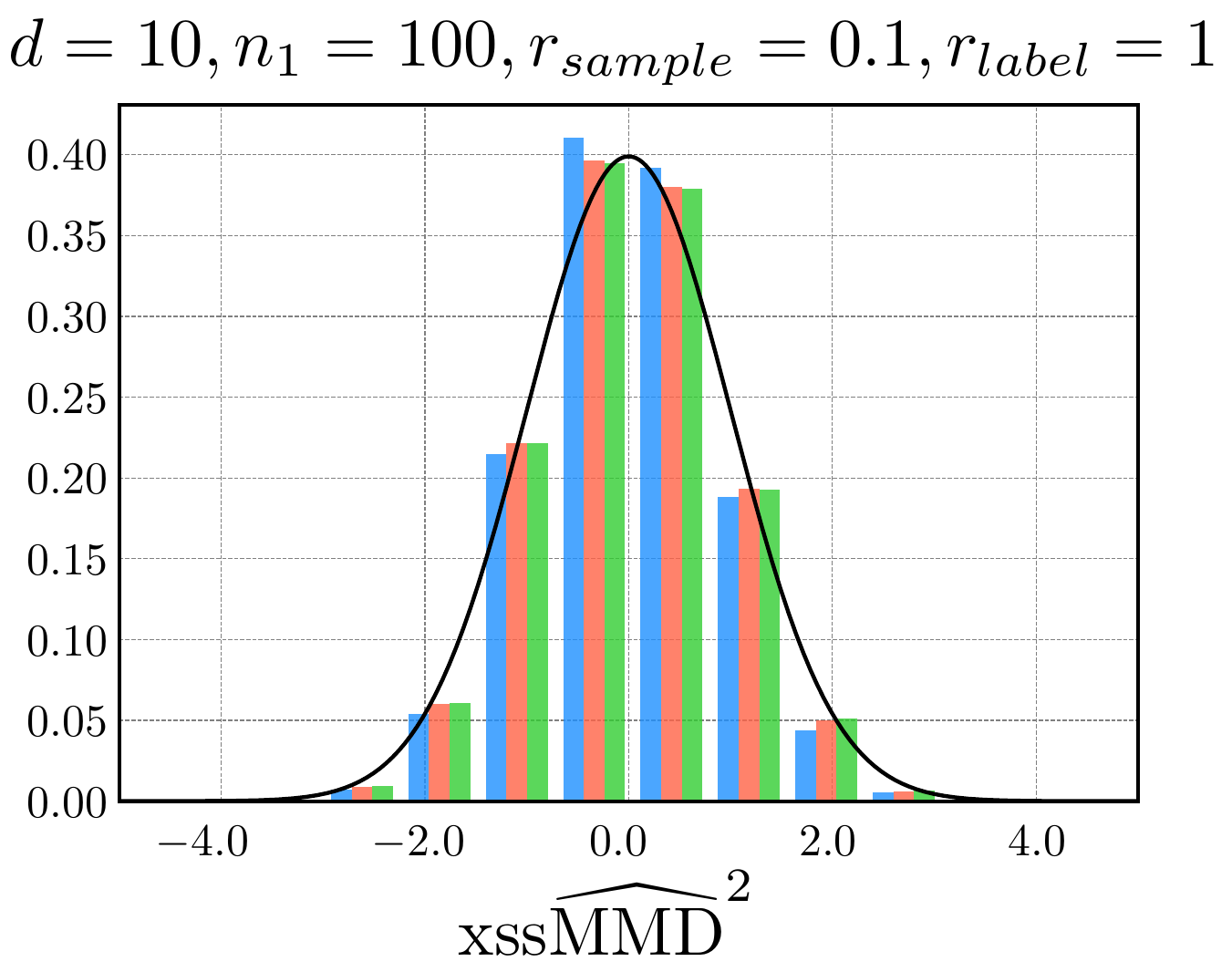}
    \end{subfigure}
    \hspace{0.01\textwidth}
    \begin{subfigure}[t]{0.21\textwidth} 
        \centering
        \includegraphics[width=\linewidth]{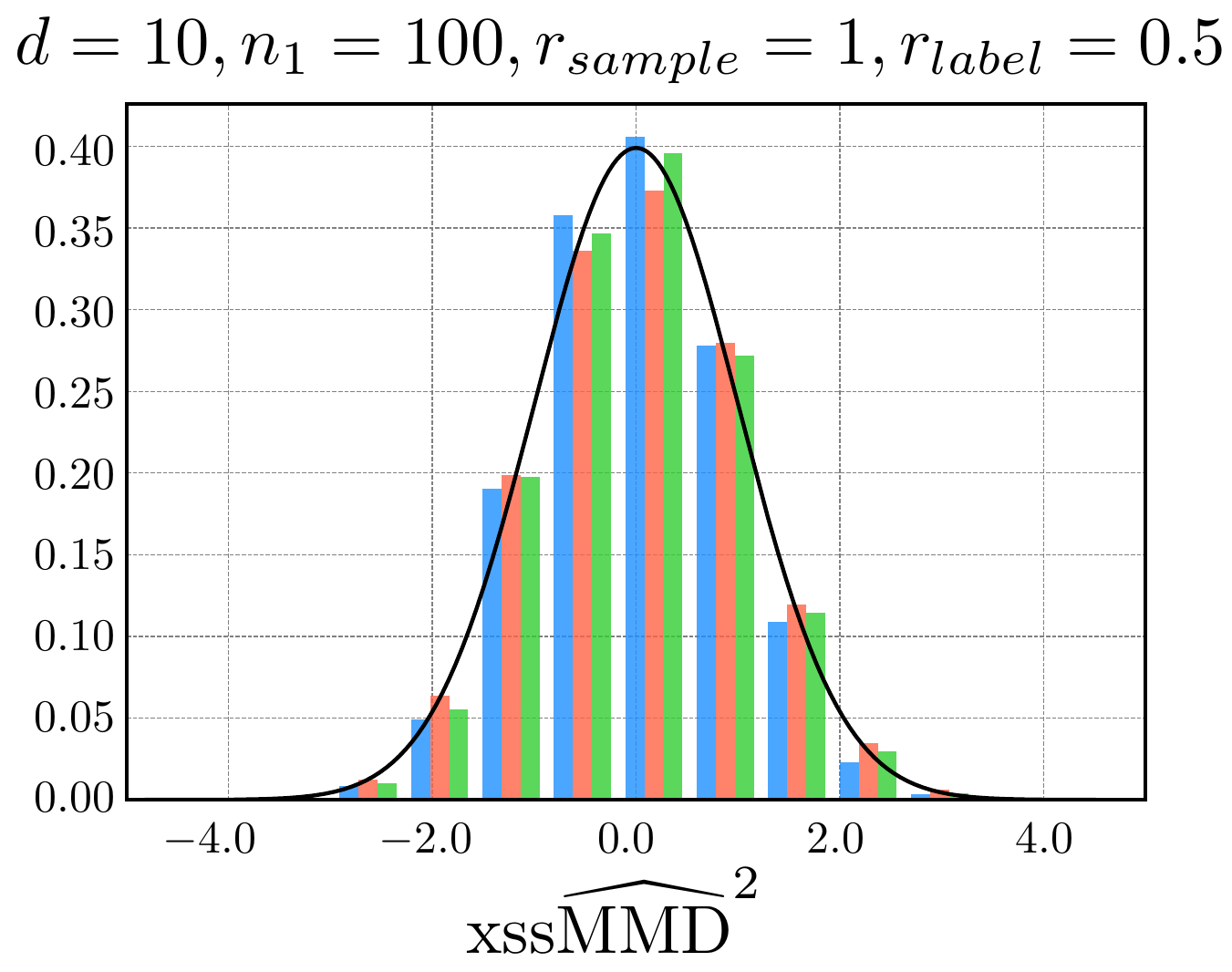}
    \end{subfigure}
    \hspace{0.01\textwidth}
    \begin{subfigure}[t]{0.21\textwidth} 
        \centering
        \includegraphics[width=\linewidth]{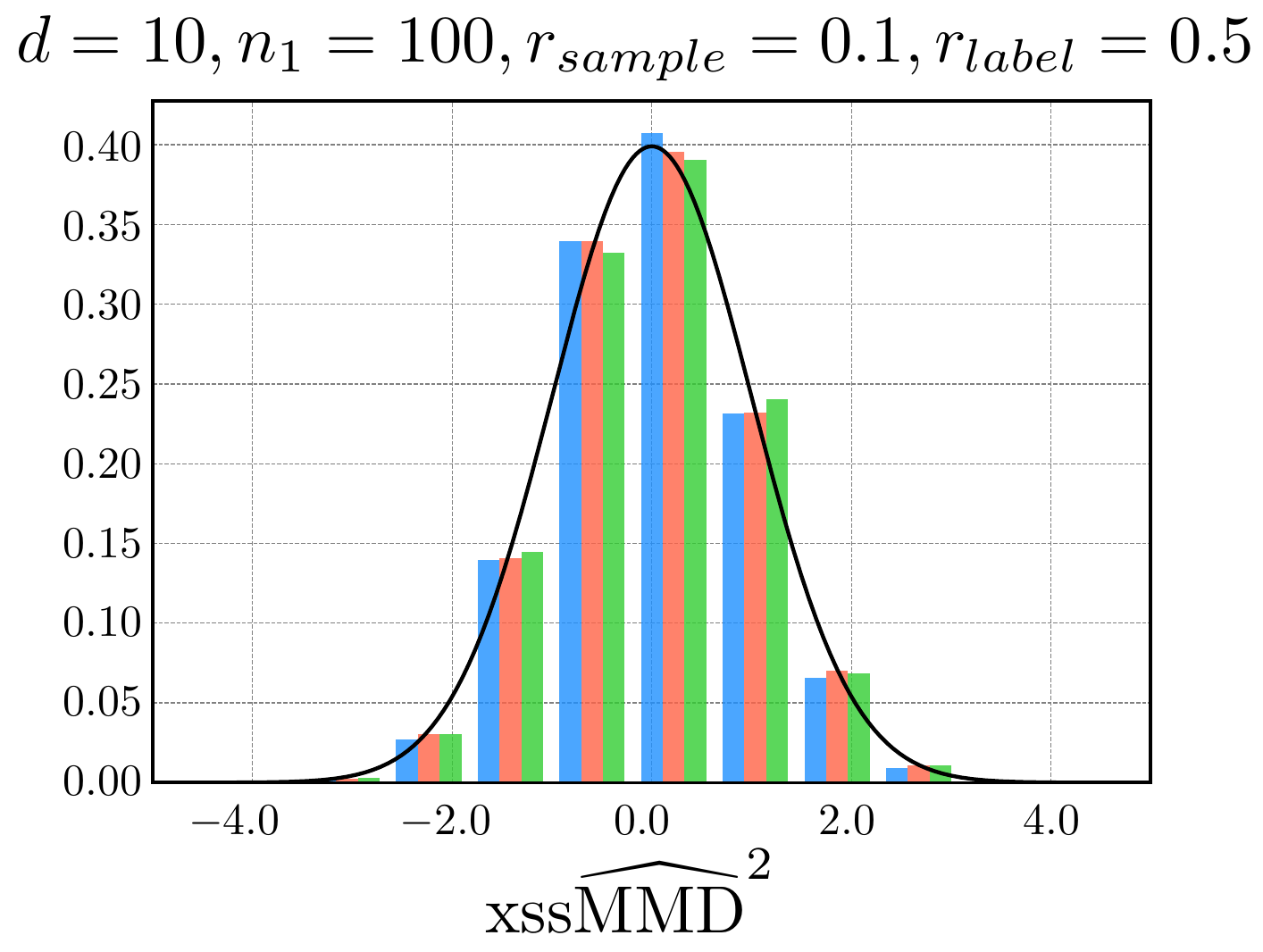}
    \end{subfigure}
    \begin{subfigure}[t]{0.22\textwidth} 
        \centering
        \includegraphics[width=\linewidth]{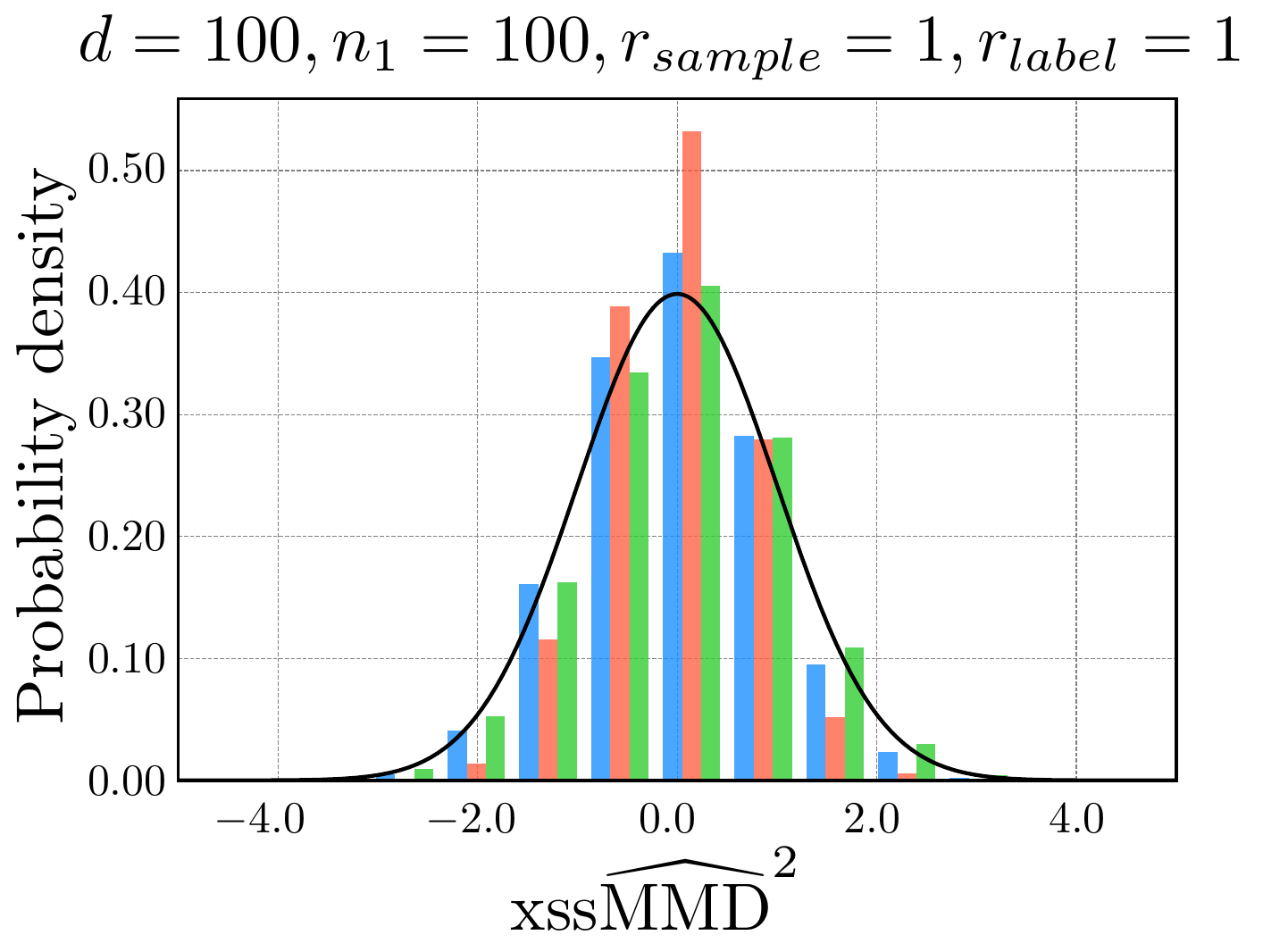}
    \end{subfigure}
    \hspace{0.01\textwidth}
    \begin{subfigure}[t]{0.21\textwidth}
        \centering
        \includegraphics[width=\linewidth]{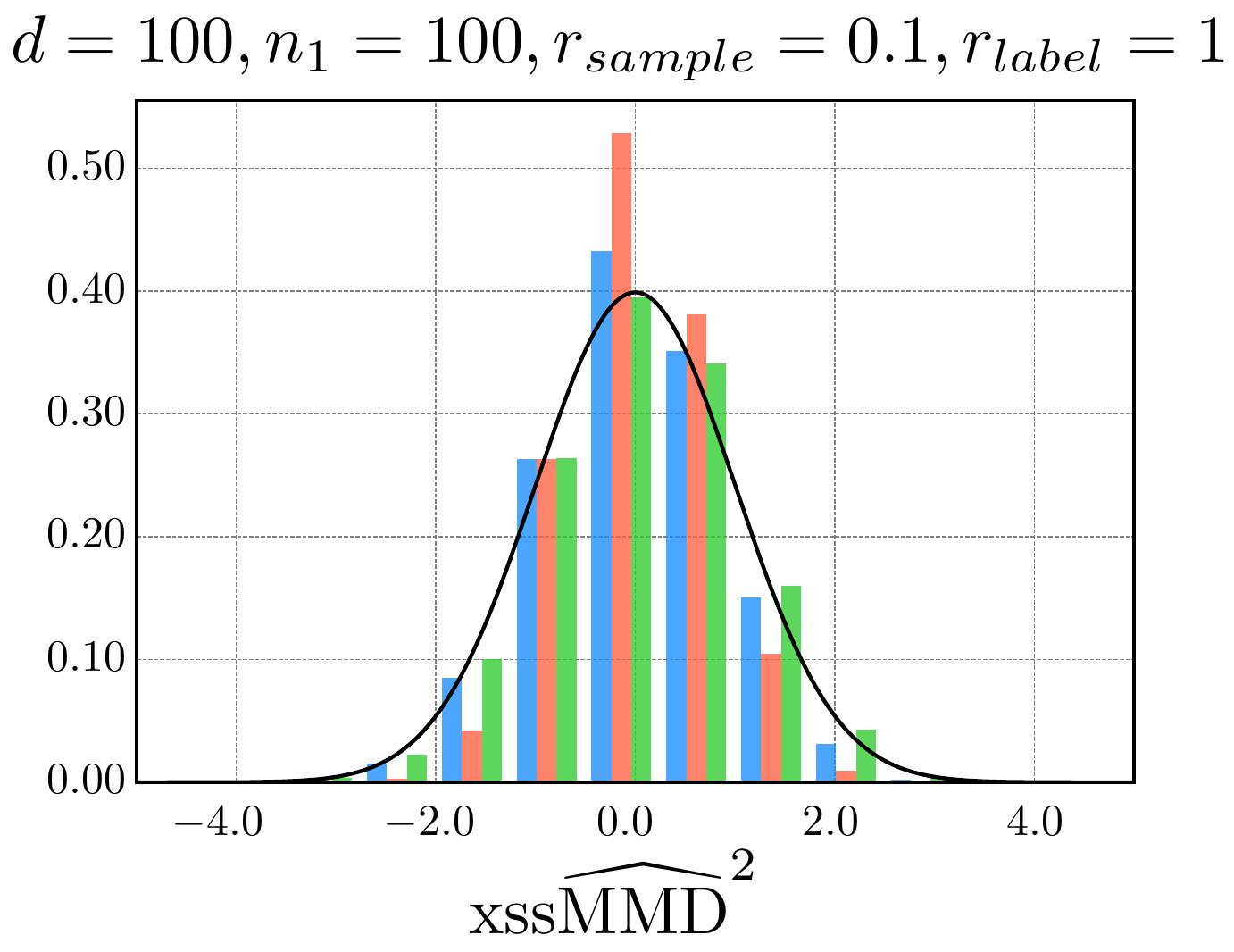}
    \end{subfigure}
    \hspace{0.01\textwidth}
    \begin{subfigure}[t]{0.21\textwidth} 
        \centering
        \includegraphics[width=\linewidth]{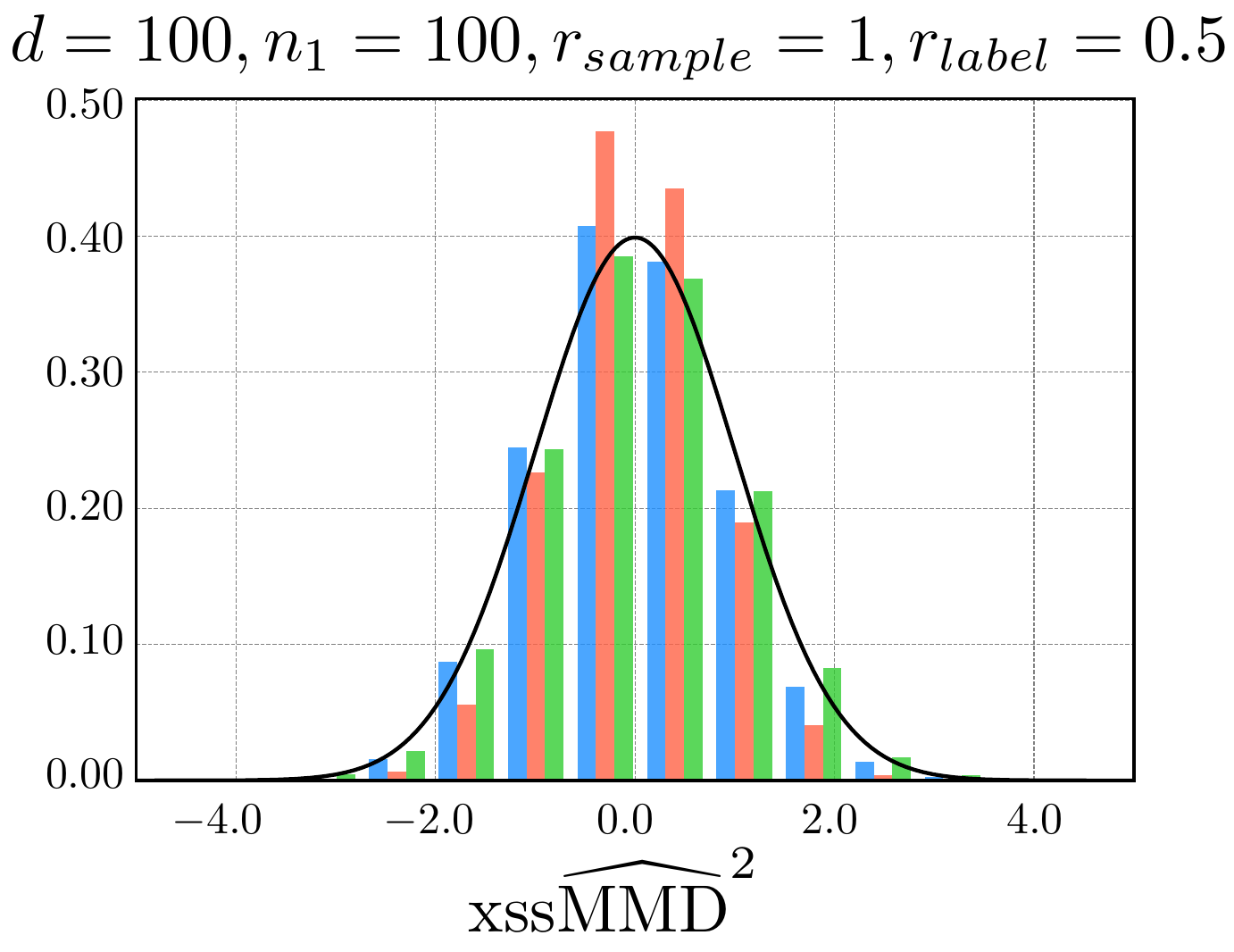}
    \end{subfigure}
    \hspace{0.01\textwidth}
    \begin{subfigure}[t]{0.21\textwidth} 
        \centering
        \includegraphics[width=\linewidth]{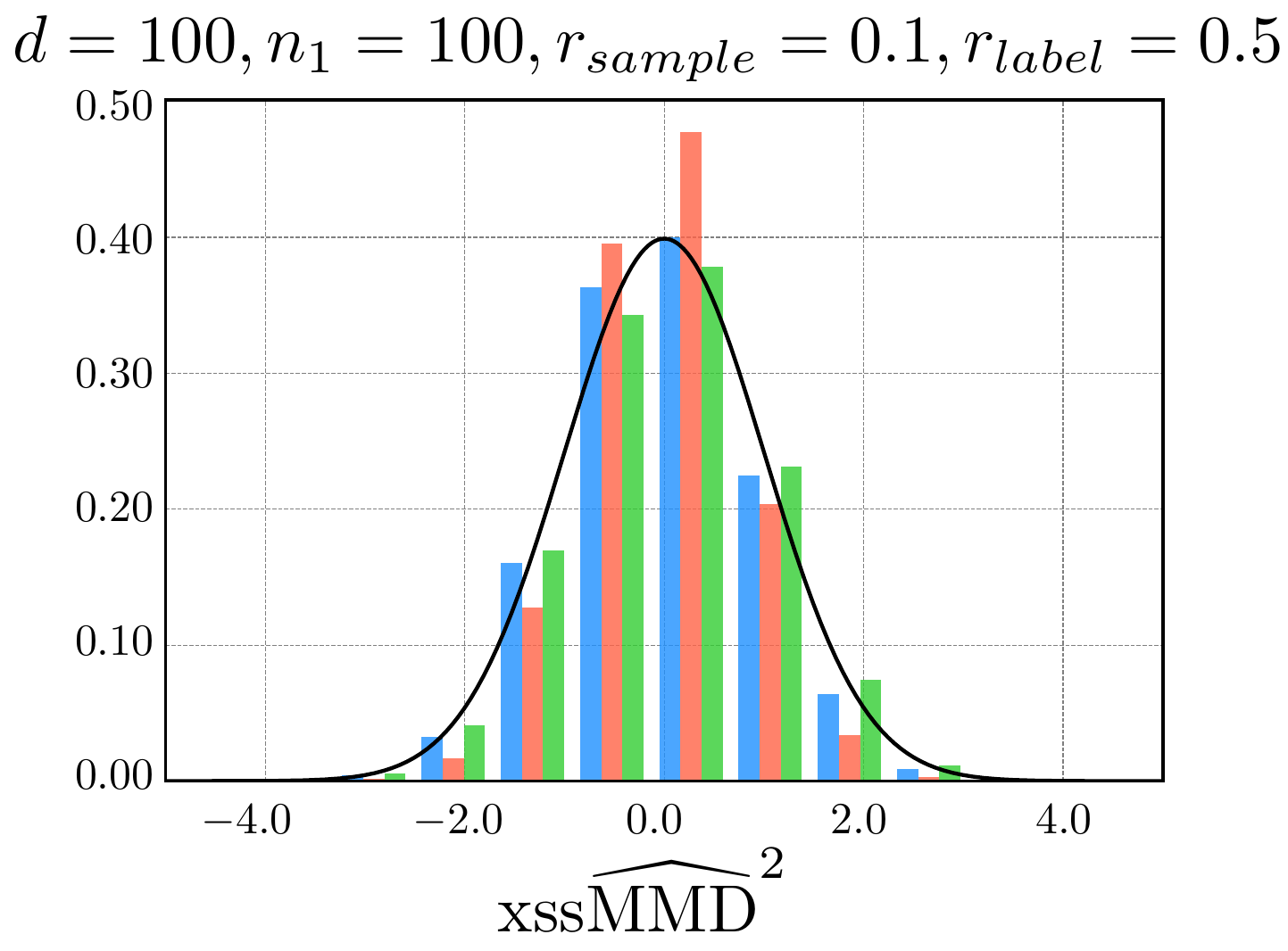}
    \end{subfigure}
    \caption{Experimental results for the distribution of $\xssMMD$ under the null hypothesis across all scenarios explained in \Cref{Section: Experiments}. The plots illustrate that the test statistic consistently adheres to a $N(0,1)$ distribution under various parameter settings. These comprehensive results confirm the validity of the xssMMD test across a broad range of conditions.}
    \label{fig:additional limiting null distributions 2}
\end{figure*}
Lastly, we investigate the robustness of the xssMMD test under different structural dependencies. We conduct experiments where $X$ and $V$ (or $Y$ and $W$) exhibit dependence which correspond to the settings used in the power analysis (Scenario 1 (Alt) to Scenario 4 (Alt) in \Cref{Section: Experiments}). \Cref{fig:additional limiting null distributions 3} displays the empirical distribution of the xssMMD test statistic in these scenarios. The results demonstrate that even in the presence of dependencies, the standardized test statistic consistently follows $N(0,1)$. This confirms that the asymptotic normality of $\xssMMD$ holds even when the covariates are not independent, thereby supporting the validity of the xssMMD test across a wide range of dependency structures.

\begin{figure*}[htb!]
    \centering
    \begin{subfigure}[t]{0.22\textwidth} 
        \centering
        \includegraphics[width=\linewidth]{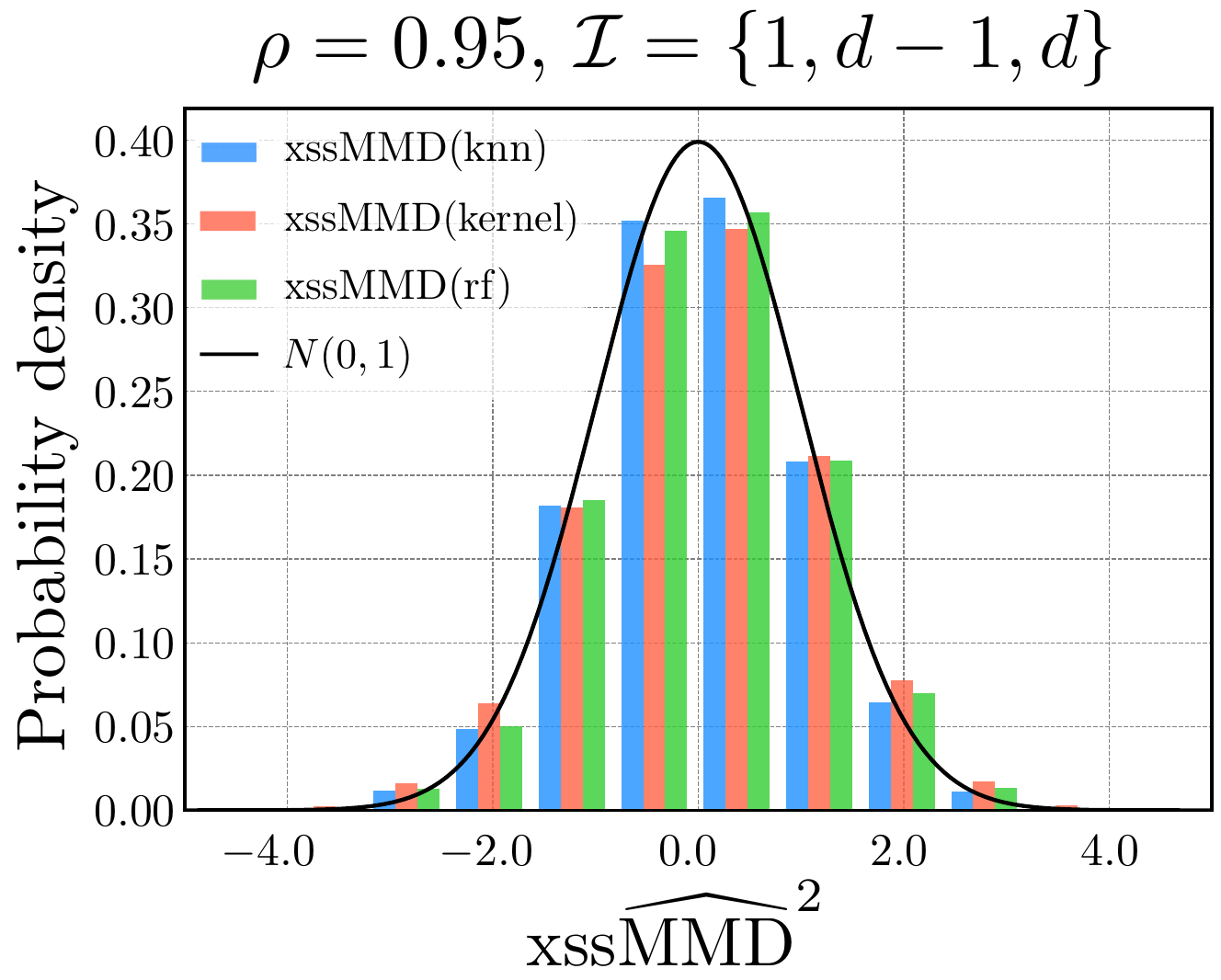}
    \end{subfigure}
    \hspace{0.01\textwidth}
    \begin{subfigure}[t]{0.21\textwidth}
        \centering
        \includegraphics[width=\linewidth]{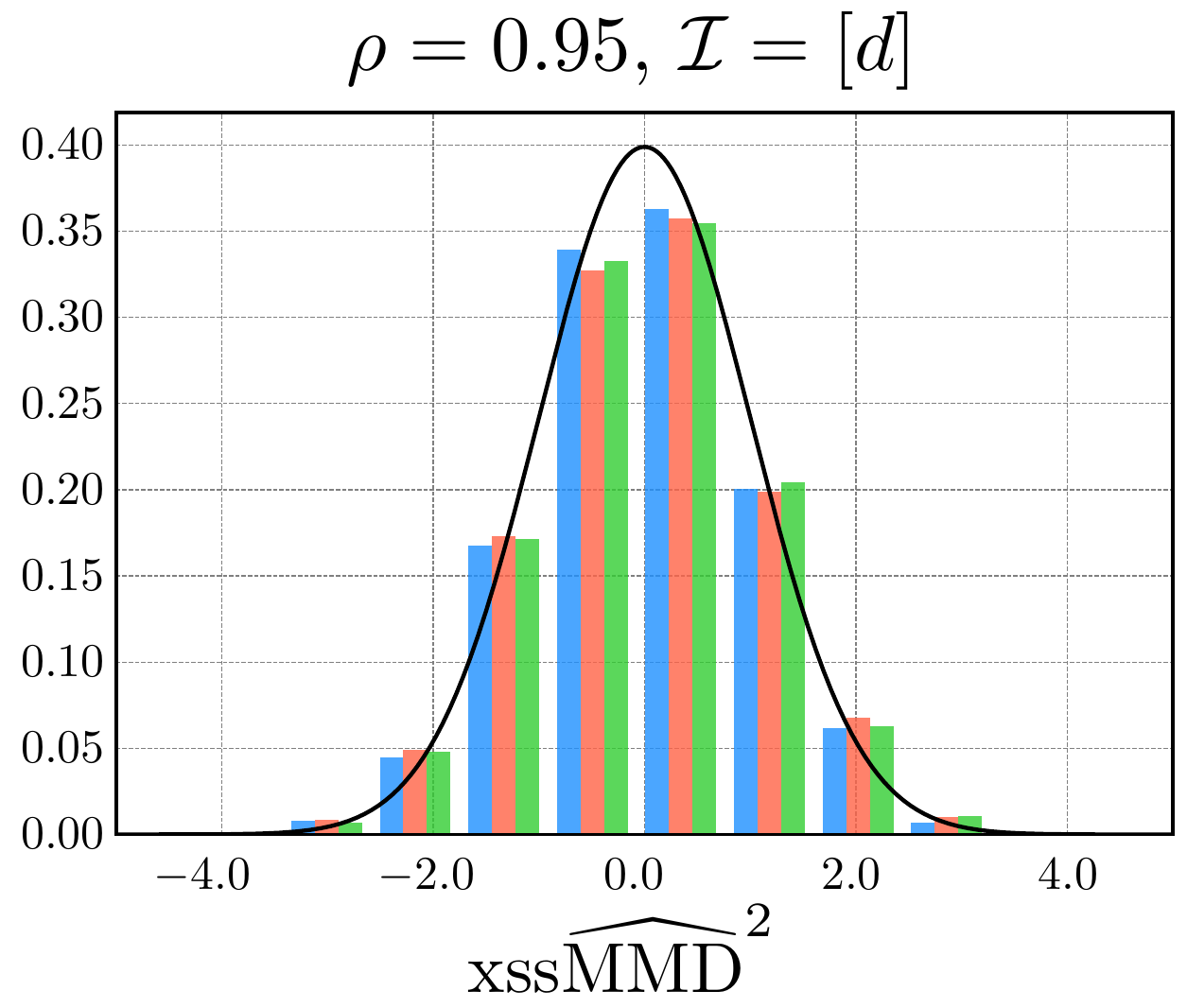}
    \end{subfigure}
    \hspace{0.01\textwidth}
    \begin{subfigure}[t]{0.21\textwidth} 
        \centering
        \includegraphics[width=\linewidth]{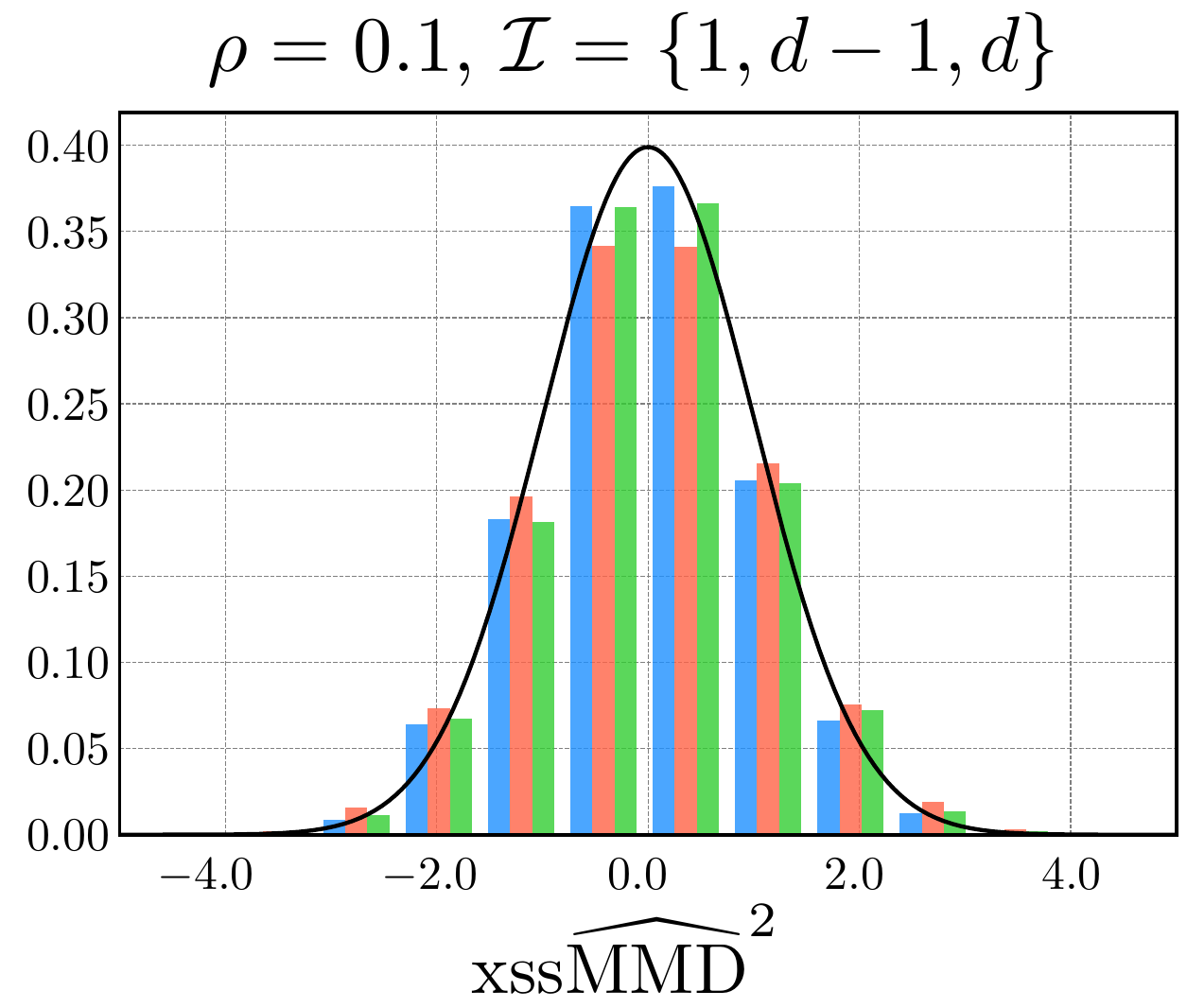}
    \end{subfigure}
    \hspace{0.01\textwidth}
    \begin{subfigure}[t]{0.21\textwidth} 
        \centering
        \includegraphics[width=\linewidth]{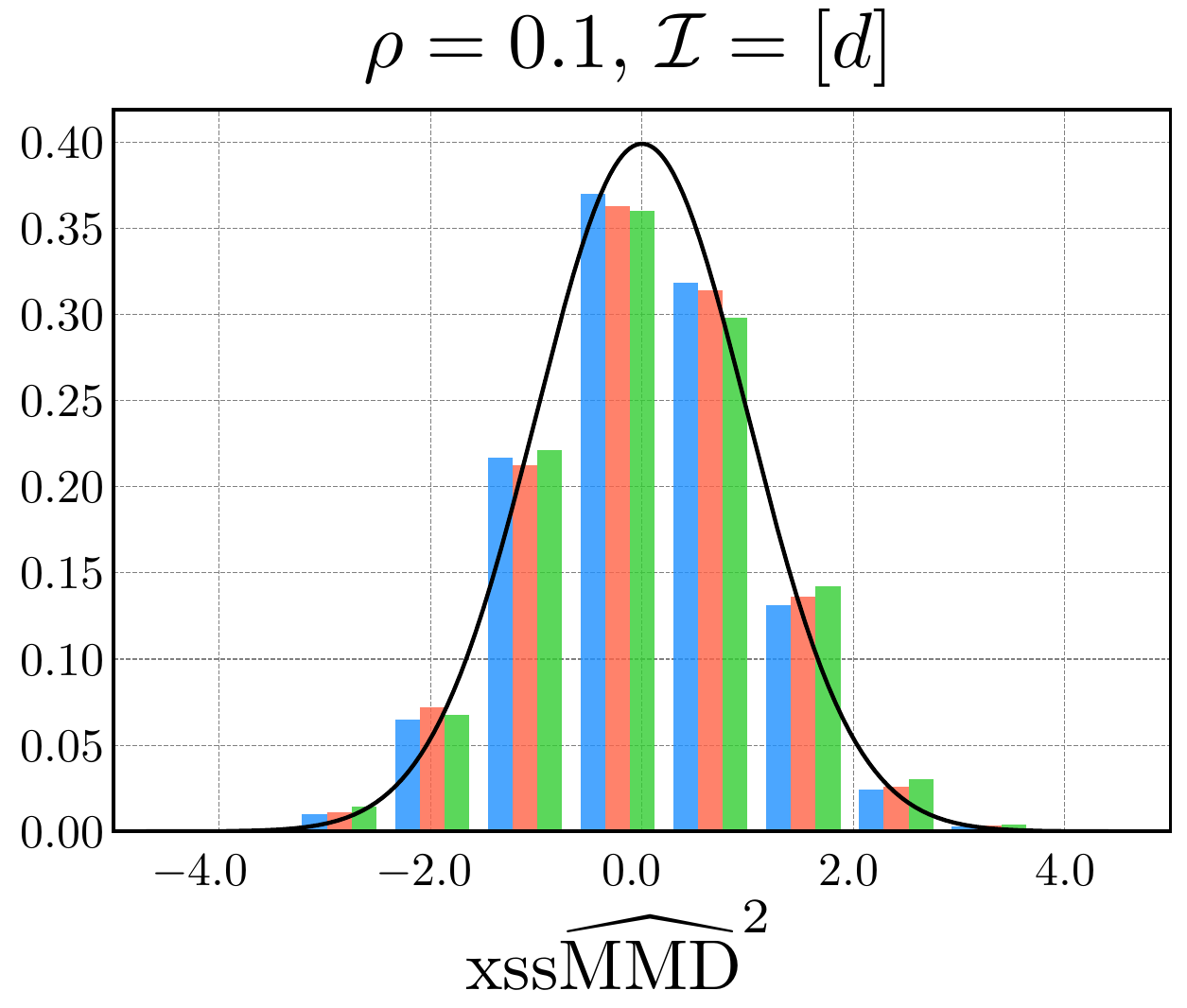}
    \end{subfigure}
    \caption{Experimental results for the distribution of $\xssMMD$ under the null across the scenarios in \Cref{Section: Experiments}, particularly from Scenario 1 (Alt) to Scenario 4 (Alt). These settings introduce dependencies between $X$ and $V$ or $Y$ and $W$, deviating from the standard independence assumption. The plots illustrate that even in the presence of such dependencies, the test statistic continues to follow a $N(0,1)$ under the null.}
    \label{fig:additional limiting null distributions 3}
\end{figure*}

These findings confirm the robustness of the asymptotic normality of $\xssMMD$ across a wide range of conditions. This further affirms the validity of our theoretical results that the xssMMD test is asymptotically level $\alpha$ under the null, as shown in \Cref{Theorem: xssMMD}.

\subsection{Power Curve with Different Settings}\label{appendix: experiments under the alternative}
In this subsection, we investigate the power curves of the xssMMD test statistic in different settings.

\begin{figure*}[htb!]
    \centering
    \begin{subfigure}[t]{0.22\textwidth} 
        \centering
        \includegraphics[width=\linewidth]{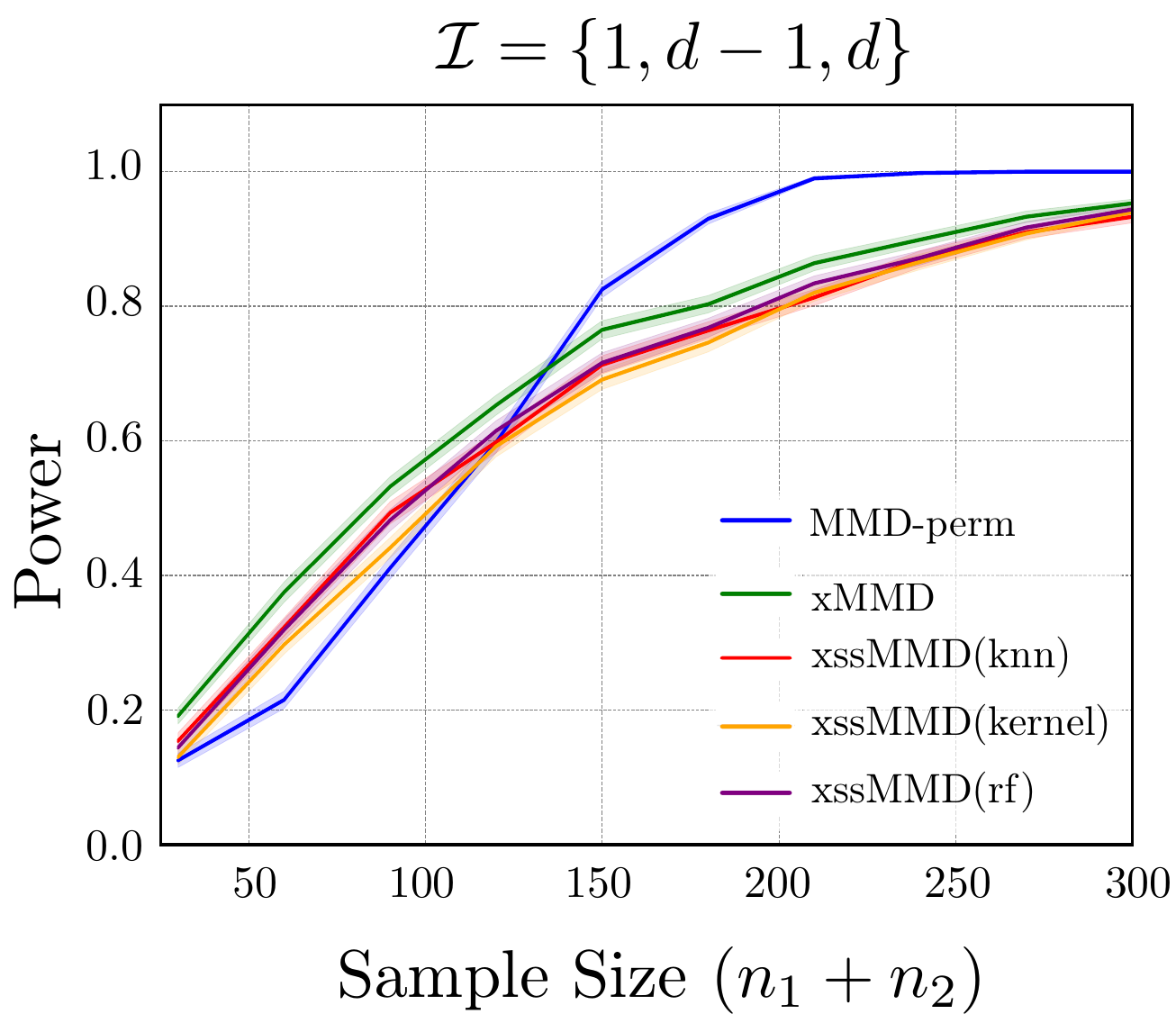}
    \end{subfigure}
    \hspace{0.01\textwidth}
    \begin{subfigure}[t]{0.21\textwidth}
        \centering
        \includegraphics[width=\linewidth]{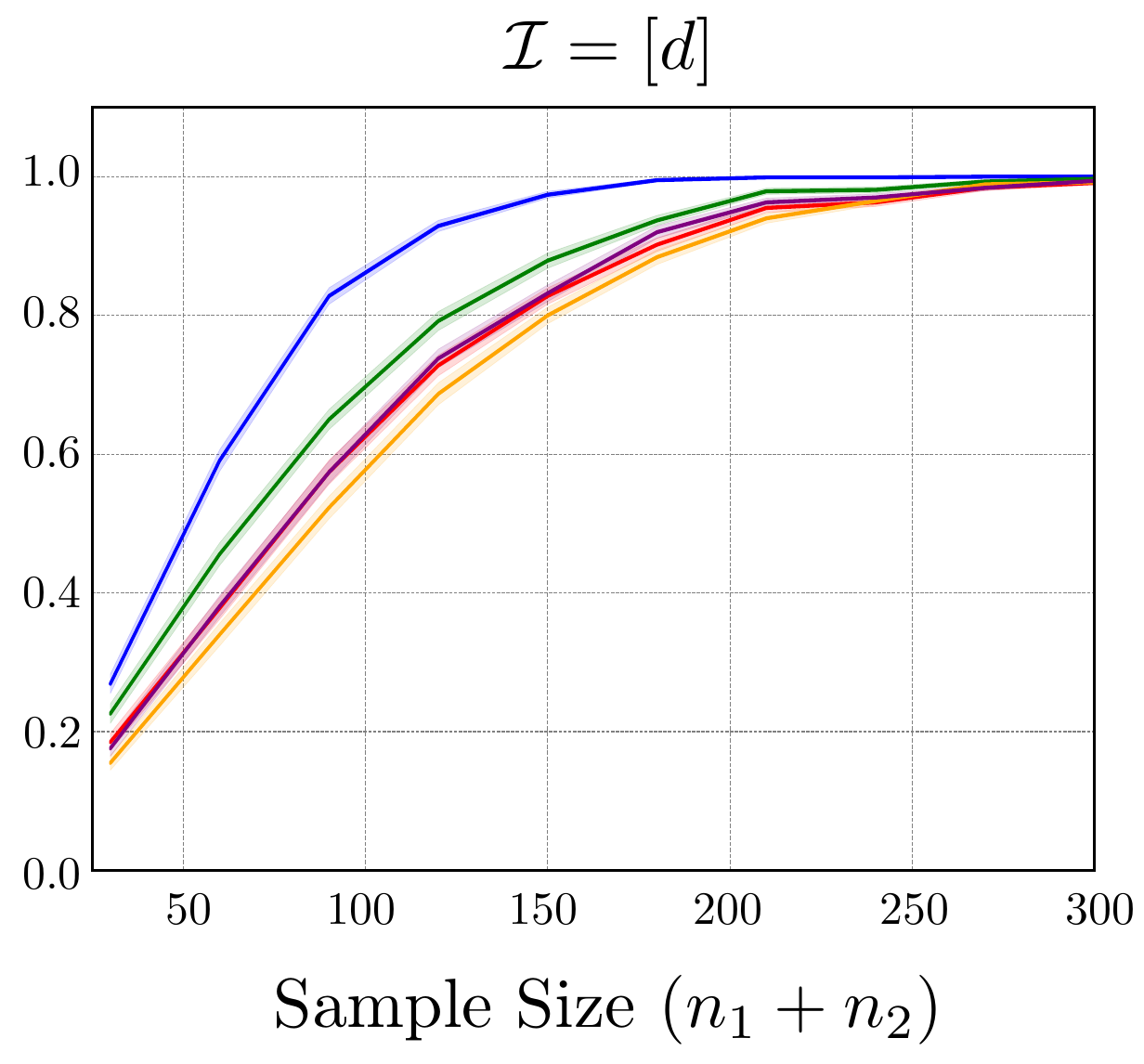}
    \end{subfigure}
    \hspace{0.01\textwidth}
    \begin{subfigure}[t]{0.21\textwidth} 
        \centering
        \includegraphics[width=\linewidth]{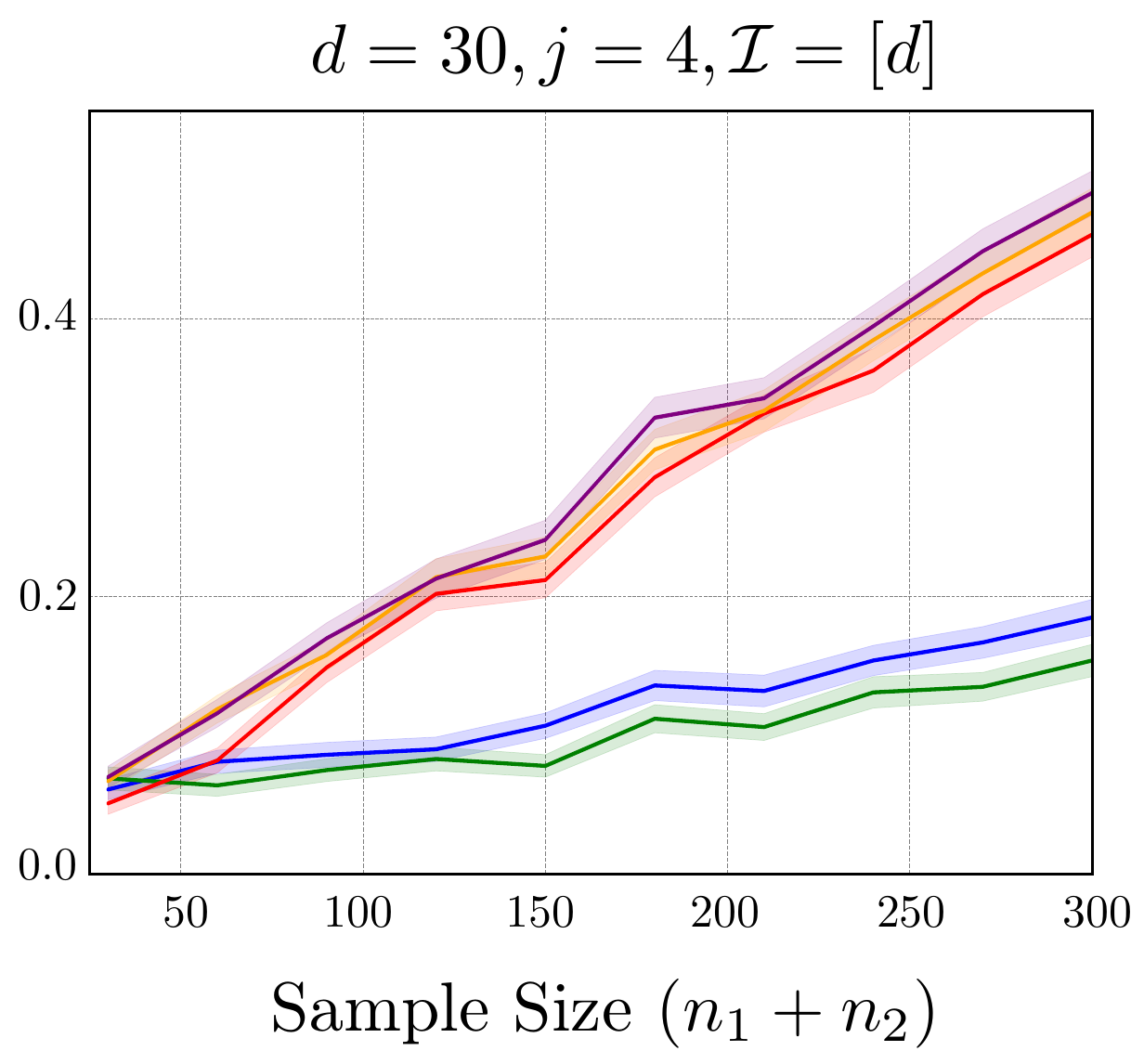}
    \end{subfigure}
    \hspace{0.01\textwidth}
    \begin{subfigure}[t]{0.21\textwidth} 
        \centering
        \includegraphics[width=\linewidth]{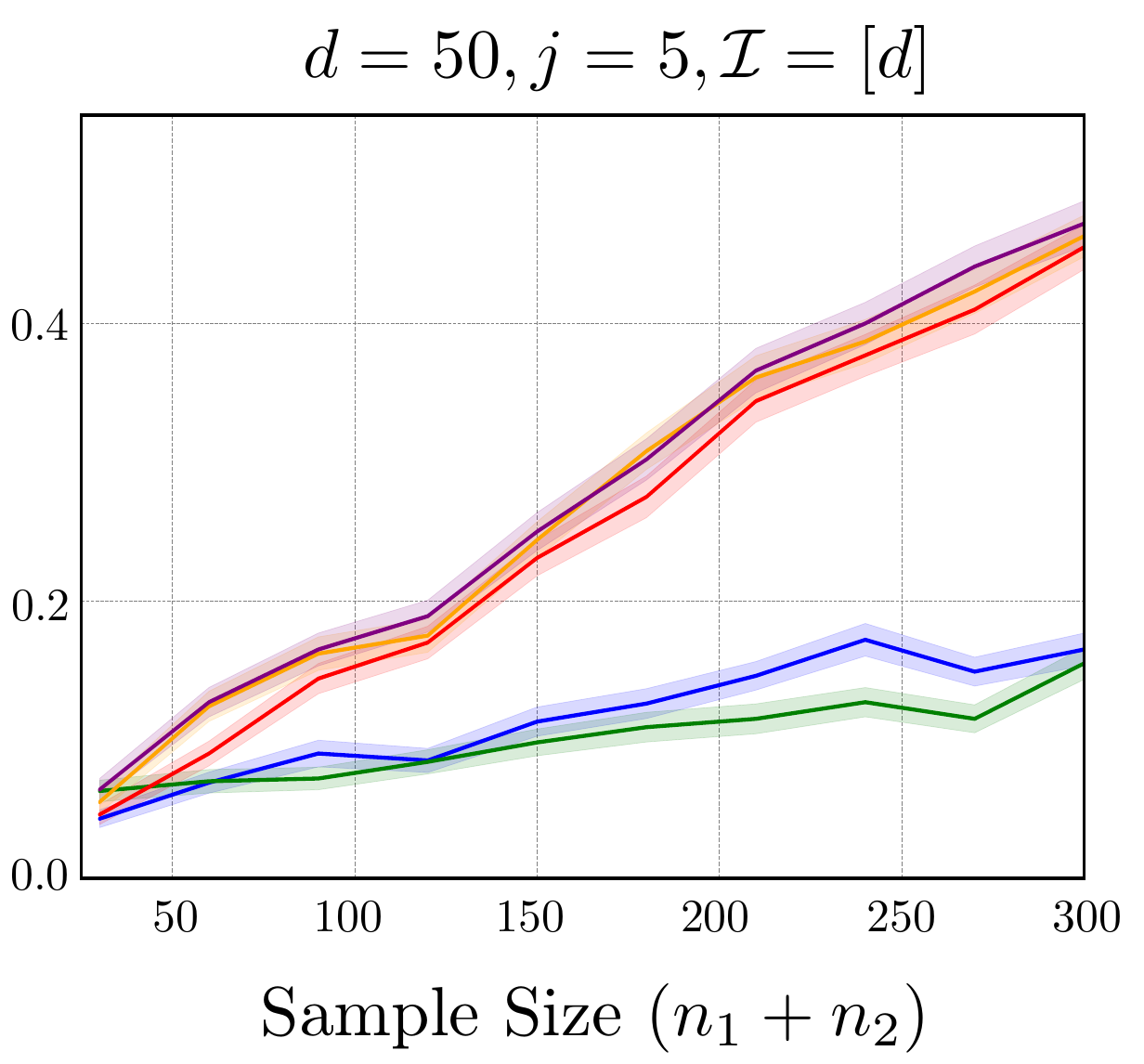}
    \end{subfigure}
    \caption{Power analysis of the xssMMD test in various settings. The first two subfigures depict scenarios in which additional covariates are independent of the labeled data. This result confirms that independent covariates do not enhance the performance of the xssMMD test. The last two subfigures illustrate the impact of varying the dimension $d$, demonstrating that the xssMMD test maintains superior power when $X$ and $Y$ show strong dependence on $V$ and $W$, even as the dimension increases.}
    \label{fig:additional power curves}
\end{figure*}
First, we evaluate the performance of the xssMMD test when additional covariates are independent of the labeled data. Unlike cases where unlabeled data provides valuable information to enhance power, this experiment demonstrates that the xssMMD test does not outperform other tests when the additional covariates contain no useful information. The first two subfigures in \Cref{fig:additional power curves} illustrate each case corresponding to Scenario 1 (Alt) and Scenario 2 (Alt), respectively, but with $X$ and $Y$ sampled independently from $P_X$ and $P_Y$, both following $N(\mathbf{0}_{d}, I_{d}).$ In such scenarios, the power curves of the xssMMD test closely align with those of the xMMD test and fall below those of the MMD-perm test, consistent with our conclusions about the impact of independent covariates.

Additionally, we conduct experiments to verify that the theoretical results on the power derived in \Cref{Theorem: xssMMD} hold across various dimensions. Using the same construction method outlined in \Cref{Section: Experiments} for Scenario 1 (Alt) through Scenario 4 (Alt), we vary the dimension $d$ and the parameter $j$, which represents the difference between the mean vectors of $V$ and $W$. Recall that both are sampled from a Gaussian distribution with mean vectors $\mathbf{0}_{d}$ and $\mathbf{a}_{\epsilon, j}$ where the first $j$ entries of $\mathbf{a}_{\epsilon, j}$ are equal to $\epsilon$ and the remaining entries are equal to $0$. The last two subfigures in \Cref{fig:additional power curves} correspond to the cases where the dimension is set to $d=30$ and $d=50$, while the mean vector difference is set to $j=4$ and $j=5$, respectively. As illustrated in the last two subfigures in \Cref{fig:additional power curves}, even as the dimension increases, the xssMMD test remains robust and consistently surpasses both the MMD-perm and xMMD tests, confirming its superior power in settings where the unlabeled data are informative.

Furthermore, we investigate how the amount of unlabeled data influences the test power. We conducted additional simulations adopting the settings of Scenario 1 (Alt) through Scenario 4 (Alt) from \Cref{Section: Experiments}. Fixing the labeled sample sizes at $n_1=n_2=100$, we varied the unlabeled sample sizes $m_1=m_2$ from 0 to 2000. As shown in Table \ref{tab:unlabeled_size}, the power of the xssMMD test increases monotonically as the size of the unlabeled data grows, confirming that our method effectively leverages additional information. In contrast, the power of MMD-perm and xMMD remains constant. Notably, when $m_1=m_2=0$, xssMMD is mathematically identical to xMMD, yielding the same power.

\begin{table}[!htbp]
\centering
\footnotesize
\setlength{\tabcolsep}{3pt}
\caption{Estimated test power for Scenario 1 (Alt) through Scenario 4 (Alt) with fixed $n_1=n_2=100$ and varying unlabeled sample sizes $m_1=m_2$.}
\label{tab:unlabeled_size}
\begin{tabular}{llrrrrrrrrrr}
\toprule
& & \multicolumn{10}{c}{$m_1=m_2$} \\
\cmidrule(lr){3-12}
Scenario & Test & $0$ & $222$ & $444$ & $666$ & $888$ & $1111$ & $1333$ & $1555$ & $1777$ & $2000$ \\
\midrule
\begin{tabular}{l}Scenario \\ 1 (Alt) \end{tabular} & 
\begin{tabular}{l}  MMD-perm        \\ xMMD     \\ xssMMD(knn)  \\ xssMMD(ker)      \\ xssMMD(rf)\end{tabular} & 
\begin{tabular}{r}  \textbf{0.136}  \\ 0.097    \\ 0.097        \\ 0.097            \\ 0.097 \end{tabular} & 
\begin{tabular}{r}  0.147           \\ 0.105    \\ 0.179        \\ \textbf{0.215}   \\ 0.193 \end{tabular} & 
\begin{tabular}{r}  0.129           \\ 0.107    \\ 0.242        \\ 0.251            \\ \textbf{0.254} \end{tabular} & 
\begin{tabular}{r}  0.157           \\ 0.125    \\ 0.279        \\ 0.286 \\\textbf{0.301} \end{tabular} & 
\begin{tabular}{r}  0.126           \\ 0.117    \\ 0.319        \\ 0.315 \\ \textbf{0.343} \end{tabular} & 
\begin{tabular}{r}  0.125           \\0.106     \\0.340         \\0.344\\\textbf{0.362}\end{tabular} & 
\begin{tabular}{r}  0.116           \\0.092     \\0.350         \\0.348\\\textbf{0.366}\end{tabular} & 
\begin{tabular}{r}  0.128           \\0.090     \\0.366         \\0.365\\\textbf{0.405}\end{tabular} & 
\begin{tabular}{r}  0.151           \\0.122     \\0.403         \\0.375\\\textbf{0.434}\end{tabular} & 
\begin{tabular}{r}  0.120           \\0.114     \\0.388         \\0.360\\\textbf{0.421}\end{tabular} \\
\midrule
\begin{tabular}{l}Scenario \\ 2 (Alt)\end{tabular} & 
\begin{tabular}{l}MMD-perm \\ xMMD \\ xssMMD(knn) \\ xssMMD(ker) \\ xssMMD(rf)\end{tabular} & 
\begin{tabular}{r}\textbf{0.129} \\ 0.088 \\ 0.088 \\ 0.088 \\ 0.088 \end{tabular} & 
\begin{tabular}{r}0.099  \\ 0.088 \\ 0.159 \\ \textbf{0.176} \\ 0.161 \end{tabular} & 
\begin{tabular}{r}0.155\\0.116\\0.215\\ \textbf{0.244}\\0.242\end{tabular} & 
\begin{tabular}{r}0.129\\0.109\\0.261\\0.261\\ \textbf{0.279} \end{tabular} & 
\begin{tabular}{r}0.126\\0.092\\0.286\\0.312\\ \textbf{0.317}\end{tabular} & 
\begin{tabular}{r}0.112\\0.092\\0.309\\0.316\\ \textbf{0.326}\end{tabular} & 
\begin{tabular}{r}0.121\\0.091\\0.331\\0.320\\ \textbf{0.358}\end{tabular} & 
\begin{tabular}{r}0.113\\0.088\\0.342\\0.325\\ \textbf{0.361}\end{tabular} & 
\begin{tabular}{r}0.107\\0.093\\0.372\\0.351\\ \textbf{0.399}\end{tabular} & 
\begin{tabular}{r}0.108\\0.089\\0.387\\0.374\\ \textbf{0.424}\end{tabular} \\ 
\midrule
\begin{tabular}{l}Scenario \\ 3 (Alt)\end{tabular} & 
\begin{tabular}{l}MMD-perm \\ xMMD \\ xssMMD(knn) \\ xssMMD(ker) \\ xssMMD(rf)\end{tabular} & 
\begin{tabular}{r}0.264\\0.208\\0.208\\0.208\\0.208
\end{tabular} & 
\begin{tabular}{r}0.257\\0.197\\0.249\\0.261\\ \textbf{0.271}\end{tabular} & 
\begin{tabular}{r}0.231\\0.193\\0.260\\0.263\\ \textbf{0.292}\end{tabular} & 
\begin{tabular}{r}0.253\\0.189\\0.265\\0.297\\ \textbf{0.317}\end{tabular} & 
\begin{tabular}{r}0.267\\0.200\\0.290\\0.293\\ \textbf{0.322}\end{tabular} &  
\begin{tabular}{r}0.230\\0.179\\0.256\\0.248\\ \textbf{0.282}\end{tabular} & 
\begin{tabular}{r}0.272\\0.205\\0.292\\0.307\\ \textbf{0.337}\end{tabular} & 
\begin{tabular}{r}0.261\\0.188\\0.279\\0.298\\ \textbf{0.322}\end{tabular} & 
\begin{tabular}{r}0.247\\0.177\\0.275\\0.292\\ \textbf{0.315}\end{tabular} & 
\begin{tabular}{r}0.251\\0.206\\0.266\\0.272\\ \textbf{0.309}\end{tabular} \\ 
\midrule
\begin{tabular}{l}Scenario \\ 4 (Alt)\end{tabular} & 
\begin{tabular}{l}MMD-perm \\ xMMD \\ xssMMD(knn) \\ xssMMD(ker) \\ xssMMD(rf)\end{tabular} & 
\begin{tabular}{r}\textbf{0.406}\\0.286\\0.286\\0.286\\0.286\end{tabular} &
\begin{tabular}{r}0.394\\0.284\\0.441\\ \textbf{0.456}\\0.385\end{tabular} & 
\begin{tabular}{r}0.389\\0.285\\0.500\\ \textbf{0.511}\\0.413\end{tabular} & 
\begin{tabular}{r}0.394\\0.283\\0.508\\ \textbf{0.540}\\0.407\end{tabular} & 
\begin{tabular}{r}0.388\\0.281\\0.557\\ \textbf{0.564}\\0.400\end{tabular} & 
\begin{tabular}{r}0.435\\0.302\\0.549\\ \textbf{0.564}\\0.453\end{tabular} & 
\begin{tabular}{r}0.406\\0.278\\0.532\\ \textbf{0.544}\\0.439\end{tabular} & 
\begin{tabular}{r}0.405\\0.295\\0.544\\ \textbf{0.559}\\0.429\end{tabular} & 
\begin{tabular}{r}0.416\\0.289\\0.533\\ \textbf{0.547}\\0.445\end{tabular} & 
\begin{tabular}{r}0.414\\0.278\\0.538\\ \textbf{0.575}\\0.448\end{tabular} \\
\bottomrule
\end{tabular}
\end{table}

\subsection{Comparison with Joint Two-Sample Tests}\label{appendix: comparison with joint test}

In this section, we highlight the specific advantages of our semi-supervised test by comparing it against a standard two-sample test applied directly to the joint distributions, $P_{XV}$ vs.\ $P_{YW}$. To systematically evaluate this, we generate $10$-dimensional Gaussian vectors with a strong correlation of $\rho=0.95$ and induce distributional differences using a mean shift of $\epsilon=0.3$ and $j=1$.

Under the null ($P_X = P_Y$), we deliberately construct a shift in the auxiliary covariates ($P_V \neq P_W$). Specifically, both $X$ and $Y$ contain the mean shift in their first coordinate. However, $V$ is extracted from the shifted first two coordinates of $X$, while $W$ is extracted from the unshifted last two coordinates of $Y$, creating a pure auxiliary shift. As shown in \Cref{tab:joint_test_null}, the joint MMD test exhibits a Type-I error rate approaching $1.0$, as it incorrectly flags the irrelevant shift as a discrepancy of interest. In contrast, xssMMD effectively ignores this shift in the auxiliary space and reliably maintains nominal control of $\alpha=0.05$. To explicitly demonstrate this robustness, \Cref{tab:type1_varying_eps} shows the effect of varying the shift magnitude $\epsilon$. As the auxiliary shift intensifies from $0.0$ to $1.0$, the joint tests rapidly collapse, whereas all variants of xssMMD consistently maintain the Type-I error near the nominal level, remaining unaffected by the auxiliary noise.

\begin{table}[!htbp]
\centering
\footnotesize 
\setlength{\tabcolsep}{3pt}
\caption{Estimated Type-I error of the joint and marginal tests under the null when $P_X = P_Y$ and $P_V \neq P_W$. The joint MMD-perm test incorrectly rejects the null, while xssMMD maintains nominal control of $\alpha=0.05$.}
\label{tab:joint_test_null}
\begin{tabular}{lrrrrrrrrrr}
\toprule
& \multicolumn{10}{c}{$n_1+n_2$} \\
\cmidrule(lr){2-11}
Test &  20 &  40 &  60 &  80 &  100 &  120 &  140 &  160 &  180 &  200 \\
\midrule
 MMD-perm (Joint) & 0.143 & 0.272 & 0.459 & 0.670 & 0.837 & 0.929 & 0.969 & 0.996 & 1.000 & 1.000 \\
 xMMD & 0.170 & 0.246 & 0.407 & 0.546 & 0.851 & 0.769 & 0.711 & 0.916 & 0.953 & 0.973 \\
\midrule
 xssMMD(knn) & \textbf{0.037} & \textbf{0.046} & 0.052 & \textbf{0.050} & 0.051 & 0.049 & 0.044 & 0.049 & \textbf{0.048} & \textbf{0.042} \\
 xssMMD(ker) & 0.046 & 0.048 & \textbf{0.043} & 0.051 & \textbf{0.047} & \textbf{0.046} & 0.053 & \textbf{0.048} & 0.051 & 0.050 \\
 xssMMD(rf) & 0.046 & 0.055 & 0.051 & 0.055 & 0.056 & \textbf{0.046} & \textbf{0.043} & 0.049 & 0.053 & 0.054 \\
\bottomrule
\end{tabular}
\end{table}

\begin{table}[!htbp]
\centering
\footnotesize
\setlength{\tabcolsep}{3pt}
\caption{Estimated Type-I error of the joint and marginal tests under the null when $P_X=P_Y$ and $P_V \neq P_W$ with varying shift magnitude $\epsilon$.}
\label{tab:type1_varying_eps}
\begin{tabular}{lrrrrrr}
\toprule
& \multicolumn{6}{c}{$\epsilon$} \\
\cmidrule(lr){2-7}
Test & $0.0$ & $0.2$ & $0.4$ & $0.6$ & $0.8$ & $1.0$ \\
\midrule
MMD-perm        & 0.079 & 0.225 & 0.800 & 1.000 & 1.000 & 1.000 \\
xMMD            & 0.178 & 0.457 & 0.736 & 0.925 & 0.984 & 1.000 \\
\midrule
xssMMD(knn)     & \textbf{0.047} & \textbf{0.038} & \textbf{0.042} & 0.057 & 0.047 & 0.055 \\
xssMMD(ker)     & 0.061 & 0.049 & 0.045 & \textbf{0.056} & \textbf{0.044} & \textbf{0.054} \\
xssMMD(rf)      & 0.056 & 0.053 & 0.054 & 0.061 & 0.045 & 0.056 \\
\bottomrule
\end{tabular}
\end{table}

Under the alternative ($P_X \neq P_Y$), we configure the environment such that the auxiliary covariates are completely identical ($P_V = P_W$). Here, only $X$ contains the mean shift, but both $V$ and $W$ are extracted strictly from the unshifted last two coordinates of $X$ and $Y$, respectively. In this scenario, the joint test yields lower power as shown in \Cref{tab:joint_test_alt} since identical auxiliary marginal distributions act as high-dimensional noise, diluting the overall signal. In contrast, xssMMD leverages the underlying correlation between $X$ and $V$ to reduce variance, substantially outperforming both MMD-perm and xMMD. To further validate the utility of our approach, we examine the test power across different correlation strengths $\rho$ in \Cref{tab:power_varying_rho}. When the auxiliary data are uninformative ($\rho=0.0$), xssMMD retains the baseline power of the standard xMMD test, avoiding any negative transfer. As the correlation increases to $0.95$, xssMMD efficiently exploits the dependency structure to achieve higher power, whereas the power of the joint tests remains low.

\begin{table}[!htbp]
\centering
\footnotesize
\setlength{\tabcolsep}{3pt}
\caption{Estimated test power of the joint and marginal tests under the alternative when $P_X \neq P_Y$ and $P_V = P_W$. xssMMD achieves higher power, whereas joint tests struggle to capture the marginal difference.}
\label{tab:joint_test_alt}
\begin{tabular}{lrrrrrrrrrr}
\toprule
& \multicolumn{10}{c}{$n_1+n_2$} \\
\cmidrule(lr){2-11}
Test &  20 &  40 &  60 &  80 &  100 &  120 &  140 &  160 &  180 &  200 \\
\midrule
MMD-perm (Joint) & 0.078 & 0.089 & 0.121 & 0.180 & 0.244 & 0.303 & 0.385 & 0.495 & 0.633 & 0.737 \\
xMMD & 0.117 & 0.175 & 0.274 & 0.360 & 0.420 & 0.496 & 0.553 & 0.637 & 0.680 & 0.716 \\
\midrule
xssMMD(knn) & 0.097 & 0.234 & 0.522 & 0.696 & 0.830 & 0.905 & 0.935 & 0.948 & \textbf{0.975} & 0.985 \\
xssMMD(ker) & \textbf{0.152} & 0.344 & 0.552 & 0.681 & 0.787 & 0.875 & 0.906 & 0.935 & 0.962 & 0.969 \\
xssMMD(rf) & 0.151 & \textbf{0.380} & \textbf{0.604} & \textbf{0.763} & \textbf{0.858} & \textbf{0.919} & \textbf{0.946} & \textbf{0.958} & \textbf{0.975} & \textbf{0.987} \\
\bottomrule
\end{tabular}
\end{table}

\begin{table}[!htbp]
\centering
\footnotesize
\setlength{\tabcolsep}{3pt}
\caption{Estimated test power of the joint and marginal tests under the alternative when $P_X \neq P_Y$ and $P_V = P_W$ across varying correlation strengths $\rho$ between target and auxiliary covariates.}
\label{tab:power_varying_rho}
\begin{tabular}{lrrrrrr}
\toprule
& \multicolumn{6}{c}{$\rho$} \\
\cmidrule(lr){2-7}
Test & $0.00$ & $0.20$ & $0.40$ & $0.60$ & $0.80$ & $0.95$ \\
\midrule
MMD-perm        & \textbf{0.185} & 0.156 & 0.132 & 0.090 & 0.097 & 0.080 \\
xMMD            & 0.146 & 0.150 & 0.176 & 0.140 & 0.165 & 0.154 \\
\midrule
xssMMD(knn)     & 0.159 & \textbf{0.165} & 0.210 & 0.195 & \textbf{0.298} & 0.414 \\
xssMMD(ker)     & 0.165 & 0.159 & \textbf{0.221} & 0.202 & 0.297 & 0.369 \\
xssMMD(rf)      & 0.155 & 0.162 & 0.201 & \textbf{0.203} & 0.290 & \textbf{0.462} \\
\bottomrule
\end{tabular}
\end{table}

\subsection{Running-Time Comparison} \label{appendix: running time}
While permutation-based MMD tests provide strong finite-sample guarantees, they are often computationally prohibitive for large datasets. To demonstrate the computational efficiency of our proposed framework, we compared the execution time (in seconds) of xssMMD against xMMD and MMD-perm. We adopted the setting from Scenario 1 (Alt) to Scenario 4 (Alt), simultaneously varying the labeled sample size $n_1=n_2$ from 10 to 100 and the unlabeled sample size $m_1=m_2$ from 100 to 1000. 

The results, presented in \Cref{tab:running_time}, show a consistent efficiency ranking across most of the settings: xMMD requires the least amount of time, followed closely by xssMMD, while MMD-perm is the most computationally intensive. Notably, there is a substantial time difference between xssMMD and MMD-perm as the sample size increases. For instance, at $n_1+n_2=200$, xssMMD(knn) is approximately 13 times faster than MMD-perm in every scenarios. These results indicate that xssMMD provides a computationally efficient semi-supervised approach. It mitigates the computational burden of standard permutation tests without sacrificing the statistical benefits of incorporating unlabeled data.

\begin{table}[!htbp]
\centering
\footnotesize
\setlength{\tabcolsep}{2pt}
\caption{Running-time comparison (in seconds) across varying sample sizes. The total labeled sample size $n_1+n_2$ is shown, with the unlabeled sample size fixed at ten times the labeled size ($m_1+m_2 = 10(n_1+n_2)$).}
\label{tab:running_time}
\begin{tabular}{llrrrrrrrrrr}
\toprule
& & \multicolumn{10}{c}{$n_1+n_2$} \\
\cmidrule(lr){3-12}
Scenario & Method & $20$ & $40$ & $60$ & $80$ & $100$ & $120$ & $140$ & $160$ & $180$ & $200$ \\
\midrule
\multirow{5}{*}{Scenario 1 (Alt)} & MMD-perm & 32.594 & 55.570 & 354.557 & 440.820 & 503.935 & 500.720 & 548.205 & 566.169 & 575.530 & 865.501 \\
& xMMD & 0.709 & 0.650 & 1.061 & 1.008 & 3.121 & 3.489 & 4.312 & 3.565 & 4.147 & 3.464 \\
& xssMMD(knn) & 9.130 & 13.794 & 16.985 & 21.670 & 28.756 & 35.236 & 45.231 & 49.880 & 57.319 & 64.845 \\
& xssMMD(ker) & 10.337 & 16.008 & 20.623 & 26.023 & 34.358 & 41.907 & 55.003 & 57.790 & 68.878 & 78.862 \\
& xssMMD(rf) & 278.681 & 270.142 & 254.706 & 256.005 & 268.618 & 281.739 & 297.154 & 308.513 & 333.594 & 321.440 \\
\midrule
\multirow{5}{*}{Scenario 2 (Alt)} 
& MMD-perm & 29.597 & 58.656 & 366.912 & 448.648 & 500.716 & 503.202 & 547.274 & 570.824 & 552.232 & 838.050 \\
& xMMD & 0.656 & 0.721 & 1.019 & 1.029 & 3.224 & 3.448 & 4.259 & 3.798 & 3.375 & 3.594 \\
& xssMMD(knn) & 8.169 & 14.346 & 17.180 & 22.723 & 30.735 & 37.411 & 49.342 & 52.368 & 56.562 & 66.373 \\
& xssMMD(ker) & 9.164 & 16.864 & 20.907 & 27.751 & 37.016 & 44.399 & 60.935 & 59.687 & 68.209 & 81.675 \\
& xssMMD(rf) & 277.267 & 270.677 & 250.552 & 259.682 & 272.121 & 286.572 & 314.231 & 317.082 & 329.089 & 325.893 \\
\midrule
\multirow{5}{*}{Scenario 3 (Alt)} 
& MMD-perm & 29.407 & 58.977 & 371.145 & 442.269 & 506.101 & 501.124 & 553.109 & 577.414 & 547.026 & 824.445 \\
& xMMD & 0.660 & 0.675 & 0.994 & 1.015 & 3.229 & 3.003 & 3.908 & 3.685 & 3.076 & 3.356 \\
& xssMMD(knn) & 7.972 & 14.525 & 16.849 & 22.682 & 30.519 & 36.164 & 47.830 & 51.671 & 55.769 & 65.830 \\
& xssMMD(ker) & 9.031 & 17.077 & 20.441 & 27.328 & 36.228 & 42.917 & 59.162 & 57.778 & 65.760 & 81.043 \\
& xssMMD(rf) & 279.430 & 273.422 & 250.447 & 261.256 & 272.815 & 287.031 & 315.216 & 317.947 & 332.128 & 331.039 \\
\midrule
\multirow{5}{*}{Scenario 4 (Alt)} 
& MMD-perm & 30.948 & 59.501 & 352.966 & 434.063 & 486.332 & 480.628 & 565.353 & 573.585 & 540.560 & 856.321 \\
& xMMD & 0.622 & 0.646 & 1.066 & 1.052 & 3.176 & 3.600 & 4.570 & 3.366 & 3.637 & 3.894 \\
& xssMMD(knn) & 7.725 & 13.961 & 17.659 & 23.076 & 30.688 & 37.273 & 49.264 & 51.885 & 58.866 & 67.369 \\
& xssMMD(ker) & 9.842 & 16.055 & 21.594 & 27.440 & 36.015 & 44.364 & 60.609 & 59.432 & 69.667 & 80.912 \\
& xssMMD(rf) & 272.036 & 265.660 & 256.013 & 258.103 & 269.625 & 292.752 & 316.483 & 321.325 & 338.474 & 339.847 \\
\bottomrule
\end{tabular}
\end{table}

\subsection{Details of the Real-World Experiment: HTRU2 Pulsar dataset} \label{appendix: HTRU2 dataset}

This subsection provides a detailed description of our experiments on real-world data using the HTRU2 pulsar dataset. The HTRU2 dataset consists of measurements from radio astronomy, with each sample characterized by eight continuous features derived from the integrated pulse profile (IP) and the DM-SNR curve (DM). The dataset contained 1639 pulsar and 16259 non-pulsar observations and we constructed a testing problem to determine whether the distribution of features differed between the two classes.

\begin{table}[tb!]
\centering
\footnotesize
\setlength{\tabcolsep}{3pt}
\caption{Estimated test power under different scenarios and increasing noise levels (independent Gaussian noise with standard deviation $\sigma \in \{0, 0.1, 0.3, 0.5, 0.7, 1.0\}$). Each value corresponds to the average test power across $1000$ trials. Our proposed method, xssMMD, was implemented using random forest for the conditional expectation estimation.}
\label{table: pulsar experiment}
\begin{tabular}{llrrrrrr}
\toprule
Labeled Data & Test & $\sigma = 0$ & $\sigma = 0.1$ & $\sigma = 0.3$ & $\sigma = 0.5$ & $\sigma = 0.7$ & $\sigma = 1.0$ \\
\midrule
\begin{tabular}{l}IP(Mean), \\ DM(Mean) \end{tabular} & 
\begin{tabular}{l}MMD-perm \\ xMMD \\ xssMMD\end{tabular} & 
\begin{tabular}{r}\textbf{1.000} \\ 0.964 \\ 0.998\end{tabular} & 
\begin{tabular}{r}\textbf{1.000} \\ 0.962 \\ 0.997\end{tabular} & 
\begin{tabular}{r}\textbf{0.999} \\ 0.891 \\ 0.983\end{tabular} & 
\begin{tabular}{r}0.859 \\ 0.690 \\ \textbf{0.869}\end{tabular} & 
\begin{tabular}{r}0.442 \\ 0.398 \\ \textbf{0.567}\end{tabular} & 
\begin{tabular}{r}0.120 \\ 0.156 \\ \textbf{0.232}\end{tabular} \\
\midrule
\begin{tabular}{l}IP(Mean, SD)\end{tabular}  & 
\begin{tabular}{l}MMD-perm \\ xMMD \\ xssMMD\end{tabular} & 
\begin{tabular}{r}0.106 \\ 0.205 \\ \textbf{0.560}\end{tabular} & 
\begin{tabular}{r}0.095 \\ 0.185 \\ \textbf{0.531}\end{tabular} & 
\begin{tabular}{r}0.052 \\ 0.123 \\ \textbf{0.342}\end{tabular} & 
\begin{tabular}{r}0.015 \\ 0.062 \\ \textbf{0.181}\end{tabular} & 
\begin{tabular}{r}0.090 \\ 0.038 \\ \textbf{0.105}\end{tabular} & 
\begin{tabular}{r}0.013 \\ 0.030 \\ \textbf{0.061}\end{tabular} \\
\midrule
\begin{tabular}{l}IP(Mean, SD, EK, Skew) \end{tabular}   & 
\begin{tabular}{l}MMD-perm \\ xMMD \\ xssMMD\end{tabular} & 
\begin{tabular}{r}0.262 \\ \textbf{0.402} \\ 0.367\end{tabular} & 
\begin{tabular}{r}0.250 \\ \textbf{0.388} \\ 0.361\end{tabular} & 
\begin{tabular}{r}0.173 \\ \textbf{0.282} \\ 0.271\end{tabular} & 
\begin{tabular}{r}0.084 \\ \textbf{0.173} \\ 0.170\end{tabular} & 
\begin{tabular}{r}0.054 \\ \textbf{0.069} \\ 0.064\end{tabular} & 
\begin{tabular}{r}0.014 \\ 0.053 \\ \textbf{0.064}\end{tabular} \\
\midrule
\begin{tabular}{l}IP(Mean, SD), \\DM(Mean, SD)\end{tabular}  & 
\begin{tabular}{l}MMD-perm \\ xMMD \\ xssMMD\end{tabular} & 
\begin{tabular}{r}\textbf{1.000} \\ 0.929 \\ 0.999\end{tabular} & 
\begin{tabular}{r}\textbf{0.999} \\ 0.927 \\ 0.998\end{tabular} & 
\begin{tabular}{r}\textbf{0.985} \\ 0.832 \\ 0.982\end{tabular} & 
\begin{tabular}{r}0.762 \\ 0.605 \\ \textbf{0.849}\end{tabular} & 
\begin{tabular}{r}0.360 \\ 0.256 \\ \textbf{0.535}\end{tabular} & 
\begin{tabular}{r}0.072 \\ 0.135 \\ \textbf{0.222}\end{tabular} \\
\midrule
\begin{tabular}{l}IP(EK, Skew), \\ DM(EK, Skew)\end{tabular} & 
\begin{tabular}{l}MMD-perm \\ xMMD \\ xssMMD\end{tabular} & 
\begin{tabular}{r}0.758 \\ 0.694 \\ \textbf{0.941}\end{tabular} & 
\begin{tabular}{r}0.726 \\ 0.682 \\ \textbf{0.944}\end{tabular} & 
\begin{tabular}{r}0.552 \\ 0.608 \\ \textbf{0.880}\end{tabular} & 
\begin{tabular}{r}0.284 \\ 0.390 \\ \textbf{0.674}\end{tabular} & 
\begin{tabular}{r}0.109 \\ 0.218 \\ \textbf{0.405}\end{tabular} & 
\begin{tabular}{r}0.043 \\ 0.085 \\ \textbf{0.183}\end{tabular} \\
\midrule
\begin{tabular}{l}IP(Mean, SD, EK), \\ DM(Mean, SD, EK) \end{tabular}& 
\begin{tabular}{l}MMD-perm \\ xMMD \\ xssMMD\end{tabular}  & 
\begin{tabular}{r}0.985 \\ 0.860 \\ \textbf{0.994}\end{tabular} & 
\begin{tabular}{r}0.983 \\ 0.845 \\ \textbf{0.987}\end{tabular} & 
\begin{tabular}{r}0.932 \\ 0.740 \\ \textbf{0.944}\end{tabular} & 
\begin{tabular}{r}0.636 \\ 0.537 \\ \textbf{0.812}\end{tabular} & 
\begin{tabular}{r}0.251 \\ 0.296 \\ \textbf{0.528}\end{tabular} & 
\begin{tabular}{r}0.057 \\ 0.118 \\ \textbf{0.208}\end{tabular} \\
\midrule
\begin{tabular}{l}IP(SD, EK, Skew), \\ DM(SD, EK, Skew) \end{tabular}& 
\begin{tabular}{l}MMD-perm \\ xMMD \\ xssMMD\end{tabular}  & 
\begin{tabular}{r}0.811 \\ 0.719 \\ \textbf{0.934}\end{tabular} & 
\begin{tabular}{r}0.786 \\ 0.713 \\ \textbf{0.932}\end{tabular} & 
\begin{tabular}{r}0.612 \\ 0.621 \\ \textbf{0.865}\end{tabular} & 
\begin{tabular}{r}0.307 \\ 0.432 \\ \textbf{0.681}\end{tabular} & 
\begin{tabular}{r}0.106 \\ 0.232 \\ \textbf{0.417}\end{tabular} & 
\begin{tabular}{r}0.022 \\ 0.089 \\ \textbf{0.165}\end{tabular} \\
\bottomrule
\end{tabular}
\end{table}

\begin{table}[tb!]
\centering
\footnotesize
\setlength{\tabcolsep}{3pt}
\caption{Estimated test power under different scenarios with different unlabeled data and increasing noise levels (independent Gaussian noise with standard deviation $\sigma \in \{0, 0.1, 0.3, 0.5, 0.7, 1.0\}$). Each value corresponds to the average test power across $1000$ trials.}
\label{table: pulsar experiment2}
\begin{tabular}{lllrrrrrr}
\toprule
\multicolumn{2}{c}{Data Features} & \multirow{2}{*}{Test} & \multirow{2}{*}{$\sigma = 0$} & \multirow{2}{*}{$\sigma = 0.1$} & \multirow{2}{*}{$\sigma = 0.3$} & \multirow{2}{*}{$\sigma = 0.5$} & \multirow{2}{*}{$\sigma = 0.7$} & \multirow{2}{*}{$\sigma = 1.0$} \\
\cmidrule(lr){1-2}
Labeled Data & Unlabeled Data & & & & & & & \\
\midrule
\multirow{6}{*}{\begin{tabular}[l]{@{}l@{}}IP(Mean, SD), \\ DM(Mean, SD)\end{tabular}} & \multirow{3}{*}{\begin{tabular}[l]{@{}l@{}}$V$: IP(EK, Skew) \\$W$: DM(EK, Skew)\end{tabular}} & MMD-perm & \textbf{1.000} & \textbf{1.000} & \textbf{0.985} & \textbf{0.757} & 0.290 & 0.072 \\
& & xMMD & 0.935 & 0.932 & 0.837 & 0.622 & 0.345 & 0.138 \\
& & xssMMD & 0.968 & 0.961 & 0.889 & 0.687 & \textbf{0.370} & \textbf{0.145} \\
\cmidrule{2-9}
& \multirow{3}{*}{\begin{tabular}[l]{@{}l@{}}$V$: IP(EK) \\ $W$: DM(EK)\end{tabular}} & MMD-perm & \textbf{1.000} & \textbf{1.000} & \textbf{0.985} & \textbf{0.757} & 0.290 & 0.072 \\
& & xMMD & 0.935 & 0.932 & 0.837 & 0.622 & 0.345 & 0.138 \\
& & xssMMD & 0.962 & 0.956 & 0.875 & 0.663 & \textbf{0.349} & \textbf{0.145} \\
\midrule
\multirow{6}{*}{\begin{tabular}[l]{@{}l@{}}IP(EK, Skew), \\DM(EK, Skew) \end{tabular}} & \multirow{3}{*}{\begin{tabular}[l]{@{}l@{}}$V$: IP(Mean, SD) \\$W$: DM(Mean, SD)\end{tabular}} & MMD-perm & 0.753 & 0.730 & 0.549 & 0.392 & 0.130 & 0.050 \\
& & xMMD & 0.729 & 0.728 & 0.640 & 0.431 & 0.259 & 0.098 \\
& & xssMMD & \textbf{0.878} & \textbf{0.872} & \textbf{0.776} & \textbf{0.569} & \textbf{0.323} & \textbf{0.135} \\
\cmidrule{2-9}
& \multirow{3}{*}{\begin{tabular}[l]{@{}l@{}}$V$: IP(Mean) \\$W$: DM(Mean)\end{tabular}} & MMD-perm & 0.753 & 0.730 & 0.549 & 0.392 & 0.130 & 0.050 \\
& & xMMD & 0.729 & 0.728 & 0.640 & 0.431 & 0.259 & 0.098 \\
& & xssMMD & \textbf{0.84} & \textbf{0.833} & \textbf{0.750} & \textbf{0.547} & \textbf{0.321} & \textbf{0.133} \\
\midrule
\multirow{6}{*}{\begin{tabular}[l]{@{}l@{}}IP(SD, EK, Skew), \\ DM(SD, EK, Skew)\end{tabular}} & \multirow{3}{*}{\begin{tabular}[l]{@{}l@{}}$V$: IP(Mean), \\ \hspace{0.5cm}DM(Mean) \\$W$: IP(Mean)\end{tabular}} & MMD-perm & 0.811 & 0.786 & 0.612 & 0.307 & 0.106 & 0.022 \\
& & xMMD & 0.719 & 0.713 & 0.621 & 0.432 & 0.232 & 0.089 \\
& & xssMMD & \textbf{0.851} & \textbf{0.836} & \textbf{0.760} & \textbf{0.544} & \textbf{0.306} & \textbf{0.122} \\
\cmidrule{2-9}
& \multirow{3}{*}{\begin{tabular}[l]{@{}l@{}}$V$: IP(Mean) \\ $W$: DM(Mean)\end{tabular}} & MMD-perm & 0.811 & 0.786 & 0.612 & 0.307 & 0.106 & 0.022 \\
& & xMMD & 0.719 & 0.713 & 0.621 & 0.432 & 0.232 & 0.089 \\
& & xssMMD & \textbf{0.833} & \textbf{0.819} & \textbf{0.728} & \textbf{0.516} & \textbf{0.274} & \textbf{0.100} \\
\bottomrule
\end{tabular}
\end{table}

For each trial, we randomly selected $n_1=100$ pulsar and $n_2=100$ non-pulsar samples as labeled data and the remaining dataset formed the unlabeled data. In detail, $X$ comprised some covariates of random pulsar samples, $V$ comprised other covariates of the remaining pulsar observations, and the same applied to $Y$ and $W$. Input features were standardized before testing.

To simulate meaningful and challenging scenarios, we chose feature subsets as the labeled data with varying information and added noise. We considered six experimental settings, each defined by combining IP and DM statistics: (1) the means of IP and DM; (2) the mean and standard deviation (SD) of the IP; (3) the mean, SD, EK, and Skew of IP; (4) the mean and SD of both IP and DM; (5) the EK and Skew of both IP and DM; and (6) the mean, SD, and EK from both IP and DM. We used the remaining covariates as the unlabeled data for each scenario. For each setting, we added independent Gaussian noise with a standard deviation from 0 to 1 to the labeled data, thereby gradually reducing their discriminative power. This process enabled an examination of how feature groups responded to increasing corruption and how effectively the auxiliary covariates could be leveraged when the primary covariates were severely degraded by noise. Finally, each method used 50 random splits with 20 repetitions per split, and we reported the average test power.

The result is summarized in \Cref{table: pulsar experiment}. Our method, xssMMD, achieved performance comparable to or significantly exceeding that of both the standard kernel test (MMD-perm) and xMMD across most of the feature settings and noise levels, demonstrating higher power, particularly under higher noise or with weakly informative features. By using unlabeled data effectively, xssMMD is more sensitive in detecting distributional differences. Incorporating auxiliary covariates is especially beneficial when the main features are weak, as shown in this setup. Using additional information clearly enhanced test sensitivity and robustness, as our results demonstrate.

We further investigated the robustness of xssMMD in a more challenging scenario. In this case, the unlabeled data for the pulsar and non-pulsar classes contained different covariates so they could not be used to test the two-sample problem alone. In other words, $V$ and $W$ came from non-comparable feature spaces. The result is summarized in \Cref{table: pulsar experiment2}. Even in this case, xssMMD consistently maintained a test power comparable to or higher than MMD-perm and xMMD across most of the levels of Gaussian noise. This result highlights the key advantage of xssMMD. It successfully integrates information from the auxiliary sets $V$ and $W$ for each class, even though $V$ and $W$ come from different feature spaces and cannot be directly compared. This integration improves the effectiveness of the primary two-sample test on labeled data.

\subsection{Details of the Real-World Experiment: Caltech-UCSD Bird dataset} \label{appendix: CUB dataset}
\begin{figure}
\centering
    \begin{subfigure}[t]{0.45\textwidth} 
        \centering
        \includegraphics[width=\linewidth]{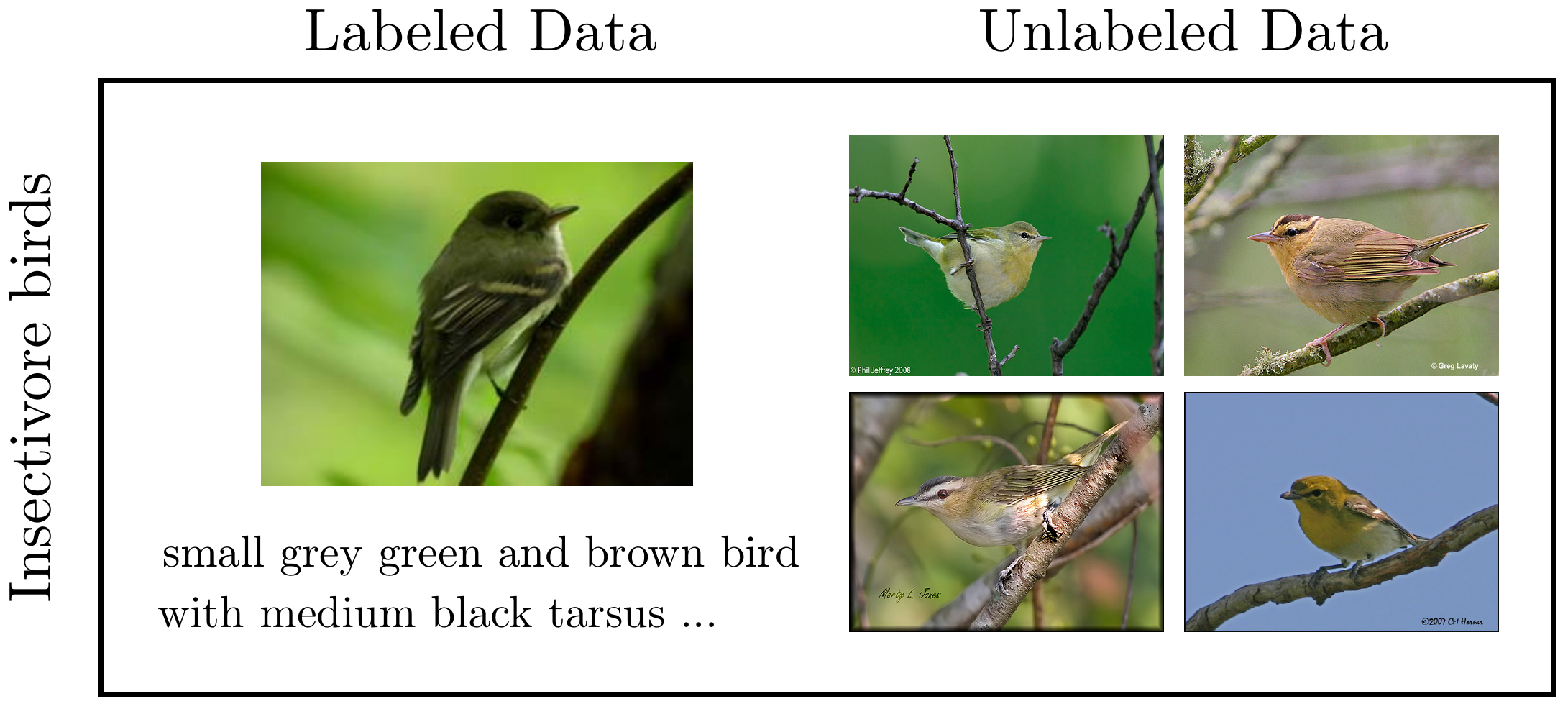}
    \end{subfigure}
    \hspace{0.02\textwidth}
    \begin{subfigure}[t]{0.45\textwidth}
        \centering
        \includegraphics[width=\linewidth]{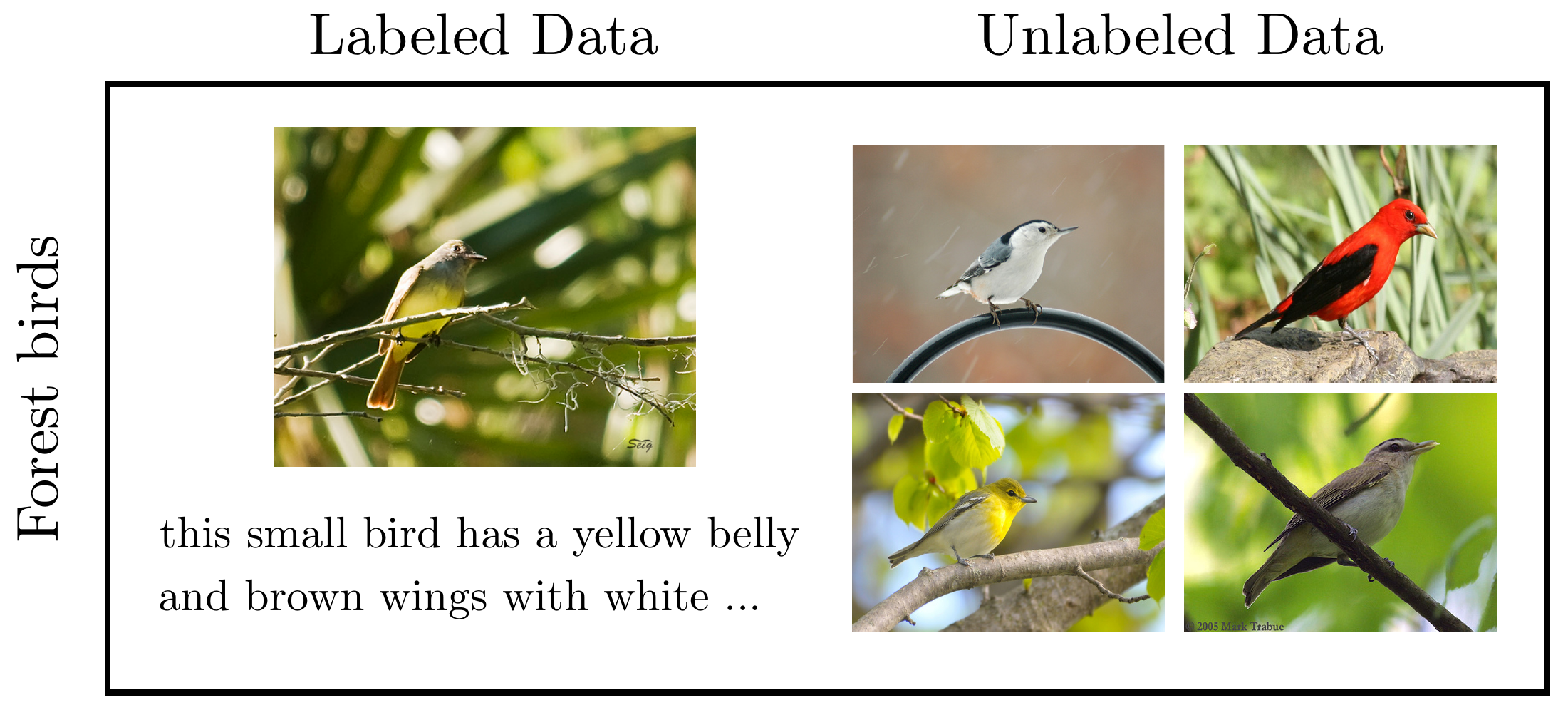}
    \end{subfigure}
    \caption{An example of data construction when testing coastal birds against grassland birds. Labeled data consists of both text and corresponding images, while unlabeled data consists only of images.}
    \label{fig:description of cub data}
\end{figure}
This subsection provides a detailed description of our experiments on real-world data using the Caltech-UCSD Bird dataset (CUB-200-2011). The dataset includes both images and text descriptions for each bird species. This setup allows us to model a two-sample testing problem using multi-modal data. There are 11,788 images and 2,000 sentences of 200 bird species. We extracted text features (primary covariates) by obtaining embeddings with a pre-trained BERT base model (`bert-base-uncased'). The CLS token embedding from BERT's last hidden state was then passed through a single-layer MLP (hidden dimension 128, output dimension 4, dropout 0.2, ReLU activation) for dimensionality reduction. For image features (auxiliary covariates), a pre-trained ResNet-152 was used, with standard preprocessing (resize to 64×64 pixels, normalization). These features were further reduced with a separate MLP (hidden dimensions 1024 and 256, output 32, dropout 0.2, ReLU). Batch normalization followed each linear layer in the text and image MLPs. Following \citet{biggs2024mmd}, these text and image embeddings were used within the two-sample testing framework.

We grouped bird species by diet (Insect, Seed, Fish) and habitat (Forest, Scrub, Wetland), forming three distinct comparison pairs. This grouping tested whether primary covariates (text) alone could distinguish species, while auxiliary covariates (image) provided complementary information. This classification was chosen specifically to create a scenario where textual descriptions, which often emphasize species-specific attributes, would be less sufficient for accurate group differentiation, whereas the image backgrounds would provide crucial contextual information about the habitat, making the auxiliary data more significant. The three pairs were Insect (12 species) vs. Forest (9), Seed (7) vs. Scrub (7), and Fish (8) vs. Wetland (7); in every comparison, 6 species overlapped, simulating realistic conditions with significant species overlap. To further investigate this phenomenon, we conducted additional experiments focusing on specific species comparisons: we tested 15 species of sparrow against 18 species of ground picker, compared 10 species of cuckoo with 13 species of foliage gleaner, and also examined 14 species of warbler against 12 species of canopy explorer. This data setup is illustrated in \Cref{fig:description of cub data}. For each trial, we randomly selected $n_1=150$ Group 1 and $n_2=150$ Group 2 labeled samples, then chose 200 samples as unlabeled data. For example, $X$ contained random text samples from insectivorous species, $V$ contained images of corresponding species, including those not in $X$; the same approach applied to $Y$ and $W$. We applied the standard kernel MMD (MMD-perm) and xMMD tests using text embeddings only, while our proposed test, xssMMD, incorporated both text and image embeddings. Each method used 50 random splits with 20 repetitions per split, and we reported average test power.

The estimated power for detecting distributional differences between bird groups is shown in \Cref{table: cub dataset} and \Cref{table: cub dataset2}. Across all testing scenarios, the permutation-based method MMD-perm generally achieved higher power than the permutation-free method xMMD. Importantly, our proposed method, xssMMD, consistently showed the highest test power in all comparisons. This underscores the advantage of our method, which leverages image embeddings to boost test power. The consistently strong performance of xssMMD across all pairs highlights the benefit of using complementary information from different modalities.

\begin{table}[t]
\footnotesize
\setlength{\tabcolsep}{3pt}
	\caption{Additional results of the estimated test power for detecting the difference between two bird groups with test level $\alpha= 0.05$. Our proposed method, xssMMD, was implemented using knn for the conditional expectation estimation.}
	\label{table: cub dataset2}
	\begin{center}
    \small
	\begin{tabular}{lllr}
	\toprule
	Group 1         & Group 2     & Test        & Power \\
	\midrule
	                &             & MMD-perm    & 0.783 \\
	Sparrow          & Ground picker      & xMMD        & 0.712 \\
                    &             & xssMMD      & \textbf{0.978} \\
	\hline
	                &               & MMD-perm    & 0.766 \\
	Cuckoo			& Foliage gleaner       & xMMD        & 0.697 \\
					&             & xssMMD      & \textbf{0.882} \\
	\hline
    	           &             & MMD-perm    & 0.883 \\
	Warbler		    & Canopy explorer      & xMMD        & 0.650 \\
					&             & xssMMD      & \textbf{0.991} \\
	\bottomrule
	\end{tabular}
	\end{center}
    \vspace{-0.4cm}
\end{table}

\subsection{Details of the Real-World Experiment: MNIST dataset} \label{appendix: MNIST dataset}

This subsection provides a detailed description of our experiments on real-world data using the MNIST digit dataset, a widely used benchmark for visual recognition. The MNIST dataset consists of $28\times 28$ grayscale images of handwritten digits ($0$ through $9$). In this experiment, we constructed a two-sample testing problem by partitioning the dataset into two distinct classes ($\mathcal{D}_1$ vs. $\mathcal{D}_2$) to detect distributional differences between the classes. The experimental design was motivated by the work of \cite{schrab2023mmd} and \cite{chatterjee2025boosting}, aiming to simulate a challenging scenario where the primary information is easily corrupted but supplemented by abundant auxiliary covariates.

For the main test, the labeled data $X$ and $Y$ consisted of pixel data from clear, original images. The unlabeled auxiliary covariates $V$ and $W$ consisted of the full pixel data of the images, into which we systematically injected independent Gaussian noise $\epsilon$ with increasing standard deviation $\sigma$. This process was motivated by examining the method's robustness against data corruption. Specifically, for each observation and each pixel entry $i,j$, the noise $\epsilon_{i,j}$ was sampled independently from a univariate normal distribution with zero mean and variance $\sigma^2$. The final noisy covariates were generated by adding this noise $\epsilon$ to the normalized original images, $V$ and $ W$. Subsequently, the resulting pixel values were constrained to remain within the valid range of $[0, 1]$. The operation for each entry of the noisy matrix is precisely defined as: 

$$ (V_{\text{noisy}})_{i,j} = \begin{cases} 0, & \text{if } V_{i,j} + \epsilon_{i,j} < 0 \\ 1, & \text{if } V_{i,j} + \epsilon_{i,j} > 1 \\ V_{i,j} + \epsilon_{i,j}, & \text{otherwise} \end{cases}$$ 

This matrix-wise operation ensures that the added noise $\epsilon$ is independent across all pixels and observations, and the resulting pixel values remain within the valid range. This setup allows us to examine how our method utilizes auxiliary information under increasing levels of corruption. 

We considered three experimental settings for partitioning the data into two classes: (1) $\mathcal{D}_1 = \{0,1,2,3,4,5,9\}$ vs $\mathcal{D}_2 = \{0,1,2,3,4,5,8\}$; (2) $\mathcal{D}_1 = \{0,1,2,3,9\}$ vs $\mathcal{D}_2 = \{0,1,2,3,6\}$; and (3) $\mathcal{D}_1 = \{0,1,2,3,5,8\}$ vs $\mathcal{D}_2 = \{0,1,2,3,5,9\}.$ For each setting, the dataset was divided by randomly sampling $n_1 = 200$ images from $\mathcal{D}_1$ and $n_2 = 200$ images from $\mathcal{D}_2$ to create the labeled data for each group. From the remaining images of each class in $\mathcal{D}_1$ and $\mathcal{D}_2$, $2000$ samples per class were randomly selected and used as the unlabeled auxiliary covariates. Each method was then evaluated on $100$ random partitions of the data, with $10$ repetitions per partition, and average test power was reported. 

The results are summarized in \Cref{table: mnist experiment}. Across all feature settings and noise levels, the xssMMD test consistently outperformed the standard kernel test (MMD-perm) and xMMD test. When the auxiliary data was clear ($\sigma=0$), xssMMD showed superior power because the additional clear images effectively increased the sample size of the informative pixel features, confirming that leveraging the full, uncorrupted image data significantly enhances test power. More importantly, even when substantial Gaussian noise was present, xssMMD maintained a significant advantage over MMD-perm and xMMD. Notably, in the highest noise scenario ($\sigma=1$), severely degrading the visual quality of the auxiliary covariates, xssMMD consistently demonstrated superior performance compared to xMMD. These results confirm the crucial utility of the auxiliary covariates and demonstrate the robustness of xssMMD in utilizing noisy information; our method successfully extracts the underlying structural differences between the two digit distributions from the corrupted auxiliary covariates, confirming that the utility of incorporating additional information persists, even when the discriminative signal in the primary labeled features is weak and the auxiliary data is corrupted.

We further investigated the scenario where the amount of unlabeled data was significantly reduced to $200$ samples per class, matching the number of labeled samples ($n_1=n_2=200$). With the same partitioning of (1), (2), and (3), we compared the performance of xssMMD against the baseline tests performed solely on the labeled data (MMD-perm($X$,$Y$), xMMD) and the unlabeled data (MMD-perm($V$,$W$)). The results are summarized in \Cref{table: mnist experiment2}. When the noise level was minimal ($\sigma=0$), testing on the unlabeled data yielded high performance since this setting is the same as testing on the labeled data. However, the performance of MMD-perm on the unlabeled data drastically decreased as the noise level $\sigma$ increased. On the other hand, xssMMD demonstrated superior power compared to the MMD-perm and xMMD tests in most scenarios when the noise is not extremely large. As the noise increases, the power of xssMMD test decreases gradually, while the power of the MMD-perm test on the unlabeled data decreases drastically, falling below that of xssMMD in high-noise scenarios. These results indicate that xssMMD effectively leverages the structural information of the limited noisy auxiliary covariates to enhance test sensitivity, even when the auxiliary covariates themselves are not sufficiently distinct for successful permutation testing.

\begin{figure}[htb!]	\centering
	\includegraphics[width=\columnwidth]{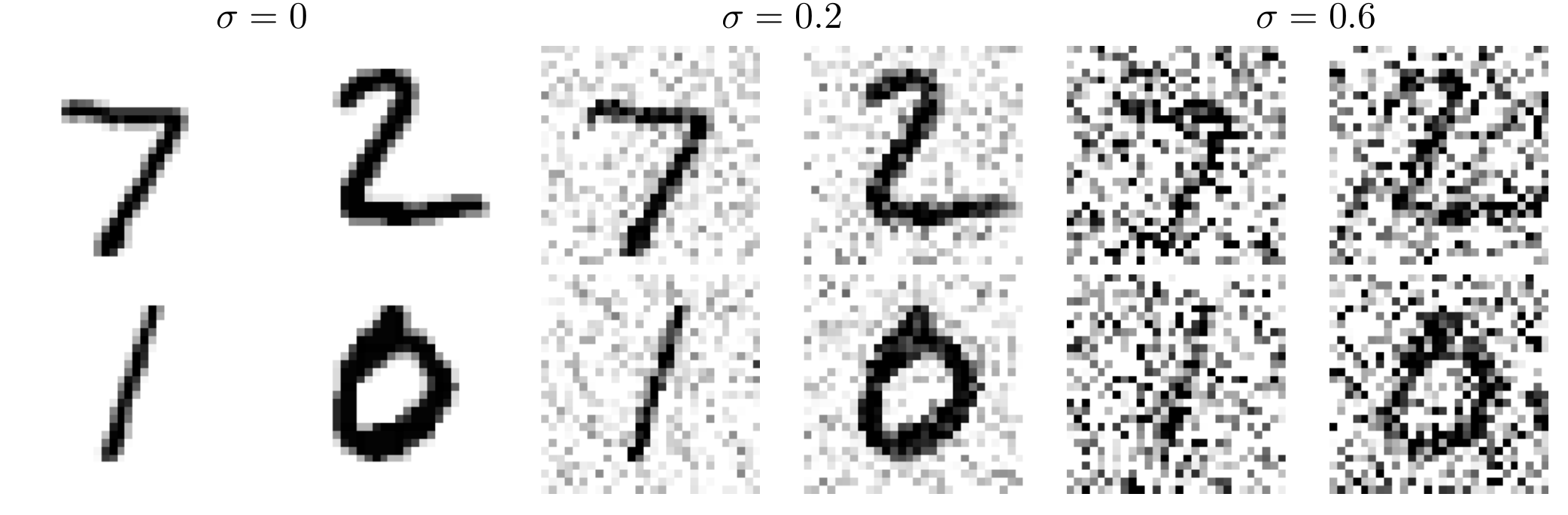}
	\caption{An example of data construction of images with Gaussian noise of $\sigma=0,\; 0.2,\;0.6$. Labeled data consists of both clear image without any noise and corresponding image with noise, while unlabeled data consists only of images with Gaussian noise.}
    \label{fig:description of mnist data} 
    \label{fig:mnist_fig}
    \vspace{-0.2cm}
\end{figure}

\begin{table}[htb!]
\centering
\footnotesize
\setlength{\tabcolsep}{3pt}
\caption{Estimated test power under different scenarios and increasing noise levels (independent Gaussian noise with standard deviation $\sigma \in \{0, 0.2, 0.4, 0.6, 0.8, 1.0\}$). Each value corresponds to the average test power across 1000 trials. Our proposed method, xssMMD, was implemented using knn for the conditional expectation estimation.}
\label{table: mnist experiment}
\begin{tabular}{lrrrrrrrr}
\toprule
\multirow{2}{*}{Scenario} & \multirow{2}{*}{MMD-perm} & \multirow{2}{*}{xMMD} & \multicolumn{6}{c}{xssMMD} \\
\cmidrule(lr){4-9}
& & & $\sigma = 0$ & $\sigma = 0.2$ & $\sigma = 0.4$ & $\sigma = 0.6$ & $\sigma = 0.8$ & $\sigma = 1.0$ \\
\midrule
\shortstack{\parbox{2.5cm}{
    $\{0,1,2,3,5,8\}$ \\ 
    \text{vs }$\{0,1,2,3,5,9\}$}} &
\begin{tabular}{r}0.731 \end{tabular} & 
\begin{tabular}{r}0.601 \end{tabular} & 
\begin{tabular}{r}\textbf{0.85} \end{tabular} & 
\begin{tabular}{r}\textbf{0.852} \end{tabular} & 
\begin{tabular}{r}\textbf{0.862} \end{tabular} & 
\begin{tabular}{r}\textbf{0.837} \end{tabular} & 
\begin{tabular}{r}\textbf{0.776} \end{tabular} & 
\begin{tabular}{r}0.721 \end{tabular} \\
\midrule
\shortstack{\parbox{2.5cm}{
    $\{0,1,2,3,9\}$ \\ 
    \text{vs }$\{0,1,2,3,6\}$}} &
\begin{tabular}{r}0.985 \end{tabular} & 
\begin{tabular}{r}0.915 \end{tabular} & 
\begin{tabular}{r}\textbf{0.999} \end{tabular} & 
\begin{tabular}{r}\textbf{0.999}  \end{tabular} & 
\begin{tabular}{r}\textbf{0.999} \end{tabular} & 
\begin{tabular}{r}\textbf{0.999}  \end{tabular} & 
\begin{tabular}{r}\textbf{0.998} \end{tabular} & 
\begin{tabular}{r}\textbf{0.997}  \end{tabular}\\
\midrule
\shortstack{\parbox{2.5cm}{
    $\{0,1,2,3,4,5,9\}$ \\ 
    \text{vs }$\{0,1,2,3,5,8\}$}} &
\begin{tabular}{r}0.55 \end{tabular} & 
\begin{tabular}{r}0.402 \end{tabular} & 
\begin{tabular}{r}\textbf{0.666} \end{tabular} & 
\begin{tabular}{r}\textbf{0.679} \end{tabular} & 
\begin{tabular}{r}\textbf{0.694} \end{tabular} & 
\begin{tabular}{r}\textbf{0.659} \end{tabular} & 
\begin{tabular}{r}\textbf{0.577} \end{tabular} & 
\begin{tabular}{r}0.496 \end{tabular}\\
\bottomrule
\end{tabular}
\end{table}

\begin{table}[tb!]
\centering
\footnotesize
\setlength{\tabcolsep}{3pt}
\caption{Estimated test power under different scenarios including kernel test on the unlabeled data only and increasing noise levels (independent Gaussian noise with standard deviation $\sigma \in \{0, 0.5, 1.0, 1.5, 2.0\}$). Each value corresponds to the average test power across $1000$ trials. Our proposed method, xssMMD, was implemented using knn for the conditional expectation estimation.}
\label{table: mnist experiment2}
\begin{tabular}{llrrrrr}
\toprule
Labeled Data & Test & $\sigma = 0$ & $\sigma = 0.5$ & $\sigma = 1.0$ & $\sigma = 1.5$ & $\sigma = 2.0$  \\
\midrule
\shortstack{\parbox{2.5cm}{
    $\{0,1,2,3,5,8\}$ \\ 
    \text{vs }$\{0,1,2,3,5,9\}$}} &
\begin{tabular}{l}MMD-perm($X$,$Y$) \\ MMD-perm($V$,$W$)\\ xMMD \\ xssMMD\end{tabular} & 
\begin{tabular}{r}0.724 \\ 0.721 \\ 0.581 \\ \textbf{0.769}\end{tabular} & 
\begin{tabular}{r} 0.724 \\ 0.662 \\ 0.581 \\ \textbf{0.779}\end{tabular} & 
\begin{tabular}{r}\textbf{0.724} \\ 0.330 \\ 0.581 \\ 0.651\end{tabular} & 
\begin{tabular}{r}\textbf{0.724} \\ 0.179 \\ 0.581 \\ 0.572\end{tabular} & 
\begin{tabular}{r}\textbf{0.724} \\ 0.118 \\ 0.581 \\ 0.545\end{tabular} \\
\midrule
\shortstack{\parbox{2.5cm}{
    $\{0,1,2,3,9\}$ \\ 
    \text{vs }$\{0,1,2,3,6\}$}} &
\begin{tabular}{l}MMD-perm($X$,$Y$) \\ MMD-perm($V$,$W$)\\ xMMD \\ xssMMD\end{tabular} & 
\begin{tabular}{r}0.988 \\ 0.982 \\ 0.912 \\ \textbf{0.999}\end{tabular} & 
\begin{tabular}{r}0.988 \\ 0.963 \\ 0.912 \\ \textbf{0.999}\end{tabular} & 
\begin{tabular}{r}0.988 \\ 0.703 \\ 0.912 \\ \textbf{0.993}\end{tabular} & 
\begin{tabular}{r}\textbf{0.988} \\ 0.352 \\ 0.912 \\ 0.958\end{tabular} & 
\begin{tabular}{r}\textbf{0.988} \\ 0.214 \\ 0.912 \\ 0.937\end{tabular} \\
\midrule
\shortstack{\parbox{2.5cm}{
    $\{0,1,2,3,4,5,9\}$ \\ 
    \text{vs }$\{0,1,2,3,5,8\}$}} &
\begin{tabular}{l}MMD-perm($X$,$Y$) \\ MMD-perm($V$,$W$)\\ xMMD \\ xssMMD\end{tabular} & 
\begin{tabular}{r}0.550 \\ \textbf{0.836} \\ 0.419 \\ 0.580\end{tabular} & 
\begin{tabular}{r}0.550 \\ \textbf{0.698} \\ 0.419 \\ 0.611\end{tabular} &
\begin{tabular}{r}\textbf{0.550} \\ 0.277 \\ 0.419 \\ 0.450\end{tabular} & 
\begin{tabular}{r}\textbf{0.550} \\ 0.114 \\ 0.419 \\ 0.390\end{tabular} & 
\begin{tabular}{r}\textbf{0.550} \\ 0.080 \\ 0.419 \\ 0.369\end{tabular} \\
\bottomrule
\end{tabular}
\end{table}

\section{Proof of Main Results}\label{appendix: proof of main results}
\subsection{Proof of \Cref{Theorem: General Power Expression}}\label{section: proof of main theorem - general power expression}

Consider the centered oracle statistic with the population variance
\begin{align}\label{eq: oracle statistic with the population variance}
	\overline{T}_{\mathrm{oracle}} = \frac{\widehat{\mu}_{X,f} - \widehat{\mu}_{Y,f} - \mE[f(X)] + \mE[f(Y)]}{\sqrt{{\sigma}^2_{X,f} + {\sigma}^2_{Y,f} }}.
\end{align}
To establish the desired result, it suffices to prove that $\overline{T}_{\mathrm{oracle}}$ is asymptotically $N(0,1)$ as $n = n_1 \wedge n_2 \to \infty$, and that the empirical variance estimates are ratio-consistent. We will prove these statements in order. Throughout this proof, we write $\mE[f(X)]$ as $\mE[f(X) \given f]$ (and similarly for other quantities) to highlight that we condition on the randomness inherent in $f$.

\noindent \textbf{Step 1: Asymptotic Normality of $\overline{T}_{\mathrm{oracle}}$.}
Note that the numerator of $\overline{T}_{\mathrm{oracle}}$ is $\widehat{\mu}_{X,f} -\widehat{\mu}_{Y,f} - \mE[f(X)] + \mE[f(Y)] = \sum_{i=1}^{n_1+m_1} G_i - \sum_{i=1}^{n_2+m_2} H_i$ where
\begin{align*}
	G_i = \begin{cases}
		\frac{1}{n_1} \big\{f(X_i) - \mE[f(X_i) \given f] \big\}  - \frac{m_1}{n_1(n_1+m_1)} \big\{\mE[f(X_i) \given V_i, f] - \mE[f(X_i) \given f ] \big\} &  \text{if $1 \leq i \leq n_1$,} \\[.5em]
		\frac{1}{n_1+m_1} \big\{ \mE[f(X_i) \given V_i, f]  - \mE[f(X_i) \given f ] \big\} \quad & \text{if $n_1+1 \leq i \leq n_1+m_1$,}
	\end{cases}
\end{align*}
which are centered and conditionally independent given $f$. Similarly, we have 
\begin{align*}
	H_i = \begin{cases}
		\frac{1}{n_2} \big\{f(Y_i) - \mE[f(Y_i) \given f] \big\}  - \frac{m_2}{n_2(n_2+m_2)} \big\{\mE[f(Y_i) \given W_i, f] - \mE[f(Y_i) \given f] \big\} \quad & \text{if $1 \leq i \leq n_2$,} \\[.5em]
		\frac{1}{n_2+m_2} \big\{ \mE[f(Y_i) \given W_i, f]  - \mE[f(Y_i) \given f ] \big\} \quad  & \text{if $n_2+1 \leq i \leq n_2+m_2$}.
	\end{cases}
\end{align*}
Since the distributions of $G_i$ and $H_i$ may vary with the sample sizes depending on the choice of $f$, we use Lyapunov central limit theorem (CLT) to prove the desired statement. To apply Lyapunov CLT, we further define the variance as  
\begin{align*}
	s_{n+m}^2 & := \mV(\widehat{\mu}_{X,f} -\widehat{\mu}_{Y,f} \given f ) \\ 
	& = \sum_{i=1}^{n_1+m_1} \mV(G_i  \given f) + \sum_{i=1}^{n_2+m_2} \mV(H_i \given f) \\
	& = \sigma^2_{X,f}+\sigma^2_{Y,f}  \\
	& = \frac{1}{n_1} \sigma_{1,X,f}^2 + \frac{1}{n_1+m_1} \sigma_{2,X,f}^2 +  \frac{1}{n_2} \sigma_{1,Y,f}^2 + \frac{1}{n_2+m_2} \sigma_{2,Y,f}^2.
\end{align*}
For some $\delta > 0$, (conditional)~Lyapunov's condition~\citep[][Lemma S8]{lundborg2024projected} stated as
\begin{align}\label{condition: Lyapunov}
	\sum_{i=1}^{n_1+m_1} \mE\biggl[ \bigg|\frac{G_i}{s_{n+m}} \bigg|^{2+\delta} \,\bigg| \, f \biggr] + \sum_{i=1}^{n_2+m_2} \mE\biggl[ \bigg|\frac{H_i}{s_{n+m}} \bigg|^{2+\delta} \,\bigg| \, f \biggr] = o_P(1)
\end{align}
ensures the asymptotic normality of $\overline{T}_{\mathrm{oracle}}$ unconditional on $f$. 

Now, we show that the above condition (\ref{condition: Lyapunov}) is satisfied. Letting $r_1 = \frac{m_1}{n_1+m_1}$ and $r_2 = \frac{m_2}{n_2 + m_2}$, the proportions of the unlabeled data in the total dataset for $X$ and $Y$, respectively, we have an upper bound for the first term with $G_i$ in (\ref{condition: Lyapunov}) as
\begin{align*}
	\sum_{i=1}^{n_1+m_1} \mE\biggl[ \bigg|\frac{G_i}{s_{n+m}} \bigg|^{2+\delta} \,\bigg| \, f \biggr] ~ \lesssim ~ & \frac{1}{n_1^{1+\delta}} \mE\biggl[ \frac{|f(X) - \mE[f(X) \given f ] |^{2+\delta}}{s_{n+m}^{2+\delta}} \,\bigg|\, f \biggr] \\
	+ ~& \frac{r_1^{2+\delta}}{n_1^{1+\delta}} \mE\biggl[ \frac{| \mE[f(X) \given V, f] - \mE[f(X) \given f]|^{2+\delta}}{s_{n+m}^{2+\delta}} \,\bigg|\, f \biggr] \\
	+ ~& \frac{r_1}{(n_1 + m_1)^{1+\delta}} \mE\biggl[ \frac{| \mE[f(X) \given V, f] - \mE[f(X) \given f]|^{2+\delta}}{s_{n+m}^{2+\delta}} \,\bigg|\, f \biggr]. 
\end{align*}
Similarly, for the second term of (\ref{condition: Lyapunov}) with $H_i,$ we have an upper bound as
\begin{align*}
	\sum_{i=1}^{n_2+m_2} \mE\biggl[ \bigg|\frac{H_i}{s_{n+m}} \bigg|^{2+\delta} \,\bigg| \, f \biggr]  ~ \lesssim ~ & \frac{1}{n_2^{1+\delta}} \mE\biggl[ \frac{|f(Y) - \mE[f(Y) \given  f ] |^{2+\delta}}{s_{n+m}^{2+\delta}} \,\bigg|\, f \biggr] \\
	+ ~& \frac{r_2^{2+\delta}}{n_2^{1+\delta}} \mE\biggl[ \frac{| \mE[f(Y) \given W, f] - \mE[f(Y) \given f]|^{2+\delta}}{s_{n+m}^{2+\delta}} \,\bigg|\, f \biggr] \\
	+ ~& \frac{r_2}{(n_2 + m_2)^{1+\delta}} \mE\biggl[ \frac{| \mE[f(Y) \given W, f] - \mE[f(Y) \given f]|^{2+\delta}}{s_{n+m}^{2+\delta}} \,\bigg|\, f \biggr]. 
\end{align*}
By conditional Jensen's inequality, we have an upper bound
\begin{align*}
	\mE \bigl[\bigl|\mE[f(X) \given V, f] - \mE[f(X) \given f]\big|^{2+\delta} \bigr] ~\leq~ \mE \bigl[\bigl|f(X) - \mE[f(X) \given f]\big|^{2+\delta} \bigr], 
\end{align*}
and similarly
\begin{align*}
	\mE \bigl[\bigl|\mE[f(Y) \given W, f] - \mE[f(Y) \given f]\big|^{2+\delta} \bigr] \leq~ \mE \bigl[\bigl|f(Y) - \mE[f(Y) \given f]\big|^{2+\delta} \bigr]. 
\end{align*}
Next, we observe that the term $s^2_{n+m}$, representing the combined sample variance, satisfies the lower bound: 
\begin{align*}
	s^2_{n+m} & \, \geq \,  \sigma^2_{X,f} = \frac{1}{n_1}\sigma^2_{1,X,f}+\frac{1}{n_1+m_1}\sigma^2_{2,X,f} \\
	&\, \geq \,  \biggl(\frac{1}{n_1}+\frac{1}{n_1+m_1}\biggr)\times \bigl(\sigma^2_{1,X,f} \wedge \sigma^2_{2,X,f}\bigr).
\end{align*} 
Similarly, the combined variance also satisfies: 
\begin{align*}
	s^2_{n+m} \geq \biggl(\frac{1}{n_2}+\frac{1}{n_2+m_2}\biggr)\times \bigl(\sigma^2_{1,Y,f} \wedge \sigma^2_{2,Y,f}\bigr).
\end{align*}
These bounds, combined with the earlier results, ensure that the terms in (\ref{condition: Lyapunov}) decay appropriately as 
\begin{align*}
	\sum_{i=1}^{n_1+m_1} \mE\biggl[ & \bigg|\frac{G_i}{s_{n+m}} \bigg|^{2+\delta}  \,\bigg| \, f \biggr]   ~\lesssim~  \frac{1}{n_1^{\delta/2}}  \frac{\mE \bigl[\bigl\{f(X) - \mE[f(X) \given f]\big\}^{2+\delta} \bigr]}{\sigma_{1,X,f}^{2+\delta} \wedge \sigma_{2,X,f}^{2+\delta}} .
\end{align*}

A similar upper bound can be obtained for 
\begin{align*}
	\sum_{i=1}^{n_2+m_2} \mE\biggl[ & \bigg|\frac{H_i}{s_{n+m}} \bigg|^{2+\delta}  \,\bigg| \, f \biggr]   ~\lesssim~  \frac{1}{n_2^{\delta/2}}  \frac{\mE \bigl[\bigl\{f(Y) - \mE[f(Y) \given f]\big\}^{2+\delta} \bigr]}{\sigma_{1,Y,f}^{2+\delta} \wedge \sigma_{2,Y,f}^{2+\delta}}.
\end{align*}
Consequently, under the moment condition (\ref{Assumption: moment condition}), Lyapunov's condition is fulfilled, which implies the asymptotic normality of $\overline{T}_{\mathrm{oracle}}$.

\noindent \textbf{Step 2: Asymptotic Normality with Sample Variance}
In this step, we aim to show that the ratio of the sample variance to the population variance converges to $1$ in probability. This result in conjunction with Slutsky's theorem and continuous mapping theorem in turn confirms that $\overline{T}_{\mathrm{oracle}}$ and $T_{\mathrm{oracle}}$ share the same asymptotic distribution.

We first formally define the estimated variance of $\widehat{\sigma}^2_{X,f} = n_1^{-1}\widehat{\sigma}_{1,X,f}^2 + (n_1+m_1)^{-1} \widehat{\sigma}_{2,X,f}^2$ where each term is defined as 
\begin{equation}
    \begin{aligned}\label{eq: semi-supervised variance estimate}
	& \widehat{\sigma}_{1,X,f}^2 = \frac{1}{n_1} \sum_{i=1}^{n_1} \big\{  f(X_i) - \mE[f(X_i) \given V_i] \big\}^2 \quad \text{and} \\
	&  \widehat{\sigma}_{2,X,f}^2 = \frac{1}{n_1+m_1} \sum_{i=1}^{n_1+m_1} \biggl\{ \mE[f(X_i) \given V_i] - \frac{1}{n_1+m_1} \sum_{j=1}^{n_1+m_1} \mE[f(X_j) \given V_j]  \biggr\}^2.
    \end{aligned} 
\end{equation}

We similarly define the variance of $\widehat{\mu}_{Y,f}$ as $\widehat{\sigma}^2_{Y,f} = n_2^{-1} \widehat{\sigma}_{1,Y,f}^2 + (n_2+m_2)^{-1} \widehat{\sigma}_{2,Y,f}^2.$

Based on these definitions, we prove the following convergence:
\begin{align*}
	\frac{\widehat{\sigma}_{X,f}^2 + \widehat{\sigma}_{Y,f}^2}{\sigma_{X,f}^2 + \sigma_{Y,f}^2} - 1 = o_P(1) \quad \Longleftrightarrow \quad \frac{\widehat{\sigma}_{X,f}^2  - \sigma_{X,f}^2 + \widehat{\sigma}_{Y,f}^2 - \sigma_{Y,f}^2}{\sigma_{X,f}^2 + \sigma_{Y,f}^2}  = o_P(1).
\end{align*}
By the triangle inequality, it suffices to show that 
\begin{align}\label{eq: convergence to prove step2 in thm 1}
	\frac{\widehat{\sigma}_{X,f}^2  - \sigma_{X,f}^2 }{\sigma_{X,f}^2} = o_P(1) \quad \text{and} \quad \frac{\widehat{\sigma}_{Y,f}^2  - \sigma_{Y,f}^2 }{\sigma_{Y,f}^2} = o_P(1).
\end{align}
Without loss of generality, we focus on the first convergence result. Using the lower bound for $\sigma^2_{X,f}$:
\begin{align}\label{lower bound for the population variance}
	\sigma^2_{X,f} = \frac{1}{n_1} \sigma^2_{1,X,f} + \frac{1}{n_1+m_1} \sigma^2_{2,X,f} \geq  \biggl(\frac{1}{n_1}+\frac{1}{n_1+m_1}\biggr)\times \bigl(\sigma^2_{1,X,f} \wedge \sigma^2_{2,X,f}\bigr),
\end{align}
as well as the definition of $\widehat{\sigma}_{X,f}^2 = n_1^{-1} \widehat{\sigma}_{1,X,f}^2 + (n_1+m_1)^{-1} \widehat{\sigma}_{2,X,f}^2$, we have
\begin{align*}
	\frac{\big|\widehat{\sigma}_{X,f}^2  - \sigma_{X,f}^2\big|}{\sigma_{X,f}^2}  \leq \underbrace{\frac{\big|\widehat{\sigma}_{1,X,f}^2 - \sigma_{1,X,f}^2\big|}{\sigma^2_{1,X,f} \wedge \sigma^2_{2,X,f}}}_{(\mathrm{I})} + \underbrace{\frac{\big|\widehat{\sigma}_{2,X,f}^2 - \sigma_{2,X,f}^2\big|}{\sigma^2_{1,X,f} \wedge \sigma^2_{2,X,f}}}_{(\mathrm{II})}.
\end{align*}
A conditional version of the weak law of large numbers~\citep[][Lemma S9]{lundborg2024projected} under \Cref{Assumption: moment condition} guarantees that the first term $(\mathrm{I})$ is $o_P(1)$. For the second term $(\mathrm{II})$, we rewrite $\widehat{\sigma}_{2,X,f}^2$ as
\begin{align*}
	 \underbrace{\frac{1}{n_1+m_1} \sum_{i=1}^{n_1+m_1}\big\{   \mE[f(X_i) \given V_i, f] - \mE[f(X_i) \given f] \big\}^2}_{(\mathrm{II})_1} - \biggl( \underbrace{\frac{1}{n_1+m_1} \sum_{i=1}^{n_1+m_1} \big\{\mE[f(X_i) \given V_i, f] -  \mE[f(X_i) \given f] \big\}}_{(\mathrm{II})_2^2}  \biggr)^2.
\end{align*}
By the weak law of large numbers, again, under \Cref{Assumption: moment condition}, it can be seen that 
\begin{align*}
	\frac{|(\mathrm{II})_1 - \sigma_{2,X,f}^2|}{\sigma^2_{1,X,f} \wedge \sigma^2_{2,X,f}} = o_P(1).
\end{align*}  
For $(\mathrm{II})_2$, we first note by Jensen's inequality that
\begin{align*}
	\mE \biggl[ \bigg|\frac{\big\{\mE[f(X_i) \given V_i, f] -  \mE[f(X_i) \given f] \big\}}{\sigma_{1,X,f} \wedge \sigma_{2,X,f}}\bigg|^{1+\delta} \,\bigg|\, f \biggr] &~\leq \sqrt{\mE \biggl[ \bigg|\frac{\big\{\mE[f(X_i) \given V_i, f] -  \mE[f(X_i) \given f] \big\}}{\sigma_{1,X,f} \wedge \sigma_{2,X,f}}\bigg|^{2+2\delta} \,\bigg|\, f \biggr]} \\
	&~=  o_P(n^{\delta}),
\end{align*}
where the last approximation holds under \Cref{Assumption: moment condition}. Hence by \citet[][Lemma S9]{lundborg2024projected}, it holds that $(\mathrm{II})_2^2 = o_P(1)$, which in turn implies $(\mathrm{II}) = o_P(1)$ as required. Combining the results for $(\mathrm{I})$ and $(\mathrm{II})$, we conclude 
\begin{align*}
	\frac{\widehat{\sigma}_{X,f}^2  - \sigma_{X,f}^2 }{\sigma_{X,f}^2} = o_P(1).
\end{align*}
A similar argument applies to $\widehat{\sigma}_{Y,f}^2$, which verifies the ratio consistency of the sample variance.

\subsection{Proof of \Cref{Corollary: power expression for cross-fit test}}
As mentioned in the main text, we assume that the sample sizes $n_1,n_2,m_1,m_2$ are even for simplicity. The other cases can be proven similarly by minor modifications. Similarly to \Cref{Theorem: General Power Expression}, we consider the centered cross-fitted statistic with the empirical variance 
\begin{align*}
	\overline{T}_{\mathrm{cross}} = \frac{\widehat{\mu}^{\dagger}_{X,f} - \widehat{\mu}^{\dagger}_{Y,f} - \mE[f(X) \given f] + \mE[f(Y) \given f]}{\sqrt{{\widehat{\sigma}}^{\dagger 2}_{X,f} + {\widehat{\sigma}}^{\dagger 2}_{Y,f} }}.
\end{align*} 
Above, we define $\widehat{\mu}^{\dagger}_{X,f}$ and $\widehat{\mu}^{\dagger}_{Y,f}$ as the counterparts of $\widehat{\mu}_{X,f}$ and $\widehat{\mu}_{Y,f}$, replacing $\mE[f(X_i) \given V_i, f]$ and $\mE[f(Y_i) \given W_i, f]$ with their estimates using $\widehat{\mE}[f(X_i) \given V_i, f]$ and $\widehat{\mE}[f(Y_i) \given W_i, f]$, respectively. Similarly, we define $\widehat{\sigma}_{1,X,f}^{\dagger 2}$ and $\widehat{\sigma}_{2,X,f}^{\dagger 2}$ as the counterparts of $\widehat{\sigma}_{1,X,f}^2$ and $\widehat{\sigma}_{2,X,f}^2$ with the estimated conditional expectations. 

Since we have already established that $\overline{T}_{\mathrm{oracle}}$ is asymptotically $N(0,1)$ as $n = n_1 \wedge n_2 \to \infty,$ it suffices to prove 
\begin{align*}
    \overline{T}_{\mathrm{cross}}-\overline{T}_{\mathrm{oracle}}=o_P(1).
\end{align*}
Once this convergence is verified, $\overline{T}_{\mathrm{cross}}$ will also converge to $N(0,1)$ by Slutsky's theorem. To this end, denote $\overline{T}_{\mathrm{cross}} = \frac{N_{\mathrm{cross}}}{D_{\mathrm{cross}}}$ and $\overline{T}_{\mathrm{oracle}} = \frac{N_{\mathrm{oracle}}}{D_\mathrm{oracle}}$. Then 
\begin{align*}
	\overline{T}_{\mathrm{cross}} - \overline{T}_{\mathrm{oracle}} & = \frac{N_{\mathrm{cross}}}{D_{\mathrm{cross}}} - \frac{N_{\mathrm{cross}}}{D_{\mathrm{oracle}}} + \frac{N_{\mathrm{cross}}}{D_{\mathrm{oracle}}} - \frac{N_{\mathrm{oracle}}}{D_{\mathrm{oracle}}} \\
	& = \frac{N_{\mathrm{cross}} - N_{\mathrm{oracle}}}{D_{\mathrm{oracle}}} \biggl( \frac{D_{\mathrm{oracle}}}{D_{\mathrm{cross}}} - 1 \biggr) + \underbrace{\frac{N_{\mathrm{oracle}}}{D_{\mathrm{oracle}}}}_{=O_P(1)} \biggl( \frac{D_{\mathrm{oracle}}}{D_{\mathrm{cross}}} - 1 \biggr) + \frac{N_{\mathrm{cross}} - N_{\mathrm{oracle}}}{D_{\mathrm{oracle}}}.
\end{align*}
Hence it suffices to show the following two claims hold:
\begin{align*}
	& \mathrm{(i)} \ \, \frac{N_{\mathrm{cross}} - N_{\mathrm{oracle}}}{D_{\mathrm{oracle}}} = o_P(1) \quad \text{and} \quad \mathrm{(ii)}~\,  \frac{D_{\mathrm{oracle}}}{D_{\mathrm{cross}}} - 1 = o_P(1),
\end{align*}
which are proved below.

\textit{Proof of claim (i).} 
We define $D_{\mathrm{oracle}}^{\star}$ as the population standard deviation. Note that $(D_{\mathrm{oracle}})^2 = \widehat{\sigma}_{X,f}^2 + \widehat{\sigma}_{Y,f}^2$ and let $(D_{\mathrm{oracle}}^{\star})^2 = \sigma_{X,f}^2 + \sigma_{Y,f}^2$ where $\sigma_{X,f}^2 = \frac{1}{n_1} \sigma_{1,X}^2 + \frac{1}{n_1 + m_1} \sigma_{2,X}^2$ and $\sigma_{Y,f}^2 = \frac{1}{n_2} \sigma_{1,Y}^2 + \frac{1}{n_2 + m_2} \sigma_{2,Y}^2$. From our previous result in \Cref{Theorem: General Power Expression}, we have
\begin{align*}
	\biggl(\frac{D_{\mathrm{oracle}}}{D_{\mathrm{oracle}}^\star}\biggr)^2 -1 = \frac{\widehat{\sigma}_{X,f}^2 + \widehat{\sigma}_{Y,f}^2}{\sigma_{X,f}^2 + \sigma_{Y,f}^2} -1 = o_P(1).
\end{align*}
By the continuous mapping theorem, we have	
\begin{align*}
	\frac{D_{\mathrm{oracle}}^\star}{D_{\mathrm{oracle}}} - 1 = o_P(1)
\end{align*}
and consequantly,
\begin{align*}
	\frac{N_{\mathrm{cross}} - N_{\mathrm{oracle}}}{D_{\mathrm{oracle}}}  = \frac{N_{\mathrm{cross}} - N_{\mathrm{oracle}}}{D_{\mathrm{oracle}}^\star} \{1 + o_P(1)\}.
\end{align*}
On the other hand, 
\begin{align*}
	\bigg|  \frac{N_{\mathrm{cross}} - N_{\mathrm{oracle}}}{D_{\mathrm{oracle}}^\star}  \bigg| \leq \frac{|R_X|}{\sigma_{X,f}}  + \frac{|R_Y|}{\sigma_{Y,f}},
\end{align*}
where
\begin{align*}
	& R_X = \frac{1}{n_1} \sum_{i=1}^{n_1} \{\widehat{\mE}[f(X_i) \given V_i, f] - \mE[f(X_i) \given V_i, f]\} \\
	& \hskip 5em + \frac{1}{n_1+m_1} \sum_{i=1}^{n_1+m_1} \{\mE[f(X_i) \given V_i, f] - \widehat{\mE}[f(X_i) \given V_i, f]\},\quad\text{and} \\
	&  R_Y = \frac{1}{n_2} \sum_{i=1}^{n_2} \{\widehat{\mE}[f(Y_i) \given W_i, f] - \mE[f(Y_i) \given W_i, f]\} \\
	& \hskip 5em + \frac{1}{n_2+m_2} \sum_{i=1}^{n_2+m_2} \{\mE[f(Y_i) \given W_i, f] - \widehat{\mE}[f(Y_i) \given W_i, f]\}.
\end{align*}
Without loss of generality, we focus on the first term $|R_X|/\sigma_{X,f}$ and the other term $|R_Y|/\sigma_{Y,f}$ can be handled similarly. 

Observe that $R_X$ takes the form $\frac{1}{n} \sum_{i=1}^n g\left(V_i\right)-\frac{1}{n+m} \sum_{j=1}^{n+m} g\left(V_j\right)$ for some function $g$. This form is invariant under location shifts in the function $g$. Specifically, for any constant $c \in \mathbb{R}$, we have $\frac{1}{n} \sum_{i=1}^n\left\{g\left(V_i\right)+c\right\}-\frac{1}{n+m} \sum_{j=1}^{n+m}\left\{g\left(V_j\right)+c\right\}=\frac{1}{n} \sum_{i=1}^n g\left(V_i\right)-\frac{1}{n+m} \sum_{j=1}^{n+m} g\left(V_j\right)$. Then, we choose $c=-\mathbb{E}[g(V)]$ where $g\left(V_j\right)+c$ has expectation of $0$. Therefore, without loss of generality, we may assume that $R_X$ has zero mean, i.e., $\mE \bigl[ \mE[f(X_i) \given V_i, f] - \widehat{\mE}[f(X_i) \given V_i,f] \bigr] = 0$ for all $i$.
Then by Chebyshev's inequality, it can be seen that
\begin{align*}
	R_X^2 = O_P \Bigl(n_1^{-1} \mE\bigl[ \{ \widehat{\mE}[f(X) \given V, f] - \mE[f(X) \given V, f] \}^2 \given f \bigr] \Bigr).
\end{align*}

Combining this with the lower bound \eqref{lower bound for the population variance} for $\sigma^2_{X,f}$, 
\begin{align*}
	\frac{R_X^2}{\sigma_{X,f}^2} \leq \biggl(\frac{1}{n_1}+\frac{1}{n_1+m_1}\biggr)^{-1}\frac{R_X^2}{\sigma^2_{1,X,f} \wedge \sigma^2_{2,X,f}} = O_P\biggl( \frac{ \mE\bigl[ \{\widehat{\mE}[f(X) \given V, f] - \mE[f(X) \given V, f]  \}^2 \given f \bigr]}{\sigma^2_{1,X,f} \wedge \sigma^2_{2,X,f}} \biggr).
\end{align*}

Hence, under the condition \eqref{Eq: cross-fit condition}, we have
\begin{align*}
    \frac{R_X^2}{\sigma_{X,f}^2}=o_P(1).
\end{align*} A similar argument applies to $R_Y^2/\sigma_{Y,f}^2$, which proves that claim (i) holds.

\textit{Proof of claim (ii).} 
By the continuous mapping theorem, it is sufficient to prove that 
\begin{align*}
	\biggl(\frac{D_{\mathrm{cross}}}{D_{\mathrm{oracle}}}\biggr)^2 - 1  = o_P(1),
\end{align*}
which, in turn, is implied by
\begin{align*}
	& \frac{1}{n_1}\frac{|\widehat{\sigma}_{1,X,f}^{\dagger 2} - \widehat{\sigma}_{1,X,f}^2|}{\widehat{\sigma}_{X,f}^2} = o_P(1), \ \frac{1}{n_1+m_1}\frac{|\widehat{\sigma}_{2,X,f}^{\dagger 2} - \widehat{\sigma}_{2,X,f}^2|}{\widehat{\sigma}_{X,f}^2} = o_P(1),\quad\text{and} \\
	& \frac{1}{n_2}\frac{|\widehat{\sigma}_{1,Y,f}^{\dagger 2} - \widehat{\sigma}_{1,Y,f}^2|}{\widehat{\sigma}_{Y,f}^2} = o_P(1), \  \frac{1}{n_2+m_2}\frac{|\widehat{\sigma}_{2,Y,f}^{\dagger 2} - \widehat{\sigma}_{2,Y,f}^2|}{\widehat{\sigma}_{Y,f}^2} = o_P(1).
\end{align*}

Without loss of generality, we focus on $\widehat{\sigma}_{1,X,f}^{\dagger 2}$ and $\widehat{\sigma}_{2,X,f}^{\dagger 2}.$ Using the Cauchy--Schwarz inequality, we observe that
\begin{align*}
	\frac{1}{n_1}\frac{|\widehat{\sigma}_{1,X,f}^{\dagger 2} - \widehat{\sigma}_{1,X,f}^2|}{\widehat{\sigma}_{X,f}^2} ~\leq~ & \frac{1}{n_1}\frac{\frac{1}{n_1} \sum_{i=1}^{n_1} \big\{ \mE[f(X_i) \given V_i, f] - \widehat{\mE}[f(X_i) \given V_i, f] \big\}^2}{\widehat{\sigma}_{X,f}^2} \\
	& + \frac{2}{n_1}\sqrt{\frac{\frac{1}{n_1} \sum_{i=1}^{n_1} \big\{ \mE[f(X_i) \given V_i, f] - \widehat{\mE}[f(X_i) \given V_i, f] \big\}^2}{\widehat{\sigma}_{X,f}^2}}.
\end{align*}
This becomes $o_P(1)$ when
\begin{align*}
	\frac{1}{n_1}\frac{ \mE\bigl[ \{\mE[f(X) \given V, f] - \widehat{\mE}[f(X) \given V, f]\}^2 \given f \bigr]}{\sigma_{X,f}^2} &~\leq \biggl(\frac{1}{n_1}+\frac{1}{n_1+m_1}\biggr)^{-1}\frac{1}{n_1}\frac{ \mE\bigl[ \{\mE[f(X) \given V, f] - \widehat{\mE}[f(X) \given V, f]\}^2 \given f \bigr]}{\sigma_{1,X,f}^{2} \wedge \sigma_{2,X,f}^{2}} \\&~= o_P(1),
\end{align*} which is satisfied under the condition \eqref{Eq: cross-fit condition}. 
On the other hand, we again use the Cauchy--Schwarz inequality to observe that
\begin{align*}
	\frac{1}{n_1+m_1}\frac{|\widehat{\sigma}_{2,X}^{\dagger 2} - \widehat{\sigma}_{2,X}^2|}{\widehat{\sigma}_{X,f}^2} ~\leq~& \frac{1}{n_1+m_1}\frac{\frac{1}{n_1+m_1} \sum_{i=1}^{n_1+m_1}A_i^2}{\widehat{\sigma}_{X,f}^2} \\
	& + 2 \sqrt{\frac{1}{n_1+m_1}\frac{\frac{1}{n_1+m_1} \sum_{i=1}^{n_1+m_1} A_i^2}{\widehat{\sigma}_{X,f}^2}} \underbrace{\sqrt{\frac{\frac{1}{n_1+m_1} \sum_{i=1}^{n_1+m_1} B_i^2}{\widehat{\sigma}_{2,X}^2}}}_{=1}
\end{align*}
where
\begin{align*}
	& A_i =  \{\mE[f(X_i) \given V_i, f] - \widehat{\mE}[f(X_i) \given V_i, f]\} - \frac{1}{n_1+m_1} \sum_{j=1}^{n_1+m_1} \{ \mE[f(X_j) \given V_j, f] - \widehat{\mE}[f(X_j) \given V_j, f]\}, \\
	& B_i = \mE[f(X_i) \given V_i, f] - \frac{1}{n_1 + m_1} \sum_{j=1}^{n_1+m_1} \mE[f(X_j) \given V_j, f]. 
\end{align*}
The above upper bound becomes $o_P(1)$ when
\begin{align*}
	& \frac{1}{n_1+m_1}\frac{ \mE\bigl[ \{\mE[f(X) \given V, f] - \widehat{\mE}[f(X) \given V, f]\}^2 \given f \bigr]}{\sigma_{X,f}^2} \\
	& \leq \biggl(\frac{1}{n_1}+\frac{1}{n_1+m_1}\biggr)^{-1}\frac{1}{n_1+m_1}\frac{ \mE\bigl[ \{\mE[f(X) \given V, f] - \widehat{\mE}[f(X) \given V, f]\}^2 \given f \bigr]}{\sigma_{1,X,f}^{2} \wedge \sigma_{2,X,f}^{2}} = o_P(1),
\end{align*} which is also satisfied under the condition \eqref{Eq: cross-fit condition}.
A similar argument applies to $\widehat{\sigma}_{1,Y,f}^{\dagger 2}$ and $\widehat{\sigma}_{2,Y,f}^{\dagger 2}$, which proves that claim (ii) holds.
Therefore, under the condition \eqref{Eq: cross-fit condition}, we conclude that $\overline{T}_{\mathrm{cross}}$ is asymptotically $N(0,1)$ as $n = n_1 \wedge n_2 \to \infty.$

\subsection{Proof of \Cref{Theorem: xssMMD}}\label{appendix: proof of the main theorem}

Before presenting the formal proof of \Cref{Theorem: xssMMD}, we first provide the explicit mathematical formulation of the cross-fitted test statistic $\xssMMD$ as defined in \eqref{Eq: xssMMD}.

Recall that to compute the cross-fitted estimators, we partition the labeled dataset $\mathcal{L}_{XV}$ into two disjoint subsets: 
$\mathcal{L}_{XV,a} = \{(X_i, V_i) : i \in \mathcal{I}_a\}$ and 
$\mathcal{L}_{XV,b} = \{(X_i, V_i) : i \in \mathcal{I}_b\}$, 
where $\mathcal{I}_a$ and $\mathcal{I}_b$ represent the odd and even indices of $\{1, \dots, n_1\}$, respectively. Similarly, we partition the unlabeled dataset $\mathcal{U}_V$ into $\mathcal{U}_{V,a} = \{V_i : i \in \mathcal{J}_a\}$ and $\mathcal{U}_{V,b} = \{V_i : i \in \mathcal{J}_b\}$, where $\mathcal{J}_a$ and $\mathcal{J}_b$ represent the odd and even indices of $\{n_1+1, \dots, n_1+m_1\}$, respectively. Let $\widehat{\mE}[\fhat(X_i) \given V_i]$ denote the conditional expectation estimator trained on $\mathcal{L}_{XV,a}$ if $i \in \mathcal{I}_b \cup \mathcal{J}_b$ (even indices), and on $\mathcal{L}_{XV,b}$ if $i \in \mathcal{I}_a \cup \mathcal{J}_a$ (odd indices). We apply an analogous partition to $\mathcal{L}_{YW}$ and $\mathcal{U}_W$ and define $\widehat{\mE}[\fhat(Y_i) \given W_i]$.

Based on the obtained estimator, the cross-fitted semi-supervised mean estimator for $X$, denoted as $\hat{\mu}_{X,\fhat}^{\dagger}$, is defined as
\begin{align*}
    \hat{\mu}_{X,\fhat}^{\dagger} = 
    &\frac{1}{n_1}\sum_{i = 1}^{n_1} \left\{ \fhat(X_i) - \widehat{\mE}[\fhat(X_i) \given V_i] \right\} + \frac{1}{n_1 + m_1}  \sum_{i = 1}^{n_1+m_1} \widehat{\mE}[\fhat(X_i) \given V_i].
\end{align*}
The estimator $\hat{\mu}_{Y,\fhat}^{\dagger}$ is defined analogously using $\mathcal{L}_{YW}$ and $\mathcal{U}_W$.

Similarly, the cross-fitted variance estimator $\widehat{\sigma}_{X,\fhat}^{\dagger 2}$ is defined as $\widehat{\sigma}^{\dagger 2}_{X,\fhat} = n_1^{-1}\widehat{\sigma}_{1,X,\fhat}^{\dagger 2} + (n_1+m_1)^{-1} \widehat{\sigma}_{2,X,\fhat}^{\dagger 2}$, where each term represents the empirical sample variance of the cross-fitted components:
\begin{equation}
    \begin{aligned}\label{eq: cross-fitted variance estimate}
    & \widehat{\sigma}_{1,X,\fhat}^{\dagger 2} = \frac{1}{n_1} \sum_{i=1}^{n_1} \left\{ \fhat(X_i) - \widehat{\mE}[\fhat(X_i) \given V_i] \right\}^2 - \left( \frac{1}{n_1} \sum_{i=1}^{n_1} \left\{ \fhat(X_i) - \widehat{\mE}[\fhat(X_i) \given V_i] \right\} \right)^2, \\
    & \widehat{\sigma}_{2,X,\fhat}^{\dagger 2} = \frac{1}{n_1+m_1} \sum_{i=1}^{n_1+m_1} \left\{ \widehat{\mE}[\fhat(X_i) \given V_i] \right\}^2 - \left( \frac{1}{n_1+m_1} \sum_{i=1}^{n_1+m_1} \widehat{\mE}[\fhat(X_i) \given V_i] \right)^2.
    \end{aligned} 
\end{equation}
The variance $\widehat{\sigma}_{Y,\fhat}^{\dagger 2}$ is computed symmetrically using $\mathcal{L}_{YW}$ and $\mathcal{U}_W$. 

The final cross-fitted semi-supervised MMD test statistic is then given by:
\begin{align*}
    \xssMMD = \frac{\hat{\mu}_{X,\fhat}^{\dagger} - \hat{\mu}_{Y,\fhat}^{\dagger}}{\sqrt{\widehat{\sigma}_{X,\fhat}^{\dagger 2} + \widehat{\sigma}_{Y,\fhat}^{\dagger 2}}}.
\end{align*}

With the precise formulation of the test statistic established, we now proceed to the main proof of \Cref{Theorem: xssMMD}. We first note that the tests are defined using the \((1-\alpha)\)-quantile of the standard normal distribution. Hence, it suffices to demonstrate that under \Cref{Assumption: xssMMD under H0} and \Cref{Assumption: consistency of conditional expectation}, the asymptotic normality of $\xMMD$ and $\xssMMD$ holds under the null, and it further holds under the alternative when \Cref{Assumption: xssMMD under H1} is satisfied. 

The asymptotic normality of \(\xMMD\) under \Cref{Assumption: xssMMD under H0} has already been established by \citet[Theorem 5,][]{shekhar2022permutation}. Thus, it remains to show that this result also holds under the alternative and that $\xssMMD$ asymptotically follows $N(0,1)$ as $n \to \infty$ under the both null and alternative given the considered conditions. 

To this end, we first prove that $\xMMD$ and the oracle statistic, $\xssMMD_{\circ} \coloneqq \overline{T}_{\mathrm{oracle}}$ are asymptotically $N(0,1)$. Then, we prove that $\xssMMD - \xssMMD_{\circ} = o_P(1)$ to conclude the asymptotic normality of $\xssMMD$. Finally, we compare their asymptotic power under the alternative to complete the proof of \Cref{Theorem: xssMMD}.

\noindent \textbf{Step 1: Asymptotic Normality of $\xssMMD_{\circ}$.} 
We proceed in similar steps as we have done in the proof of \Cref{Theorem: General Power Expression}. We first show the asymptotic normality when using the true variance, then show the same result is valid with the sample variance.

For the oracle version, we prove the asymptotic normality for general distributions $P_X$ and $P_Y$, which ensures that the result holds under both the null and alternative hypotheses.

Recall from the proof of \Cref{Theorem: General Power Expression} that
\begin{align*}
	\xssMMD_{\circ} = \frac{\widehat{\mu}_{X,\fhat} - \widehat{\mu}_{Y,\fhat}-\mE[\fhat(X)|\fhat]+\mE[\fhat(Y)|\fhat]}{s_{n+m}}=\frac{\sum_{i=1}^{n_1+m_1} G_i- \sum_{i=1}^{n_2+m_2} H_i}{\sqrt{{\sigma}^2_{X,\fhat} + {\sigma}^2_{Y,\fhat} }},
\end{align*}
where
\begin{align*}
	G_i = \begin{cases}
		\frac{1}{n_1} \big\{\fhat(X_i) - \mE[\fhat(X_i) \given \fhat] \big\}  - \frac{m_1}{n_1(n_1+m_1)} \big\{\mE[\fhat(X_i) \given V_i, \fhat] - \mE[\fhat(X_i) \given \fhat ] \big\} &  \text{if $1 \leq i \leq n_1$,} \\[.5em]
		\frac{1}{n_1+m_1} \big\{ \mE[\fhat(X_i) \given V_i, \fhat]  - \mE[\fhat(X_i) \given \fhat ] \big\} \quad & \text{if $n_1+1 \leq i \leq n_1+m_1$,}
	\end{cases}
\end{align*} and
\begin{align*}
	H_i = \begin{cases}
		\frac{1}{n_2} \big\{\fhat(Y_i) - \mE[\fhat(Y_i) \given \fhat ] \big\}  - \frac{m_2}{n_2(n_2+m_2)} \big\{\mE[\fhat(Y_i) \given W_i, \fhat] - \mE[\fhat(Y_i) \given \fhat] \big\} \quad & \text{if $1 \leq i \leq n_2$,} \\[.5em]
		\frac{1}{n_2+m_2} \big\{ \mE[\fhat(Y_i) \given W_i, \fhat]  - \mE[\fhat(Y_i) \given \fhat ] \big\} \quad  &  \text{if $n_2+1 \leq i \leq n_2+m_2$}.
	\end{cases}
\end{align*}
The denominator $s_{n+m}$ is recalled as 
\begin{align*}
	s_{n+m}^2 & = \mV(\widehat{\mu}_{X,\fhat} -\widehat{\mu}_{Y,\fhat} \given \fhat ) = \sum_{i=1}^{n_1+m_1} \mV(G_i  \given \fhat) + \sum_{i=1}^{n_2+m_2} \mV(H_i \given \fhat) \\
	& = \underbrace{\frac{1}{n_1} \sigma_{1,X,\fhat}^2 + \frac{1}{n_1+m_1} \sigma_{2,X,\fhat}^2}_{\sigma^2_{X,\fhat}} +  \underbrace{\frac{1}{n_2} \sigma_{1,Y,\fhat}^2 + \frac{1}{n_2+m_2} \sigma_{2,Y,\fhat}^2}_{\sigma^2_{Y,\fhat}}.
\end{align*}
To ensure that $\xssMMD_{\circ}$ is asymptotically $N(0,1)$ distributed, it suffices to show that Lyapunov's condition \eqref{condition: Lyapunov} is satisfied. For simplicity, we take $\delta = 2$ and show the following convergence holds under the given conditions:
\begin{align}\label{condition: Lyapunov when delta=2}
	\sum_{i=1}^{n_1+m_1} \mE\biggl[ \bigg|\frac{G_i}{s_{n+m}} \bigg|^{4} \,\bigg| \, \fhat \biggr] + \sum_{i=1}^{n_2+m_2} \mE\biggl[ \bigg|\frac{H_i}{s_{n+m}} \bigg|^{4} \,\bigg| \, \fhat \biggr] =o_P(1). 
\end{align}

By symmetry, we focus on the first term involving $G_i$ values. Letting $r_1 = \frac{m_1}{n_1+m_1}$, we obtain an upper bound for the first term as
\begin{align}\label{eq: upper bound for Lyapunov condition}
	\sum_{i=1}^{n_1+m_1} \mE\biggl[ \bigg|\frac{G_i}{s_{n+m}} \bigg|^{4} \,\bigg| \, \fhat \biggr] ~ \lesssim ~ & \frac{1}{n_1^3} \mE\biggl[ \frac{|\fhat(X) - \mE[\fhat(X) \given \fhat ] |^4}{s_{n+m}^4} \,|\, \fhat \biggr] \nonumber \\ 
    + ~& \frac{r_1^4}{n_1^3} \mE\biggl[ \frac{| \mE[\fhat(X) \given V, \fhat] - \mE[\fhat(X) \given \fhat]|^4}{s_{n+m}^4} \,|\, \fhat \biggr] \nonumber \\
	+ ~& \frac{r_1}{(n_1 + m_1)^3} \mE\biggl[ \frac{| \mE[\fhat(X) \given V, \fhat] - \mE[\fhat(X) \given \fhat]|^4}{s_{n+m}^4} \,|\, \fhat \biggr]\nonumber\\
    ~ \lesssim ~ &\biggl( \frac{1+r_1^4}{n_1^2}+\frac{r_1 n_1}{(n_1+m_1)^3}\biggr)\mE\biggl[ \frac{|\fhat(X) - \mE[\fhat(X) \given \fhat ] |^4}{n_1 s_{n+m}^4} \,|\, \fhat \biggr],
\end{align} 
where we used conditional Jensen's inequality for the last inequality.

On the other hand, by the law of total variance, $s_{n+m}^2$ term is lower bounded as
\begin{align*}
    s_{n+m}^2\geq \sigma^2_{X,\fhat}&=\frac{1}{n_1} \sigma_{1,X,\fhat}^2 + \frac{1}{n_1+m_1} \sigma_{2,X,\fhat}^2 \\
    &\geq \frac{1}{n_1+m_1} \sigma_{1,X,\fhat}^2 + \frac{1}{n_1+m_1} \sigma_{2,X,\fhat}^2 = \frac{1}{n_1+m_1}\mV(\fhat(X)\given \fhat)=\frac{1}{n_1+m_1}\sigma^2_{X,\fhat}.
\end{align*}
Hence in order to show that the first term in \eqref{condition: Lyapunov when delta=2} is $o_P(1)$ under $n_1 \asymp m_1$, it suffices to show that the two claims hold:
\begin{align*}
	 \mathrm{(i)} ~\, \frac{\mE [\{\fhat(X) - \mE[\fhat(X) \given \fhat]\}^4]}{n_1 \{\mE[\sigma_{X,\fhat}^2]\}^2}  = o_P(1) \quad \text{and} \quad  \mathrm{(ii)} ~  \, \frac{\mE[\sigma_{X,\fhat}^2]}{\sigma_{X,\fhat}^2} = O_P(1).
\end{align*}
We shall prove these two claims in order. 

\textit{Proof of claim (i) }
Denote $X \sim P_X$ and $Y \sim P_Y$. Given a kernel $k$ and its feature map $\psi$ so that $k(x,y) \coloneqq \langle \psi(x), \psi(y) \rangle_{\mathcal{H}_k}$ (we will drop the dependence on $\mathcal{H}_k$ for brevity), we define its centered version $\overline{k}_X$ with respect to $P_X$ as
\begin{align*}
    \overline{k}_X(x_1,x_2) & = k(x_1,x_2) - \mE[k(x_1,X)] - \mE[k(X,x_2)] + \mE[k(X_1,X_2)] \\
	& = \langle \psi(x_1) - \mE_{P_X}[\psi(X)], \psi(x_2) - \mE_{P_X}[\psi(X)] \rangle \\
	& = \sum_{i=1}^\infty \lambda_i \phi_i(x_1) \phi_i(x_2),
\end{align*} 
where we use spectral decomposition to denote the centered kernel $\overline{k}_X(x,y) = \sum_{i=1}^\infty \lambda_i \phi_i(x) \phi_i(y)$ with orthonormal basis $\{\phi_i\}_{i=1}^\infty$ and corresponding eigenvalues $\{\lambda_i\}_{i=1}^\infty$. Similarly, we define the centered kernel with respect to $P_Y$ as $\overline{k}_Y(x,y) \coloneqq \langle \psi(x) - \mE[\psi(Y)], \psi(y) - \mE[\psi(Y)] \rangle  = \sum_{i=1}^\infty \check{\lambda}_i \check{\phi}_i(x) \check{\phi}_i(y).$

We express the witness function in terms of the inner product of feature maps as follows:
\begin{align*}
	\fhat(x) = \frac{1}{n_1} \sum_{i=1}^{n_1} k(X_i',x) - \frac{1}{n_2} \sum_{i=1}^{n_2} k(Y_i',x) = \langle \bar{\psi}_X - \bar{\psi}_Y,  \psi(x) \rangle,
\end{align*}
where we denote the sample mean of the feature map $\bar{\psi}_X\coloneqq \frac{1}{n_1} \sum_{i=1}^{n_1} k(X_i,\cdot)$ and $\bar{\psi}_Y\coloneqq \frac{1}{n_1} \sum_{i=1}^{n_2} k(Y_i,\cdot)$ as $\bar{\psi}_X$ and $\bar{\psi}_Y$, respectively. 
We also let 
\begin{align*}
	\overline{k}_{Y,X}(y,x) = \langle \psi(y) - \mE_{P_Y}[\psi(Y)], \psi(x) - \mE_{P_X}[\psi(X)] \rangle
\end{align*}
from which we observe that
\begin{align*}
    \mathrm{MMD}^2 = \mE[\overline{k}_X(Y,Y')] = \mE[\overline{k}_Y(X,X')] = \sum_{i=1}^\infty \lambda_i \mE[\phi_i(Y)]^2 =  \sum_{i=1}^\infty \check{\lambda}_i \mE[\check{\phi}_i(X)]^2.
\end{align*}
Given the notation and denoting $\overline{\phi}_{i,X}$ and $\overline{\phi}_{i,X}$ as the sample mean of $\{\phi_{i}(X_j')\}_{j=1}^{n_1}$ and $\{\phi_{i}(Y_j')\}_{j=1}^{n_2}$, respectively, we obtain the upper bound for the numerator as
\begin{align*}
	 & \mE [\{\fhat(X) - \mE[\fhat(X) \given \fhat]\}^4]  = \mE\biggl[ \biggl\{ \sum_{i=1}^{\infty} \lambda_i (\overline{\phi}_{i,X} - \overline{\phi}_{i,Y}) \phi_i(X)   \biggr\}^4 \biggr] \\
	 & = \mE\biggl[  \biggl\{ \frac{1}{n_1} \sum_{i=1}^{n_1} \overline{k}_X(X_i,X)  -  \frac{1}{n_2} \sum_{i=1}^{n_2} \overline{k}_X(Y_i,X)  \bigg\}^4 \biggr] \\
	 & = \mE\biggl[  \biggl\{ \frac{1}{n_1} \sum_{i=1}^{n_1} \overline{k}_X(X_i,X)  -  \frac{1}{n_2} \sum_{i=1}^{n_2} \overline{k}_{Y,X}(Y_i,X) + \langle \mE_{P_X}[\psi(X)] - \mE_{P_Y}[\psi(Y)] , \psi(X) - \mE_{P_X}[\psi(X)] \rangle  \bigg\}^4  \biggr] \\
	 & \lesssim \mE\biggl[ \biggl\{  \frac{1}{n_1} \sum_{i=1}^{n_1} \overline{k}_X(X_i,X) \bigg\}^4 \biggr] +  \mE\biggl[ \biggl\{  \frac{1}{n_2} \sum_{i=1}^{n_2} \overline{k}_{Y,X}(Y_i,X) \bigg\}^4 \biggr] \\ 
     & \hskip 18em + \mE\bigl[\langle \mE_{P_X}[\psi(X)] - \mE_{P_Y}[\psi(Y)] , \psi(X) - \mE_{P_X}[\psi(X)] \rangle^4  \bigr] \\
	 & \lesssim \frac{1}{n_1^3} \mE[\overline{k}_X(X_1,X_2)^4] + \frac{1}{n_1^2}\mE[\overline{k}_X(X_1,X_2)^2 \overline{k}_X(X_1,X_3)^2] +  \frac{1}{n_2^3} \mE[\overline{k}_{Y,X}(Y,X)^4]  \\
	 & \hskip 15em + \frac{1}{n_2^2}\mE[\overline{k}_{Y,X}(Y_1,X)^2 \overline{k}_{Y,X}(Y_2,X)^2] + \mathrm{MMD}^4 \times \mE[\overline{k}_{X}(X,X)^2].
\end{align*}
Next, we compute the conditional variance of $\fhat(X)$ as
\begin{align*}
	\mV[\fhat(X) \given \fhat] & = \mE[ \langle \bar{\psi}_X - \bar{\psi}_Y,  \psi(X) - \mE_P[\psi(X)] \rangle^2 \given \bar{\psi}_X, \bar{\psi}_Y] \\
	& = \mE[ \langle \bar{\psi}_X - \mE[\psi(X)] + \mE[\psi(X)] - \bar{\psi}_Y,  \psi(X) - \mE_P[\psi(X)] \rangle^2 \given \bar{\psi}_X, \bar{\psi}_Y] \\
	& = \mE\biggl[ \biggl\{ \sum_{i=1}^{\infty} \lambda_i  \biggl( \frac{1}{n_1} \sum_{j=1}^{n_1} \phi_i(X_j') \biggr) \phi_i(X)  - \sum_{i=1}^{\infty} \lambda_i  \biggl( \frac{1}{n_2} \sum_{j=1}^{n_2} \phi_i(Y_j') \biggr) \phi_i(X)  \biggr\}^2 \,\bigg| \, (X_j'),(Y_j') \biggr] \\
	& = \mE\biggl[ \biggl\{ \sum_{i=1}^{\infty} \lambda_i (\overline{\phi}_{i,X} - \overline{\phi}_{i,Y}) \phi_i(X)   \biggr\}^2 \,\bigg| \, (X_j'),(Y_j') \biggr] \\
	& = \sum_{i=1}^\infty \lambda_i^2 (\overline{\phi}_{i,X} - \overline{\phi}_{i,Y})^2.
\end{align*}

Moreover, we denote 
\begin{align*}
     & \overline{g}_X(x,y) = \mE[\overline{k}_X(x,X)\overline{k}_X(y,X)] = \sum_{i=1}^{\infty} \lambda_i^2 \phi_i(x) \phi_i(y) \quad \text{and} \\
     & \overline{g}_Y(x,y) = \mE[\overline{k}_Y(x,Y)\overline{k}_Y(y,Y)]= \sum_{i=1}^{\infty} \check{\lambda}_i^2 \check{\phi}_i(x) \check{\phi}_i(y),
\end{align*}
and compute the lower bound for the denominator as
\begin{align*}
	\mE[\sigma_{X,\fhat}^2] = \mE \{\mV[\fhat(X) \given \fhat]\} &= \sum_{i=1}^\infty \lambda_i^2 \mE \bigl[ (\overline{\phi}_{i,X} - \overline{\phi}_{i,Y})^2 \bigr] = \sum_{i=1}^\infty \lambda_i^2 \biggl( \frac{1}{n_1} + \mE\bigl[\overline{\phi}_{i,Y}^2\bigr] \biggr) \\
	& = \sum_{i=1}^\infty \lambda_i^2 \biggl( \frac{1}{n_1} + \frac{\mV[\phi_i(Y)]}{n_2}  + \mE[\phi_i(Y)]^2 \biggr) \\
	& = \frac{\mE[\overline{k}_X(X_1,X_2)^2]}{n_1} + \frac{\mE[\overline{g}_X(Y,Y)] - \mE[\overline{g}_X(Y_1,Y_2)]}{n_2} +\mE[\overline{g}_X(Y_1,Y_2)] \\
	& \gtrsim \frac{\mE[\overline{g}_X(X,X)]}{n_1} + \frac{\mE[\overline{g}_X(Y,Y)]}{n_2} + \mE[\overline{g}_X(Y_1,Y_2)].
\end{align*}

Combining these and letting $n_1 \leq n_2,$ it suffices to show that the following convergence results hold:

\begin{equation}
    \begin{aligned}\label{condition: condition for the claim (i) to prove asymptotic normality of xssMMD}
        \text{(a)} ~\, & \frac{\mathrm{MMD}^4\mE[\overline{k}_X(X,X)^2] }{n_1\{(n_1^{-1} + n_2^{-1})\mE[\overline{g}_X(X,X)] + \mE[\overline{g}_X(Y_1,Y_2)]\}^2} =o_P(1),\\[.5em]
        \text{(b)} ~\,& \frac{\mE[\overline{k}_{Y,X}(Y,X)^4]}{n_1n_2^3\{n_1^{-1}\mE[\overline{g}_X(X,X)] + n_2^{-1}\mE[\overline{g}_X(Y,Y)] + \mE[\overline{g}_X(Y_1,Y_2)]\}^2} =o_P(1) \\[.5em]
	   \text{(c)} ~\, & \frac{\mE[\overline{k}_{Y,X}(Y_1,X)^2 \overline{k}_{Y,X}(Y_2,X)^2]}{n_1n_2^2\{n_1^{-1}\mE[\overline{g}_X(X,X)] + n_2^{-1}\mE[\overline{g}_X(Y,Y)] + \mE[\overline{g}_X(Y_1,Y_2)]\}^2} =o_P(1),\\
        \text{(d)} ~\,& \frac{\mE[\overline{k}_X(X_1,X_2)^4]}{n_1^2\{\mE[\overline{g}_X(X,X)]\}^2} =o_P(1), \quad \text{and}\quad \text{(e)} ~\, \frac{\mE[\overline{k}_X(X_1,X_2)^2 \overline{k}_X(X_1,X_3)^2]}{n_1\{\mE[\overline{g}_X(X,X)]\}^2} =o_P(1).
    \end{aligned}
\end{equation}
Let us verify that these convergence results hold. With $n_1 \leq n_2,$ we obtain from \Cref{Assumption: xssMMD under H1}
\begin{align*}
	\frac{\mathrm{MMD}^4\mE[\overline{k}_X(X,X)^2] }{n_1\{(n_1^{-1} + n_2^{-1})\mE[\overline{g}_X(X,X)] + \mE[\overline{g}_X(Y_1,Y_2)]\}^2}  \leq \frac{\mathrm{MMD}^4\mE[\overline{k}_X(X,X)^2] }{\{n_1\mE[\overline{g}_X(X,X)] + n_1^2\mE[\overline{g}_X(Y_1,Y_2)]\}^2} =o_P(1),
\end{align*} which implies that (a) holds. 

Since we assume that $P_X$ and $P_Y$ have density functions $p$ and $q$ and $\|p/q\|_{L_{\infty}} \vee \|q/p\|_{L_{\infty}} \leq C,$ 
\begin{align*}
	\mE[\overline{g}_X(Y,Y)] \asymp \mE[\overline{g}_X(X,X)] \quad \text{and} \quad \mE[\overline{k}_{Y,X}(Y,X)^4] \lesssim \mathrm{MMD}^4 \mE[\overline{k}_X(X,X)^2].
\end{align*}

Combining these results with \Cref{Assumption: xssMMD under H1}, we show that (b) and (c) hold. Lastly, from \Cref{Assumption: xssMMD under H0}, (d) and (e) are satisfied. Hence we prove that the claim $(\mathrm{i})$ holds. 

\textit{Proof of claim (ii)} We show that the ratio converges to one in probability  
\begin{align*}
 		\frac{\mE[\sigma_{X,\fhat}^2]}{\sigma_{X,\fhat}^2} = 1 + o_P(1).
\end{align*} which directly shows that the claim $(\mathrm{ii})$ holds.

Recall 
\begin{align*}
	& \sigma_{X,\fhat}^2 = \sum_{i=1}^\infty \lambda_i^2 (\overline{\phi}_{i,X} - \overline{\phi}_{i,Y})^2 \quad \text{and} \\
	& \mE[\sigma_{X,\fhat}^2] \gtrsim \frac{\mE[\overline{g}_X(X,X)]}{n_1} + \frac{\mE[\overline{g}_X(Y,Y)]}{n_2} + \mE[\overline{g}_X(Y_1,Y_2)].
\end{align*}
Letting $Z \sim N(0,1)$, we have for any $\epsilon >0$ that 
\begin{align*}
	& \mP\Biggl( \frac{\mE[\sigma_{X,\fhat}^2]}{\sigma_{X,\fhat}^2} \geq \epsilon \Biggr) = \mP\Biggl( \frac{\sigma_{X,\fhat}^2}{\mE[\sigma_{X,\fhat}^2]} \leq \epsilon^{-1} \Biggr) \\
	& \leq  \mP\Biggl( \frac{\lambda_1^2(\overline{\phi}_{1,X} - \overline{\phi}_{1,Y})^2}{\mE[\sigma_{X,\fhat}^2]} \leq \epsilon^{-1} \Biggr) = \mP\Biggl( (\overline{\phi}_{1,X} - \overline{\phi}_{1,Y})^2 \leq \frac{\mE[\sigma_{X,\fhat}^2]}{\epsilon \lambda_1^2} \Biggr) \\
	& = \mP \Biggl( - \sqrt{\frac{\mE[\sigma_{X,\fhat}^2]}{\epsilon \lambda_1^2\mV[\overline{\phi}_{1,X} - \overline{\phi}_{1,Y}]}} \leq \frac{\overline{\phi}_{1,X} - \overline{\phi}_{1,Y}}{\sqrt{\mV[\overline{\phi}_{1,X} - \overline{\phi}_{1,Y}]}} \leq \sqrt{\frac{\mE[\sigma_{X,\fhat}^2]}{\epsilon \lambda_1^2 \mV[\overline{\phi}_{1,X} - \overline{\phi}_{1,Y}] }}  \Biggr) \\
	& \overset{(\text{a})}{\leq}  \mP \Biggl( - \sqrt{\frac{\mE[\sigma_{X,\fhat}^2]}{\epsilon \lambda_1^2\mV[\overline{\phi}_{1,X} - \overline{\phi}_{1,Y}]}} \leq Z - \frac{\mu_1}{\sqrt{\mV[\overline{\phi}_{1,X} - \overline{\phi}_{1,Y}]}} \leq \sqrt{\frac{\mE[\sigma_{X,\fhat}^2]}{\epsilon \lambda_1^2 \mV[\overline{\phi}_{1,X} - \overline{\phi}_{1,Y}] }}  \Biggr) + C \sqrt{\frac{\mE[\phi_1(X)^4]}{n_1}} \\
	& \overset{(\text{b})}{\lesssim}  \sqrt{\frac{1}{\epsilon} \times \frac{\sum_{i=1}^\infty \lambda_i^2(1 + n_1 \mu_i^2)}{\lambda_1^2}} + C \sqrt{\frac{\mE[\phi_1(X)^4]}{n_1}} \\
	& \overset{(\text{c})}{\lesssim} \sqrt{\frac{1}{\epsilon} \times \frac{\sum_{i=1}^\infty \lambda_i^2(1 + n_1 \mu_i^2)}{\lambda_1^2}} + \frac{\sum_{i=1}^\infty \lambda_i^2}{\lambda_1^2} \times o(1),
\end{align*}
where step (a) uses the Berry--Esseen bound and step (b) and (c) hold by the following reasoning: First of all, we use the observation that for $b \geq 0$
\begin{align*}
	\sup_{a \in \mathbb{R}}\int_{a-b}^{a+b} \frac{1}{\sqrt{2\pi}}e^{-\frac{x^2}{2}} dx = \int_{-b}^{b}  \frac{1}{\sqrt{2\pi}}e^{-\frac{x^2}{2}} dx \leq \frac{2b}{\sqrt{2\pi}},
\end{align*}
which can be verified by calculus. Hence, the first term in step (a) can be bounded above by 
\begin{align*}
    \sqrt{\frac{\mE[\sigma_{X,\fhat}^2]}{\epsilon \lambda_1^2 \mV[\overline{\phi}_{1,X} - \overline{\phi}_{1,Y}]}} \quad \text{up to a constant.}
\end{align*}
Hence step (b) follows since
\begin{align*}
	\frac{\mE[\sigma_{X,\fhat}^2]}{\epsilon \lambda_1^2 \mV[\overline{\phi}_{1,X} - \overline{\phi}_{1,Y}]} & \lesssim \frac{n_1^{-1} \mE[\overline{g}_X(X,X)] + \mE[\overline{g}_X(Y_1,Y_2)]}{\epsilon \lambda_1^2 (n_1^{-1} \mV[\phi_1(X)] + n_2^{-1}\mV[\phi_1(Y)])} \\
	& \lesssim \frac{ \mE[\overline{g}_X(X,X)] + n_1\mE[\overline{g}_X(Y_1,Y_2)]}{\epsilon \lambda_1^2} = \frac{1}{\epsilon} \times \frac{\sum_{i=1}^\infty \lambda_i^2(1 + n_1 \mu_i^2)}{\lambda_1^2}
\end{align*} 
where $\mu_i = \mE[\phi_i(Y)]$. Step (c) uses the observation that
\begin{align*}
	\biggl(\frac{\lambda_1^2}{\sum_{i=1}^\infty \lambda_i^2}\biggr)^2 \frac{\mE[\phi_1(X)^4]}{n_1} \leq \frac{\mE[\overline{k}_X(X_1,X_2)^2 \overline{k}_X(X_1,X_3)^2]}{n_1\{\mE[\overline{g}_X(X,X)]\}^2} =o_P(1) ,
\end{align*}
and thus
\begin{align*}
	\sqrt{\frac{\mE[\phi_1(X)^4]}{n_1}} = \frac{\sum_{i=1}^\infty \lambda_i^2}{\lambda_1^2} o(1),
\end{align*} where the first convergence is derived from \Cref{Assumption: xssMMD under H0}. Combining these results, we prove that the claim $(\mathrm{ii})$ is valid.

A similar argument applies to the second term of each equation in (\ref{condition: Lyapunov when delta=2}). Therefore, assuming that \Cref{Assumption: xssMMD under H0} and \Cref{Assumption: xssMMD under H1} hold, we prove that $\xssMMD_{\circ}$ is asymptotically $N(0,1)$ under the alternative.


Based on the previous results, we now focus on the null hypothesis and verify the asymptotic normality of $\xssMMD_\circ$ under the null when \Cref{Assumption: xssMMD under H0} holds. Through the same reasoning, it suffices to prove that claims ($\mathrm{i}$) and ($\mathrm{ii}$) are valid.

To show that claim ($\mathrm{i}$) holds, we prove that condition \eqref{condition: condition for the claim (i) to prove asymptotic normality of xssMMD} is satisfied. Under the null, $\textrm{MMD}=0$ which satisfies (a), (b), and (c). From \Cref{Assumption: xssMMD under H0}, (d) and (e) are satisfied. Therefore, we prove that the claim ($\mathrm{i}$) holds.

Next, we show that the claim ($\mathrm{ii}$) holds. Under the null, we obtain the upper bound
\begin{align*}
	\frac{\mE[\sigma_{X,\fhat}^2]}{\epsilon \lambda_1^2 \mV[\overline{\phi}_{1,X} - \overline{\phi}_{1,Y}]}   \lesssim \frac{ \mE[\overline{g}_X(X,X)] }{\epsilon \lambda_1^2} = \frac{1}{\epsilon} \times \frac{\sum_{i=1}^\infty \lambda_i^2}{\lambda_1^2}
\end{align*}
which leads to 
\begin{align*}
    \mP\Biggl( \frac{\mE[\sigma_{X,\fhat}^2]}{\sigma_{X,\fhat}^2} \geq \epsilon \Biggr) \lesssim \sqrt{\frac{1}{\epsilon} \times \frac{\sum_{i=1}^\infty \lambda_i^2}{\lambda_1^2}} + \frac{\sum_{i=1}^\infty \lambda_i^2}{\lambda_1^2} o(1) \quad \text{for sufficiently large $n_1$.}
\end{align*}
Hence we show that the claim ($\mathrm{ii}$) is valid. Therefore, we conclude that $\xssMMD$ is asymptotically normal under the null hypothesis as well.

A similar argument applies to the second term of each equation in (\ref{condition: Lyapunov when delta=2}). Therefore, assuming that \Cref{Assumption: xssMMD under H0} hold, we prove that $\xssMMD_{\circ}$ is asymptotically $N(0,1)$ under the null.

In addition, we discuss the asymptotic normality of $\xMMD$ under the alternative. Recall that in \Cref{Section: Consistency in Power}, $\xMMD$ with the same witness function $\fhat$ is defined as \begin{align*}
    \xMMD = \frac{\widetilde{\mu}_{X,\fhat} - \widetilde{\mu}_{Y,\fhat}}{\sqrt{\widetilde{\sigma}^2_{X,\fhat} + \widetilde{\sigma}^2_{Y,\fhat}}}.
\end{align*} 
 Similar to the previous proof, we show that the Lyapunov condition (\ref{condition: Lyapunov}) is satisfied. For simplicity, we take $\delta=2$ and show
\begin{align*}
	\frac{1}{\sigma_{n_1,n_2}^4} \Biggl[ \frac{1}{n_1^4}\sum_{i=1}^{n_1}  \mE [\{\fhat(X_i) - \mE[\fhat(X) \given \fhat]\}^4 \given \fhat] + \frac{1}{n_2^4} \sum_{i=1}^{n_2} \mE [\{\fhat(Y_i) - \mE[\fhat(Y) \given \fhat]\}^4 \given \fhat] \Biggr] =o_P(1)
\end{align*} where we denote \begin{align*}
	\sigma_{n_1,n_2}^2 = \widetilde{\sigma}^2_{X,\fhat} + \widetilde{\sigma}^2_{Y,\fhat}. 
 \end{align*}

In the previous proof, we have already shown that the following convergence holds under \Cref{Assumption: xssMMD under H0} and \Cref{Assumption: xssMMD under H1}:
\begin{align*}
	\frac{\mE [\{\fhat(X) - \mE[\fhat(X) \given \fhat]\}^4]}{n_1 \{\mE[\sigma_{X,\fhat}^2]\}^2} + \frac{\mE [\{\fhat(Y) - \mE[\fhat(Y) \given \fhat]\}^4]}{n_2 \{\mE[\sigma_{Y,\fhat}^2]\}^2} = o_P(1) \quad \text{and} \quad \frac{\mE[\sigma_{X,\fhat}^2]}{\sigma_{X,\fhat}^2} + \frac{\mE[\sigma_{Y,\fhat}^2]}{\sigma_{Y,\fhat}^2} = O_P(1).
\end{align*}

Hence the asymptotic normality of $\xMMD$ also follows under the alternative.

\noindent \textbf{Step 2: Asymptotic Normality with Sample Variance.} 
In this step, we show that the ratio of the sample variance to the population variance converges to 1 in probability. Following the same approach as in Step 2 of the proof of \Cref{Theorem: General Power Expression}, it suffices to show that \eqref{eq: convergence to prove step2 in thm 1} holds.

Without loss of generality, we focus on the first convergence result. From the definition of $\widehat{\sigma}_{X,\fhat}^2$ and $\widehat{\sigma}_{Y,\fhat}^2$, we obtain 
\begin{align*}
	\frac{\big|\widehat{\sigma}_{X,\fhat}^2  - \sigma_{X,\fhat}^2\big|}{\sigma_{X,\fhat}^2}  \leq \underbrace{\frac{1}{n_1}\frac{\big|\widehat{\sigma}_{1,X,\fhat}^2 - \sigma_{1,X,\fhat}^2\big|}{\sigma_{X,\fhat}^2}}_{(\mathrm{I})} + \underbrace{\frac{1}{n_1 + m_1} \frac{\big|\widehat{\sigma}_{2,X,\fhat}^2 - \sigma_{2,X,\fhat}^2\big|}{\sigma_{X,\fhat}^2}}_{(\mathrm{II})}.
\end{align*}
Since $\mE[\widehat{\sigma}_{1,X,\fhat}^2 \given \fhat] = \sigma_{1,X,\fhat}^2 ,$ the first term above satisfies $(\mathrm{I}) = o_P(1)$ if 
\begin{align*}
	\frac{1}{n^2_1}\frac{\mE[\{\widehat{\sigma}_{1,X,\fhat}^2 - \sigma_{1,X,\fhat}^2\}^2 \given \fhat]}{\sigma_{X,\fhat}^4} = o_P(1).
\end{align*}
This can be seen using the pieces established before in step 1. When showing the asymptotic normality of $\xssMMD_{\circ},$ we have proved the the both claims ($\mathrm{i}$) and ($\mathrm{ii}$) hold. From this, we obtain that
\begin{align*}
	\frac{1}{n^2_1}\frac{\mE[\{\widehat{\sigma}_{1,X,\fhat}^2 - \sigma_{1,X,\fhat}^2\}^2 \given \fhat]}{\sigma_{X,\fhat}^4} \lesssim  \underbrace{\frac{1}{n^3_1}\frac{\mE[\{\fhat(X) - \mE[\fhat(X) \given \fhat] \}^4]}{\{\mE[\overline{k}^2(X_1,X_2)]\}^2}}_{o_P(1)} \times \underbrace{\frac{\{\mE[\overline{k}^2(X_1,X_2)]\}^2}{\sigma_{X,\fhat}^4}}_{O_P(1)}= o_P(1).
\end{align*}
For the second term $(\mathrm{II})$, we may similarly proceed using conditional Jensen's inequality as
\begin{align*}
	(\mathrm{II}) \lesssim \, & \frac{1}{n^3_1}\frac{\mE[\{\mE[\fhat(X) \given V, \fhat] - \mE[\fhat(X) \given \fhat] \}^4\given \fhat]}{\{\mE[\overline{k}^2(X_1,X_2)]\}^2} \times \frac{\{\mE[\overline{k}^2(X_1,X_2)]\}^2}{\sigma_{X,\fhat}^4} \\
	\lesssim \, &  \frac{1}{n^3_1}\frac{\mE[\{\fhat(X) - \mE[\fhat(X) \given \fhat] \}^4]}{\{\mE[\overline{k}^2(X_1,X_2)]\}^2} \times \frac{\{\mE[\overline{k}^2(X_1,X_2)]\}^2}{\sigma_{X,\fhat}^4}= o_P(1).
\end{align*}
Combining the results, we use Slutsky's theorem to conclude that
\begin{align*}
	\xssMMD_{\circ} = \frac{\widehat{\mu}_{X,\fhat} - \widehat{\mu}_{Y,\fhat}}{\sqrt{\widehat{\sigma}_{X,\fhat}^2 + \widehat{\sigma}_{Y,\fhat}^2}}=\frac{\widehat{\mu}_{X,\fhat} - \widehat{\mu}_{Y,\fhat}}{\sqrt{{\sigma}^2_{X,\fhat} + {\sigma}^2_{Y,\fhat} }}\times \frac{\sqrt{\sigma_{X,\fhat}^2 + \sigma_{Y,\fhat}^2}}{\sqrt{\widehat{\sigma}_{X,\fhat}^2 + \widehat{\sigma}_{Y,\fhat}^2}}\convD N(0,1).
\end{align*}

\noindent \textbf{Step 3: Asymptotic Normality of $\xssMMD$.}

The aim of this subsection is to identify condition on these estimators under which 
\begin{align*}
	\xssMMD - \xssMMD_{\circ} = o_P(1).
\end{align*}
Once this condition is fulfilled, $\xssMMD$ converges to $N(0,1)$ by Slutsky's theorem. Denote $\xssMMD = \frac{\widehat{N}}{\widehat{D}}$ and $\xssMMD_{\circ} = \frac{N_{\circ}}{D_{\circ}}$. Then 
\begin{align*}
	\xssMMD - \xssMMD_{\circ} & = \frac{\widehat{N}}{\widehat{D}} - \frac{\widehat{N}}{D_{\circ}} + \frac{\widehat{N}}{D_{\circ}} - \frac{N_{\circ}}{D_{\circ}} \\
	& = \frac{\widehat{N} - N_{\circ}}{D_{\circ}} \biggl( \frac{D_{\circ}}{\widehat{D}} - 1 \biggr) + \underbrace{\frac{N_{\circ}}{D_{\circ}}}_{=O_P(1)} \biggl( \frac{D_{\circ}}{\widehat{D}} - 1 \biggr) + \frac{\widehat{N} - N_{\circ}}{D_{\circ}}
\end{align*} where we have proved that $\xssMMD_{\circ} = N_{\circ}/D_{\circ} = O_P(1)$ in the previous step.

Hence it suffices to show that two claims hold:
\begin{align*}
	& \mathrm{(i^\prime)} ~ \frac{\widehat{N} - N_{\circ}}{D_{\circ}} = o_P(1) \quad \text{and} \quad \mathrm{(ii^\prime)}~  \frac{D_{\circ}}{\widehat{D}} - 1 = o_P(1).
\end{align*}

\textit{Proof of claim ($i^\prime$).} 
Note that $D_{\circ}^2 = \widehat{\sigma}_{X,\fhat}^2 + \widehat{\sigma}_{Y,\fhat}^2$ using the sample variance and let $D_{\star}^2 = \sigma_{X,\fhat}^2 + \sigma_{Y,\fhat}^2$ using the population variance where $\sigma_{X,\fhat}^2 = \frac{1}{n_1} \sigma_{1,X}^2 + \frac{1}{n_1 + m_1} \sigma_{2,X}^2$ and $\sigma_{Y,\fhat}^2 = \frac{1}{n_2} \sigma_{1,Y}^2 + \frac{1}{n_2 + m_2} \sigma_{2,Y}^2$. From our previous result obtained in Step 2, 
\begin{align*}
	\frac{D_{\circ}^2}{D^2_\star} = \frac{\widehat{\sigma}_{X,\fhat}^2 + \widehat{\sigma}_{Y,\fhat}^2}{\sigma_{X,\fhat}^2 + \sigma_{Y,\fhat}^2} \convP 1.
\end{align*}
By the continuous mapping theorem, 
\begin{align*}
	\frac{D_\star}{D_{\circ}} \convP 1,
\end{align*}
and thus
\begin{align*}
	\frac{\widehat{N} - N_{\circ}}{D_{\circ}}  = \frac{\widehat{N} - N_{\circ}}{D_\star} \{1 + o_P(1)\}.
\end{align*}
On the other hand, 
\begin{align*}
	\bigg|  \frac{\widehat{N} - N_{\circ}}{D_\star}  \bigg| \leq \frac{|R_X|}{\sigma_{X,\fhat}}  + \frac{|R_Y|}{\sigma_{Y,\fhat}}
\end{align*}
where
\begin{align*}
	& R_X = \frac{1}{n_1} \sum_{i=1}^{n_1} \{\widehat{\mE}[\fhat(X_i) \given V_i, \fhat] - \mE[\fhat(X_i) \given V_i, \fhat]\} \\
	& \hskip 5em + \frac{1}{n_1+m_1} \sum_{i=1}^{n_1+m_1} \{\mE[\fhat(X_i) \given V_i, \fhat] - \widehat{\mE}[\fhat(X_i) \given V_i, \fhat]\}, \\
	&  R_Y = \frac{1}{n_2} \sum_{i=1}^{n_2} \{\widehat{\mE}[\fhat(Y_i) \given W_i, \fhat] - \mE[\fhat(Y_i) \given W_i, \fhat]\} \\
	& \hskip 5em + \frac{1}{n_2+m_2} \sum_{i=1}^{n_2+m_2} \{\mE[\fhat(Y_i) \given W_i, \fhat] - \widehat{\mE}[\fhat(Y_i) \given W_i, \fhat]\}.
\end{align*}
Using the fact that $R_X$ and $R_Y$ are location-shift invariant
, it can be seen that
\begin{align*}
	& R_X^2 = O_P \Bigl(n_1^{-1} \mE\bigl[ \{\mE[\fhat(X) \given V, \fhat] - \widehat{\mE}[\fhat(X) \given V, \fhat]\}^2 \given \fhat \bigr] \Bigr), \\
	& R_Y^2 = O_P \Bigl(n_2^{-1} \mE\bigl[ \{\mE[\fhat(Y) \given W, \fhat] - \widehat{\mE}[\fhat(Y) \given W, \fhat]\}^2 \given \fhat \bigr] \Bigr).
\end{align*}
Thus
\begin{align*}
	\frac{R_X^2}{\sigma_{X,\fhat}^2} = O_P\biggl( \frac{ \mE\bigl[ \{\mE[\fhat(X) \given V, \fhat] - \widehat{\mE}[\fhat(X) \given V, \fhat]\}^2 \given \fhat \bigr]}{\mV\{\fhat(X) \given \fhat\}} \biggr),  \\
	\frac{R_Y^2}{\sigma_{Y,\fhat}^2} = O_P\biggl( \frac{ \mE\bigl[ \{\mE[\fhat(Y) \given W, \fhat] - \widehat{\mE}[\fhat(Y) \given W, \fhat]\}^2 \given \fhat \bigr]}{\mV\{\fhat(Y) \given \fhat\}} \biggr).
\end{align*}

Since \Cref{Assumption: consistency of conditional expectation} holds, both terms converges to zero with in probability, which proves that the claim ($\mathrm{i}^\prime$) holds.

\textit{Proof of claim ($ii^\prime$).} 
Note that $\widehat{D}^2 = \widehat{\sigma}^{\dagger 2}_{X,\fhat} + \widehat{\sigma}^{\dagger 2}_{Y,\fhat}$. By the continuous mapping theorem, it is sufficient to prove that 
\begin{align*}
	\frac{\widehat{D}^2}{D_{\circ}^2} - 1  = o_P(1),
\end{align*}
which, in turn, is implied by
\begin{align*}
	& \frac{1}{n_1}\frac{|\widehat{\sigma}^{\dagger 2}_{1,X} - \widehat{\sigma}_{1,X}^2|}{\widehat{\sigma}_{X,\fhat}^2} = o_P(1), \ \frac{1}{n_1+m_1}\frac{|\widehat{\sigma}^{\dagger 2}_{2,X} - \widehat{\sigma}_{2,X}^2|}{\widehat{\sigma}_{X,\fhat}^2} = o_P(1), \\
	& \frac{1}{n_2}\frac{|\widehat{\sigma}^{\dagger 2}_{1,Y} - \widehat{\sigma}_{1,Y}^2|}{\widehat{\sigma}_{Y,\fhat}^2} = o_P(1), \  \frac{1}{n_2+m_2}\frac{|\widehat{\sigma}^{\dagger 2}_{2,Y} - \widehat{\sigma}_{2,Y}^2|}{\widehat{\sigma}_{Y,\fhat}^2} = o_P(1).
\end{align*}
Without loss of generality, we focus on the terms on the first row with $\widehat{\sigma}^{\dagger 2}_{1,X}$ and $\widehat{\sigma}^{\dagger 2}_{2,X}.$

Observe that by the Cauchy--Schwarz inequality,
\begin{align*}
	\frac{1}{n_1}\frac{|\widehat{\sigma}^{\dagger 2}_{1,X} - \widehat{\sigma}_{1,X}^2|}{\widehat{\sigma}_{X,\fhat}^2} ~\leq~& \frac{1}{n_1}\frac{\frac{1}{n_1} \sum_{i=1}^{n_1}A_i^2}{\widehat{\sigma}_{X,\fhat}^2} \\
	& + 2 \sqrt{\frac{1}{n_1}\frac{\frac{1}{n_1} \sum_{i=1}^{n_1} A_i^2}{\widehat{\sigma}_{X,\fhat}^2}} \underbrace{\sqrt{\frac{\frac{1}{n_1} \sum_{i=1}^{n_1} B_i^2}{\widehat{\sigma}_{1,X}^2}}}_{=1}
\end{align*}
where $A_i$ and $B_i$ are defined as
\begin{align*}
	& A_i =  \{\mE[\fhat(X_i) \given V_i, \fhat] - \widehat{\mE}[\fhat(X_i) \given V_i, \fhat]\} - \frac{1}{n_1} \sum_{j=1}^{n_1} \{ \mE[\fhat(X_j) \given V_j, \fhat] - \widehat{\mE}[\fhat(X_j) \given V_j, \fhat]\}, \\
	& B_i = \{\fhat(X_i) - \mE[\fhat(X_i) \given V_i, \fhat]\} - \frac{1}{n_1} \sum_{j=1}^{n_1} \{\fhat(X_j) - \mE[\fhat(X_j) \given V_j, \fhat]\}. 
\end{align*}
Note that $\frac{1}{n_1} \sum_{i=1}^{n_1} A_i^2 \leq \frac{1}{n_1} \sum_{i=1}^{n_1} \{\mE[\fhat(X_i) \given V_i, \fhat] - \widehat{\mE}[\fhat(X_i) \given V_i, \fhat]\}^2$. Thus, this term becomes $o_P(1)$ when
\begin{align*}
	\frac{ \mE\bigl[ \{\mE[\fhat(X) \given V, \fhat] - \widehat{\mE}[\fhat(X) \given V, \fhat]\}^2 \given \fhat \bigr]}{\mV\{\fhat(X) \given \fhat\} } = o_P(1).
\end{align*}

Similarly, for the second term, we have
\begin{align*}
	\frac{1}{n_1+m_1}\frac{|\widehat{\sigma}^{\dagger 2}_{2,X} - \widehat{\sigma}_{2,X}^2|}{\widehat{\sigma}_{X,\fhat}^2} ~\leq~& \frac{1}{n_1+m_1}\frac{\frac{1}{n_1+m_1} \sum_{i=1}^{n_1+m_1}\widetilde{A}_i^2}{\widehat{\sigma}_{X,\fhat}^2} \\
	& + 2 \sqrt{\frac{1}{n_1+m_1}\frac{\frac{1}{n_1+m_1} \sum_{i=1}^{n_1+m_1} \widetilde{A}_i^2}{\widehat{\sigma}_{X,\fhat}^2}} \underbrace{\sqrt{\frac{\frac{1}{n_1+m_1} \sum_{i=1}^{n_1+m_1} \widetilde{B}_i^2}{\widehat{\sigma}_{2,X}^2}}}_{=1}
\end{align*}
where $\widetilde{A}_i$ and $\widetilde{B}_i$ are analogously defined as
\begin{align*}
	& \widetilde{A}_i =  \{\mE[\fhat(X_i) \given V_i, \fhat] - \widehat{\mE}[\fhat(X_i) \given V_i, \fhat]\} - \frac{1}{n_1+m_1} \sum_{j=1}^{n_1+m_1} \{ \mE[\fhat(X_j) \given V_j, \fhat] - \widehat{\mE}[\fhat(X_j) \given V_j, \fhat]\}, \\
	& \widetilde{B}_i = \mE[\fhat(X_i) \given V_i, \fhat] - \frac{1}{n_1 + m_1} \sum_{j=1}^{n_1+m_1} \mE[\fhat(X_j) \given V_j, \fhat]. 
\end{align*}
By the same logic, this also becomes $o_P(1)$ when 
\begin{align*}
	\frac{ \mE\bigl[ \{\mE[\fhat(X) \given V, \fhat] - \widehat{\mE}[\fhat(X) \given V, \fhat]\}^2 \given \fhat \bigr]}{\mV\{\fhat(X) \given \fhat\} } = o_P(1).
\end{align*}

A similar argument applies to $\widehat{\sigma}^{\dagger 2}_{1,Y}$ and $\widehat{\sigma}^{\dagger 2}_{2,Y},$ which proves the claim (ii) holds. Thus, assuming \Cref{Assumption: consistency of conditional expectation} holds, we conclude that $\xssMMD - \xssMMD_{\circ} = o_p(1)$, implying the asymptotic normality of $\xssMMD.$

\noindent \textbf{Step 4: Power Comparison.}

Assuming that \Cref{Assumption: xssMMD under H0}, \Cref{Assumption: consistency of conditional expectation}, and \Cref{Assumption: xssMMD under H1} hold, we showed that both $\xMMD$ and $\xssMMD$ converges to a normal distribution under the null and alternative. This derives the explicit expression of the  asymptotic power of each test statistics.
Recall the definition of each test statistic. The power function of the xMMD test approximates that
\begin{align*}
		\Phi \biggl( z_{\alpha} + \frac{\mE[\fhat(X)] - \mE[\fhat(Y)]}{\sqrt{\smash[b]{n_1^{-1}\mV(\fhat(X))+n_2^{-1}\mV(\fhat(Y))}}} \biggr) \quad \text{as $n \rightarrow \infty$}.
	\end{align*} 
On the other hand, the power function of the xssMMD test approximates
\begin{align*}
		\Phi \biggl( z_{\alpha} + \frac{\mE[\fhat(X)] - \mE[\fhat(Y)]}{\sqrt{\smash[b]{{\sigma}_{X,\fhat}^{ 2} + {\sigma}_{Y,\fhat}^{2}}}} \biggr) \quad \text{as $n \rightarrow \infty$}.
	\end{align*} 
Observe that the only difference results in the denominator which consists of the variance. Hence it suffices to show that 
\begin{align*}
    {\sigma}_{X,\fhat}^{ 2} + {\sigma}_{Y,\fhat}^{2} \leq \mV(\fhat(X))+\mV(\fhat(Y)).
\end{align*}
Using the total law of variance, we show 
\begin{align*}
    {\sigma}_{X,\fhat}^{2} &~ =\frac{1}{n_1} \mE [\mV\{\fhat(X) \given V,\fhat \} \given \fhat]+\frac{1}{n_1+m_1} \mV[\mE\{\fhat(X) \given V, \fhat\} \given \fhat]\\
    &~ \leq \frac{1}{n_1} \mE [\mV\{\fhat(X) \given V,\fhat \} \given \fhat]+\frac{1}{n_1} \mV[\mE\{\fhat(X) \given V, \fhat\} \given \fhat]\\
    &~  = \frac{1}{n_1} \mV(\fhat(X)).
\end{align*}
A similar computation applies to ${\sigma}_{X,\fhat}^{2}$ and $\frac{1}{n_2} \mV(\fhat(X)).$ Therefore, we conclude that the power of the xssMMD test is asymptotically greater than or equal to that of the xMMD test.

\subsection{Proof of \Cref{Lemma: consistency}}
The proof of \Cref{Lemma: consistency} follows the same lines of argument as in the proof of \citet[][Theorem 8]{shekhar2022permutation} with the additional observation that the numerator of $\xssMMD$ has an expectation equal to $\mathrm{MMD}(P_X,P_Y)^2$. Hence we omit the proof here.

\subsection{Proof of \Cref{Theorem: consistency condition}} \label{Section: proof of Theorem consistency condition}
We will prove that under the condition \eqref{Eq: consistency condition}, the following two inequalities hold:
\begin{align} \label{Eq: consistency inequality 1}
	& \mE_{P_{X,n},P_{Y,n}}\bigl[\widehat{\sigma}^{\dagger 2}_{X,\fhat} + \widehat{\sigma}^{\dagger 2}_{Y,\fhat}\bigr] \, \lesssim \,  \mE_{P_{X,n},P_{Y,n}}\bigl[\widetilde{\sigma}^2_{X,\fhat} + \widetilde{\sigma}^2_{Y,\fhat}\bigr] \quad \text{and} \\[.5em]
	& \mathrm{Var}_{P_{X,n},P_{Y,n}}\bigl[\widehat{\mu}_{X,\fhat}^\dagger - \widehat{\mu}_{Y,\fhat}^\dagger \bigr] \, \lesssim \, \mathrm{Var}_{P_{X,n},P_{Y,n}}\bigl[\widetilde{\mu}_{X,\fhat} - \widetilde{\mu}_{Y,\fhat}\bigr], \label{Eq: consistency inequality 2}
\end{align}
which directly implies the claim of \Cref{Theorem: consistency condition}. For simplicity, we will omit the dependence on $P_{X,n}$ and $P_{Y,n}$ in subsequent expressions. As mentioned in the main text, we focus on the case where $n_1, n_2, m_1, m_2$ are even for simplicity and the general case can be handled similarly with minor modifcations. 

\noindent \textbf{Verification of \eqref{Eq: consistency inequality 1}.} Starting with the first inequality \eqref{Eq: consistency inequality 1}, it suffices to prove that $\mE[\widehat{\sigma}^{\dagger 2}_{X,\fhat}] \lesssim \mE[\widetilde{\sigma}^2_{X,\fhat}]$ or equivalently,
\begin{align*}
	& \mE\biggl[ \frac{1}{n_1^2} \sum_{i=1}^{n_1} \big\{ \fhat(X_i) - \widehat{u}_X(V_i) \big\}^2 + \frac{1}{(n_1+m_1)^2} \sum_{i=1}^{n_1+m_1}\biggl\{ \widehat{u}_X(V_i) - \frac{1}{n_1+m_1}\sum_{j=1}^{n_1+m_1} \widehat{u}_X(V_j) \biggr\}^2 \biggr]\\
	\lesssim \  & \mE\biggl[\frac{1}{n_1^2} \sum_{i=1}^{n_1}\biggl\{\fhat(X_i) - \frac{1}{n_1} \sum_{j=1}^{n_1} \fhat(X_j)\biggr\}^2 \biggr] = \frac{n_1-1}{n_1^2}\mE\bigl[\mV\{\fhat(X) \given \fhat\} \bigr] \lesssim  \frac{1}{n_1}\mE\bigl[\mV\{\fhat(X) \given \fhat\} \bigr].
\end{align*}
Under the condition \eqref{Eq: consistency condition}, we have 
\begin{align*}
	\mE\biggl[ \frac{1}{n_1^2} \sum_{i=1}^{n_1} \big\{ \fhat(X_i) - \widehat{u}_X(V_i) \big\}^2\biggr] \lesssim \frac{1}{n_1}\mE\bigl[\mV\bigl\{\fhat(X) \given \fhat \bigr\} \bigr]
\end{align*}
and  
\begin{align*}
	& \mE\biggl[ \frac{1}{(n_1+m_1)^2} \sum_{i=1}^{n_1+m_1}\biggl\{ \widehat{u}_X(V_i) - \frac{1}{n_1+m_1}\sum_{j=1}^{n_1+m_1} \widehat{u}_X(V_j) \biggr\}^2 \biggr] \lesssim \frac{1}{n_1+m_1}\mE\bigl[\mV\bigl\{\widehat{u}_X(V) \given \widehat{u}_X \bigr\} \bigr] \\
	\lesssim \ & \frac{1}{n_1}\mE\bigl[\mV\bigl\{\widehat{u}_X(V) \given \widehat{u}_X \bigr\} \bigr] = \frac{1}{n_1}\mE\bigl[\mV\bigl\{\widehat{u}_X(V) - \fhat(X) + \fhat(X) \given \fhat, \widehat{u}_X \bigr\} \bigr] \\
	\lesssim \ & \frac{1}{n_1}\mE\bigl[\mV\bigl\{\widehat{u}_X(V) - \fhat(X) \given \fhat, \widehat{u}_X \bigr\} \bigr] + \frac{1}{n_1}\mE\bigl[\mV\bigl\{\fhat(X) \given \fhat \bigr\} \bigr] \lesssim  \frac{1}{n_1}\mE\bigl[\mV\bigl\{\fhat(X) \given \fhat \bigr\} \bigr].
\end{align*}
This completes the proof of the first inequality \eqref{Eq: consistency inequality 1}.

\noindent \textbf{Verification of \eqref{Eq: consistency inequality 2}.} For the second inequality \eqref{Eq: consistency inequality 2}, the law of total variance gives
\begin{align*}
	\mV\bigl[\widehat{\mu}_{X,\fhat}^\dagger - \widehat{\mu}_{Y,\fhat}^\dagger\bigr] = \mE \bigl[ \mV\bigl\{\widehat{\mu}_{X,\fhat}^\dagger - \widehat{\mu}_{Y,\fhat}^\dagger \given \fhat \bigr\} \bigr] + \mV \bigl[ \underbrace{\mE\bigl\{\widehat{\mu}_{X,\fhat}^\dagger - \widehat{\mu}_{Y,\fhat}^\dagger \given \fhat \bigr\}}_{= \mE\bigl\{\widetilde{\mu}_{X,\fhat} - \widetilde{\mu}_{Y,\fhat} \given \fhat \bigr\}} \bigr].
\end{align*}
Therefore it suffices to show that
\begin{align*}
	\mE\bigl[ \mV\bigl\{\widehat{\mu}_{X,\fhat}^\dagger - \widehat{\mu}_{Y,\fhat}^\dagger \given \fhat \bigr\} \bigr]  \lesssim  \mE\bigl[\mV\bigl\{\widetilde{\mu}_{X,\fhat} - \widetilde{\mu}_{Y,\fhat} \given \fhat \bigr\}\bigr].
\end{align*}
Using independence between $\widehat{\mu}_{X,\fhat}^\dagger$ and $\widehat{\mu}_{Y,\fhat}^\dagger$ (also $\widetilde{\mu}_{X,\fhat}$ and $\widetilde{\mu}_{Y,\fhat}$) conditional on $\fhat$, we have
\begin{align*}
	& \mV\bigl\{\widehat{\mu}_{X,\fhat}^\dagger - \widehat{\mu}_{Y,\fhat}^\dagger \given \fhat \bigr\} = \mV\bigl\{\widehat{\mu}_{X,\fhat}^\dagger \given \fhat \bigr\} + \mV\bigl\{\widehat{\mu}_{Y,\fhat}^\dagger \given \fhat \bigr\} \quad \text{and} \\
	& \mV\bigl\{\widetilde{\mu}_{X,\fhat} - \widetilde{\mu}_{Y,\fhat} \given \fhat \bigr\} = \mV\bigl\{\widetilde{\mu}_{X,\fhat} \given \fhat \bigr\} + \mV\bigl\{\widetilde{\mu}_{Y,\fhat} \given \fhat \bigr\}.
\end{align*}
Without loss of generality, we focus on $\mV\bigl\{\widehat{\mu}_{X,\fhat}^\dagger \given \fhat \bigr\}$ and $\mV\bigl\{\widetilde{\mu}_{X,\fhat} \given \fhat \bigr\}$, and show that $\mE\bigl[ \mV\bigl\{\widehat{\mu}_{X,\fhat}^\dagger \given \fhat \bigr\}\bigr] \lesssim \mE\bigl[\mV\bigl\{\widetilde{\mu}_{X,\fhat} \given \fhat \bigr\}\bigr]$. The other terms can be handled similarly. Since the cross-fit estimator has smaller variance than the single-split estimator and a constant factor is not of interest, we may assume that $\widehat{\mE}[\fhat(X) \given V, \fhat] \coloneqq \widehat{u}_X(V)$ is trained on an auxiliary dataset. Under this simplification, another application of the law of total variance gives
\begin{align*}
	\mV\bigl\{\widehat{\mu}_{X,\fhat}^\dagger \given \fhat \bigr\} = \mE\bigl[\mV\bigl\{\widehat{\mu}_{X,\fhat}^\dagger \given \fhat, \widehat{u}_X \bigr\} \given \fhat \bigr] + \mV\bigl[\underbrace{\mE\bigl\{\widehat{\mu}_{X,\fhat}^\dagger \given \fhat, \widehat{u}_X \bigr\}}_{= \mE \bigl\{ \widetilde{\mu}_{X,\fhat} \given \fhat \bigr\} } \given \fhat \bigr]
\end{align*}
and thus we focus on the expectation of the conditional variance above. A direct calculation yields
\begin{align*}
	\mE\bigl[\mV\bigl\{\widehat{\mu}_{X,\fhat}^\dagger \given \fhat, \widehat{u}_X \bigr\} \given \fhat \bigr] =  \frac{1}{n_1} \mV\bigl\{\fhat(X) \given \fhat\bigr\} & + \frac{m_1}{n_1(n_1+m_1)} \mE\bigl[\mV\bigl\{\widehat{u}_X(V) \given \fhat, \widehat{u}_X \bigr\} \given \fhat \bigr] \\ & - \frac{2m_1}{n_1(n_1+m_1)}\mE \bigl[ \mathrm{Cov}\bigl\{\fhat(X), \widehat{u}_X(V) \given \fhat, \widehat{u}_X \bigr\} \given \fhat \bigr].
\end{align*}
Now, in order to prove the second inequality \eqref{Eq: consistency inequality 2}, we need to ensure that
\begin{align*}
	\mE\bigl[\mV\bigl\{\widehat{u}_X(V) \given \fhat, \widehat{u}_X \bigr\} \bigr] - 2\mE \bigl[ \mathrm{Cov}\bigl\{\fhat(X), \widehat{u}_X(V) \given \fhat, \widehat{u}_X \bigr\} \bigr] \lesssim \mE\bigl[\mV\bigl\{\fhat(X) \given \fhat\bigr\}\bigr].
\end{align*}
This follows from the condition \eqref{Eq: consistency condition} as we have
\begin{align*}
	& \mE\bigl[\mV\bigl\{\widehat{u}_X(V) \given \fhat, \widehat{u}_X \bigr\} - 2 \mathrm{Cov}\bigl\{\fhat(X), \widehat{u}_X(V) \given \fhat, \widehat{u}_X \bigr\} \bigr] \\
	= \ & \mE\bigl[\mV\bigl\{\fhat(X) - \widehat{u}_X(V) \given \fhat, \widehat{u}_X \bigr\} \bigr] - \mE\bigl[\mV\bigl\{\fhat(X) \given \fhat \bigr\}\bigr] \\
	\leq \ & \mE\bigl[\bigl\{\fhat(X) - \widehat{u}_X(V) \bigr\}^2 \bigr] - \mE\bigl[\mV\bigl\{\fhat(X) \given \fhat \bigr\}\bigr] \\
	\lesssim \ & \mE\bigl[\mV\bigl\{\fhat(X) \given \fhat \bigr\}\bigr].
\end{align*}
Thus, the proof of \Cref{Theorem: consistency condition} is complete.
\subsection{Proof of~\Cref{corollary:fixed-alternative}} 
\label{proof:fixed-alternative}
    To prove this result, we will show that in the case of fixed alternative with $P_{X_n}=P_X$ and $P_{Y_n} = P_Y$ with $\mathrm{MMD}(P_X, P_Y)=\gamma>0$ for all $n \geq 1$, we have 
    \begin{align}
        \max \left\{ \mathbb{E}_{P_X}[k(X_1, X_1)],\; \mathbb{E}_{P_Y}[k(Y_1, Y_1)]  \right\} < \infty \quad \implies \quad \text{Equation~\eqref{Eq: xMMD consistency}}.  \label{eq:proof-fixed-alternative-1}
    \end{align}
    This fact, along with~\Cref{Theorem: consistency condition} implies~\Cref{corollary:fixed-alternative}. The proof of the above implication is essentially the same as that of \citet[Theorem 7]{shekhar2022permutation}. However, we include the details for completeness and also because the fourth moment condition required by~\citet[Theorem~7]{shekhar2022permutation} is unnecessary. 

    Since $\gamma>0$ for all $n$, the conditions in~\eqref{Eq: xMMD consistency} are equivalent to the following (we drop the subscripts $P_{X,n}$ and $P_{Y,n}$ from $\mE$ and $\mV$ to simplify the notation)
    \begin{align*}
        \lim_{n_1 \to \infty} \mathbb{E}\left[  \widetilde{\sigma}^2_{X, \fhat} \right] = 0, \quad 
        \lim_{n_2 \to \infty} \mathbb{E}\left[  \widetilde{\sigma}^2_{Y, \fhat} \right] = 0,\quad \text{and} \quad 
        \lim_{n_1 \wedge n_2 \to \infty} \mV\left[  \widetilde{\mu}_{X, \fhat} - \widetilde{\mu}_{Y, \fhat} \right] = 0. 
    \end{align*}
    We will first show that $\lim_{n_1 \to \infty} \mE[\widetilde{\sigma}^2_{X, \fhat}]=0$. To do this, we introduce the notation $\ghat_X = \frac{1}{n_1} \sum_{i=1}^{n_1} k(X_i, \cdot)$, and observe the following: 
    \begin{align}
        \mE\left[ \widetilde{\sigma}_{X, \fhat}^2 \right] & = \mE \left[ \frac{1}{n_1^2} \sum_{i=1}^{n_1} (\fhat(X_i) - \widetilde{\mu}_{X, f})^2 \right] 
        = \mathbb{E}\lb \frac{1}{n_1^2} \sum_{i=1}^{n_1} (\langle k(X_i, \cdot), \fhat \rangle_{\mathcal{H}_k} - \langle \ghat_X, \, \fhat \rangle_{\mathcal{H}_k} )^2 \rb \nonumber \\
        & = \frac{1}{n_1} \mE\lb \lp \langle k(X_1, \cdot) - \ghat_X, \, \fhat \rangle_{\mathcal{H}_k} \rp^2\rb \stackrel{(i)}{\leq} \frac{1}{n_1} \mE\lb \|k(X_1, \cdot) - \ghat_X\|_{\mathcal{H}_k}^2 \|\fhat\|_{\mathcal{H}_k}^2 \rb \nonumber \\
        & \stackrel{(ii)}{\lesssim} \frac{1}{n_1} \mE \lb \|k(X_1, \cdot)\|_{\mathcal{H}_k}^2  + \|\ghat_X\|_{\mathcal{H}_k}^2 \rb \mE \lb \|\fhat\|_{\mathcal{H}_k}^2 \rb.  \label{eq:proof-fixed-alternative-2}
    \end{align}
    Here $(i)$ uses the Cauchy--Schwarz inequality, and $(ii)$ follows from the independence of $k(X_1, \cdot)-\ghat_X$ and $\fhat$.  Now, $\|k(X_1, \cdot)\|_{\mathcal{H}_k}^2 = k(X_1, X_1)$ by the reproducing property, and $\|\ghat_X\|_{\mathcal{H}_k}^2 \leq \frac{1}{n_1}\sum_{i=1}^{n_1} k(X_i, X_i)$ by Jensen's inequality and the convexity of $\|\cdot\|_{\mathcal{H}_k}^2$. Thus, $\mE[\|k(X_1, \cdot)\|_{\mathcal{H}_k}^2] + \mE[\|\ghat_X\|_{\mathcal{H}_k}^2] \lesssim \mathbb{E}[k(X_1, X_1)]$. Now, consider the term $\|\fhat\|_{\mathcal{H}_k}^2$, and observe that 
    \begin{align*} 
        \|\fhat\|_{\mathcal{H}_k}^2 &= \left\| \frac{1}{n_1} \sum_{i=1}^{n_1} k(\widetilde{X}_i, \cdot) - \frac{1}{n_2} \sum_{i=1}^{n_2} k(\widetilde{Y}_i, \cdot)\right\|_{\mathcal{H}_k}^2 \lesssim \left\| \frac{1}{n_1} \sum_{i=1}^{n_1} k(\widetilde{X}_i, \cdot)\right\|_{\mathcal{H}_k}^2 + \left\| \frac{1}{n_2} \sum_{i=1}^{n_2} k(\widetilde{Y}_i, \cdot)\right\|_{\mathcal{H}_k}^2 \\
        & \leq \frac{1}{n_1} \sum_{i=1}^n k(\widetilde{X}_i, \widetilde{X}_i)  + \frac{1}{n_2} \sum_{i=1}^{n_2} k(\widetilde{Y}_i, \widetilde{Y}_i), 
    \end{align*}
    where the last inequality again uses Jensen's inequality along with the convexity of the mapping $x \mapsto \|x\|_{\mathcal{H}_k}^2$. This implies $\mE[\|\fhat\|_{\mathcal{H}_k}^2] \lesssim \mE[k(X_1, X_1)] + \mE[k(Y_1, Y_1)]$.  Plugging this back into~\eqref{eq:proof-fixed-alternative-2}, we get 
    \begin{align}
        \mE\lb \widetilde{\sigma}_{X, \fhat}^2 \rb \lesssim \frac{1}{n_1} \lp \mE[k(X_1, X_1)] \rp \lp \mE[k(X_1, X_1)] + \mE[k(Y_1, Y_1)]\rp. \label{eq:proof-fixed-alternative-3}
    \end{align}
    Thus, under the assumption that $\mE[k(X_1, X_1)]<\infty$ and $\mE[k(Y_1, Y_1)]<\infty$, the above inequality implies that $\lim_{n_1 \to \infty} \mE[\widetilde{\sigma}_{X, \fhat}^2] = 0$. An exactly analogous argument implies that $\lim_{n_2 \to \infty} \mE[\widetilde{\sigma}_{Y, \fhat}^2] = 0$.

    We now show that $\lim_{n_1 \wedge n_2 \to \infty} \mV\bigl[ \widetilde{\mu}_{X, \fhat} - \widetilde{\mu}_{Y, \fhat} \bigr] = 0$. To show this, we introduce some new notation: let $\mu_X = \mE[k(X, \cdot)]$ and $\mu_Y = \mE[k(Y, \cdot)]$ denote the kernel mean embeddings corresponding to $P_X$ and $P_Y$, and let $\ghat_Y = \frac{1}{n_2}\sum_{i=1}^{n_2} k(Y_i, \cdot)$. Then, we have the following: 
    \begin{align*}
        \mV\lb \widetilde{\mu}_{X, \fhat} - \widetilde{\mu}_{Y, \fhat} \rb & = \mE \lb \lp \widetilde{\mu}_{X, \fhat} - \widetilde{\mu}_{Y, \fhat} - \langle \mu_X - \mu_Y, \mu_X - \mu_Y 
 - \fhat + \fhat \rangle_{\mathcal{H}_k} \rp^2 \rb  \\
    & = \mE \lb \lp  \langle \ghat_X - \mu_X, \fhat\rangle_{\mathcal{H}_k} - \langle \ghat_Y-\mu_Y, \fhat\rangle_{\mathcal{H}_k} + \langle \fhat - (\mu_X-\mu_Y), \mu_X - \mu_Y \rangle_{\mathcal{H}_k} \rp^2 \rb  \\
    & \lesssim \mE \lb\langle \ghat_X - \mu_X, \fhat \rangle_{\mathcal{H}_k}^2 \rb +    \mE \lb\langle \ghat_Y - \mu_Y, \fhat \rangle_{\mathcal{H}_k}^2 \rb  + \mE \lb\langle \fhat - (\mu_X-\mu_Y), \mu_X - \mu_Y \rangle_{\mathcal{H}_k}^2 \rb. 
    \end{align*}
    Applying the Cauchy--Schwarz inequality on all the three terms, we get 
    \begin{align*}
        \mV\lb \widetilde{\mu}_{X, \fhat} - \widetilde{\mu}_{Y, \fhat} \rb &\lesssim \mE[\|\fhat\|_{\mathcal{H}_k}^2] \lp \mE[\|\ghat_X - \mu_X\|_{\mathcal{H}_k}^2] + \mE[\|\ghat_Y - \mu_Y\|_{\mathcal{H}_k}^2] \rp \\
        & \hskip 14em + \|\mu_X-\mu_Y\|_{\mathcal{H}_k}^2 \mE\lb \|\fhat - (\mu_X - \mu_Y)\|_{\mathcal{H}_k}^2 \rb.
    \end{align*}
    Now, we can break up $\fhat$ into $\gtilde_X - \gtilde_Y$, with $\gtilde_X = \frac{1}{n_1} \sum_{i=1}^{n_1} k(\widetilde{X}_i, \cdot)$ and $\gtilde_Y = \frac{1}{n_2} \sum_{i=1}^{n_2} k(\widetilde{Y}_i, \cdot)$, and get the following bound $\mE\lb \|\fhat - (\mu_X - \mu_Y)\|_{\mathcal{H}_k}^2 \rb \lesssim \mE \lb \|\gtilde_X - \mu_X\|_{\mathcal{H}_k}^2 \rb + \mE \lb \|\gtilde_Y - \mu_Y\|_{\mathcal{H}_k}^2 \rb$. To summarize, we have proved that  
    \begin{align*}
        \mV\lb \widetilde{\mu}_{X, \fhat} - \widetilde{\mu}_{Y, \fhat} \rb &\lesssim \mE[\|\fhat\|_{\mathcal{H}_k}^2] \lp \mE[\|\ghat_X - \mu_X\|_{\mathcal{H}_k}^2] + \mE[\|\ghat_Y - \mu_Y\|_{\mathcal{H}_k}^2] \rp \\
        & \hskip 13em + \gamma^2 \lp \mE[\|\gtilde_X - \mu_X\|_{\mathcal{H}_k}^2] + \mE[\|\gtilde_Y - \mu_Y\|_{\mathcal{H}_k}^2]\rp. 
    \end{align*}
    We have already proved that $\mE[\|\fhat\|_{\mathcal{H}_k}^2] \lesssim \mE[k(X_1, X_1)] + \mE[k(Y_1, Y_1)]<\infty$ under the assumptions of this corollary. Thus, to complete the proof, we need to show that 
    \begin{align}
        &\lim_{n_1\to \infty} \mE[\|\ghat_X - \mu_X\|_{\mathcal{H}_k}^2] = 0, \; \lim_{n_2 \to \infty} \mE[\|\ghat_Y - \mu_Y\|_{\mathcal{H}_k}^2],  \; 
        \lim_{n_1\to \infty} \mE[\|\gtilde_X - \mu_X\|_{\mathcal{H}_k}^2] = 0, \; \text{and}\; \lim_{n_2 \to \infty} \mE[\|\gtilde_Y - \mu_Y\|_{\mathcal{H}_k}^2]. \label{eq:proof-fixed-alternative-4}
    \end{align}
    We present the details of the first of these four conditions, since the steps for proving the other three are exactly the same. 
    \begin{align}
        \mE\lb \|\ghat_X - \mu_X\|_{\mathcal{H}_k}^2 \rb & = \mE \lb \frac{1}{n_1^2} \sum_{i=1}^{n_1} \sum_{j=1}^{n_1} \langle k(X_i, \cdot) - \mu_X, \, k(X_j, \cdot) - \mu_X \rangle_{\mathcal{H}_k} \rb \nonumber \\
        &= \frac{1}{n_1^2} \lp \sum_{i=1}^{n_1} \mE[\langle k(X_i, \cdot) - \mu_X, k(X_i, \cdot) - \mu_X \rangle_{\mathcal{H}_k}] + \sum_{i\neq j} \mE \langle k(X_i, \cdot)-\mu_X, k(X_j, \cdot)-\mu_X \rangle_{\mathcal{H}_k}] \rp. \label{eq:proof-fixed-alternative-5}
    \end{align}
    Observe that by the Cauchy--Schwarz inequality, we have $\mE [\langle k(X_i, \cdot) - \mu_X, k(X_i, \cdot) - \mu_X \rangle_{\mathcal{H}_k}] \lesssim \mE[k(X_1, X_1)] + \gamma^2$. Furthermore, for any $i \neq j$, we have $\mE \langle k(X_i, \cdot)-\mu_X, k(X_j, \cdot)-\mu_X \rangle_{\mathcal{H}_k}] = 0$. Plugging these back into~\eqref{eq:proof-fixed-alternative-5}, we get 
    \begin{align*}
        \lim_{n_1 \to \infty} \mE\lb \|\ghat_X - \mu_X\|_{\mathcal{H}_k}^2 \rb & \leq \lim_{n_1 \to \infty} \frac{1}{n_1} \lp \mE[k(X_1, X_1)] + \gamma^2\rp = 0, 
    \end{align*}
    under the assumption that $\mE[k(X_1, X_1)]<\infty$. The remaining three terms in~\eqref{eq:proof-fixed-alternative-4} can also be shown to go to zero similarly. This completes the proof of~\Cref{corollary:fixed-alternative}. 
\subsection{Proof of \Cref{Corollary: cross-fitting using linear operator}}\label{Section: proof of Corollary linear operator}

We show that using a linear operator, the condition (\ref{Eq: linear operator}) results in \Cref{Assumption: consistency of conditional expectation}. Without loss of generality, we focus on the first convergence result of (\ref{Eq: cross-fitting condition}). Leveraging the linearity of the estimators of conditional expectations, combined with the spectral decomposition of the centered kernel $\overline{k}$, yields:
\begin{align*}
	& \mE[\fhat(X) \given V, \fhat] - \widehat{\mE}[\fhat(X) \given V, \fhat] = \sum_{i=1}^\infty \lambda_i (\overline{\phi}_{i,X} - \overline{\phi}_{i,Y}) \{\mE[\phi_i(X) \given V]  - \widehat{\mE}[\phi_i(X) \given V]\}, \\
	&  \mV[\fhat(X) \given \fhat] = \sum_{i=1}^\infty \lambda_i^2  (\overline{\phi}_{i,X} - \overline{\phi}_{i,Y})^2.
\end{align*}
Write 
\begin{align*}
	a_i = \lambda_i (\overline{\phi}_{i,X} - \overline{\phi}_{i,Y}), \ b_i = \overline{\phi}_{i,X} - \overline{\phi}_{i,Y}  \quad \text{and} \quad  \Delta_{X,i} = \mE[\phi_i(X) \given V]  - \widehat{\mE}[\phi_i(X) \given V].
\end{align*}
Then the first convergence condition in (\ref{Eq: cross-fitting condition}) is equivalent to
\begin{align*}
	\frac{\mE[\{\sum_{i=1}^\infty a_i \Delta_{X,i} \}^2 \given (a_i)_{i=1}^\infty]}{\sum_{i=1}^\infty a_i^2} = o_P(1).
\end{align*}
We decompose the above ratio into two terms as
\begin{align*}
	\frac{\mE[\{\sum_{i=1}^\infty a_i \Delta_{X,i} \}^2 \given (a_i)_{i=1}^\infty]}{\sum_{i=1}^\infty a_i^2} = \frac{\sum_{i=1}^\infty a_i^2 \mE[\Delta_{X,i}^2]}{\sum_{i=1}^\infty a_i^2}  + \frac{\sum_{i \neq j}a_ia_j \mE[\Delta_{X,i} \Delta_{X,j}]}{\sum_{i=1}^\infty a_i^2},
\end{align*} and show that each term converges to zero in probability. For the first term, we have
\begin{align*}
	& \frac{\sum_{i=1}^\infty a_i^2 \mE[\Delta_{X,i}^2]}{\sum_{i=1}^\infty a_i^2} \leq \sup_{i \geq 1} \mE[\Delta_{X,i}^2] \times  \frac{\sum_{i=1}^\infty a_i^2}{\sum_{i=1}^\infty a_i^2}  = \sup_{i \geq 1} \mE[\Delta_{X,i}^2] = o_P(1),
\end{align*}
where the last equality follows by $\sup_{i \geq 1} \mE[\Delta_{X,i}^2] = o(1)$. 


Next we decompose the second term into 
\begin{align}
    \label{equation used to show corollary c.4}
	\frac{\sum_{i \neq j}a_ia_j \mE[\Delta_{X,i} \Delta_{X,j}]}{\sum_{i=1}^\infty a_i^2} = \frac{n_1n_2}{n_1 + n_2} \frac{\sum_{i \neq j}a_ia_j \mE[\Delta_{X,i} \Delta_{X,j}]}{\sum_{i=1}^\infty \lambda_i^2} \times \frac{n_1+n_2}{n_1n_2} \frac{\sum_{i=1}^\infty \lambda_i^2}{\sum_{i=1}^\infty a_i^2}.
\end{align}

We first show that the first term of \eqref{equation used to show corollary c.4} converges to zero in probability
\begin{align*}
	 \frac{n_1n_2}{n_1 + n_2} \frac{\sum_{i \neq j}a_ia_j \mE[\Delta_{X,i} \Delta_{X,j}]}{\sum_{i=1}^\infty \lambda_i^2} = \frac{n_1n_2}{n_1 + n_2} \frac{\sum_{i \neq j} \lambda_i \lambda_j b_i b_j \mE[\Delta_{X,i} \Delta_{X,j}]}{\sum_{i=1}^\infty \lambda_i^2} = o_P(1).
\end{align*}
By including the constant function 1 as an eigenfunction corresponding to the eigenvalue zero, it can be shown that the expectation of the above expression is zero:
\begin{align*}
	\mE\biggl[  \frac{n_1n_2}{n_1 + n_2} \frac{\sum_{i \neq j} \lambda_{X,i} \lambda_{X,j} b_i b_j \mE[\Delta_{X,i} \Delta_{X,j}]}{\sum_{i=1}^\infty \lambda_i^2} \biggr] =  \frac{n_1n_2}{n_1 + n_2} \frac{\sum_{i \neq j} \lambda_{X,i} \lambda_{X,j} \mE[b_i b_j] \mE[\Delta_{X,i} \Delta_{X,j}]}{\sum_{i=1}^\infty \lambda_i^2}  = 0.
\end{align*}
On the other hand, the variance satisfies
\begin{align*}
	\mE\biggl[  \biggl\{\frac{n_1n_2}{n_1 + n_2} \frac{\sum_{i \neq j} \lambda_i \lambda_j b_i b_j \mE[\Delta_{X,i} \Delta_{X,j}]}{\sum_{i=1}^\infty \lambda_i^2} \biggr\}^2 \biggr] & = \frac{n_1^2n_2^2}{(n_1 + n_2)^2} \frac{\sum_{i \neq j} \lambda_i^2 \lambda_j^2 \mE[b_i^2 b_j^2] \{\mE[\Delta_{X,i} \Delta_{X,j}]\}^2}{(\sum_{i=1}^\infty \lambda_i^2)^2} \\
	& \leq \frac{n_1^2n_2^2}{(n_1 + n_2)^2} \frac{\sum_{i \neq j} \lambda_i^2 \lambda_j^2 \mE[b_i^2 b_j^2] }{(\sum_{i=1}^\infty \lambda_i^2)^2} \times \sup_{i \geq 1} \mE[\Delta_{X,i}^2].
\end{align*}
The expectation of the product of $b_i^2$ and $b_j^2$ can be bounded as
\begin{align*}
	 \mE[b_i^2 b_j^2] & = \mE[(\overline{\phi}_{i,X} - \overline{\phi}_{i,Y})^2(\overline{\phi}_{j,X} - \overline{\phi}_{j,Y})^2] \\
	 & \lesssim \, \mE[\overline{\phi}_{i,X}^2 \overline{\phi}_{j,X}^2] + \mE[\overline{\phi}_{i,Y}^2 \overline{\phi}_{j,Y}^2] + \mE[\overline{\phi}_{i,X}^2] \mE[\overline{\phi}_{j,Y}^2] + \mE[\overline{\phi}_{j,X}^2] \mE[\overline{\phi}_{i,Y}^2] \\
	 & \lesssim \, \biggl(\frac{1}{n_1^3} + \frac{1}{n_2^3} \biggr) \mE[\phi_i^2(X)\phi_j^2(X)] + \frac{1}{n_1^2} + \frac{1}{n_2^2} + \frac{1}{n_1n_2}.
\end{align*}
Thus the variance is bounded above by
\begin{align*}
	& \mE\biggl[  \biggl\{\frac{n_1n_2}{n_1 + n_2} \frac{\sum_{i \neq j} \lambda_i \lambda_j b_i b_j \mE[\Delta_i \Delta_j]}{\sum_{i=1}^\infty \lambda_i^2} \biggr\}^2 \biggr]  \\
	\lesssim ~ & \frac{n_1^2n_2^2}{(n_1 + n_2)^2} \biggl(\frac{1}{n_1^3} + \frac{1}{n_2^3} \biggr) \frac{\sum_{i \neq j} \lambda_i^2 \lambda_j^2 \mE[\phi_i^2(X)\phi_j^2(X)] }{(\sum_{i=1}^\infty \lambda_i^2)^2} \times \sup_{i \geq 1} \mE[\Delta_{X,i}^2] \\
	+ ~ & \frac{n_1^2n_2^2}{(n_1 + n_2)^2} \biggl( \frac{1}{n_1^2} + \frac{1}{n_2^2} + \frac{1}{n_1n_2} \biggr) \frac{\sum_{i \neq j} \lambda_i^2 \lambda_j^2 }{(\sum_{i=1}^\infty \lambda_i^2)^2} \times \sup_{i \geq 1} \mE[\Delta_{X,i}^2] \\
	\lesssim ~ & \biggl[ \biggl(\frac{1}{n_1} + \frac{1}{n_2} \biggr) \frac{\sum_{i \neq j} \lambda_i^2 \lambda_j^2 \mE[\phi_i^2(X)\phi_j^2(X)] }{(\sum_{i=1}^\infty \lambda_i^2)^2} + 1 \biggr] \times \sup_{i \geq 1} \mE[\Delta_{X,i}^2].
\end{align*}
Moreover, under \Cref{Assumption: xssMMD under H0}, we have
\begin{align*}
	 \biggl(\frac{1}{n_1} + \frac{1}{n_2} \biggr) \frac{\sum_{i \neq j} \lambda_i^2 \lambda_j^2 \mE[\phi_i^2(X)\phi_j^2(X)] }{(\sum_{i=1}^\infty \lambda_i^2)^2}  \leq \biggl(\frac{1}{n_1} + \frac{1}{n_2} \biggr) \frac{\mE[\overline{k}^2(X_1,X_3) \overline{k}^2(X_2,X_3)]}{\{\mE[\overline{k}(X_1,X_2)]\}^2} \to 0,
\end{align*}
which implies that the variance term converges to zero. Thus, the second term is also $o_P(1)$ by Chebyshev's inequality. 

Next we show the second term of \eqref{equation used to show corollary c.4} is bounded in probability
\begin{align*}
    \frac{n_1+n_2}{n_1n_2} \frac{\sum_{i=1}^\infty \lambda_i^2}{\sum_{i=1}^\infty a_i^2}=O_P(1).
\end{align*}
As in \citet[][page 55]{kim2024dimension}, let $\widetilde{\lambda}_i = \lambda_i^2 / \sum_{i'=1}^\infty \lambda_{i'}^2 \geq 0$, the reciprocal of the above is
\begin{align*}
	\sum_{i=1}^\infty \widetilde{\lambda}_i \frac{n_1n_2}{n_1+n_2}(\overline{\phi}_{i,X} - \overline{\phi}_{i,Y})^2,
\end{align*}
whose expectation is
\begin{align*}
	\sum_{i=1}^\infty \widetilde{\lambda}_i \frac{n_1n_2}{n_1+n_2}\mE[(\overline{\phi}_{i,X} - \overline{\phi}_{i,Y})^2] = \sum_{i=1}^\infty \widetilde{\lambda}_i  = 1.
\end{align*}
On the other hand, its variance satisfies
\begin{align*}
	\mV \Biggl( \sum_{i=1}^\infty \widetilde{\lambda}_i \frac{n_1n_2}{n_1+n_2}(\overline{\phi}_{i,X} - \overline{\phi}_{i,Y})^2 \Biggr) ~\lesssim~ \frac{n_1^2n_2^2}{(n_1+n_2)^2} \Biggl[ \mV\Biggl( \sum_{i=1}^\infty \widetilde{\lambda}_i \overline{\phi}_{i,X}^2 \Biggr) +  \mV\Biggl( \sum_{i=1}^\infty \widetilde{\lambda}_i \overline{\phi}_{i,Y}^2 \Biggr) \Biggr]
\end{align*}
Letting $g(x,y) = \mE[\overline{k}(x,X)\overline{k}(y,X)],$
we may see
\begin{align*}
	 \mV\Biggl( \sum_{i=1}^\infty \widetilde{\lambda}_i \overline{\phi}_{i,X}^2 \Biggr) + \mV\Biggl( \sum_{i=1}^\infty \widetilde{\lambda}_i \overline{\phi}_{i,Y}^2 \Biggr) ~ \lesssim ~&  \frac{\mE[\overline{k}^2(X_1,X_3)\overline{k}^2(X_2,X_3)]}{\{\mE[\overline{k}^2(X_1,X_2)]\}^2} \biggl(\frac{1}{n_1^3} + \frac{1}{n_2^3}\biggr) \\
	 & + \frac{\mE[g^2(X_1,X_2)]}{\{\mE[\overline{k}^2(X_1,X_2)]\}^2}\biggl(\frac{1}{n_1^2} + \frac{1}{n_2^2}\biggr).
\end{align*}
Hence, we obtain the upper bound of the variance as
\begin{align*}
	& \mV \Biggl( \sum_{i=1}^\infty \widetilde{\lambda}_i \frac{n_1n_2}{n_1+n_2}(\overline{\phi}_{i,X} - \overline{\phi}_{i,Y})^2 \Biggr) \\
	\lesssim ~ & \underbrace{\frac{\mE[\overline{k}^2(X_1,X_3)\overline{k}^2(X_2,X_3)]}{\{\mE[\overline{k}^2(X_1,X_2)]\}^2} \biggl(\frac{n_1+n_2}{n_1n_2}\biggr)}_{\rightarrow 0} + \underbrace{\frac{\sum_{i=1}^\infty \lambda_i^4}{\{\sum_{i=1}^\infty \lambda_i^2\}^2}}_{\leq \tilde{\lambda}_1}.
\end{align*}

On the other hand, 
\begin{align*}
	\frac{\sum_{i=1}^\infty \lambda_i^2}{\sum_{i=1}^\infty \lambda_i^2 \frac{n_1n_2}{n_1+n_2}(\overline{\phi}_{i,X} - \overline{\phi}_{i,Y})^2} \leq \frac{1}{\widetilde{\lambda}_1  \frac{n_1n_2}{n_1+n_2}(\overline{\phi}_{1,X} - \overline{\phi}_{1,Y})^2}.
\end{align*}
Using this inequality, for any $t>0$, 
\begin{align*}
	\mP(T \geq t) ~ \leq ~ & \mP\biggl( \frac{n_1n_2}{n_1+n_2}(\overline{\phi}_{1,X} - \overline{\phi}_{1,Y})^2 \leq  \widetilde{\lambda}_1^{-1} t^{-1} \biggr).
\end{align*}
Letting
\begin{align*}
	\sqrt{\frac{n_1n_2}{n_1+n_2}} (\overline{\phi}_{1,X} - \overline{\phi}_{1,Y}) = \sum_{i=1}^{n_1+n_2} Z_i,
\end{align*}
where the summands are mutually independent, given as
\begin{align*}
	Z_ i = \begin{cases}
		\sqrt{\frac{n_2}{n_1(n_1+n_2)}} \phi_1(X_i) \quad & \text{if $1 \leq i \leq n_1$,} \\[.5em]
		\sqrt{\frac{n_1}{n_2(n_1+n_2)}} \phi_1(Y_{i - n_1}) \quad & \text{if $n_1+1 \leq i \leq n_1+n_2$},
	\end{cases}
\end{align*}
we may expect that it converges to a normal distribution. In particular, using the Berry--Esseen bound for independent but not identically distributed summands~\citep{Berry1941}, we have
\begin{align}\label{eq: upper bound for thm 2 using BE bound}
	 & \mP\biggl( \frac{n_1n_2}{n_1+n_2}(\overline{\phi}_{1,X} - \overline{\phi}_{1,Y})^2 \leq  \widetilde{\lambda}_1^{-1} t^{-1} \biggr) \nonumber\\[.5em]
	 \lesssim ~ & \mP(\xi^2 \leq \widetilde{\lambda}^{-1} t^{-1}) + \frac{\mE[|Z_1|^3]}{\sqrt{n_1}\{\mV[Z_1]\}^{3/2}}  + \frac{\mE[|Z_{n_1+1}|^3]}{\sqrt{n_2}\{\mV[Z_{n_1+1}]\}} \nonumber\\[.5em]
	 \lesssim ~ & \mP(\xi^2 \leq \widetilde{\lambda}^{-1} t^{-1}) + \mE[|\phi_1(X)|^3] \sqrt{\frac{1}{n^2_1} + \frac{1}{n^2_2}}
\end{align}
where $\xi \sim N(0,1)$. Using the fact that
\begin{align*}
	& \widetilde{\lambda}_1^2 \mE[\phi_1^4(X)] \biggl(\frac{1}{n_1} + \frac{1}{n_2} \biggr) \rightarrow 0 \quad \text{and} \\
	& \mE[|\phi_1(X)|^3] = \mE[|\phi_1(X)||\phi_1(X)|^2] \leq \{\mE[\phi_1^2(X)]\}^{1/2} \{ \mE[\phi_1^4(X)] \}^{1/2} =  \{ \mE[\phi_1^4(X)] \}^{1/2},
\end{align*}
we observe the convergence rate of the upper bound of the second term of \eqref{eq: upper bound for thm 2 using BE bound}
\begin{align*}
	 \mE[|\phi_1(X)|^3] \sqrt{\frac{1}{n^2_1} + \frac{1}{n^2_2}} \leq \sqrt{\mE[|\phi_1(X)|^4]\biggl( \frac{1}{n^2_1} + \frac{1}{n^2_2} \biggr)} \leq \sqrt{\mE[|\phi_1(X)|^4]\biggl( \frac{1}{n_1} + \frac{1}{n_2} \biggr)} = o(\widetilde{\lambda}_1^{-1}).
\end{align*}
We then follow the proof of \citet[][Theorem 4.2]{kim2024dimension} and show that 
\begin{align*}
	\frac{n_1+n_2}{n_1n_2}\frac{\sum_{i=1}^\infty \lambda_i^2}{\sum_{i=1}^\infty \lambda_i^2 (\overline{\phi}_{i,X} - \overline{\phi}_{i,Y})^2}  = O_P(1).
\end{align*}

Consequently, the first convergence condition in (\ref{Eq: cross-fitting condition}) is satisfied. A similar argument applies to $Y$ and $W$, demonstrating that the second convergence condition is also satisfied. Thus, we establish that if condition (\ref{Eq: linear operator}) and \Cref{Assumption: xssMMD under H0} hold, then \Cref{Assumption: consistency of conditional expectation} is valid as claimed.

\subsection{Proof of \Cref{thm: asymptotic power expression}}\label{Section: proof of power using linear kernel}

Using the test statistic, power function could be written as 
\begin{align*}
    \mP(\xssMMD>z_{1-\alpha})
    &=\mP\bigg(\frac{\widehat{\mu}^{\dagger}_{X,\fhat} - \widehat{\mu}^{\dagger}_{Y,\fhat}}{\sqrt{\widehat{\sigma}^{\dagger 2}_{X,\fhat} + \widehat{\sigma}^{\dagger 2}_{Y,\fhat} }}>z_{1-\alpha}\bigg)\\
    &=\mP\bigg(\frac{\big(\widehat{\mu}^{\dagger}_{X,\fhat}-\mE[\fhat(X)|\fhat]\big) - \big(\widehat{\mu}^{\dagger}_{Y,\fhat}-\mE[\fhat(Y)|\fhat]\big)}{\sqrt{\widehat{\sigma}^{\dagger 2}_{X,\fhat} + \widehat{\sigma}^{\dagger 2}_{Y,\fhat} }}>z_{1-\alpha}-\frac{\mE[\fhat(X)|\fhat]-\mE[\fhat(Y)|\fhat]}{\sqrt{\widehat{\sigma}^{\dagger 2}_{X,\fhat} + \widehat{\sigma}^{\dagger 2}_{Y,\fhat} }}\bigg)
\end{align*}

Assume that we have bilinear kernel $k(x,y)=x^\top y.$ We first show 
\begin{align*}
T=\frac{\big(\widehat{\mu}^{\dagger}_{X,\fhat}-\mE[\fhat(X)|\fhat]\big) - \big(\widehat{\mu}^{\dagger}_{Y,\fhat}-\mE[\fhat(Y)|\fhat]\big)}{\sqrt{\widehat{\sigma}^{\dagger 2}_{X,\fhat} + \widehat{\sigma}^{\dagger 2}_{Y,\fhat} }} 
\end{align*} follows asymptotically normal distribution. It suffices to show that the above satisfies \Cref{Assumption: xssMMD under H0} and \Cref{Assumption: xssMMD under H1}.

From the Gaussianity of the data and the independence of $X_1, X_2, X_3$, we compute the key moments of the centralized kernel: 
\begin{align*} 
&\mE[\overline{g}_X(X,X)]=\mathbb{E}[\overline{k}(X,X_1)^2] = \mathrm{tr}(\Sigma_{11}^2) \asymp d, \\  
&\mathbb{E}[\overline{k}(X_1, X_2)^4] = 2\mathrm{tr}(\Sigma_{11}^4) + (\mathrm{tr}(\Sigma_{11}^2))^2\asymp d^2, \\  
&\mathbb{E}[\overline{k}(X_1, X_2)^2 \overline{k}(X_1, X_3)^2] = 2\mathrm{tr}(\Sigma_{11}^4) + 2 (\mathrm{tr}(\Sigma_{11}^2))^2\asymp d^2,  
\end{align*} since the eigenvalues of $\Sigma_{11}$ are bounded from \Cref{Assumption: assumptions for asymptotic power expression}.

We substitute these results into the following term
\begin{align*} 
	\frac{\mathbb{E}[\overline{k}(X_1,X_2)^4] + n_1\mathbb{E}[\overline{k}(X_1,X_2)^2 \overline{k}(X_1,X_3)^2]}{n_1^2\{\mathbb{E}[\overline{g}_X(X,X)]\}^2}\asymp\frac{1}{n^2}+\frac{1}{n}=o_P(1),
\end{align*} which implies that \Cref{Assumption: xssMMD under H0} is satisfied.

Similarly, we compute other moments:
\begin{align*} 
&\mathrm{MMD}(P_X, P_Y)^2 = 2\mathrm{tr}(\Sigma_{11})\asymp d, \\  
&\mE[\overline{g}_X(Y_1,Y_2)]=\mathbb{E}[\overline{k}(X,Y_1)\overline{k}(X,Y_2)] = \mathrm{tr}(\Sigma^2_{11})\asymp d.  
\end{align*}
Combining these results with the previous computation, we obtain the following convergence
\begin{align*} 
	\frac{\mathrm{MMD}(P_X,P_Y)^4\times \mathbb{E}[\overline{k}_X(X,X)^2] }{\{n_1\mathbb{E}[\overline{g}_X(X,X)] + n_1^2\mathbb{E}[\overline{g}_X(Y_1,Y_2)]\}^2}\asymp \frac{d^4}{(nd^2+n^2 d)^2}=\frac{\tau^2}{\tau^2+2\tau+1}\frac{1}{n^2}= o_P(1),
\end{align*}
which shows that \Cref{Assumption: xssMMD under H1} holds.

Therefore, together with \Cref{Assumption: consistency of conditional expectation}, we conclude that the test statistic $T$ follows asymptotically $ N(0,1) $.

Next, we analyze asymptotic behavior of $$\frac{\mE[\fhat(X)|\fhat]-\mE[\fhat(Y)|\fhat]}{\sqrt{\widehat{\sigma}^{\dagger 2}_{X,\fhat} + \widehat{\sigma}^{\dagger 2}_{Y,\fhat} }}.$$

Since \Cref{Assumption: xssMMD under H0}, \ref{Assumption: consistency of conditional expectation}, and \ref{Assumption: xssMMD under H1} holds, we know that \begin{align*}
    \frac{\widehat{\sigma}_{X,\fhat}^2 + \widehat{\sigma}_{Y,\fhat}^2}{\sigma_{X,\fhat}^2 + \sigma_{Y,\fhat}^2} \convP 1,\quad\text{and}\quad\frac{\widehat{\sigma}_{X,\fhat}^2 + \widehat{\sigma}_{Y,\fhat}^2}{\widehat{\sigma}_{X,\fhat}^{\dagger 2} + \widehat{\sigma}_{Y,\fhat}^{\dagger 2}} \convP 1,
\end{align*} from the proof of Step 2 and Step 3, claim ($\mathrm{ii}^\prime$) in \Cref{Theorem: xssMMD}, respectively.
From these facts, it follows that 
\begin{align*}
    \frac{\mE[\fhat(X)|\fhat]-\mE[\fhat(Y)|\fhat]}{\sqrt{\widehat{\sigma}^{\dagger 2}_{X,\fhat} + \widehat{\sigma}^{\dagger 2}_{Y,\fhat} }} &~= \frac{\mE[\fhat(X)|\fhat]-\mE[\fhat(Y)|\fhat]}{\sqrt{\sigma^{2}_{X,\fhat} + \sigma^{2}_{Y,\fhat} }} \times \frac{\sqrt{\sigma_{X,\fhat}^2 + \sigma_{Y,\fhat}^2}}{\sqrt{\widehat{\sigma}_{X,\fhat}^2 + \widehat{\sigma}_{Y,\fhat}^2}} \times \frac{\sqrt{\widehat{\sigma}_{X,\fhat}^2 + \widehat{\sigma}_{Y,\fhat}^2}}{\sqrt{\widehat{\sigma}_{X,\fhat}^{\dagger 2} + \widehat{\sigma}_{Y,\fhat}^{\dagger 2}}}\\
    &~=\frac{\mE[\fhat(X)|\fhat]-\mE[\fhat(Y)|\fhat]}{\sqrt{\sigma^{2}_{X,\fhat} + \sigma^{2}_{Y,\fhat} }} \times \{1+o_P(1)\}.
\end{align*}
Note that with normal distribution assumption, $\overline{X}-\overline{Y} \sim N(\mu_X-\mu_Y, 2n^{-1}\Sigma_{11})$ where $\overline{X}=\frac{1}{n}\sum_{i=1}^{n}X_i$ and $\overline{Y}=\frac{1}{n}\sum_{i=1}^{n}Y_i.$
We compute \begin{align*}
    \mE[\fhat(X)|\fhat]-\mE[\fhat(Y)|\fhat]
    =~&(\overline{X}-\overline{Y})^\top(\mu_X-\mu_Y)\\
    =~& (\mu_X-\mu_Y)^\top(\mu_X-\mu_Y) + O_P\big(\sqrt{(\mu_X-\mu_Y)^\top \Sigma (\mu_X-\mu_Y)}\big).
\end{align*} 
On the other hand, we compute $\sigma^{2}_{X,\fhat}$
: \begin{align*}
    \sigma^{2}_{X,\fhat}=&\frac{1}{n} \mE [\mV\{\fhat(X) \given V,\fhat \} \given \fhat] + \frac{1}{2n} \mV[\mE\{\fhat(X) \given V, \fhat\} \given \fhat]\\
    =~&\frac{1}{n} \mE [(\overline{X}-\overline{Y})^\top\mV(X\given V)(\overline{X}-\overline{Y})] + \frac{1}{2n} \mV((\overline{X}-\overline{Y})^\top\mE[X\given V])\\
    =~& (\overline{X}-\overline{Y})^\top\Lambda(\overline{X}-\overline{Y})\\
    =~& 2n^{-1}\mathrm{tr}(\Lambda\Sigma_{11})+(\mu_X-\mu_Y)^\top\Lambda(\mu_X-\mu_Y) \\
    & \hskip 6em+ O_P\big(\sqrt{8n^{-2}\mathrm{tr}(\Lambda\Sigma_{11}\Lambda\Sigma_{11})+8n^{-1}(\mu_X-\mu_Y)^\top\Lambda\Sigma_{11}\Lambda(\mu_X-\mu_Y)}\big)
\end{align*}
where $\Lambda=n^{-1}\Sigma_{11}-(2n)^{-1}\Sigma_{12}\Sigma^{-1}_{22}\Sigma_{21}$.

Assuming that $(X,V)^\top$ and $(V,W)^\top$ have equal sample sizes and equal covariancce matrices, a similar computation derives the same result for $\sigma^2_{Y,\fhat}$ :
\begin{align*}
    \sigma^2_{Y,\fhat} &= 2n^{-1}\mathrm{tr}(\Lambda\Sigma_{11})+(\mu_X-\mu_Y)^\top\Lambda(\mu_X-\mu_Y) \\
    & \hskip 6em+ O_P\big(\sqrt{8n^{-2}\mathrm{tr}(\Lambda\Sigma_{11}\Lambda\Sigma_{11})+8n^{-1}(\mu_X-\mu_Y)^\top\Lambda\Sigma_{11}\Lambda(\mu_X-\mu_Y)}\big).
\end{align*}
Hence we obtain 
\begin{align*}
    \frac{\mE[\fhat(X_i)|\fhat]-\mE[\fhat(Y_i)|\fhat]}{\sqrt{\sigma^2_{X,\fhat} + \sigma^2_{Y,\fhat} }}
    =~&\frac{(\mu_X-\mu_Y)^\top(\mu_X-\mu_Y) + O_P(n^{-3/4})}{\sqrt{4n^{-1}\mathrm{tr}(\Lambda\Sigma_{11})+O_P(n^{-3/2})}}\{1+o_P(1)\}\\
    =~&\frac{(\mu_X-\mu_Y)^\top(\mu_X-\mu_Y)}{\sqrt{4n^{-1}\mathrm{tr}(\Lambda\Sigma_{11})}}+o_P(1),
\end{align*}
using the fact that $(\mu_X-\mu_Y)^\top(\mu_X-\mu_Y)=O_P(n^{-1/2})$ and $\tr(\Lambda\Sigma_{11})=O_P(1)$ from \Cref{Assumption: assumptions for asymptotic power expression}.

Combining with the normal approximation, we conclude that
\begin{align*}
    \mP(T^\star>z_{1-\alpha})
    &=\Phi\Biggl(z_\alpha+\frac{(\mu_X-\mu_Y)^\top(\mu_X-\mu_Y)}{\sqrt{4n^{-1}\mathrm{tr}(\Lambda\Sigma_{11})}}\Biggr)+o(1)\\
    &=\Phi\Biggl(z_\alpha + \frac{n(\mu_X-\mu_Y)^\top(\mu_X-\mu_Y)}{\sqrt{4\tr(\Sigma_{11}^2)-2\tr(\Sigma_{12}\Sigma^{-1}_{22}\Sigma_{21}\Sigma_{11})}}\Biggr)+o(1).
\end{align*}

\end{document}